\documentclass[11pt,letterpaper]{article}

\usepackage[
  letterpaper,
  left=1in,
  right=1in,
  top=1in,
  bottom=0.8in,
  headheight=14pt,
  headsep=10pt,
  footskip=30pt
]{geometry}

\usepackage[
  protrusion=true,
  expansion=false
]{microtype}

\usepackage{amsmath}
\usepackage{amssymb}
\usepackage{amsthm}
\usepackage{mathtools}
\usepackage{bm}
\usepackage{setspace}
\usepackage{etoolbox}

\AtBeginEnvironment{algorithmic}{\setstretch{1.25}}

\newtheoremstyle{bayatiplain}
  {6pt}
  {6pt}
  {\itshape}
  {}
  {\scshape}
  {.}
  {0.5em}
  {}

\theoremstyle{bayatiplain}
\newtheorem{theorem}{Theorem}
\newtheorem{lemma}[theorem]{Lemma}
\newtheorem{proposition}[theorem]{Proposition}
\newtheorem{corollary}[theorem]{Corollary}

\theoremstyle{definition}
\newtheorem{assumption}[theorem]{Assumption}

\theoremstyle{remark}
\newtheorem{remark}[theorem]{Remark}

\usepackage{algorithm}
\usepackage{algorithmic}
\usepackage{tikz}
\usepackage{graphicx}
\usepackage{booktabs}
\usepackage{multicol}
\usepackage{multirow}
\usepackage{float}
\usepackage{longtable}
\usepackage{makecell}
\usepackage{pdflscape}
\usepackage{placeins}
\usepackage[table]{xcolor}
\usepackage{caption}
\usepackage{subcaption}

\graphicspath{{Figures/}{./Figures/}}
\DeclareGraphicsExtensions{.pdf,.png,.jpg,.jpeg}

\usepackage{titlesec}

\titleformat{\section}
  {\normalfont\large\bfseries}
  {\thesection.}
  {0.65em}
  {}

\titleformat{\subsection}
  {\normalfont\normalsize\bfseries}
  {\thesubsection.}
  {0.65em}
  {}

\titleformat{\subsubsection}
  {\normalfont\normalsize\bfseries}
  {\thesubsubsection.}
  {0.65em}
  {}

\titleformat{\paragraph}[runin]
  {\normalfont\itshape}
  {}
  {0pt}
  {}

\titlespacing*{\section}
  {0pt}
  {2.0ex plus 0.4ex minus 0.2ex}
  {0.8ex plus 0.1ex}

\titlespacing*{\subsection}
  {0pt}
  {1.3ex plus 0.3ex minus 0.15ex}
  {0.55ex plus 0.1ex}

\titlespacing*{\subsubsection}
  {0pt}
  {0.9ex plus 0.2ex minus 0.1ex}
  {0.4ex}

\titlespacing*{\paragraph}
  {0pt}
  {0pt}
  {0.6em}

\usepackage[authoryear,round]{natbib}

\bibpunct{(}{)}{;}{a}{,}{,}

\usepackage{xurl}
\usepackage[hypertexnames=false]{hyperref}
\hypersetup{
  colorlinks=true,
  linkcolor=black,
  citecolor=black,
  filecolor=black,
  urlcolor=black,
  pdftitle={The Greedy Advantage in Finite-Horizon Bandits},
  pdfauthor={Kai Zhou, Michael Lingzhi Li, and Kai Wang}
}

\usepackage{fancyhdr}
\fancypagestyle{plain}{%
  \fancyhf{}
  \fancyfoot[C]{\small\thepage}

}

\allowdisplaybreaks
\DeclareCaptionFont{captionstretch}{\setstretch{1.25}}
\begin{document}
%%%%%%%%%%%%%%%%%%%%%%%%%%%%%%%%%%%%%%%%%%%%%%%%%%%%%%%%%%%%%%%%%%%%%%%%%%

\thispagestyle{plain}

%%%%%%%%%%%%%%%%%%%%%%%%%%%%%%%%%%%%%%%%%%%%%%%%%%%%%%%%%%%%%%%%%%%%%%%%%%
%% Title, authors, abstract, and keywords
%%%%%%%%%%%%%%%%%%%%%%%%%%%%%%%%%%%%%%%%%%%%%%%%%%%%%%%%%%%%%%%%%%%%%%%%%%

\begin{center}
\vspace*{0.20in}
{\LARGE The Greedy Advantage in Finite-Horizon Bandits\par}
\vspace{1.25em}

{\large Kai Zhou\par}
{\small Tsinghua University, \href{mailto:zhouk23@mails.tsinghua.edu.cn}{zhouk23@mails.tsinghua.edu.cn}\par}
\vspace{0.55em}

{\large Michael Lingzhi Li\par}
{\small Harvard Business School, \href{mailto:mili@hbs.edu}{mili@hbs.edu}\par}
\vspace{0.55em}

{\large Kai Wang\par}
{\small Tsinghua University, \href{mailto:cwangkai@tsinghua.edu.cn}{cwangkai@tsinghua.edu.cn}\par}
\end{center}

\vspace{1.2em}

\begin{quote}
\small
% \noindent\textbf{Abstract.}\enspace
Organizations increasingly rely on sequential experimentation to improve decision-making. While the multi-armed bandit literature has developed algorithms with strong asymptotic regret guarantees, many practical applications operate over finite and externally imposed horizons. Motivated by the finite-horizon setting, we develop a class of regularized greedy algorithms for multi-armed Bernoulli bandits. We derive the first finite-horizon regret envelopes for regularized greedy bandits, showing that finite-horizon regret decomposes into transient exploration costs and a suboptimal convergence term that decays exponentially with the regularization strength. This characterization yields principled calibration rules for the regularization parameters and, as a limiting case, sharper regret guarantees for the classical greedy policy. Across extensive numerical experiments, calibrated regularized greedy policies consistently match or outperform state-of-the-art algorithms. These results suggest that regularized greedy policies can provide an effective approach for finite-horizon bandit problems.

\end{quote}

\vspace{0.5em}
\noindent\rule{\textwidth}{0.4pt}
\vspace{0.5em}

\section{Introduction}\label{sec:Intro}

Organizations increasingly use sequential experimentation to improve decision-making. For example, pharmaceutical companies use adaptive trial designs to learn about treatment performance while patients are still being enrolled, and digital platforms test recommendations and interface designs by allocating user traffic adaptively. In all these settings, the dilemma is: each action generates information that can improve future decisions, but each action also affects current performance.

Multi-armed bandits provide a natural mathematical model for this trade-off \citep{robbins1952some}. A decision-maker repeatedly selects among competing actions, observes stochastic rewards, and updates future choices based on the accumulated data. The classical bandit literature has long established that greedy policies can converge to suboptimal actions and incur regret that is linear in $T$ relative to an oracle \citep{lai1985asymptotically,den2013simultaneously,keskin2014dynamic,bastani2020mostly}. In response, the literature has developed a rich class of algorithms with strong asymptotic regret guarantees, including upper confidence bound policies \citep{agrawal1995sample,auer2002finite,garivier2011kl}, Thompson sampling \citep{thompson1933likelihood,russo2018tutorial}, and extensions to other settings \citep{slivkins2011contextual,agarwal2012contextual}. %\citep{lu2010contextual,slivkins2011contextual,agarwal2012contextual,lattimore2020bandit}.
These policies aim to balance exploration and exploitation to achieve sublinear regret as the horizon $T \to \infty$.

However, the operational environments that motivate bandit algorithms often impose \emph{finite and externally imposed horizons}. In online experimentation, the horizon may be determined by the duration of a product launch. In clinical trials, it is often bounded by operational constraints on trial length. In these settings, the relevant objective is to perform well within the finite window in which decisions actually matter.

This observation changes the principle of algorithmic design. Greedy policies are asymptotically vulnerable because noise in the observed rewards can induce permanent commitment to a suboptimal arm, leading to linear regret that eventually dominates any policy with sublinear regret as $T\to\infty$. Over a finite operational horizon, however, the relevant comparison is not determined by asymptotic rates alone. Exploration-based policies incur sampling costs throughout the horizon, while greedy-like policies incur large regret only on sample paths that lead to suboptimal convergence. The relative performance of these algorithms therefore depends on the probability and regret contribution of such suboptimal-convergence paths. The central question is whether this finite-horizon perspective can change the relative attractiveness of different policies.

This paper answers the question in the affirmative. Many asymptotically optimal bandit algorithms aim to drive the probability of permanent convergence to a suboptimal arm to zero as $T\to\infty$, thus achieving sublinear regret. We argue that this asymptotic objective does not necessarily lead to the best finite-horizon performance. Driving the probability of suboptimal convergence to zero requires continued exploration, and this exploration is itself costly within a finite horizon.

Instead, we seek policies that explicitly balance the cost of continued exploration against the cost of occasional suboptimal convergence. The objective is no longer to eliminate suboptimal convergence asymptotically, but to make its probability sufficiently small that its finite-horizon contribution is outweighed by the exploration cost it avoids.

This principle motivates a class of regularized greedy algorithms. Rather than relying on continued exploration to avoid suboptimal convergence, these algorithms retain the greedy rule while using regularization to reduce the probability of premature commitment. The regularization parameters determine how aggressively the algorithm abandons competing arms, and thus balance the transient cost of delayed abandonment against the probability of convergence to a suboptimal arm.

We formalize the above argument in the context of multi-armed Bernoulli bandits. Bernoulli rewards capture operational experimentation settings in which actions generate binary outcomes, such as conversion or non-conversion, treatment response or non-response, and success or failure relative to a predefined operational target. In this setting, the regularized greedy algorithm pulls each arm once and then repeatedly selects the arm with the largest regularized empirical mean,
\[
\widehat p_i(t)=\frac{S_i(t)+\alpha}{N_i(t)+\beta},
\qquad i=1,\ldots,K,
\]
where $S_i(t)$ and $N_i(t)$ denote the cumulative successes and pulls of arm $i$ at each time period $t=1,\ldots,T$, while $(\alpha,\beta)$ encode pseudo-successes and pseudo-trials. The pure greedy algorithm corresponds to $\alpha=\beta=0$. Positive regularization increases the amount of evidence required before an arm is abandoned, reducing the probability of premature suboptimal convergence while preserving the greedy structure of the policy.

To quantify how regularization changes the trade-off between exploration and suboptimal convergence, we study the absorbing probabilities of the greedy dynamics. Let $Q_i$ denote the event that the algorithm permanently converges to arm $i$, and index the arms so that $p_1\ge p_2\ge\cdots\ge p_K$. On $Q_i$ with $i \geq 2$ and $p_i < p_1$, the algorithm spends a finite transient period sampling the remaining arms before allocating asymptotically all future pulls to arm $i$, thereby incurring linear regret at rate $\delta_i=p_1-p_i$. Finite-horizon regret therefore decomposes into the linear component
\[
    R_{\mathrm{linear}}
    =
    T\sum_{\substack{i=2\\p_i<p_1}}^{K}\delta_i\mathbb P(Q_i),
\]
together with finite transient abandonment costs $R_{\mathrm{transient}}$.

Our main result derives analytical finite-horizon expressions for both regret components that characterize how regularization affects performance over a fixed decision horizon. The analysis considers a fixed horizon $T$ while the regularization strength grows, yielding the envelope
\[
    R_{\mathrm{linear}}(1-o(1))-o(1)
    \le
    R(T)
    \le
    (R_{\mathrm{linear}}+
    R_{\mathrm{transient}})(1+o(1)),
\]
where the small-$o$ terms are taken with respect to the regularization strength rather than the horizon. These expressions reveal that regularization affects the two terms in opposite directions. Increasing the regularization strength suppresses $R_{\mathrm{linear}}$ exponentially by reducing the probability of suboptimal convergence, while increasing $R_{\mathrm{transient}}$ by delaying abandonment. The finite-horizon design problem therefore reduces to choosing $(\alpha,\beta)$ to balance an exponentially decreasing suboptimal-convergence term against an increasing transient term.

We then propose calibration rules for $(\alpha,\beta)$ under different levels of information, ranging from oracle settings in which the problem instance is known, to fully adaptive implementations requiring no knowledge of either the horizon or the arm means. Across extensive numerical experiments, calibrated regularized greedy policies consistently match or outperform state-of-the-art bandit algorithms, including Thompson Sampling~\citep{thompson1933likelihood}, OGI (Optimistic Gittins Indices; \citealp{Farias2022OGI}), and IDS (Information-Directed Sampling, \citealp{russo2014ids}). These results demonstrate that appropriately calibrated greedy policies are competitive for finite-horizon bandit problems, combining ease of implementation with computational tractability.

\subsection{Literature Review}
%The multi-armed bandit literature studies the exploration-exploitation trade-off in sequential decision-making problems. 
Classical multi-armed bandit work established asymptotic regret lower bounds and developed policies that attain sublinear regret, including upper confidence bound algorithms, Thompson sampling, and their many extensions \citep{robbins1952some,lai1985asymptotically,agrawal1995sample,auer2002finite,garivier2011kl,thompson1933likelihood,russo2018tutorial,lattimore2020bandit}. Most of this literature evaluates performance through asymptotic regret as the horizon grows large. A related line of work studies finite-horizon and limited-adaptivity settings through Bayesian dynamic programming, Gittins-index methods, explore-then-commit policies, batched bandits, and experimental designs that balance online performance against other objectives such as inference or tail risk \citep{gittins1979bandit,nino2011computing,russo2014ids, garivier2016explore,perchet2016batched,jin2021double,Farias2022OGI, simchilevi2025experimental,simchilevi2025safety,fan2025fragility}. This literature recognizes that exploration is itself costly over finite horizons. Our analysis focuses on a different way of controlling this tradeoff. The policy remains greedy with respect to a regularized estimator, so exploration is not imposed through confidence bounds, randomization, or prescribed exploration phases. Instead, regularization changes the learning dynamics by reducing the probability of suboptimal absorption while increasing the transient cost of abandoning competing arms.

Our work is particularly related to the literature on greedy bandits. Pure greedy policies eliminate exploration costs but are vulnerable to permanent commitment to a suboptimal arm. In dynamic pricing settings, \citet{den2013simultaneously} and \citet{keskin2014dynamic,keskin2017chasing} showed that greedy learning can converge to suboptimal decisions with positive probability and therefore incur linear regret. In bandits with correlated arm rewards, \citet{mersereau2009structured} showed that greedy policies can exploit cross-arm learning to achieve logarithmic cumulative Bayes risk. In contextual settings, \citet{bastani2020mostly} identified conditions under which greedy learning remains asymptotically optimal, while \citet{slivkins2025greedy} characterized broader conditions governing the success and failure of greedy algorithms. Recent work has also emphasized that greedy or near-greedy policies can be effective when structural features of the environment generate sufficient learning or reduce the cost of initial exploration \citep{cao2025collaborative,bayati2026coldstart}. These papers identify costly exploration and suboptimal convergence as central issues in sequential learning. We complement this literature by explicitly characterizing the absorbing probabilities of suboptimal arms and studying how regularization suppresses their aggregate regret contribution.

The regularized greedy policies studied in this paper select actions according to the estimator $(S_i+\alpha)/(N_i+\beta)$. 
Estimators of this form have appeared in prior work through optimistic initialization heuristics, pseudo-observations, and posterior means under Beta priors \citep{gittins1979bandit,liu2015prior,sutton2018reinforcement}. Existing work, however, primarily employs these estimators as algorithmic or Bayesian modeling devices rather than analyzing their finite-horizon behavior. 
%To the best of our knowledge, this paper provides the first finite-horizon analytical characterization of regularized greedy policies for $K$-armed Bernoulli bandits. The analysis characterizes suboptimal absorbing probabilities, derives two-sided regret envelopes, substantially sharpens existing guarantees for the classical greedy policy as a limiting case, and yields closed-form calibration rules that perform strongly across a broad range of finite-horizon settings.

\subsection{Our Contributions}
Overall, this paper makes two major contributions.

First, the paper develops a finite-horizon theory for regularized greedy $K$-armed Bernoulli bandits (Section~\ref{sec:Method}). Motivated by the trade-off between continued exploration and occasional suboptimal convergence, we characterize finite-horizon regret with two components: a linear term arising from convergence to a suboptimal arm and a transient term arising from delayed abandonment. To the best of our knowledge, this is the first finite-horizon theory for regularized greedy bandits. The analysis develops new tools combining absorption analysis, score-minimum reductions, boundary-crossing arguments for drifted random walks, Lundberg-root change-of-measure estimates, and local Stieltjes replacement arguments to characterize the probability of convergence to each suboptimal arm. This yields regret envelopes, explicit calibration rules for the regularization parameters, and, as a limiting case, sharper regret guarantees for the classical greedy algorithm. The analytical tools may also be useful for studying other finite-horizon sequential learning problems.

Second, the paper demonstrates that the proposed algorithms perform strongly across a broad range of finite-horizon settings (Section~\ref{sec:calibration}). We derive calibration procedures for $(\alpha,\beta)$ under varying levels of problem information, ranging from oracle settings to fully adaptive ones requiring no knowledge of either the horizon or the arm means.  Extensive numerical experiments show that calibrated regularized greedy policies consistently match or outperform modern state-of-the-art bandit algorithms, including Thompson sampling, OGI, and IDS, across a wide range of instances and horizons. In all, these results establish regularized greedy as an efficient, practical and theoretically grounded approach to finite-horizon implementations.

Section~\ref{sec:conclusion} concludes the paper, and the electronic companion provides all proofs.

\section{Bounding Regret for Bernoulli Bandits}\label{sec:Method}

\subsection{Problem Setup} \label{subsec:setup}

We consider a $K$-armed Bernoulli bandit over a finite horizon $T$, where $K\ge 2$ is fixed and $T\ge K$. The arms have unknown success probabilities $p_i\in[\epsilon_{\mathrm{p}},1-\epsilon_{\mathrm{p}}]$ for some fixed $\epsilon_{\mathrm{p}}\in(0,1/2)$, indexed without loss of generality so that $p_1\ge p_2\ge\cdots\ge p_K$ and $p_1>p_K$. Conditional on selecting arm $i\in\{1,\ldots,K\}$, the decision-maker observes a binary reward at each time period $t=1,\ldots,T$:
\[
X_t\sim \mathrm{Bernoulli}(p_i).
\]

Bernoulli bandits arise naturally in a range of operational settings involving sequential allocation with binary feedback. In healthcare operations, a hospital system may sequentially allocate patients across treatment protocols and observe binary clinical outcomes such as recovery or deterioration. In facility deployment problems, an organization may open facilities sequentially and observe binary indicators of operational success, such as whether a facility achieves a target utilization threshold. In online experimentation and recommendation systems, the decision-maker allocates traffic across alternatives and observes binary engagement outcomes such as clicks or purchases. In all such settings, the decision-maker repeatedly allocates resources under uncertainty and updates future decisions using binary feedback from prior allocations.

At each time period \(t=1,\ldots,T\), the decision-maker selects an arm \(A_t\in\{1,\ldots,K\}\) using a policy based on the available history and then observes a reward \(X_t\sim\mathrm{Bernoulli}(p_{A_t})\) independently across time, conditional on the selected actions.

We define the action and reward histories through time \(t\) as
\[
    A_{1:t}=(A_1,\ldots,A_t),
    \qquad
    X_{1:t}=(X_1,\ldots,X_t),
\]
and let \(H_t=(A_{1:t},X_{1:t})\) denote the information available after period \(t\). A valid, possibly randomized, policy is therefore a sequence of history-dependent decision rules
\[
    \pi_t: H_{t-1}\to\mathcal S_K,
    \qquad
    \mathcal S_K
    :=
    \left\{
    u\in[0,1]^K:\sum_{i=1}^{K}u_i=1
    \right\},
\]
and \(\pi_{t,i}(H_{t-1})\) denotes the probability of selecting arm \(i\) at time \(t\) conditional on the history \(H_{t-1}\).

The decision-maker seeks a policy minimizing cumulative regret over the finite horizon $T$. For each arm \(i=1,\ldots,K\), define
\[
    N_i(t)=\sum_{s=1}^t\mathbf 1\{A_s=i\},
    \qquad
    S_i(t)=\sum_{s=1}^t X_s\mathbf 1\{A_s=i\},
\]
as the cumulative number of pulls and cumulative number of successes through time \(t\), respectively. 
Since arms are indexed so that \(p_1\ge p_2\ge\cdots\ge p_K\), define the regret gaps $\delta_i:=p_1-p_i\ge 0$ for $i=1,\cdots,K$, so that  \(\delta_1=0\). The cumulative expected regret of policy \(\pi\) over horizon \(T\) is
\[
% \label{eq:setup_regret_def}
    R_{\pi}(T)
    =
    \mathbb E_{\pi}\left[
    \sum_{t=1}^{T}(p_1-p_{A_t})
    \right]
    =
    \sum_{i=2}^{K}\delta_i\,\mathbb E_{\pi}[N_i(T)],
\]
where \(\mathbb E_{\pi}[\cdot]\) denotes expectation taken with respect to the probability measure induced by policy \(\pi\) and the Bernoulli reward realizations.

\subsection{Regularized Greedy Policies}
\label{subsec:algorithm}

We study a class $\Pi$ of regularized greedy policies. Each policy $\pi_{\alpha,\beta}\in\Pi$ is indexed by two parameters $(\alpha,\beta)$. After an initialization step, policy $\pi_{\alpha,\beta}$ assigns each arm $i=1,\ldots,K$ the score
\[
    \widehat p_i(t)
    =
    \frac{S_i(t)+\alpha}{N_i(t)+\beta},
\]
and selects an arm with maximal score, breaking ties uniformly at random. Algorithm~\ref{alg:regularized_greedy} gives the full policy specification. The classical greedy policy corresponds to \(\alpha=\beta=0\).

The parameters $(\alpha,\beta)$ regularize the empirical comparison that determines which arms remain competitive. When $\beta>0$ and $N_i(t)>0$, the score can be written as
\[
    \widehat p_i(t)
    =
    \frac{N_i(t)}{N_i(t)+\beta}\cdot \frac{S_i(t)}{N_i(t)}
    +
    \frac{\beta}{N_i(t)+\beta}\cdot \frac{\alpha}{\beta},
    \qquad i=1,\ldots,K.
\]
Thus the regularized score is a weighted average of the empirical mean and the baseline value $\alpha/\beta$, with weight on the baseline decreasing as arm $i$ is sampled. The ratio $\alpha/\beta$ determines the level toward which the score is initially pulled, while $\beta$ determines how quickly the regularization vanishes. Along any history in which an arm continues to be sampled, the regularization effect disappears asymptotically and the score converges to the empirical mean.

The next section characterizes how the choice of $(\alpha,\beta)$ affects finite-horizon regret through the transient exploration cost and the probability of eventual convergence to a suboptimal arm.
\begin{algorithm}[ht]
\caption{Regularized Greedy Policy}
\label{alg:regularized_greedy}
\small
\begin{algorithmic}[1]
\STATE \textbf{Input:} Horizon \(T\), number of arms \(K\), regularization parameters \((\alpha,\beta)\).
\STATE \textbf{Initialize:} Pull each arm once during periods \(t=1,\ldots,K\). Observe rewards \(X_1,\ldots,X_K\), set \(A_t=t\) for \(t=1,\ldots,K\), and initialize
\[
    N_i(K)=1,
    \qquad
    S_i(K)=X_i,
    \qquad
    i=1,\ldots,K.
\]

\FOR{\(t=K+1,\ldots,T\)}
    \STATE Compute the regularized estimates
    \[
        \widehat p_i(t-1)
        =
        \frac{S_i(t-1)+\alpha}{N_i(t-1)+\beta},
        \qquad
        i=1,\ldots,K.
    \]
    \STATE Let
    \[
        \mathcal G_t
        :=
        \arg\max_{1\le i\le K}\widehat p_i(t-1)
    \]
    be the set of current score maximizers.
    \STATE Select \(A_t\) uniformly at random from \(\mathcal G_t\).
    \STATE Observe \(X_t\sim\mathrm{Bernoulli}(p_{A_t})\).
    \STATE Update
    \[
        N_i(t)=N_i(t-1)+\mathbf{1}\{A_t=i\},
        \qquad
        S_i(t)=S_i(t-1)+X_t\mathbf{1}\{A_t=i\},
        \qquad
        i=1,\ldots,K.
    \]
\ENDFOR
\end{algorithmic}
\end{algorithm}

Throughout the analysis, we impose the feasibility condition
\[
    \alpha\ge0,
    \qquad
    \beta\ge0,
    \qquad
    \alpha \ge p_1\beta.
\]
The pure-greedy policy corresponds to \(\alpha=\beta=0\).

\subsection{A Closed-Form Two-Sided Regret Envelope} \label{subsec:main_results}

We now derive our main analytical characterization of the regret $R_{\pi_{\alpha,\beta}}(T)$. For convenience, define
\[
    \Delta_i:=\alpha-p_i\beta,
    \qquad
    \sigma_i^2:=p_i(1-p_i),
    \qquad
    \lambda_i:=\frac{2\Delta_i}{\sigma_i^2},
    \qquad
    i=1,\ldots,K.
\]
For $1\le i<j\le K$, let
\[
    \delta_{ij}:=p_i-p_j,
\]
and recall that $\delta_i=\delta_{1i}=p_1-p_i$ denotes the regret gap of arm $i$. The quantity $\Delta_i$ represents the effective regularization margin of arm $i$, while $\lambda_i$ is the corresponding regularization-adjusted rate parameter that appears in the absorbing-probability characterization below.

The regret decomposition developed in this paper consists of two components: a linear term $R_{\text{linear}}$, which scales with $T$ and captures regret from convergence to a suboptimal arm, and a transient term $R_{\text{transient}}$, which is independent of $T$ and captures the regret incurred during learning. Although this decomposition holds without additional assumptions, obtaining explicit expressions for the two components is challenging because of the discrete Bernoulli observations and the nonlinear regularization induced by $(\alpha,\beta)$. We therefore study a large-regularization asymptotic regime in which the horizon remains fixed while the regularization scale grows. Unlike the classical asymptotic regime with $T\to\infty$, this regime preserves both regret components and permits explicit characterization of the score-minimum and boundary-crossing probabilities. Throughout the remainder of the paper, we consider the following asymptotic regime.
\begin{assumption}[Asymptotic regime]
\label{ass:asymptotic}
As $\Delta_1\to\infty$, the regularization scale satisfies
\[
    \Delta_i=\Theta(\Delta_1),
    \qquad
    \beta=O(\Delta_1),
    \qquad
    i=1,\ldots,K.
\]
Moreover, for a common \(\kappa\in(1/2, 1)\), every nonzero pairwise gap satisfies
\[
    p_i-p_j
    =
    \Theta(\Delta_1^{-\kappa}),
    \qquad
    1\le i<j\le K \text{ such that }p_i>p_j.
\]
\end{assumption}
The asymptotics are taken with respect to the regularization scale $\Delta_1$, while the horizon $T$ remains fixed throughout the analysis. The conditions $\Delta_i=\Theta(\Delta_1)$ and $\beta=O(\Delta_1)$ ensure that the regularized scores of all arms grow on the same asymptotic scale.

The assumption $p_i-p_j=\Theta(\Delta_1^{-\kappa})$ specifies how distinct arm means evolve as the regularization increases. This scaling ensures that both the transient component and the linear component of the regret admit explicit asymptotic characterizations within the same asymptotic regime. Consequently, the resulting asymptotic analysis yields explicit regret expressions while retaining the finite-horizon regret decomposition. Section~\ref{subsec:empirical_tightness} further demonstrates that the resulting regret envelopes remain accurate well beyond the asymptotic regime, including the moderate regularization levels used in our numerical experiments.

Given this asymptotic regime, we now present the principal result of the paper. Theorem \ref{thm:main_sandwich} characterizes the finite-horizon regret of the regularized policies with $\alpha>0$ through upper and lower envelopes, and then derives the corresponding finite-horizon regret characterization for the classical greedy policy. The proof roadmap of the theorem will be discussed in Section \ref{subsec:roadmap}.

\begin{theorem}[Two-sided finite-horizon regret envelope]
\label{thm:main_sandwich}
Consider a $K$-armed Bernoulli bandit with success probabilities satisfying $1-\epsilon_{\mathrm p} \ge p_1\ge p_2\ge\cdots\ge p_K \ge\epsilon_{\mathrm p}$ and $p_1>p_K$. Let $\pi_{\alpha,\beta}$ denote the regularized greedy policy with parameters $(\alpha,\beta)$.

\begin{enumerate}
    \item \textbf{Regularized greedy algorithm $(\alpha>0)$.} Suppose $\alpha\ge p_1\beta\ge0$ and Assumption~\ref{ass:asymptotic} holds. For $i=2,\ldots,K$ with $p_i < p_1$, define the closed-form approximation to the absorption probability of suboptimal arm $i$ by
    \[
    \widetilde{\mathbb P}(Q_i):=\sum_{m=i}^{K}\frac{1}{m}\exp\left\{-\sum_{h=1}^{m}\lambda_h\delta_{hm}\right\}\left[1-\exp\left\{-\left(\sum_{h=1}^{m}\lambda_h\right)\delta_{m,m+1}\right\}\right],
    \]
where \(\delta_{hh}:=0\), \(\delta_{K,K+1}:=\infty\), and the second exponential in the bracket is interpreted as zero when \(m=K\).
Define
\[R_{\mathrm{linear}}(T):=T\sum_{\substack{i=2\\p_i<p_1}}^{K}\delta_i\widetilde{\mathbb P}(Q_i), \qquad
    R_{\mathrm{transient}}:=\mathcal C_\Delta:=(K-1)\Delta_1.
    \]
    Then, for every finite horizon $T$,
    \[
    R_{\mathrm{linear}}(T)(1-o(1))-o(1)\le R_{\pi_{\alpha,\beta}}(T)\le (R_{\mathrm{linear}}(T)+R_{\mathrm{transient}})(1+o(1)).
    \]

\item \textbf{Pure greedy algorithm $(\alpha=\beta=0)$.} For the classical greedy policy, define
    \[
    \underline R_{\mathrm{linear}}^{\mathrm{pg}}(T):=T\sum_{\substack{i=2\\p_i<p_1}}^{K}\delta_i\underline{\mathbb P}_{\mathrm{pg}}(Q_i),\qquad \overline R_{\mathrm{linear}}^{\mathrm{pg}}(T):=T\sum_{\substack{i=2\\p_i<p_1}}^{K}\delta_i\overline{\mathbb P}_{\mathrm{pg}}(Q_i),
    \]
    and
    \[
    \underline R_{\mathrm{transient}}^{\mathrm{pg}}(T):=\sum_{i=1}^{K} \sum_{\substack{j \neq i\\p_j>p_i}} (p_j-p_i)\overline{\mathcal U}_{j|i,T}^{\mathrm{pg}},\qquad \overline R_{\mathrm{transient}}^{\mathrm{pg}}(T):=\sum_{i=1}^{K} \sum_{\substack{j \neq i\\p_j<p_i}} (p_i-p_j)\overline{\mathcal U}_{j|i,T}^{\mathrm{pg}}.
    \]
    Then
    \[
    \underline R_{\mathrm{linear}}^{\mathrm{pg}}(T)-\underline R_{\mathrm{transient}}^{\mathrm{pg}}(T)\le R_{\pi_{0,0}}(T)\le \overline R_{\mathrm{linear}}^{\mathrm{pg}}(T)+\overline R_{\mathrm{transient}}^{\mathrm{pg}}(T),
    \]
    where $\underline{\mathbb P}_{\mathrm{pg}}(Q_i)$, $\overline{\mathbb P}_{\mathrm{pg}}(Q_i)$, and $\overline{\mathcal U}_{j|i,T}^{\mathrm{pg}}$ are the explicit branchwise quantities defined in Appendix~\ref{app:pure_greedy_regime}.
\end{enumerate}
\end{theorem}

Theorem~\ref{thm:main_sandwich} establishes upper and lower bounds for the finite-horizon regret of the regularized greedy policy. The regret is decomposed into two explicit terms. The first,
\[
R_{\mathrm{linear}}(T)=T\sum_{\substack{i=2\\p_i<p_1}}^{K}\delta_i\widetilde{\mathbb P}(Q_i),
\]
is the regret contributed by eventual convergence to suboptimal arms. The second, $R_{\mathrm{transient}}$, is the finite cost incurred before the remaining arms are abandoned.

The role of regularization is visible directly from the expression for $\widetilde{\mathbb P}(Q_i)$. Each term is a weighted sum of exponentials whose exponent is proportional to the rate parameters $\lambda_i=2\Delta_i/\sigma_i^2$. Increasing the regularization margin therefore increases $\lambda_i$ linearly, causing the probabilities of suboptimal convergence to decay exponentially. At the same time, larger regularization delays abandonment and increases $R_{\mathrm{transient}}$. The theorem therefore quantifies the finite-horizon trade-off between an exponentially decreasing linear-regret component and an increasing transient component through upper and lower bounds. The theorem also gives the corresponding finite-horizon regret characterization for the classical greedy policy. 

The remainder of this section first demonstrates that the first-order regret envelopes accurately characterize finite-horizon regret through numerical experiments. We then outline the main ideas of the proof of the theorem, with its complete technical arguments deferred to the online appendix. Section~\ref{sec:calibration} then uses the regret decomposition to derive calibration rules for $(\alpha,\beta)$ that balance the linear and transient regret components.

\subsection{Empirical Tightness of the Regret Envelope}
\label{subsec:empirical_tightness}

To illustrate the performance of these bounds, we examine the finite-horizon accuracy of the regret envelope in Theorem~\ref{thm:main_sandwich} for representative \(K\)-armed Bernoulli bandits. For each displayed value of \(K\), we generate one random instance with arm means sampled independently from \(\mathrm{Uniform}[0.01,0.99]\). Every Monte Carlo curve is averaged over \(N=10{,}000\) replications, and the horizon is \(T=5{,}000\). Since the policy first pulls each arm once, all figures display only the region \(T\ge K\).

\paragraph{Regularized greedy.}

Figure~\ref{fig:rt_holistic} compares the regret envelope of Theorem~\ref{thm:main_sandwich} with realized regret under three representative regularization choices. The left column considers the two-arm case $(K=2)$, and the right column considers the ten-arm case $(K=10)$. The first two rows vary $\alpha$ with $\beta=0$, while the final row considers balanced regularization with $\alpha=\beta$.

Across all six configurations, the regret envelope closely matches the realized regret. The agreement remains strong even for moderate regularization levels and randomly generated arm means, suggesting that the asymptotic approximation remains accurate well beyond the asymptotic regime of Theorem~\ref{thm:main_sandwich}. In a few panels, the asymptotic upper envelope lies slightly below the Monte Carlo curve; this small discrepancy reflects finite-regularization effects from the omitted \(o(1)\) remainders.

Although the regret appears to plateau over the displayed horizon, it is not constant. Theorem~\ref{thm:main_sandwich} shows that the regret remains asymptotically linear in $T$ for any fixed regularization parameters. The apparent saturation arises because the linear coefficient is exponentially small in the regularization level, so over practical horizons the linear growth is dominated by the transient term. Increasing the regularization further reduces this coefficient by suppressing the probability of permanent commitment to suboptimal arms, causing the realized regret to appear nearly flat even though its asymptotic growth remains linear.

\begin{figure}[ht]
\centering
\captionsetup[subfigure]{font=small}
\begin{subfigure}[t]{0.48\linewidth}
\includegraphics[width=\linewidth]{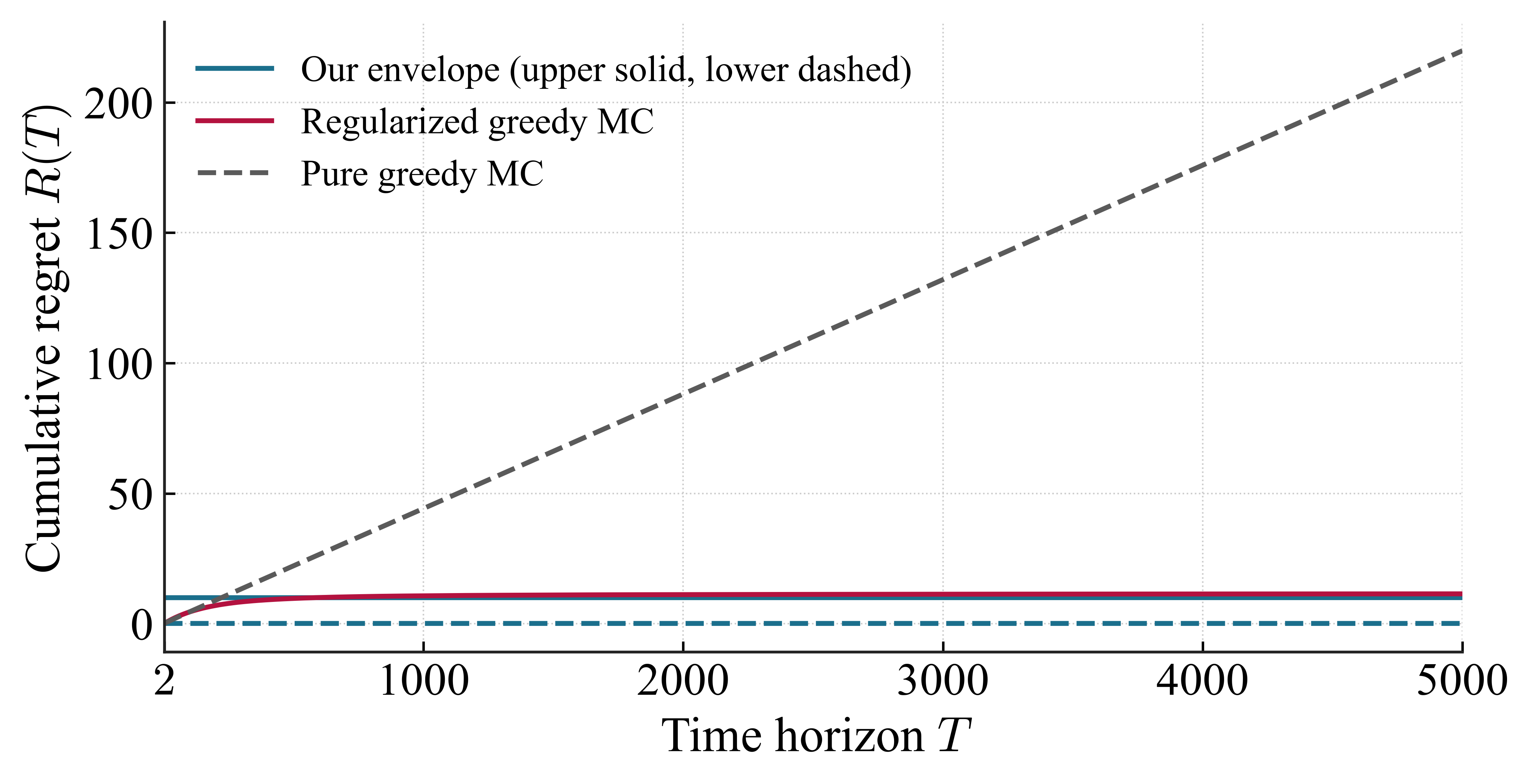}
\caption{\(K=2,\;(\alpha,\beta)=(10,0)\)}
\label{fig:rt_reg_K2_a10_b0}
\end{subfigure}
\hfill
\begin{subfigure}[t]{0.48\linewidth}
\includegraphics[width=\linewidth]{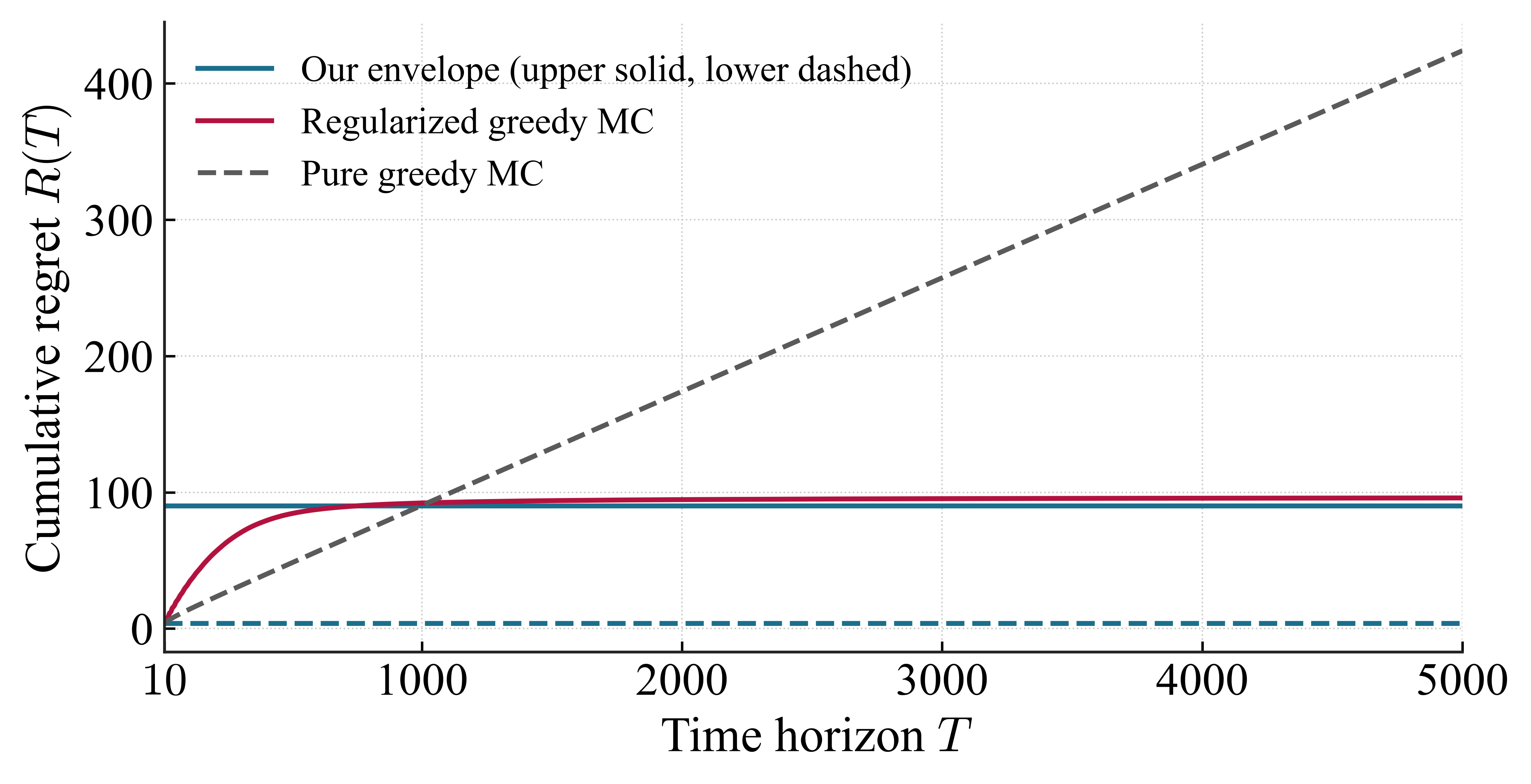}
\caption{\(K=10,\;(\alpha,\beta)=(10,0)\)}
\label{fig:rt_reg_K10_a10_b0}
\end{subfigure}

\vspace{0.5em}

\begin{subfigure}[t]{0.48\linewidth}
\includegraphics[width=\linewidth]{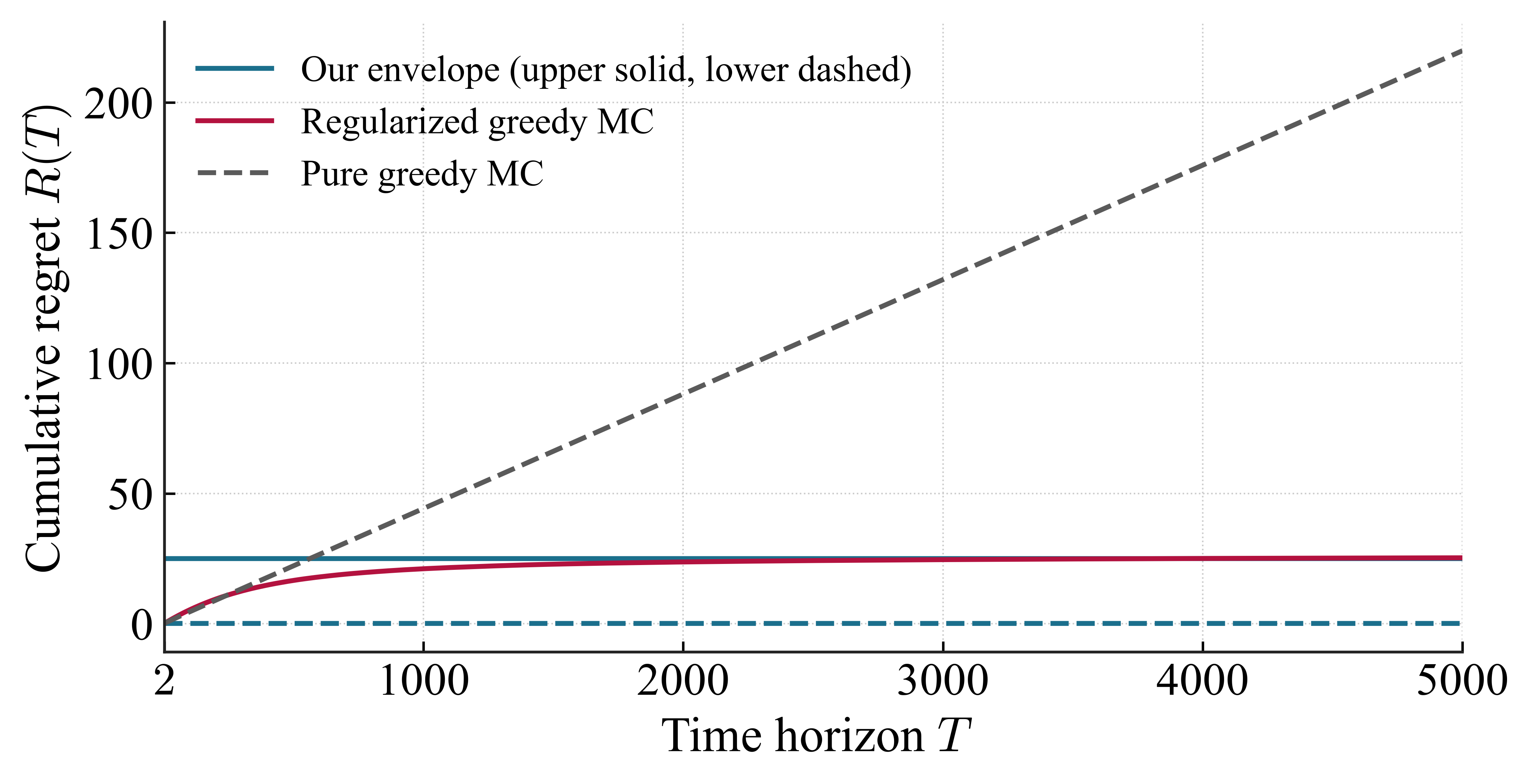}
\caption{\(K=2,\;(\alpha,\beta)=(25,0)\)}
\label{fig:rt_reg_K2_a25_b0}
\end{subfigure}
\hfill
\begin{subfigure}[t]{0.48\linewidth}
\includegraphics[width=\linewidth]{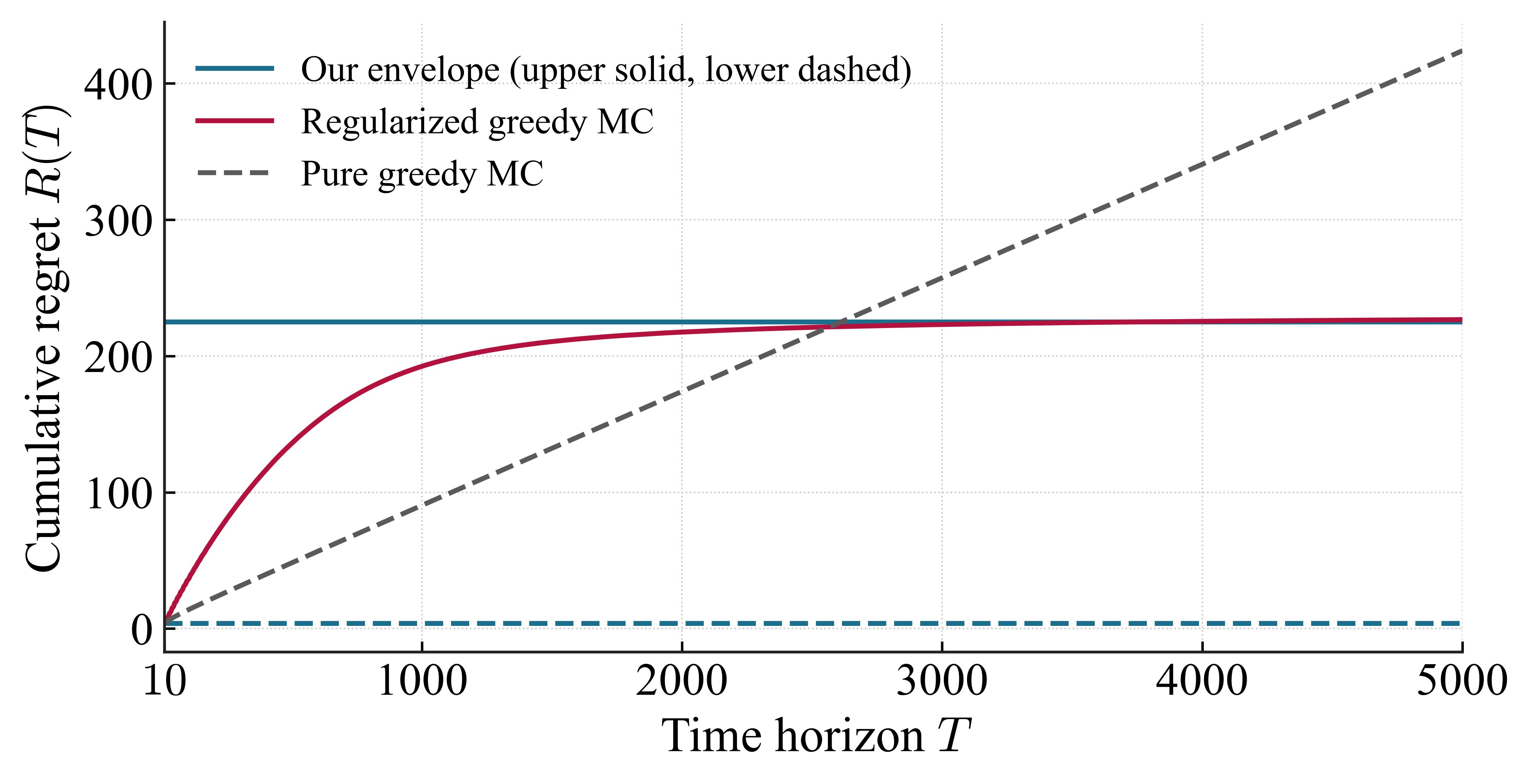}
\caption{\(K=10,\;(\alpha,\beta)=(25,0)\)}
\label{fig:rt_reg_K10_a25_b0}
\end{subfigure}

\vspace{0.5em}

\begin{subfigure}[t]{0.48\linewidth}
\includegraphics[width=\linewidth]{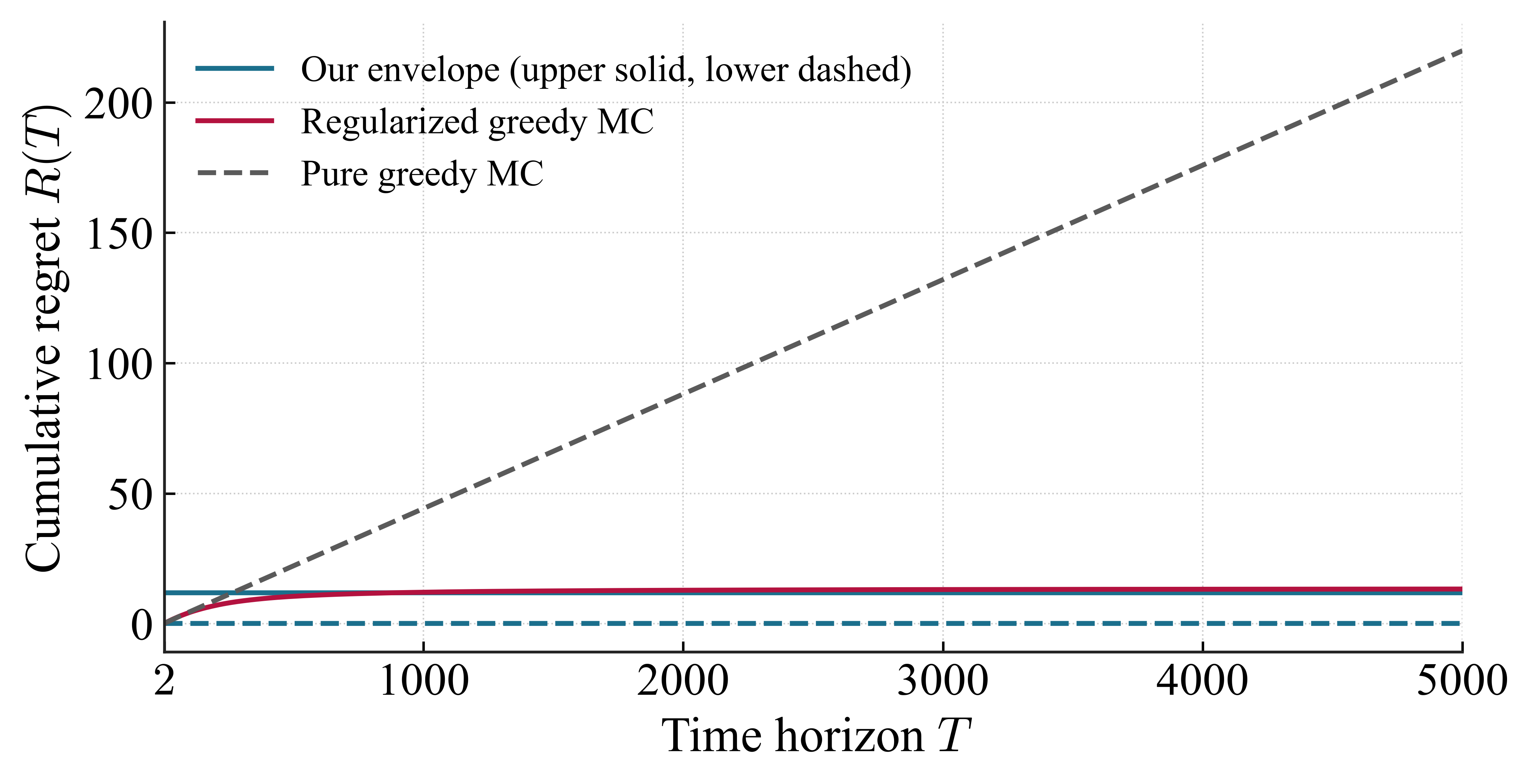}
\caption{\(K=2,\;(\alpha,\beta)=(25,25)\)}
\label{fig:rt_reg_K2_a25_b25}
\end{subfigure}
\hfill
\begin{subfigure}[t]{0.48\linewidth}
\includegraphics[width=\linewidth]{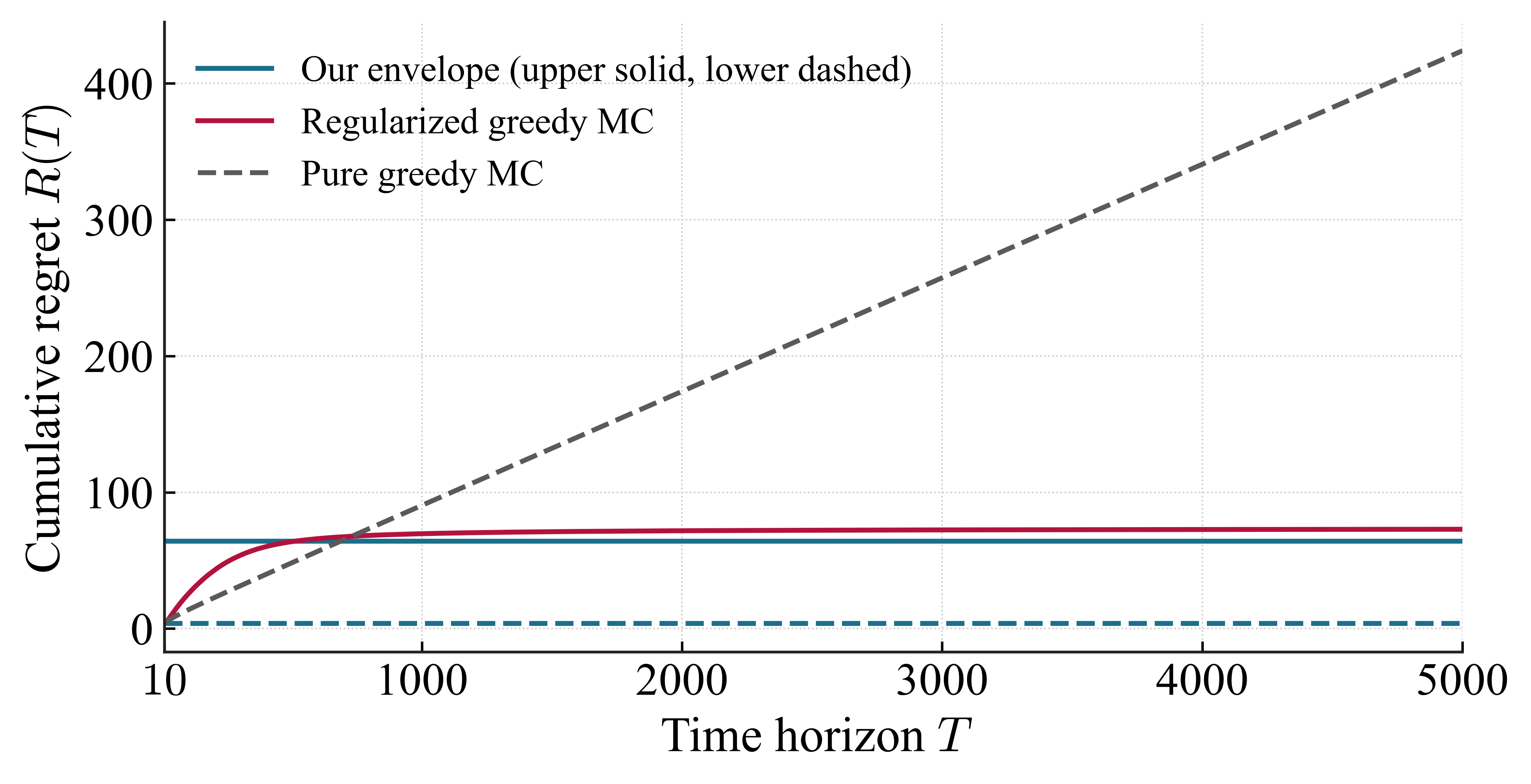}
\caption{\(K=10,\;(\alpha,\beta)=(25,25)\)}
\label{fig:rt_reg_K10_a25_b25}
\end{subfigure}

\caption{Finite-horizon accuracy of the regularized-greedy regret envelope in Theorem~\ref{thm:main_sandwich}. Each panel compares Monte Carlo (MC) regret, the analytical two-sided envelope, and the pure-greedy reference for the displayed \(K\) and \((\alpha,\beta)\).}
\label{fig:rt_holistic}
\end{figure}

\paragraph{Pure greedy.}

We next consider the special case $\alpha=\beta=0$. This regime permits a direct comparison with the existing analytical literature because several regret characterizations are available for the classical greedy policy. We compare Part~2 of Theorem~\ref{thm:main_sandwich} with the upper bound of \citet{jedor2021greedy}, the singleton-failure lower benchmark motivated by their greedy failure example, and the Bayesian greedy formula of \citet{bayati2020unreasonable}. We note that the analysis of \citet{bayati2020unreasonable} is derived under a many-arm asymptotic requiring $K\ge 30\log T/c_0$ (e.g., $K\ge256$ when $T=5{,}000$ and $c_0=1$), whereas our experiments consider substantially smaller values of $K$. We therefore include their result as an analytical point of reference rather than a benchmark. 

Figure~\ref{fig:pg_grid} reports the comparison for $K\in\{2,5,8,10\}$. Across all configurations, the branchwise characterization of Theorem~\ref{thm:main_sandwich} yields substantially tighter regret bounds than the existing analytical benchmarks. The upper envelope closely tracks the realized regret, while the lower envelope substantially improves upon the previously available analytical lower benchmark.

Figures~\ref{fig:rt_holistic} and~\ref{fig:pg_grid} together show that the regret envelopes of Theorem~\ref{thm:main_sandwich} remain accurate across a broad range of regularization strengths and numbers of arms. In the regularized regime, the envelopes accurately characterize both the magnitude of finite-horizon regret and the effect of regularization. In the pure-greedy regime, the same analytical framework yields substantially sharper regret bounds than the existing literature.

\begin{figure}[ht]
\centering
\captionsetup[subfigure]{font=small}
\begin{subfigure}[t]{0.48\linewidth}
\includegraphics[width=\linewidth]{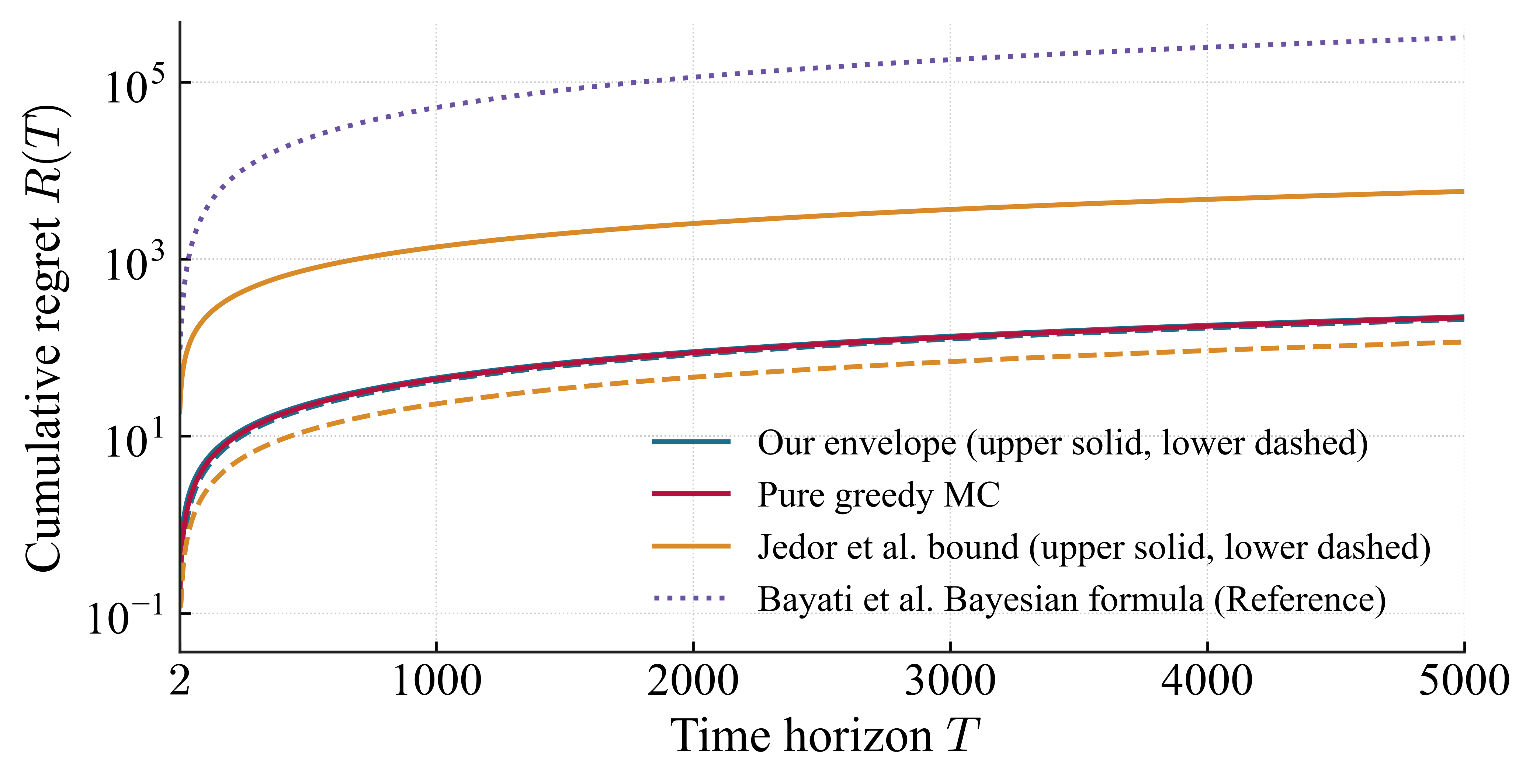}
\caption{\(K=2\)}
\label{fig:pg_K2}
\end{subfigure}
\hfill
\begin{subfigure}[t]{0.48\linewidth}
\includegraphics[width=\linewidth]{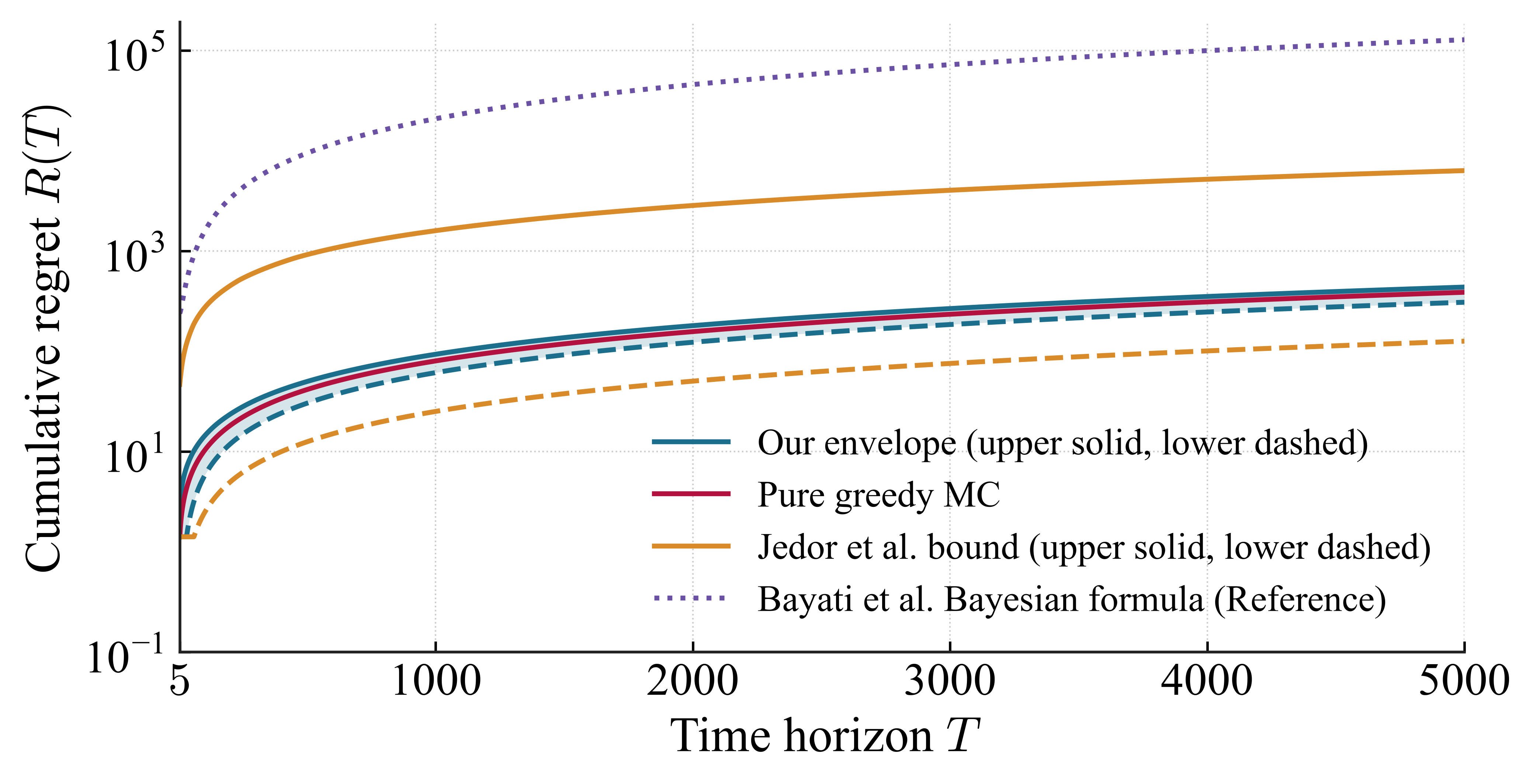}
\caption{\(K=5\)}
\label{fig:pg_K5}
\end{subfigure}

\vspace{0.5em}

\begin{subfigure}[t]{0.48\linewidth}
\includegraphics[width=\linewidth]{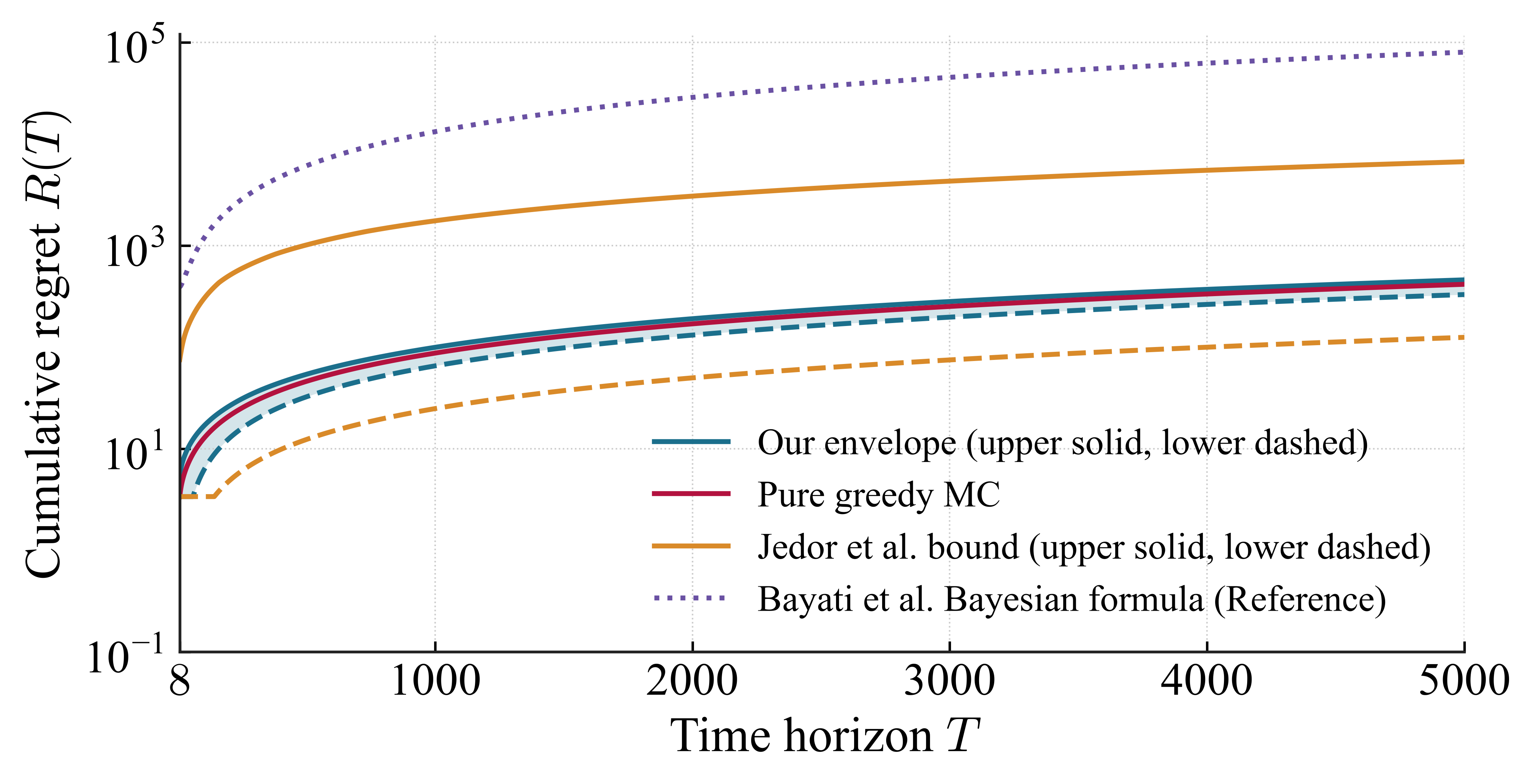}
\caption{\(K=8\)}
\label{fig:pg_K8}
\end{subfigure}
\hfill
\begin{subfigure}[t]{0.48\linewidth}
\includegraphics[width=\linewidth]{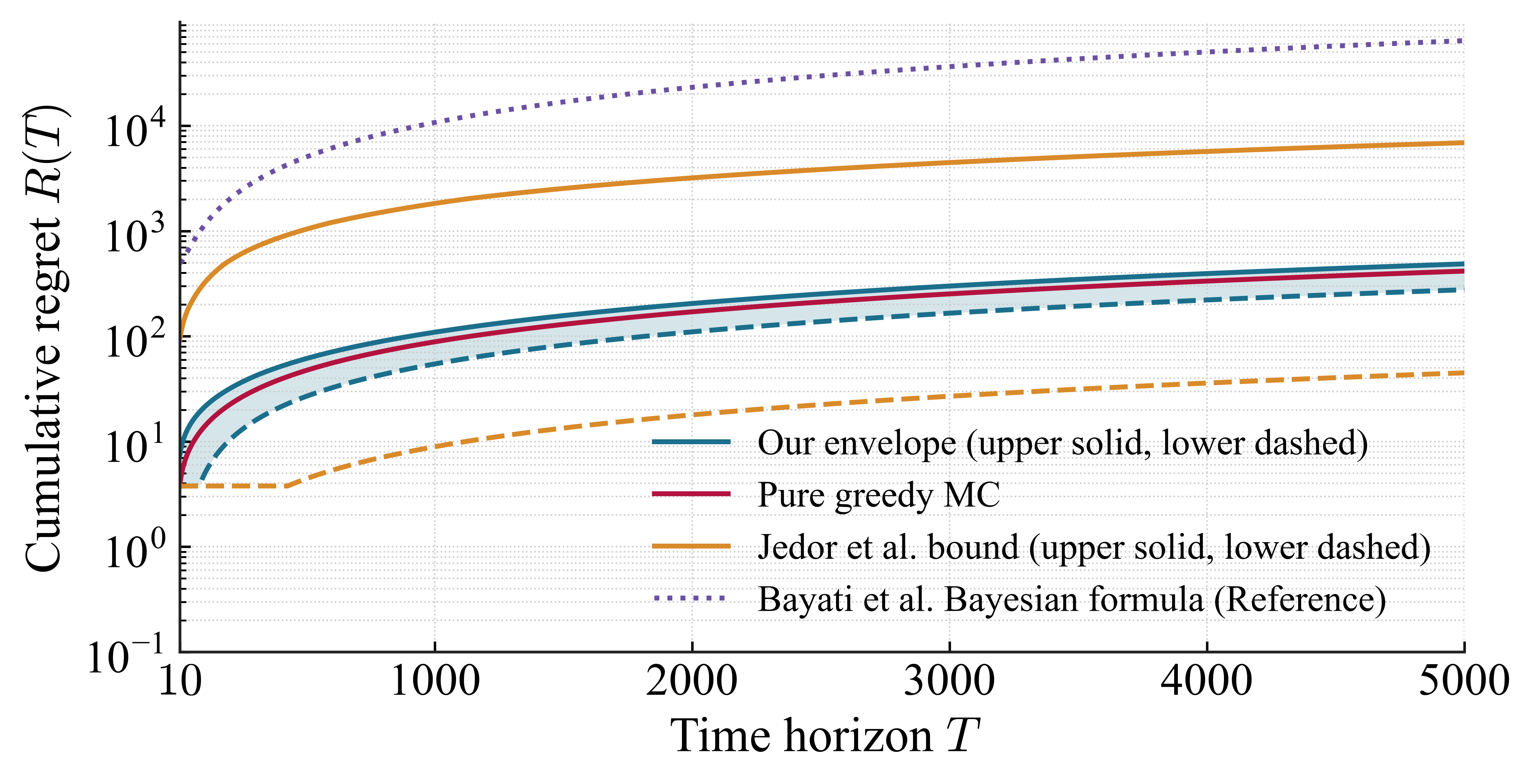}
\caption{\(K=10\)}
\label{fig:pg_K10}
\end{subfigure}

\caption{Pure-greedy regret envelope from Part~2 of Theorem~\ref{thm:main_sandwich}. Each panel compares MC regret, our branchwise envelope, and existing analytical benchmarks for the displayed \(K\).}
\label{fig:pg_grid}
\end{figure}

We next outline the proof of Theorem~\ref{thm:main_sandwich}. The analysis proceeds by characterizing the probabilities of absorption into each arm together with the transient regret accumulated before the remaining arms are abandoned. It shows how the finite-horizon regret decomposition reduces to these two quantities.

\subsection{Proof Roadmap: Reduction to Absorbing Probabilities and Conditional Abandonment Costs}
\label{subsec:roadmap}

We recall that the regret of the regularized greedy policy can be written as
\[
R(T)=\sum_{i=2}^{K}\delta_i\,\mathbb E[N_i(T)].
\]
The proof begins by separating the pulls of suboptimal arms into two sources. The first source is suboptimal absorption. If the policy eventually commits to a suboptimal arm, it incurs regret at rate $\delta_i$ for the remainder of the horizon. The second source is transient sampling: arms that are eventually abandoned may still be pulled a finite number of times before leaving the competition.

Our first step is to show that the greedy trajectory almost surely absorbs into a single arm (Lemma~\ref{lem:topological_collapse_karm}). For each arm $i$, define
\[
Q_i:=\{\exists\,\widetilde T<\infty\text{ such that }A_t=i\text{ for all }t\ge\widetilde T\},\qquad i=1,\ldots,K,
\]
where $\widetilde T$ denotes the (path-dependent) absorption time. The events $Q_1,\ldots,Q_K$ form an almost-sure partition of the sample space. Conditional on $Q_i$, define the capped transient pulls of arm $j\neq i$ by
\[
\mathcal U_{j|i,T}:=\mathbb E\!\left[\min\{T,N_j\}\,\middle|\,Q_i\right],
\]
where $N_j$ is the total number of pulls of arm $j$ over the infinite trajectory. Terms conditioned on probability-zero events are interpreted as zero.

Conditioning on the absorbing partition yields the regret sandwich
\begin{equation}
\label{eq:main_karm_sandwich_symbolic}
\begin{aligned}
& T\sum_{\substack{i=2\\p_i<p_1}}^{K}\delta_i\,\mathbb P(Q_i)
-
\sum_{i=1}^{K}\mathbb P(Q_i)
\sum_{\substack{j\neq i\\p_j>p_i}}
(p_j-p_i)\,
\mathcal U_{j|i,T} \\
\leq &
R(T) \\
\leq &
T\sum_{\substack{i=2\\p_i<p_1}}^{K}\delta_i\,\mathbb P(Q_i)
+
\sum_{i=1}^{K}\mathbb P(Q_i)
\sum_{\substack{j\neq i\\p_j<p_i}}
(p_i-p_j)\,
\mathcal U_{j|i,T}.
\end{aligned}
\end{equation}

The leading term is the linear absorption regret,
\[
R_{\mathrm{linear}}(T)
=
T\sum_{\substack{i=2\\p_i<p_1}}^{K}\delta_i\,\mathbb P(Q_i),
\]
while the remaining terms capture the transient abandonment cost. The remainder of the proof thus reduces to two tasks: characterizing the absorbing probabilities $\mathbb P(Q_i)$ and bounding the conditional abandonment costs $\mathcal U_{j|i,T}$. Section~\ref{subsubsec:main_absorbing_probability_evaluation} provides the main ingredients in our characterization of  $\mathbb P(Q_i)$, and Section~\ref{subsubsec:main_abandonment_and_synthesis} establishes the corresponding conditional abandonment costs $\mathcal U_{j|i,T}$.

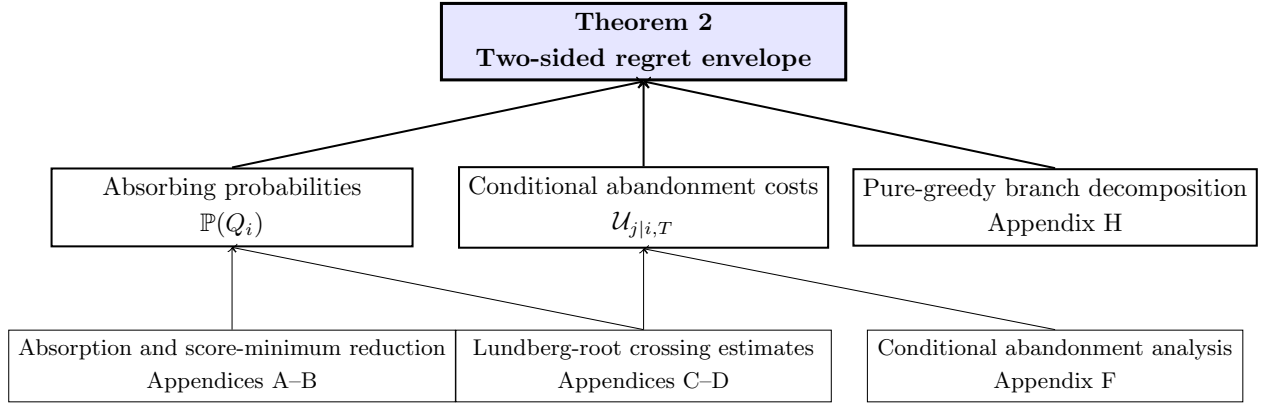
\begin{figure}[ht]
\centering
\captionsetup{justification=justified, singlelinecheck=false}
\resizebox{\linewidth}{!}{%
\begin{tikzpicture}[
  every node/.style={font=\small},
  topnode/.style={
    rectangle, draw, very thick,
    minimum height=0.9cm, minimum width=5.6cm,
    fill=blue!10, align=center, font=\small\bfseries
  },
  midnode/.style={
    rectangle, draw, thick,
    minimum height=0.95cm, minimum width=5.0cm,
    align=center, font=\small
  },
  botnode/.style={
    rectangle, draw,
    minimum height=0.9cm, minimum width=5.2cm,
    align=center, font=\footnotesize
  }
]
\node[topnode] (main) at (0,5.4)
{Theorem~\ref{thm:main_sandwich}\\[2pt]
Two-sided regret envelope};

\node[midnode] (prob) at (-5.7,3.1)
{Absorbing probabilities\\[2pt]
$\mathbb P(Q_i)$};

\node[midnode] (cost) at (0,3.1)
{Conditional abandonment costs\\[2pt]
$\mathcal U_{j|i,T}$};

\node[midnode] (pure) at (5.7,3.1)
{Pure-greedy branch decomposition\\[2pt]
Appendix~\ref{app:pure_greedy_regime}};

\node[botnode] (score) at (-5.7,0.9)
{Absorption and score-minimum reduction\\[2pt]
Appendices~\ref{app:absorption}--\ref{app:score_minimum_reduction}};

\node[botnode] (lundberg) at (0,0.9)
{Lundberg-root crossing estimates\\[2pt]
Appendices~\ref{app:cgf_drawdown_geometry}--\ref{app:lundberg_overshoot_bounds}};

\node[botnode] (abandon) at (5.7,0.9)
{Conditional abandonment analysis\\[2pt]
Appendix~\ref{app:conditional_abandonment_costs}};

\draw[->, thick] (prob.north) -- (main.south);
\draw[->, thick] (cost.north) -- (main.south);
\draw[->, thick] (pure.north) -- (main.south);

\draw[->] (score.north) -- (prob.south);
\draw[->] (lundberg.north) -- (prob.south);
\draw[->] (lundberg.north) -- (cost.south);
\draw[->] (abandon.north) -- (cost.south);

\end{tikzpicture}%
}
\caption{Dependency structure of the proof of Theorem~\ref{thm:main_sandwich}. Arrows read ``provides input to''. The regularized proof combines the absorbing-probability analysis with conditional abandonment-cost bounds. The pure-greedy part is handled separately through the initialization-branch decomposition.}
\label{fig:proof_roadmap}
\end{figure}

\subsubsection{Absorbing Probabilities: Score Minima and Exponential Evaluation.}
\label{subsubsec:main_absorbing_probability_evaluation}

We first characterize the absorbing probabilities $\mathbb P(Q_i)$ that determine the linear regret component $R_{\mathrm{linear}}(T)$. While absorption is defined through the entire adaptive trajectory of the greedy policy, it admits an equivalent characterization through the minimum regularized score attained by each arm. This reduction converts the adaptive multi-arm process into a collection of independent one-arm objects, which can then be analyzed explicitly.

For each arm $i$, let $S_i(n)$ denote the cumulative successes along an independent potential reward stream after $n$ pulls, and define
\[
M_i:=\inf_{n\ge1}\frac{S_i(n)+\alpha}{n+\beta},\qquad L_i:=p_i-M_i.
\]
Here $M_i$ is the minimum regularized score attained by arm $i$, while $L_i$ is the corresponding drawdown below the arm mean. For each $x\in(0,p_i)$, define the associated one-arm boundary
\[
b_i(x):=\alpha-(p_i-x)\beta=\Delta_i+\beta x.
\]

The following theorem shows that these score minima determine the absorbing events and characterizes their tail probabilities.

\begin{theorem}[Score-minimum reduction and one-arm tail envelope]
\label{thm:main_score_minimum_reduction}
For every $i=1,\ldots,K$, the absorbing event $Q_i$ satisfies
\[
\left\{L_j>L_i+p_j-p_i,\ \forall j\neq i\right\}
\subseteq
Q_i
\subseteq
\left\{L_j\ge L_i+p_j-p_i,\ \forall j\neq i\right\}.
\]
Moreover, for every $x\in(0,p_i)$, the strict and weak one-arm tails satisfy
\[
\exp\left\{-\theta_i(x)\bigl[b_i(x)+p_i-x\bigr]\right\}
\le
\mathbb P(L_i>x)
\le
\mathbb P(L_i\ge x)
\le
\exp\left\{-\theta_i(x)b_i(x)\right\},
\]
where $\theta_i(x)>0$ is the unique positive Lundberg root satisfying
\[
(1-p_i)\exp\{\theta_i(x)(p_i-x)\}
+
p_i \exp\{-\theta_i(x)(1-p_i+x)\}
=
1.
\]
\end{theorem}

The proof of Theorem~\ref{thm:main_score_minimum_reduction} is given in Appendices~\ref{app:score_minimum_reduction}--\ref{app:lundberg_overshoot_bounds}. The theorem consists of two steps. The first reduces the adaptive absorbing event to a comparison of the independent score minima $L_1,\ldots,L_K$, eliminating the need to analyze the full greedy trajectory directly. The second characterizes the distribution of each score minimum through a one-arm boundary-crossing problem. Specifically, the event $\{L_i>x\}$ is equivalent to a Bernoulli random walk crossing the boundary $b_i(x)$, whose probability is governed by the Lundberg exponent $\theta_i(x)$.

To illustrate the main idea, we give an informal derivation of the absorbing-probability formula; the rigorous approximation arguments are deferred to the appendix. Under the large-regularization regime, the relevant score minima lie near the lower endpoint $x=0$. As $x\downarrow0$,
\[
\theta_i(x)=\frac{2}{p_i(1-p_i)}x+O(x^2),\qquad
b_i(x)=\Delta_i+\beta x,
\]
so
\[
\theta_i(x)b_i(x)=\lambda_i x+O(\Delta_i x^2+\beta x^2),\qquad
\lambda_i:=\frac{2\Delta_i}{p_i(1-p_i)}.
\]
Consequently,
\[
\mathbb P(L_i>x)=\exp\{-\lambda_i x\}[1+o(1)],
\]
so each score minimum is locally exponential with rate $\lambda_i$. Larger regularization margins therefore make large score drawdowns exponentially less likely.

For arm $i$ to absorb, its score minimum must be sufficiently small, while the score minima of the competing arms must exceed the shifted levels $L_i+p_j-p_i$. Conditioning on $L_i=\ell$ therefore gives the Stieltjes representation
\[
\mathbb P(Q_i)\approx\int_{[0,p_i)}\prod_{j\neq i}\mathbb P\!\left(L_j\ge\ell+p_j-p_i\right)\,dF_i(\ell),
\]
where $F_i$ denotes the distribution of $L_i$. Under large regularization, the integral is concentrated on $\ell=O(\Delta_i^{-1})$, allowing the local exponential approximation above to be substituted into the integrand. This reduces the absorbing probability to the reference integral
\[
\lambda_i\int_0^\infty
\exp\left\{-\lambda_i\ell-\sum_{j\neq i}\lambda_j(\ell+p_j-p_i)_+\right\}\,d\ell.
\]
The positive-part terms partition the integral according to the ordered gaps $\delta_{ij}$. For $j<i$, the shift is always nonnegative and contributes the constant factor $\exp\{-\lambda_j\delta_{ji}\}$. For $j>i$, the term $(\ell-\delta_{ij})_+$ becomes active only after $\ell$ crosses the ordered gap $\delta_{ij}$. Evaluating the integral over these successive regions yields the following asymptotic absorbing-probability formula.

\begin{theorem}[Asymptotic absorbing-probability formula]
\label{thm:main_absorbing_probability_formula}
Suppose Assumption~\ref{ass:asymptotic} holds. For every suboptimal arm $i=2,\ldots,K$ with $p_i < p_1$,
\[\begin{aligned}
\mathbb P(Q_i)
&=
\sum_{m=i}^{K}\frac{1}{m}\exp\left\{-\sum_{h=1}^{m}\lambda_h\delta_{hm}\right\}\left[1-\exp\left\{-\left(\sum_{h=1}^{m}\lambda_h\right)\delta_{m,m+1}\right\}\right]
(1+o(1)),
\end{aligned}\]
where empty sums are interpreted as zero, $\delta_{hh}:=0$, and $\delta_{K,K+1}:=\infty$. For $m=K$, the second exponential in the bracket is interpreted as zero.
\end{theorem}

The proof is given in Appendix~\ref{app:asymptotic_absorbing_probability_formula}. Substituting these absorbing probabilities into the decomposition of Section~\ref{subsec:roadmap} yields the linear regret contribution
\[
R_{\mathrm{linear}}(T)
=
T\sum_{\substack{i=2\\p_i<p_1}}^{K}\delta_i\,\widetilde{\mathbb P}(Q_i),
\]
where $\widetilde{\mathbb P}(Q_i)$ denotes the asymptotic expression in Theorem~\ref{thm:main_absorbing_probability_formula}. The remaining task is therefore to characterize the transient abandonment cost, which determines the second component of the finite-horizon regret envelope.

\subsubsection{Conditional Abandonment Costs and Regret Synthesis.}
\label{subsubsec:main_abandonment_and_synthesis}

It remains to characterize the transient abandonment cost appearing in the regret decomposition of Section~\ref{subsec:roadmap}. Conditional on the absorbing event $Q_i$, the absorbing arm determines the score level that every remaining arm must eventually cross before it is abandoned. Once the absorbing arm is fixed, the evolution of each nonabsorbing arm thus decouples and reduces to a collection of one-arm first-passage problems.

Relative to the realized score minimum of the absorbing arm, the comparison walk for arm $j$ has drift
\[
p_i-p_j-\ell,
\]
where $\ell$ denotes the realized centered score minimum of the absorbing arm. When $p_i>p_j$, this drift changes sign at $\ell=\delta_{ij}$, separating the analysis into positive- and negative-drift regimes. When $p_j>p_i$, the drift remains negative throughout. When \(p_i=p_j\), the finite-horizon cap directly gives \(\mathcal U_{j|i,T}\le T\). Applying the one-arm boundary-crossing analysis to these two regimes yields the following conditional abandonment bounds.

\begin{theorem}[Conditional abandonment envelope]
\label{thm:main_conditional_abandonment_envelope}
Under Assumption~\ref{ass:asymptotic}, for every \(i\neq j\),
\[
% \label{eq:main_conditional_abandonment_piecewise}
\mathcal U_{j|i,T}
\le
\begin{cases}
\displaystyle
\frac{\Delta_i}{\delta_{ij}}
\left(1+O(\Delta_1^{\kappa-1})\right)
+
T\,o\!\left(\dfrac{\mathbb P(Q_j)}{\mathbb P(Q_i)}\right),
&
p_i>p_j,
\\[3ex]
T, 
&
p_i=p_j,
\\[3ex]
\displaystyle
\frac{\Delta_i}{\delta_{ji}}
\left(1+o(1)\right)
+
T\,o(1),
&
p_i<p_j.
\end{cases}
\]
Moreover, on the optimal branches \(Q_i,\, i \in [K]\) with \(p_i=p_1\),
\[
\sum_{\substack{i=1\\p_i=p_1}}^K \mathbb P(Q_i) \sum_{\substack{j \neq i \\ p_j<p_i}} (p_i-p_j) \mathcal{U}_{j|i,T} \leq \mathcal{C}_\Delta(1+o(1)) +  R_{\mathrm{linear}}(T) o(1)
\]
\end{theorem}

The proof of Theorem~\ref{thm:main_conditional_abandonment_envelope} appears in Appendix~\ref{app:conditional_abandonment_costs}. There we integrate the pointwise first-passage bounds against the conditional distribution of the absorbing score minimum. The same endpoint localization used in the absorbing-probability analysis shows that only the lower-endpoint behavior of the absorbing score minimum contributes at first order, which yields the explicit conditional abandonment bounds in the theorem. We treat the optimal branches $Q_i$ with \(p_i=p_1\) separately because they contribute to the transient term at first order.

With both the absorbing probabilities and conditional abandonment costs characterized, only the synthesis step remains. Appendix~\ref{app:cumulative_regret_envelope} combines Theorems~\ref{thm:main_absorbing_probability_formula} and~\ref{thm:main_conditional_abandonment_envelope} to establish our main result (Theorem~\ref{thm:main_sandwich}). The absorbing probabilities determine the linear regret component $R_{\mathrm{linear}}(T)$, while the conditional abandonment bounds determine the transient component $R_{\mathrm{transient}}$.

The same regret decomposition also yields the pure-greedy result. Its analysis follows the same structure but requires a different initialization-branch decomposition, and is deferred to Appendix~\ref{app:pure_greedy_regime}.

\section{Horizon-Calibrated Greedy Policies}
\label{sec:calibration}

Theorem~\ref{thm:main_sandwich} characterizes the finite-horizon regret of regularized greedy as a function of the prior $(\alpha,\beta)$. The natural question is how the prior should be chosen. The regret bound shows that increasing the prior reduces suboptimal absorption but increases transient regret, so the prior is selected by optimizing this finite-horizon trade-off.

We first study an oracle setting in which the horizon and arm means are known. The resulting oracle policy serves as a benchmark and motivates the practical algorithms developed later for settings in which these quantities are unknown.

\subsection{Oracle Calibration}
\label{subsec:oracle}

Suppose the horizon $T$ and the arm means $p_1 \geq \cdots \geq p_K$ are known. We calibrate the regularization parameters by minimizing the leading-order regret envelope of Theorem~\ref{thm:main_sandwich}. Dropping the $o(1)$ remainder, define
\begin{equation}
\label{eq:cal_main_bound_def}
\overline R(T;\alpha,\beta)=T\sum_{\substack{i=2\\ p_i<p_1}}^{K}\delta_i\,\widetilde{\mathbb P}(Q_i)+(K-1)\Delta_1,
\end{equation}
where $\delta_i=p_1-p_i$ and $\widetilde{\mathbb P}(Q_i)$ denotes the asymptotic absorbing probability in Theorem~\ref{thm:main_sandwich}. The first term is the linear regret contribution, while the second is the transient abandonment cost.

Write $\beta=\zeta\alpha$, where $0\le\zeta<1/p_1$. Under this parameterization, the implied baseline value is $1/\zeta$. For $m=2,\ldots,K$, define
\[
\varrho_m(\zeta)=\sum_{h=1}^{m}\frac{2(1-p_h\zeta)}{p_h(1-p_h)}(p_h-p_m),\qquad
\gamma_m=\frac{\sum_{i=2, p_i<p_1}^{m}(p_1-p_i)}{m}-\frac{\sum_{i=2, p_i<p_1}^{m-1}(p_1-p_i)}{m-1},
\]
where the second sum is interpreted as zero when $m=2$. Substituting the asymptotic absorbing probabilities into~\eqref{eq:cal_main_bound_def} gives
\begin{equation}
\label{eq:cal_main_fixed_tilt_objective}
\overline R(T;\alpha,\zeta)
=
T\sum_{m=2}^{K}\gamma_m \exp\{-\alpha\varrho_m(\zeta)\}
+
(K-1)\alpha(1-p_1\zeta).
\end{equation}

The following proposition characterizes the optimal calibration.

\begin{proposition}[Horizon calibration with known gaps]
\label{prop:cal_main_fixed_tilt_certificate}
Consider
\[
\min_{\alpha\ge0,\;0\le\zeta<1/p_1}\overline R(T;\alpha,\zeta).
\]
For each fixed \(0\le\zeta<1/p_1\), define
\(T_0(\zeta):=(K-1)(1-p_1\zeta)/\sum_{m=2, p_m<p_1}^{K}\gamma_m\varrho_m(\zeta)\)
and
\(\varrho_{\min}(\zeta):=\min_{2\le m\le K, p_m<p_1}\varrho_m(\zeta)\).
The optimizer over \(\alpha\ge0\) is
\[
\alpha^\star(T,\zeta)=
\begin{cases}
0, & T\le T_0(\zeta),\\[1ex]
\text{the unique positive solution of }\displaystyle \\
\sum_{m=2}^{K}\gamma_m\varrho_m(\zeta)\exp\{-\alpha\varrho_m(\zeta)\}
=
\frac{(K-1)(1-p_1\zeta)}{T}, & T>T_0(\zeta).
\end{cases}
\]
The optimization over $\zeta$ attains its infimum at the upper boundary:
\[
\inf_{0\le\zeta<1/p_1}
\min_{\alpha\ge0}\overline R(T;\alpha,\zeta)
=
\lim_{\zeta\uparrow 1/p_1}
\min_{\alpha\ge0}\overline R(T;\alpha,\zeta).
\]
Moreover, for every fixed \(0\le\zeta<1/p_1\) and every \(T>T_0(\zeta)\),
\[
\overline R(T;\alpha^\star(T,\zeta),\zeta)
\le
\frac{(K-1)(1-p_1\zeta)}
{\varrho_{\min}(\zeta)}
\left[
1+
\log\left(
\frac{T}{T_0(\zeta)}
\right)
\right].
\]
\end{proposition}

The proof is given in Appendix~\ref{app:calibration_proofs}. Proposition~\ref{prop:cal_main_fixed_tilt_certificate} completely characterizes the oracle regularization. When $T\le T_0(\zeta)$, regularization cannot improve the regret upper bound, so the optimizer is the classical greedy policy with $\alpha=0$. When $T>T_0(\zeta)$, the optimal regularization is the unique solution of the scalar equation in Proposition~\ref{prop:cal_main_fixed_tilt_certificate}, balancing the exponentially decreasing absorption term against the increasing transient cost.

The proposition also shows that the optimal tilt satisfies $\zeta^\star(T)\uparrow1/p_1$. Since $\zeta=1/p_1$ lies outside the admissible region of our asymptotic analysis, we instead use
\[
\zeta_{\epsilon^\circ}=\frac1{p_1}-\epsilon^\circ,
\]
where $\epsilon^\circ>0$ is fixed. Throughout the numerical experiments, we set $\epsilon^\circ=0.2$. The resulting oracle calibration is
\begin{equation}
\label{eq:cal_main_near_boundary_pair}
\alpha_{\epsilon^\circ}^\star(T)\in\arg\min_{\alpha\ge0}\overline R(T;\alpha,\zeta_{\epsilon^\circ}),\qquad
\beta_{\epsilon^\circ}^\star(T)=\zeta_{\epsilon^\circ}\alpha_{\epsilon^\circ}^\star(T).
\end{equation}

Algorithm~\ref{alg:oracle} summarizes the resulting oracle policy. Throughout this section, let $(\alpha,\beta)=\operatorname{Cal}(T,\widehat p)$ denote the calibrated pair obtained from Proposition~\ref{prop:cal_main_fixed_tilt_certificate} using design horizon $T$ and ordered design vector $\widehat p$. When all components of $\widehat p$ are equal, we set $\operatorname{Cal}(T,\widehat p)=(0,0)$. In this case, $\gamma_m=0$ for every $m=2,\ldots,K$, so the objective in \eqref{eq:cal_main_fixed_tilt_objective} is strictly increasing in $\alpha$ for every feasible $\zeta$ and is minimized at zero regularization.

Finally, Proposition~\ref{prop:cal_main_fixed_tilt_certificate} implies that, for every fixed problem instance and every fixed backoff $\epsilon^\circ>0$, recalibrating to the design horizon $T$ yields an upper bound of order
\[
1+\log\!\left(\frac{T}{T_0(\zeta_{\epsilon^\circ})}\right)
\]
for the regret certificate~\eqref{eq:cal_main_bound_def}. Thus, the oracle calibration increases only logarithmically with the horizon under the finite-horizon regret envelope.

\begin{algorithm}[H]
\caption{Oracle}
\label{alg:oracle}
\small
\begin{algorithmic}[1]
\STATE \textbf{Input:} mean multiset $\{p_1,\ldots,p_K\}$; horizon $T$.
\STATE Sort the mean values in decreasing order to obtain $\tilde p$.
\STATE Set $(\alpha,\beta)=\operatorname{Cal}(T,\tilde p)$.
\STATE Pull each arm once.
\FOR{$t=K+1,\ldots,T$}
    \STATE Pull $A_t\in\arg\max_i (S_i+\alpha)/(N_i+\beta)$.
    \STATE Observe $X_t$ and update $S_{A_t}$ and $N_{A_t}$.
\ENDFOR
\end{algorithmic}
\end{algorithm}

\subsection{Operationalizing the Calibration}
\label{subsec:operational}

The oracle policy assumes that the arm means are known. In most real-world settings, however, they are unknown and must be estimated online. We therefore replace the unknown means by online estimates while retaining the calibration rule from the previous subsection. We consider two settings depending on whether the horizon $T$ is known. 

\paragraph{Horizon-Aware.}
Suppose the horizon $T$ is known but the arm means are not. We replace the unknown means by the Jeffreys-smoothed estimates
\[
\tilde p_i(t)=\frac{S_i(t)+1/2}{N_i(t)+1},
\]
sort them to obtain the design vector $\tilde p(t)$, and calibrate $(\alpha,\beta)$ using the true horizon $T$. Following the standard doubling technique, the calibration is recomputed only when $\max_i N_i(t)$ has doubled since the previous recalibration. Between recalibrations, the same pair $(\alpha,\beta)$ is used.
\begin{algorithm}[ht]
\caption{Horizon-Aware}
\label{alg:horizon_aware}
\small
\begin{algorithmic}[1]
\STATE \textbf{Input:} horizon $T$.
\STATE Pull each arm once.
\STATE Form $\tilde p_i=(S_i+1/2)/(N_i+1)$ and sort to obtain $\tilde p$.
\STATE Set $(\alpha,\beta)=\operatorname{Cal}(T,\tilde p)$ and $\mathcal I_{\mathrm{cal}}=\max_i N_i$.
\FOR{$t=K+1,\ldots,T$}
    \IF{$\max_i N_i\ge 2\mathcal I_{\mathrm{cal}}$}
        \STATE Form $\tilde p_i=(S_i+1/2)/(N_i+1)$ and sort to obtain $\tilde p$.
        \STATE Set $(\alpha,\beta)=\operatorname{Cal}(T,\tilde p)$ and $\mathcal I_{\mathrm{cal}}=\max_i N_i$.
    \ENDIF
    \STATE Pull $A_t\in\arg\max_i (S_i+\alpha)/(N_i+\beta)$.
    \STATE Observe $X_t$ and update $S_{A_t}$ and $N_{A_t}$.
\ENDFOR
\end{algorithmic}
\end{algorithm}

\paragraph{Fully Adaptive.}
Suppose neither the horizon nor the arm means are known. We estimate the arm means as in the horizon-aware policy. To remove the dependence on the unknown horizon, we follow a standard anytime-style construction from the bandit literature and replace $T$ by the growing design horizon
\[
\widehat T_t=\max \left\{\varphi K,\left\lceil t \sqrt{\log(e+t)}\right\rceil\right\},
\]
where $\varphi$ is a fixed coefficient. At each recalibration, we apply the same calibration rule with $T$ replaced by $\widehat T_t$. Recalibration follows the same doubling schedule as in the horizon-aware policy.

\begin{algorithm}[ht]
\caption{Fully Adaptive}
\label{alg:fully_adaptive}
\small
\begin{algorithmic}[1]
\STATE \textbf{Input:} coefficient $\varphi$.
\STATE Pull each arm once.
\STATE Set $\widehat T=\max\left\{\varphi K,\left\lceil K\sqrt{\log(e+K)}\right\rceil\right\}$.
\STATE Form $\tilde p_i=(S_i+1/2)/(N_i+1)$ and sort to obtain $\tilde p$.
\STATE Set $(\alpha,\beta)=\operatorname{Cal}(\widehat T,\tilde p)$ and $\mathcal I_{\mathrm{cal}}=\max_i N_i$.
\FOR{$t=K+1,K+2,\ldots$}
    \IF{$\max_i N_i\ge 2\mathcal I_{\mathrm{cal}}$}
        \STATE Set $\widehat T=\max\left\{\varphi K,\left\lceil (t-1)\sqrt{\log(e+t-1)}\right\rceil\right\}$.
        \STATE Form $\tilde p_i=(S_i+1/2)/(N_i+1)$ and sort to obtain $\tilde p$.
        \STATE Set $(\alpha,\beta)=\operatorname{Cal}(\widehat T,\tilde p)$ and $\mathcal I_{\mathrm{cal}}=\max_i N_i$.
    \ENDIF
    \STATE Pull $A_t\in\arg\max_i (S_i+\alpha)/(N_i+\beta)$.
    \STATE Observe $X_t$ and update $S_{A_t}$ and $N_{A_t}$.
\ENDFOR
\end{algorithmic}
\end{algorithm}

\subsection{Numerical Experiments}
\label{subsec:numerical}

The previous sections developed horizon-calibrated regularized greedy policies under three information settings. We now evaluate their empirical performance.

\paragraph{Experimental Design.}
\label{subsubsec:exp_design}

We consider $K\in\{2,5\}$ and horizon-per-arm values $T/K\in\{600,\allowbreak 800,\allowbreak 1{,}000,\allowbreak 1{,}200,\allowbreak 1{,}400,\allowbreak 1{,}600\}$. We evaluate the proposed policies under two arm-mean environments, one uniform and one scaled Poisson-generated. For each configuration, reported quantities are averaged over $M=100$ randomly generated Bernoulli bandit instances and $N=5{,}000$ independent replications per instance. In the uniform environment, arm means are drawn independently from $\mathrm{Uniform}[0.01,0.99]$; in the scaled Poisson environment, arm means are drawn independently from a scaled $\operatorname{Poisson}(10)/30$ distribution restricted to $[0.01,0.99]$. In both environments, the arm means are sorted in decreasing order. All policies are evaluated on the same instance ensemble under common random numbers with master seed $20260630$. Oracle, Horizon-Aware, and Fully Adaptive use $\epsilon^\circ=0.2$, bisection accuracy $\varepsilon_\alpha=10^{-3}$, and $\varphi=500$. Mean regret and mean running time are reported across instances, and configurations exceeding the 24-hour time limit are marked as unfinished.
\paragraph{Benchmarks.} We compare against thirteen standard policies covering the main classes of exploration strategies used in the bandit literature. Each algorithm is initialized by pulling each arm once, with ties in any $\arg\max$ broken uniformly at random: 
\begin{itemize}
\item \textbf{Pure Greedy}~\citep{jedor2021greedy}, which selects $A_t=\arg\max_i S_i/N_i$ and coincides with our policy at $\alpha=0$.
\item \textbf{$\epsilon$-Greedy (Fixed)}~\citep{kuleshov2014algorithms}, which selects $A_t=\arg\max_i S_i/N_i$ with probability $1-\epsilon$ and a uniformly random arm otherwise, with $\epsilon=0.1$.
\item \textbf{$\epsilon$-Greedy (Decay)}~\citep{kuleshov2014algorithms}, which applies the same rule as $\epsilon$-Greedy (Fixed) with $\epsilon$ replaced by the decaying rate $\epsilon_t=\min(1,c/t)$ and $c=1$.
\item \textbf{UCB1} (Upper Confidence Bound 1; \citealp{auer2002finite}), which selects $A_t=\arg\max_i S_i/N_i+\sqrt{c\log t/N_i}$ with $c=2$.
\item \textbf{KL-UCB}~\citep{olivier2013KLUCB}, which selects $A_t=\arg\max_i\max\{q\in[S_i/N_i,1]:N_i\,d(S_i/N_i,q)\le\log t+c\log\log t\}$ with $c=3$, where $d(x,y)=x\log\tfrac{x}{y}+(1-x)\log\tfrac{1-x}{1-y}$ is the Kullback-Leibler divergence.
\item \textbf{MOSS} (the Minimax Optimal Strategy in the Stochastic case; \citealp{audibert2009minimax}), which selects $A_t=\arg\max_i S_i/N_i+\sqrt{\max(\log(T/(KN_i)),0)/N_i}$.
\item \textbf{Thompson Sampling}~\citep{thompson1933likelihood}, which draws $\theta_i\sim\mathrm{Beta}(1+S_i,1+N_i-S_i)$ from a $\mathrm{Beta}(1,1)$ prior and selects $A_t=\arg\max_i\theta_i$.
\item \textbf{BayesUCB}~\citep{kaufmann2012BayesUCB}, which uses the same $\mathrm{Beta}(1,1)$ prior and selects $A_t=\arg\max_i Q\big(1-1/t;\mathrm{Beta}(1+S_i,1+N_i-S_i)\big)$, where $Q(\cdot\,;\cdot)$ is the posterior quantile function.
\item \textbf{IDS} (Information-Directed Sampling; \citealp{russo2014ids}), which uses a $\mathrm{Beta}(1,1)$ prior, computes for each arm the posterior instantaneous regret $\Delta_i(t)=\mathbb E_t[R_{t,A^\ast}-R_{t,i}]$ and information gain $g_i(t)=I_t(A^\ast;Y_{t,i})$, and samples from a distribution $\pi_t$ minimizing the information ratio $(\sum_i \pi_i\Delta_i(t))^2/(\sum_i \pi_i g_i(t))$, with the optimizer implemented by searching over two-arm mixtures.
\item \textbf{OGI} (Optimistic Gittins Indices; \citealp{Farias2022OGI}), which uses a tractable optimistic approximation to the Gittins index with a time-varying discount factor and is implemented in its one-step form with a \(\mathrm{Beta}(1,1)\) prior by setting \(a_i=1+S_i(t)\), \(b_i=1+N_i(t)-S_i(t)\), and \(\gamma_t=1-1/t\) at time \(t\), computing the index \(v_{i,t}\) as the solution \(\lambda\in[0,1]\) of \(\lambda=\frac{a_i}{a_i+b_i}[1-\gamma_t F_{\mathrm{Beta}(a_i+1,b_i)}(\lambda)]+\gamma_t\lambda F_{\mathrm{Beta}(a_i,b_i)}(\lambda)\), and selecting \(A_t=\arg\max_i v_{i,t}\).
\item \textbf{IRS.FH} (Information Relaxation Sampling with finite-horizon penalty; \citealp{Min2024TSIRP}), which modifies Thompson sampling by using a finite-horizon posterior predictive score and, at time \(t\), sets the remaining horizon to \(\tau_t=T-t+1\), draws \(\widetilde p_i\sim\mathrm{Beta}(1+S_i(t),1+N_i(t)-S_i(t))\) and \(\widetilde R_i\sim\mathrm{Binomial}(\tau_t-1,\widetilde p_i)\) for each arm \(i\), computes \(\widetilde m_{i,t}^{\mathrm{FH}}=(1+S_i(t)+\widetilde R_i)/(2+N_i(t)+\tau_t-1)\), and selects \(A_t=\arg\max_i \widetilde m_{i,t}^{\mathrm{FH}}\).
\item \textbf{ETC} (Explore-then-Commit; \citealp{colton1963etc,garivier2016explore}), which first explores in round-robin until each arm has \(m_{\mathrm{ETC}}=\min\left\{\left\lfloor T/K\right\rfloor,\,\left\lceil \left(T/K\right)^{2/3}\right\rceil\right\}\) pulls, and then commits to \(A_t=\arg\max_i S_i/N_i\) for the remaining periods.
\item  \textbf{DETC} (Double Explore-then-Commit; \citealp{jin2021double}), which first pulls every arm until each has \(m_1=\lceil\sqrt{\log T}\rceil\) pulls, sets \(i^{(0)}=\arg\max_i S_i/N_i\), pulls \(i^{(0)}\) for \(m_2=\lceil(\log T)^2\rceil\) additional pulls, then checks each arm \(j\neq i^{(0)}\) with at most \(m_2\) additional pulls, and finally commits to \(A_t=i^{(0)}\) if the check succeeds, or otherwise pulls every arm for another \(m_2\) rounds and commits to \(A_t=\arg\max_i S_i^{\mathrm{fb}}/N_i^{\mathrm{fb}}\) based on this fallback sample.
\end{itemize}

\paragraph{Performance metrics.}
We report mean regret over the ensemble of instances and replications. Throughout this section, the strongest benchmark denotes the classical benchmark with the lowest mean regret in the corresponding configuration. Let $R_{\mathrm{best}}$ denote its mean regret. For each calibrated policy, we report the relative regret $(R_{\mathrm{ours}}-R_{\mathrm{best}})/R_{\mathrm{best}}$, where negative values indicate lower regret than the strongest benchmark. Because Oracle uses the true arm means, it serves only as a reference. We therefore also report the relative regret increase of Horizon-Aware and Fully Adaptive over Oracle, $(R_{\mathrm{ours}}-R_{\mathrm{Oracle}})/R_{\mathrm{Oracle}}$.

We also report average wall-clock time per instance. Let $C_{\mathrm{best}}$ denote the average running time of the strongest benchmark and $C_{\mathrm{ours}}$ that of the calibrated policy. The reported relative time difference is $(C_{\mathrm{ours}}-C_{\mathrm{best}})/C_{\mathrm{best}}$, where negative values indicate lower running time than the strongest benchmark.

\paragraph{Computing environment.}
All experiments used CPU-only computation. Each policy–configuration pair was run as a separate job on a server with two Intel Xeon 6258R processors, 56 CPU cores, and 192 GB of memory. Within each job, instance-level simulations were parallelized by assigning one CPU core to each task.

\subsubsection{Numerical Results.}
\label{subsubsec:exp_results}

Tables~\ref{tab:holistic_uniform_K2}--\ref{tab:holistic_poisson_K5} report the performance of the proposed policies under the uniform and scaled Poisson-generated arm-mean distributions over the values of $K$ and $T/K$. For readability, we omit confidence intervals from the tables. All highlighted comparisons below are statistically significant at the $1\%$ level. Horizon-Aware is the strongest implementable policy in all $24$ configurations and reduces mean regret by $13.5\%$ on average relative to the strongest completed benchmark.
% Fully Adaptive remains competitive without knowledge of either the time horizon or the arm means. It outperforms the strongest completed benchmark in all $24$ configurations and reduces mean regret by $10.2\%$ on average. 
Fully Adaptive exhibits a similar pattern. It also outperforms the strongest completed benchmark in every configuration and reduces mean regret by $10.2\%$ on average, demonstrating the broad competitiveness of the anytime calibration. Both policies reduce mean running time by at least $98.8\%$ relative to the corresponding strongest benchmarks. 

Across all configurations, the strongest benchmark in terms of regret is either IDS or OGI. Both are substantially more computationally intensive in practice: IDS requires information-directed calculations, whereas OGI requires index optimization. We also observe that OGI and IRS.FH perform better under the uniform arm-mean distribution, consistent with the fact that the uniform distribution coincides with the prior specification assumed by these policies.

The empirical behavior closely follows the finite-horizon regret decomposition developed in Section~\ref{sec:Method}. The proposed policies calibrate regularization to balance transient abandonment costs against the probability of suboptimal convergence over a finite horizon. The empirical improvements indicate that this tradeoff is a principal determinant of finite-horizon performance.

The comparison with existing horizon-aware algorithms further illustrates this point. MOSS, IRS.FH, ETC, and DETC already incorporate knowledge of the horizon, yet Horizon-Aware remains the strongest implementable policy across all tested configurations. The improvement therefore does not come from incorporating horizon information itself. Rather, it comes from using the finite-horizon regret envelope to determine the amount of regularization warranted by the horizon.

The benefit of calibration tends to increase with the horizon because reducing the probability of suboptimal absorption has the largest cumulative effect on regret. As the number of arms increases, the improvement becomes more moderate, consistent with the finite-arm approximation becoming less accurate in larger problems. This pattern suggests that further gains are more likely to come from sharper, albeit more complex, finite-arm approximations. Nevertheless, the proposed policies remain competitive with the strongest benchmark algorithms throughout the experimental design.

Finally, Oracle Calibration provides a useful reference for assessing the proposed methodology. As expected, it achieves the lowest regret in every configuration. Horizon-Aware and Fully Adaptive nevertheless remain close to this oracle benchmark, trailing Oracle by only $7.08\%$ and $11.1\%$ on average, respectively. This result demonstrates that most of the benefit of oracle calibration can be retained without prior knowledge of the underlying problem instance.

% Tighter spacing around floating tables
\setlength{\textfloatsep}{6pt plus 2pt minus 2pt}
\setlength{\floatsep}{6pt plus 2pt minus 2pt}
\setlength{\intextsep}{6pt plus 2pt minus 2pt}

% Allow LaTeX to place more floats on one page
\setcounter{topnumber}{5}
\setcounter{bottomnumber}{5}
\setcounter{totalnumber}{8}
\renewcommand{\topfraction}{0.95}
\renewcommand{\bottomfraction}{0.90}
\renewcommand{\textfraction}{0.05}
\renewcommand{\floatpagefraction}{0.80}
\newcommand{\ltpct}[1]{\ensuremath{<#1\%}}

\begin{table}[ht]
\centering
\setlength{\fboxsep}{1.5pt}
\caption{Performance comparison for $K=2$ under uniformly distributed arm means. Each entry averages over $M=100$ Bernoulli bandit instances and $N=5{,}000$ replications per instance. Time reports mean seconds per instance. Rows labeled ``vs. strongest benchmark'' report percentage differences in regret and running time relative to the lowest-regret completed benchmark, while rows labeled ``vs. Oracle'' report relative regret increases over Oracle. Negative values indicate improvements. Boldface identifies the lowest-regret non-Oracle policy, boxes identify the strongest completed benchmark, slashes indicate unfinished configurations, and double dashes indicate inapplicable entries. Running-time differences at or below $-99.9\%$ are reported as $-99.9\%$.}
\label{tab:holistic_uniform_K2}
\fontsize{8.0pt}{9.2pt}\selectfont
\setlength{\tabcolsep}{1.8pt}
\renewcommand{\arraystretch}{1.04}
\begin{tabular*}{\linewidth}{@{\extracolsep{\fill}}l rr|rr|rr|rr|rr|rr@{}}
\toprule
& \multicolumn{2}{c}{$T/K=600$}
& \multicolumn{2}{c}{$T/K=800$}
& \multicolumn{2}{c}{$T/K=1{,}000$}
& \multicolumn{2}{c}{$T/K=1{,}200$}
& \multicolumn{2}{c}{$T/K=1{,}400$}
& \multicolumn{2}{c}{$T/K=1{,}600$} \\
\cmidrule(lr){2-3}
\cmidrule(lr){4-5}
\cmidrule(lr){6-7}
\cmidrule(lr){8-9}
\cmidrule(lr){10-11}
\cmidrule(lr){12-13}
Policy & Mean & Time & Mean & Time & Mean & Time & Mean & Time & Mean & Time & Mean & Time \\
\midrule
Horizon-Aware
& \textbf{4.11}
& 2.81
& \textbf{4.43}
& 3.44
& \textbf{4.71}
& 3.92
& \textbf{4.98}
& 4.60
& \textbf{5.19}
& 5.30
& \textbf{5.39}
& 5.67 \\
\quad vs. strongest benchmark
& \textbf{-13.7\%}
& -99.9\%
& \textbf{-15.3\%}
& -99.9\%
& \textbf{-16.3\%}
& -99.9\%
& \textbf{-16.7\%}
& -99.9\%
& \textbf{-17.4\%}
& -99.9\%
& \textbf{-17.9\%}
& -99.9\% \\
\quad vs. Oracle
& +4.39\%
& --
& +4.25\%
& --
& +5.10\%
& --
& +6.11\%
& --
& +6.89\%
& --
& +7.55\%
& -- \\
Fully Adaptive
& 4.16
& 2.64
& 4.49
& 3.21
& 4.78
& 3.73
& 5.07
& 4.61
& 5.31
& 5.34
& 5.54
& 5.65 \\
\quad vs. strongest benchmark
& -12.6\%
& -99.9\%
& -14.2\%
& -99.9\%
& -15.1\%
& -99.9\%
& -15.1\%
& -99.9\%
& -15.4\%
& -99.9\%
& -15.6\%
& -99.9\% \\
\quad vs. Oracle
& +5.68\%
& --
& +5.65\%
& --
& +6.66\%
& --
& +8.16\%
& --
& +9.45\%
& --
& +10.5\%
& -- \\

\midrule
Pure Greedy
& 42.5
& 0.946
& 56.5
& 1.23
& 70.4
& 1.40
& 84.4
& 1.69
& 98.4
& 1.96
& 112
& 2.23 \\
$\epsilon$-Greedy (Fixed)
& 23.4
& 1.07
& 30.0
& 1.23
& 36.6
& 1.54
& 43.1
& 1.83
& 49.6
& 2.51
& 56.1
& 2.44 \\
$\epsilon$-Greedy (Decay)
& 18.6
& 1.05
& 23.5
& 1.25
& 28.2
& 1.56
& 32.9
& 1.89
& 37.4
& 2.16
& 41.9
& 2.48 \\
UCB1
& 21.4
& 1.00
& 24.4
& 1.20
& 26.9
& 1.47
& 29.0
& 1.82
& 30.9
& 2.04
& 32.6
& 2.34 \\
KL-UCB
& 10.6
& 9.52
& 11.8
& 12.4
& 12.8
& 15.7
& 13.7
& 18.5
& 14.4
& 21.3
& 15.1
& 24.4 \\
MOSS
& 10.1
& 1.12
& 11.1
& 1.34
& 11.9
& 1.69
& 12.6
& 2.00
& 13.1
& 2.31
& 13.7
& 2.65 \\
Thompson Sampling
& 6.04
& 1.98
& 6.55
& 2.45
& 6.96
& 3.11
& 7.32
& 3.74
& 7.63
& 4.43
& 7.92
& 4.91 \\
BayesUCB
& 5.62
& 14.6
& 6.24
& 19.3
& 6.75
& 24.1
& 7.21
& 29.3
& 7.61
& 35.7
& 7.98
& 39.4 \\
IDS
& \ensuremath{\boxed{4.75}}
& 6{,}211
& \ensuremath{\boxed{5.23}}
& 8{,}692
& \ensuremath{\boxed{5.63}}
& 10{,}800
& \ensuremath{\boxed{5.97}}
& 13{,}244
& \ensuremath{\boxed{6.28}}
& 15{,}640
& \ensuremath{\boxed{6.56}}
& 17{,}854 \\
OGI
& 4.78
& 179
& 5.27
& 245
& 5.68
& 314
& 6.04
& 381
& 6.37
& 454
& 6.66
& 522 \\
IRS.FH
& 5.71
& 3.61
& 6.20
& 4.57
& 6.62
& 5.64
& 6.99
& 6.71
& 7.29
& 7.78
& 7.58
& 8.82 \\
ETC
& 25.2
& 0.318
& 30.2
& 0.311
& 34.6
& 0.384
& 38.9
& 0.463
& 43.2
& 0.551
& 46.9
& 0.624 \\
DETC
& 25.0
& 0.518
& 27.5
& 0.594
& 29.4
& 0.710
& 31.2
& 0.824
& 32.9
& 0.943
& 34.4
& 1.10 \\

\midrule
Oracle
& 3.93
& 0.846
& 4.25
& 1.22
& 4.48
& 1.50
& 4.69
& 1.69
& 4.86
& 1.98
& 5.01
& 2.22 \\
\quad vs. strongest benchmark
& -17.3\%
& -99.9\%
& -18.8\%
& -99.9\%
& -20.4\%
& -99.9\%
& -21.5\%
& -99.9\%
& -22.7\%
& -99.9\%
& -23.6\%
& -99.9\% \\
\bottomrule
\end{tabular*}
\end{table}

\begin{table}[ht]
\centering
\setlength{\fboxsep}{1.5pt}
\caption{Performance comparison for $K=5$ under the uniform arm-mean distribution. The format is the same as Table~\ref{tab:holistic_uniform_K2}.}
\label{tab:holistic_uniform_K5}
\fontsize{8.0pt}{9.2pt}\selectfont
\setlength{\tabcolsep}{1.8pt}
\renewcommand{\arraystretch}{1.04}
\begin{tabular*}{\linewidth}{@{\extracolsep{\fill}}l rr|rr|rr|rr|rr|rr@{}}
\toprule
& \multicolumn{2}{c}{$T/K=600$}
& \multicolumn{2}{c}{$T/K=800$}
& \multicolumn{2}{c}{$T/K=1{,}000$}
& \multicolumn{2}{c}{$T/K=1{,}200$}
& \multicolumn{2}{c}{$T/K=1{,}400$}
& \multicolumn{2}{c}{$T/K=1{,}600$} \\
\cmidrule(lr){2-3}
\cmidrule(lr){4-5}
\cmidrule(lr){6-7}
\cmidrule(lr){8-9}
\cmidrule(lr){10-11}
\cmidrule(lr){12-13}
Policy & Mean & Time & Mean & Time & Mean & Time & Mean & Time & Mean & Time & Mean & Time \\
\midrule
Horizon-Aware
& \textbf{17.2}
& 9.32
& \textbf{18.9}
& 10.7
& \textbf{20.1}
& 13.7
& \textbf{21.3}
& 15.1
& \textbf{22.4}
& 17.1
& \textbf{23.3}
& 18.2 \\
\quad vs. strongest benchmark
& \textbf{-10.6\%}
& -98.9\%
& \textbf{-10.7\%}
& -99.1\%
& \textbf{-11.2\%}
& -99.1\%
& \textbf{-11.2\%}
& -99.2\%
& \textbf{-10.4\%}
& -99.2\%
& \textbf{-10.5\%}
& -99.2\% \\
\quad vs. Oracle
& +5.51\%
& --
& +6.84\%
& --
& +7.84\%
& --
& +8.44\%
& --
& +9.77\%
& --
& +10.4\%
& -- \\
Fully Adaptive
& 18.0
& 9.78
& 19.5
& 10.9
& 21.7
& 13.3
& 22.9
& 14.7
& 24.0
& 16.5
& 25.0
& 17.5 \\
\quad vs. strongest benchmark
& -6.77\%
& -98.8\%
& -7.97\%
& -99.1\%
& -4.26\%
& -99.1\%
& -4.30\%
& -99.2\%
& -4.04\%
& -99.2\%
& -3.67\%
& -99.3\% \\
\quad vs. Oracle
& +10.0\%
& --
& +10.1\%
& --
& +16.3\%
& --
& +16.8\%
& --
& +17.6\%
& --
& +18.8\%
& -- \\

\midrule
Pure Greedy
& 132
& 2.82
& 175
& 3.76
& 218
& 4.71
& 261
& 5.63
& 304
& 6.55
& 347
& 7.54 \\
$\epsilon$-Greedy (Fixed)
& 113
& 2.99
& 146
& 4.04
& 178
& 4.93
& 211
& 5.99
& 243
& 7.19
& 275
& 7.93 \\
$\epsilon$-Greedy (Decay)
& 90.6
& 3.04
& 117
& 4.09
& 143
& 5.09
& 169
& 6.15
& 195
& 7.19
& 220
& 8.08 \\
UCB1
& 114
& 3.10
& 129
& 4.11
& 141
& 5.21
& 151
& 6.26
& 160
& 7.22
& 168
& 8.03 \\
KL-UCB
& 47.0
& 55.9
& 51.6
& 72.0
& 55.3
& 89.5
& 58.3
& 106
& 61.0
& 125
& 63.4
& 146 \\
MOSS
& 40.8
& 3.68
& 44.4
& 4.91
& 47.3
& 6.39
& 49.8
& 7.47
& 51.9
& 8.61
& 53.8
& 10.1 \\
Thompson Sampling
& 24.1
& 8.40
& 25.8
& 11.2
& 27.1
& 14.1
& 28.2
& 16.8
& 29.2
& 20.9
& 30.1
& 22.4 \\
BayesUCB
& 23.3
& 91.1
& 25.6
& 122
& 27.4
& 156
& 28.9
& 187
& 30.3
& 221
& 31.5
& 252 \\
IDS
& 19.5
& 37{,}062
& 21.3
& 49{,}177
& 22.7
& 61{,}642
& 24.0
& 74{,}120
& /
& /
& /
& / \\
OGI
& \ensuremath{\boxed{19.3}}
& 847
& \ensuremath{\boxed{21.1}}
& 1{,}154
& \ensuremath{\boxed{22.7}}
& 1{,}468
& \ensuremath{\boxed{23.9}}
& 1{,}788
& \ensuremath{\boxed{25.0}}
& 2{,}081
& \ensuremath{\boxed{26.0}}
& 2{,}413 \\
IRS.FH
& 23.0
& 17.6
& 24.7
& 23.3
& 26.1
& 28.6
& 27.2
& 34.2
& 28.2
& 40.2
& 29.1
& 45.5 \\
ETC
& 128
& 0.590
& 154
& 0.801
& 176
& 0.988
& 198
& 1.18
& 220
& 1.49
& 238
& 1.64 \\
DETC
& 193
& 1.34
& 207
& 1.68
& 222
& 1.99
& 233
& 2.34
& 244
& 2.78
& 253
& 3.02 \\

\midrule
Oracle
& 16.3
& 2.77
& 17.7
& 3.77
& 18.6
& 4.65
& 19.6
& 5.59
& 20.4
& 6.58
& 21.1
& 7.38 \\
\quad vs. strongest benchmark
& -15.3\%
& -99.7\%
& -16.4\%
& -99.7\%
& -17.7\%
& -99.7\%
& -18.1\%
& -99.7\%
& -18.4\%
& -99.7\%
& -18.9\%
& -99.7\% \\
\bottomrule
\end{tabular*}
\end{table}

\begin{table}[ht]
\centering
\setlength{\fboxsep}{1.5pt}
\caption{Performance comparison for $K=2$ under the scaled Poisson-generated arm-mean distribution. The format is the same as Table~\ref{tab:holistic_uniform_K2}.}
\label{tab:holistic_poisson_K2}
\fontsize{8.0pt}{9.2pt}\selectfont
\setlength{\tabcolsep}{1.8pt}
\renewcommand{\arraystretch}{1.04}
\begin{tabular*}{\linewidth}{@{\extracolsep{\fill}}l rr|rr|rr|rr|rr|rr@{}}
\toprule
& \multicolumn{2}{c}{$T/K=600$}
& \multicolumn{2}{c}{$T/K=800$}
& \multicolumn{2}{c}{$T/K=1{,}000$}
& \multicolumn{2}{c}{$T/K=1{,}200$}
& \multicolumn{2}{c}{$T/K=1{,}400$}
& \multicolumn{2}{c}{$T/K=1{,}600$} \\
\cmidrule(lr){2-3}
\cmidrule(lr){4-5}
\cmidrule(lr){6-7}
\cmidrule(lr){8-9}
\cmidrule(lr){10-11}
\cmidrule(lr){12-13}
Policy & Mean & Time & Mean & Time & Mean & Time & Mean & Time & Mean & Time & Mean & Time \\
\midrule
Horizon-Aware
& \textbf{7.21}
& 4.14
& \textbf{7.95}
& 5.24
& \textbf{8.58}
& 5.67
& \textbf{9.10}
& 6.85
& \textbf{9.52}
& 7.83
& \textbf{9.97}
& 8.51 \\
\quad vs. strongest benchmark
& \textbf{-9.13\%}
& -99.9\%
& \textbf{-11.1\%}
& -99.9\%
& \textbf{-12.3\%}
& -99.9\%
& \textbf{-13.4\%}
& -99.9\%
& \textbf{-14.7\%}
& -99.9\%
& \textbf{-15.2\%}
& -99.9\% \\
\quad vs. Oracle
& +6.19\%
& --
& +6.40\%
& --
& +6.97\%
& --
& +7.22\%
& --
& +6.80\%
& --
& +7.56\%
& -- \\
Fully Adaptive
& 7.28
& 3.95
& 8.05
& 4.63
& 8.68
& 5.67
& 9.27
& 7.07
& 9.78
& 8.00
& 10.2
& 8.25 \\
\quad vs. strongest benchmark
& -8.27\%
& -99.9\%
& -9.97\%
& -99.9\%
& -11.2\%
& -99.9\%
& -11.8\%
& -99.9\%
& -12.4\%
& -99.9\%
& -12.8\%
& -99.9\% \\
\quad vs. Oracle
& +7.19\%
& --
& +7.74\%
& --
& +8.25\%
& --
& +9.16\%
& --
& +9.76\%
& --
& +10.5\%
& -- \\

\midrule
Pure Greedy
& 43.1
& 0.906
& 57.5
& 1.20
& 71.8
& 1.44
& 86.1
& 1.74
& 100
& 2.01
& 114
& 2.28 \\
$\epsilon$-Greedy (Fixed)
& 15.2
& 0.981
& 18.3
& 1.25
& 21.2
& 1.59
& 23.9
& 1.87
& 26.6
& 2.18
& 29.2
& 2.47 \\
$\epsilon$-Greedy (Decay)
& 27.3
& 0.960
& 35.5
& 1.28
& 43.5
& 1.59
& 51.5
& 1.92
& 59.4
& 2.21
& 67.2
& 2.51 \\
UCB1
& 22.8
& 0.942
& 27.3
& 1.20
& 31.1
& 1.49
& 34.6
& 1.80
& 37.7
& 2.08
& 40.6
& 2.37 \\
KL-UCB
& 15.2
& 9.42
& 17.6
& 12.6
& 19.7
& 15.3
& 21.5
& 18.7
& 23.1
& 21.8
& 24.6
& 24.6 \\
MOSS
& 13.0
& 1.07
& 14.7
& 1.35
& 16.1
& 1.71
& 17.4
& 2.01
& 18.5
& 2.35
& 19.4
& 2.70 \\
Thompson Sampling
& 9.62
& 1.95
& 10.7
& 2.48
& 11.7
& 3.18
& 12.5
& 3.75
& 13.2
& 4.37
& 13.8
& 4.96 \\
BayesUCB
& 9.36
& 13.9
& 10.7
& 18.3
& 11.8
& 22.9
& 12.7
& 27.2
& 13.6
& 32.0
& 14.4
& 36.4 \\
IDS
& \ensuremath{\boxed{7.93}}
& 6{,}806
& \ensuremath{\boxed{8.94}}
& 9{,}360
& \ensuremath{\boxed{9.78}}
& 11{,}634
& \ensuremath{\boxed{10.5}}
& 14{,}254
& \ensuremath{\boxed{11.2}}
& 16{,}783
& \ensuremath{\boxed{11.7}}
& 19{,}274 \\
OGI
& 8.13
& 202
& 9.19
& 278
& 10.1
& 357
& 10.8
& 438
& 11.5
& 521
& 12.2
& 602 \\
IRS.FH
& 8.91
& 3.48
& 9.99
& 4.44
& 10.9
& 5.54
& 11.6
& 6.52
& 12.3
& 7.66
& 13.0
& 8.72 \\
ETC
& 14.9
& 0.269
& 17.5
& 0.317
& 19.8
& 0.391
& 22.0
& 0.463
& 24.1
& 0.561
& 25.8
& 0.641 \\
DETC
& 20.4
& 0.480
& 23.9
& 0.600
& 27.2
& 0.722
& 30.2
& 0.840
& 33.0
& 0.950
& 35.7
& 1.10 \\

\midrule
Oracle
& 6.79
& 0.988
& 7.47
& 1.22
& 8.02
& 1.40
& 8.49
& 1.77
& 8.91
& 1.97
& 9.26
& 2.26 \\
\quad vs. strongest benchmark
& -14.4\%
& -99.9\%
& -16.4\%
& -99.9\%
& -18.0\%
& -99.9\%
& -19.2\%
& -99.9\%
& -20.2\%
& -99.9\%
& -21.1\%
& -99.9\% \\
\bottomrule
\end{tabular*}
\end{table}

\begin{table}[ht]
\centering
\setlength{\fboxsep}{1.5pt}
\caption{Performance comparison for $K=5$ under the scaled Poisson-generated arm-mean distribution. The format is the same as Table~\ref{tab:holistic_uniform_K2}.}
\label{tab:holistic_poisson_K5}
\fontsize{8.0pt}{9.2pt}\selectfont
\setlength{\tabcolsep}{1.8pt}
\renewcommand{\arraystretch}{1.04}
\begin{tabular*}{\linewidth}{@{\extracolsep{\fill}}l rr|rr|rr|rr|rr|rr@{}}
\toprule
& \multicolumn{2}{c}{$T/K=600$}
& \multicolumn{2}{c}{$T/K=800$}
& \multicolumn{2}{c}{$T/K=1{,}000$}
& \multicolumn{2}{c}{$T/K=1{,}200$}
& \multicolumn{2}{c}{$T/K=1{,}400$}
& \multicolumn{2}{c}{$T/K=1{,}600$} \\
\cmidrule(lr){2-3}
\cmidrule(lr){4-5}
\cmidrule(lr){6-7}
\cmidrule(lr){8-9}
\cmidrule(lr){10-11}
\cmidrule(lr){12-13}
Policy & Mean & Time & Mean & Time & Mean & Time & Mean & Time & Mean & Time & Mean & Time \\
\midrule
Horizon-Aware
& \textbf{32.3}
& 13.5
& \textbf{35.6}
& 15.4
& \textbf{38.3}
& 19.8
& \textbf{40.4}
& 22.4
& \textbf{42.6}
& 25.3
& \textbf{44.2}
& 25.9 \\
\quad vs. strongest benchmark
& \textbf{-7.06\%}
& -99.9\%
& \textbf{-7.72\%}
& -99.9\%
& \textbf{-8.40\%}
& -99.9\%
& \textbf{-20.5\%}
& -99.0\%
& \textbf{-20.7\%}
& -99.1\%
& \textbf{-21.5\%}
& -99.2\% \\
\quad vs. Oracle
& +7.43\%
& --
& +7.87\%
& --
& +7.67\%
& --
& +7.26\%
& --
& +7.92\%
& --
& +7.55\%
& -- \\
Fully Adaptive
& 33.6
& 13.7
& 36.6
& 15.7
& 40.6
& 19.4
& 43.1
& 21.6
& 45.0
& 24.8
& 46.7
& 24.9 \\
\quad vs. strongest benchmark
& -3.18\%
& -99.9\%
& -5.14\%
& -99.9\%
& -2.90\%
& -99.9\%
& -15.3\%
& -99.1\%
& -16.2\%
& -99.1\%
& -17.1\%
& -99.2\% \\
\quad vs. Oracle
& +11.9\%
& --
& +10.9\%
& --
& +14.1\%
& --
& +14.2\%
& --
& +14.0\%
& --
& +13.6\%
& -- \\

\midrule
Pure Greedy
& 197
& 2.99
& 262
& 3.89
& 327
& 5.13
& 392
& 5.84
& 457
& 6.88
& 522
& 7.89 \\
$\epsilon$-Greedy (Fixed)
& 73.5
& 3.05
& 89.8
& 4.10
& 105
& 5.14
& 120
& 6.12
& 135
& 7.20
& 149
& 8.20 \\
$\epsilon$-Greedy (Decay)
& 141
& 3.16
& 185
& 4.31
& 228
& 5.42
& 271
& 6.58
& 313
& 7.63
& 355
& 8.37 \\
UCB1
& 133
& 3.17
& 156
& 4.17
& 177
& 5.26
& 194
& 6.27
& 210
& 7.29
& 225
& 8.30 \\
KL-UCB
& 85.1
& 57.0
& 97.2
& 73.3
& 107
& 92.6
& 115
& 112
& 123
& 131
& 129
& 149 \\
MOSS
& 55.5
& 3.69
& 61.7
& 4.90
& 66.7
& 6.25
& 71.1
& 7.49
& 74.8
& 8.57
& 78.1
& 10.0 \\
Thompson Sampling
& 46.7
& 8.61
& 51.4
& 11.3
& 55.2
& 14.2
& 58.3
& 17.0
& 61.1
& 19.8
& 63.5
& 22.5 \\
BayesUCB
& 49.3
& 91.8
& 55.4
& 122
& 60.5
& 152
& 64.8
& 182
& 68.6
& 213
& 72.0
& 243 \\
IDS
& \ensuremath{\boxed{34.8}}
& 42{,}261
& \ensuremath{\boxed{38.6}}
& 55{,}733
& \ensuremath{\boxed{41.8}}
& 70{,}331
& /
& /
& /
& /
& /
& / \\
OGI
& 39.2
& 1{,}060
& 43.8
& 1{,}460
& 47.6
& 1{,}877
& \ensuremath{\boxed{50.8}}
& 2{,}278
& \ensuremath{\boxed{53.7}}
& 2{,}708
& \ensuremath{\boxed{56.3}}
& 3{,}127 \\
IRS.FH
& 43.9
& 17.2
& 48.4
& 22.9
& 52.1
& 28.1
& 55.2
& 33.3
& 57.8
& 38.6
& 60.4
& 44.1 \\
ETC
& 83.9
& 0.609
& 99.4
& 0.790
& 112
& 1.01
& 125
& 1.20
& 137
& 1.38
& 148
& 1.58 \\
DETC
& 120
& 1.30
& 137
& 1.73
& 154
& 1.92
& 170
& 2.25
& 184
& 2.58
& 198
& 2.97 \\

\midrule
Oracle
& 30.1
& 2.84
& 33.0
& 3.83
& 35.6
& 4.72
& 37.7
& 5.58
& 39.5
& 6.73
& 41.1
& 7.49 \\
\quad vs. strongest benchmark
& -13.5\%
& -99.9\%
& -14.5\%
& -99.9\%
& -14.9\%
& -99.9\%
& -25.8\%
& -99.8\%
& -26.5\%
& -99.8\%
& -27.0\%
& -99.8\% \\
\bottomrule
\end{tabular*}
\end{table}

\section{Conclusion}
\label{sec:conclusion}

This paper studies finite-horizon experimentation through a class of regularized greedy policies. We derive analytical expressions for the finite-horizon regret, characterize the probability of suboptimal convergence, and use these results to calibrate the regularization parameters under varying levels of information, ranging from oracle settings to fully adaptive implementations.

The analysis suggests a different perspective on finite-horizon bandit problems. Classical bandit algorithms are designed to achieve asymptotically optimal regret by continuing to explore as the horizon grows. Over finite operational horizons, however, the dominant source of regret is often not insufficient exploration, but premature commitment to a suboptimal arm. The resulting design problem is thus to balance the transient cost of delaying commitment against the probability of suboptimal convergence. Regularization provides a direct mechanism for controlling this trade-off.

The numerical experiments reinforce this characterization. Across a broad range of horizons and numbers of arms, the proposed policies consistently match or outperform state-of-the-art bandit algorithms, including Thompson Sampling, OGI, IDS and KL-UCB.

In general, these results suggest that simple, well-calibrated algorithms can compete with state-of-the-art bandit methods in finite-horizon settings. We hope this work motivates further study of simple algorithmic designs tailored to finite-horizon objectives.
%Several limitations remain. The analysis is confined to two arms with Bernoulli rewards. A natural extension is to more than two arms and to sub-Gaussian or otherwise general reward distributions. The calibration likewise admits refinement. At each level the prior is held just inside the feasibility boundary, because the optimal tilt lies exactly on that boundary and is itself inadmissible. This backoff is slightly suboptimal, so tuning the prior closer to the boundary at each level is a direct route to lower regret.

\section*{Code Availability}
A repository is available at \url{https://osf.io/96udf/overview?view_only=4ecd77023d5147a8bdca4be9e9528454}.

\bibliographystyle{apalike}
\bibliography{references}

\clearpage
\setcounter{page}{1}
\appendix

\begin{center}
{\Large\bfseries The Greedy Advantage in Finite-Horizon Bandits\par}
\vspace{0.5em}
{\large\bfseries Electronic Companion\par}
\end{center}

\vspace{1.2em}

%%%%%%%%%%%%%%%%%%%%%%%%%%%%%%%%%%%%%%%%%%%%%%%%%%%%%%%%%%%%%%%%%%%%%%%%%%
%% Preliminary electronic-companion numbering:
%% (EC.1), (EC.2), ...
%%%%%%%%%%%%%%%%%%%%%%%%%%%%%%%%%%%%%%%%%%%%%%%%%%%%%%%%%%%%%%%%%%%%%%%%%%

\setcounter{equation}{0}
\renewcommand{\theequation}{EC.\arabic{equation}}
\renewcommand{\theHequation}{EC.\arabic{equation}}

Throughout the appendices, arms are indexed without loss of generality so that
\begin{equation}
\label{eq:app_ordered_means}
1-\epsilon_{\mathrm p}
\ge
p_1
\ge
p_2
\ge
\cdots
\ge
p_K
\ge
\epsilon_{\mathrm p},
\qquad
p_1>p_K,
\end{equation}
where \(\epsilon_{\mathrm p}\in(0,1/2)\) is an arbitrarily small but fixed constant. Thus the arm means remain in a fixed compact subset of \((0,1)\), ties are permitted, and \(p_1>p_K\) ensures that at least one arm is suboptimal.

For \(1\le i<j\le K\), define the pairwise gap
\[
\delta_{ij}:=p_i-p_j\ge0.
\]
In particular, when arm \(1\) is compared with arm \(i\), write
\[
\delta_i:=\delta_{1i}=p_1-p_i,
\qquad
i=2,\ldots,K,
\]
and set \(\delta_1:=0\).

The regularization parameters satisfy
\[
\alpha\ge0,
\qquad
\beta\ge0,
\qquad
\Delta_i:=\alpha-p_i\beta\ge0,
\qquad
i=1,\ldots,K.
\]
Under the ordering of arm means, the conditions
\(\Delta_i\ge0\) for all \(i\) are equivalent to
\(\alpha\ge p_1\beta\). The pure-greedy policy is included as the special case
\(\alpha=\beta=0\).

% % Appendix headings follow the reference paper: "Appendix A: Title".
% \titleformat{\section}
%   {\normalfont\large\bfseries}
%   {Appendix \thesection:}
%   {0.5em}
%   {}

%%%%%%%%%%%%%%%%%%%%%%%%%%%%%%%%%%%%%%%%%%%%%%%%%%%%%%%%%%%%%%%%%%%%%%%%%%
%% Switch to appendix-section numbering:
%% (A.1), (A.2), ...; Theorem A.1, Lemma A.2, ...
%%%%%%%%%%%%%%%%%%%%%%%%%%%%%%%%%%%%%%%%%%%%%%%%%%%%%%%%%%%%%%%%%%%%%%%%%%

\numberwithin{equation}{section}
\numberwithin{theorem}{section}

% Hyperref-safe appendix anchors
\renewcommand{\theHsection}{appendix.\Alph{section}}
\renewcommand{\theHsubsection}{appendix.\Alph{section}.\arabic{subsection}}
\renewcommand{\theHsubsubsection}{appendix.\Alph{section}.\arabic{subsection}.\arabic{subsubsection}}
\renewcommand{\theHequation}{appendix.\Alph{section}.\arabic{equation}}

\providecommand{\theHtheorem}{}
\renewcommand{\theHtheorem}{appendix.\Alph{section}.\arabic{theorem}}

\section{Absorption of Greedy Trajectories}
\label{app:absorption}

This section establishes the pathwise absorption structure of the greedy dynamics. We allow both regularized greedy policies and the pure-greedy policy. 
For any fixed parameter configuration, we prove that the process cannot keep switching among two or more arms indefinitely. More precisely, there exists an almost surely finite random time \(T_0\) and an almost surely unique random absorbing arm \(I_{\mathrm{abs}}\in\{1,\ldots,K\}\) such that
\[
% \label{eq:abs_eventual_commitment}
A_t=I_{\mathrm{abs}},
\qquad
\forall t\ge T_0 .
\]
This conclusion does not require all arm means to be distinct. Even when several arms have identical success probabilities, the handover process almost surely terminates, and the trajectory commits to one of the tied arms. The proof proceeds by ruling out pairwise entanglement: for any pair of arms \(i\neq j\), the event that both arms are sampled infinitely often has probability zero. If \(p_i\neq p_j\), this follows from concentration of the two scores around different limits. If \(p_i=p_j\), the argument uses repeated negative fluctuations and the monotonicity of handover thresholds to rule out perpetual cycling.

Consequently, the absorbing events
\begin{equation}
\label{eq:abs_events_def}
Q_i:=\{I_{\mathrm{abs}}=i\},
\qquad
i=1,\ldots,K,
\end{equation}
form an almost-sure partition of the sample space:
\[
% \label{eq:abs_partition}
\sum_{i=1}^K \mathbb P(Q_i)=1.
\]
For each arm \(i\), we also define the complementary absorbing event
\(
Q_i^c
:=
\bigcup_{j\neq i}Q_j
\). 
Equivalently, up to a null set,
\(
Q_i^c
=
\{I_{\mathrm{abs}}\neq i\}
\). 
This absorption result is the starting point for the score-minimum reduction in the next section: once the eventual winner is known to be a single absorbing arm, comparisons among the arms' potential score minima, equivalently their centered score-minimum variables, can be used to sandwich the absorbing event \(Q_i\) and its complement \(Q_i^c\).

\subsection{Topological Collapse to Absorbing States}
\label{app:topological_collapse}

We prove that the greedy trajectory almost surely commits to a single arm. This conclusion applies both to regularized greedy policies and to the pure-greedy policy. The proof does not identify which arm is selected; it only rules out the possibility that two or more arms are sampled infinitely often.

After \(n\ge1\) pulls of arm \(i\), define its score by
\[
% \label{eq:tc_score_def}
\widehat p_i(n)
:=
\frac{S_i(n)+\alpha}{n+\beta}.
\]
Recall that the algorithm first pulls each arm once. Thus all scores below are
well-defined on their own clocks \(n\ge1\). When \(\alpha=\beta=0\), this reduces to
\(
\widehat p_i(n)
=
{S_i(n)}/{n}
\).
At each decision time, the policy selects an arm with maximal current score, breaking ties uniformly at random.

For each arm \(i\), define the event
\(
E_i
:=
\{N_i(t)\to\infty\}
\). 
For each pair \(i\neq j\), define the pairwise entanglement event
\(
E_{ij}
:=
E_i\cap E_j
\). 
The event that at least two arms are sampled infinitely often is
\begin{equation}
\label{eq:tc_multi_entangle_union}
E_{\ge2}
:=
\left\{
\sum_{i=1}^K\mathbf 1_{E_i}\ge2
\right\}
=
\bigcup_{1\le i<j\le K}E_{ij}.
\end{equation}

\begin{lemma}[Topological collapse to absorbing states]
\label{lem:topological_collapse_karm}
For every fixed parameter configuration, there exist an almost surely finite random time \(T_0\) and an almost surely unique random arm \(I_{\mathrm{abs}}\in\{1,\ldots,K\}\) such that
\begin{equation}
\label{eq:tc_absorption_statement}
A_t=I_{\mathrm{abs}},
\qquad
\forall t\ge T_0 .
\end{equation}
Consequently, the absorbing events form an almost-sure partition:
\begin{equation}
\label{eq:tc_absorbing_partition}
\sum_{i=1}^{K}\mathbb P(Q_i)=1.
\end{equation}
\end{lemma}

\begin{proof}\leavevmode
We prove
\[
% \label{eq:tc_pair_zero_target}
\mathbb P(E_{ij})=0,
\qquad
1\le i<j\le K.
\]
Then $\mathbb{P}(E_{\geq 2})=0$ follows from the finite union bound applied to \eqref{eq:tc_multi_entangle_union}.

\textit{Step 1. Almost-sure score concentration on each arm's own clock.}

We first prove an arm-wise own-clock convergence statement. This step does not
assert that any particular arm is sampled infinitely often along the actual
trajectory. Instead, for each arm \(i\), we consider its potential score sequence
indexed by its own pull count \(n\). Later, on the event \(E_i=\{N_i(t)\to\infty\}\),
this own-clock convergence can be evaluated along the random subsequence
\(n=N_i(t)\). If \(E_i\) does not occur, no asymptotic score statement for arm
\(i\) along calendar time is needed.

For each arm \(i\), decompose the regularized score around its true mean:
\[
\widehat p_i(n)-p_i
=
\frac{S_i(n)+\alpha}{n+\beta}-p_i
=
\frac{S_i(n)-p_i n}{n+\beta}
+
\frac{\alpha-p_i\beta}{n+\beta}.
\]
Since \(\beta\ge0\),
\(
{1}/{(n+\beta)}
\le
{1}/{n}
\). 
Hence
\begin{equation}
\left|\widehat p_i(n)-p_i\right|
\le
\left|
\frac{S_i(n)-p_i n}{n+\beta}
\right|
+
\frac{|\alpha-p_i\beta|}{n+\beta}
\le
\left|
\frac{S_i(n)}{n}-p_i
\right|
+
\frac{|\alpha-p_i\beta|}{n+\beta}.
\label{eq:tc_score_deviation_bound}
\end{equation}

Fix an arbitrary \(\varepsilon>0\). Since \(\alpha\), \(\beta\), and \(p_i\) are fixed,
\(
% \label{eq:tc_bias_vanishes}
{|\alpha-p_i\beta|}/{(n+\beta)}
\longrightarrow 0
\). 
Therefore there exists \(n_{i,\varepsilon}<\infty\) such that, for all \(n\ge n_{i,\varepsilon}\),
\begin{equation}
\label{eq:tc_bias_eps_half}
\frac{|\alpha-p_i\beta|}{n+\beta}
\le
\frac{\varepsilon}{2}.
\end{equation}
Combining \eqref{eq:tc_score_deviation_bound} and \eqref{eq:tc_bias_eps_half}, for all \(n\ge n_{i,\varepsilon}\),
\[\begin{aligned}
\mathbb P\left(
\left|\widehat p_i(n)-p_i\right|>\varepsilon
\right)
&\le
\mathbb P\left(
\left|
\frac{S_i(n)}{n}-p_i
\right|
>
\frac{\varepsilon}{2}
\right).
% \label{eq:tc_reduce_to_empirical_mean}
\end{aligned}\]
By Hoeffding's inequality,
\[
% \label{eq:tc_hoeffding_empirical_mean}
\mathbb P\left(
\left|
\frac{S_i(n)}{n}-p_i
\right|
>
\frac{\varepsilon}{2}
\right)
\le
2\exp\left\{-\frac{n\varepsilon^2}{2}\right\}.
\]
The right-hand side is summable. Hence
\begin{equation}
\label{eq:tc_score_deviation_summable}
\sum_{n=1}^{\infty}
\mathbb P\left(
\left|\widehat p_i(n)-p_i\right|>\varepsilon
\right)
<\infty .
\end{equation}
Define the deviation event
\(
% \label{eq:tc_Ain_eps_def}
A_{i,n}^{\varepsilon}
:=
\left\{
\left|\widehat p_i(n)-p_i\right|>\varepsilon
\right\}
\). 
The event that this deviation occurs infinitely often is the limsup event
\(
% \label{eq:tc_limsup_deviation_event}
\limsup_{n\to\infty}A_{i,n}^{\varepsilon}
=
\bigcap_{N=1}^{\infty}
\bigcup_{n\ge N}
A_{i,n}^{\varepsilon}
\). 
By \eqref{eq:tc_score_deviation_summable} and the Borel--Cantelli lemma,
\(
% \label{eq:tc_borel_cantelli_limsup}
\mathbb P\left(
\limsup_{n\to\infty}A_{i,n}^{\varepsilon}
\right)
=
0
\). 
Equivalently,
\begin{equation}
\label{eq:tc_eventual_eps_control}
\mathbb P\left(
\exists N_{i,\varepsilon}<\infty
\text{ such that }
\left|\widehat p_i(n)-p_i\right|\le\varepsilon
\text{ for all }n\ge N_{i,\varepsilon}
\right)
=
1.
\end{equation}
Since \(\varepsilon>0\) is arbitrary, applying
\eqref{eq:tc_eventual_eps_control} to the countable sequence
\(\varepsilon_m=1/m\) gives, for each fixed arm \(i\),
\[
\widehat p_i(n)\to p_i
\qquad
\text{almost surely as }n\to\infty .
\]
Because \(K\) is finite, there exists a probability-one event
\(\Omega_{\mathrm c}\) on which
\begin{equation}
\label{eq:tc_score_convergence}
\widehat p_i(n)\to p_i
\qquad
\text{as }n\to\infty
\end{equation}
holds simultaneously for all arms \(i=1,\ldots,K\).

This is an own-clock statement. Consequently, on \(\Omega_{\mathrm c}\cap E_i\),
where \(N_i(t)\to\infty\), we may evaluate
\eqref{eq:tc_score_convergence} along the random own-clock sequence
\(n=N_i(t)\) and obtain
\begin{equation}
\label{eq:tc_score_convergence_on_Ei}
\widehat p_i(N_i(t))\to p_i
\qquad
\text{as }t\to\infty .
\end{equation}
Thus, on \(\Omega_{\mathrm c}\cap E_{ij}\), both
\(\widehat p_i(N_i(t))\to p_i\) and
\(\widehat p_j(N_j(t))\to p_j\). This is the only way Step~1 is used below.
The fact that at least one arm must be sampled infinitely often follows later
from the identity \(\sum_{i=1}^{K}N_i(t)=t\).

\textit{Step 2. Pairwise impossibility when \(p_i\neq p_j\).}

Consider any pair \(i<j\) with \(p_i>p_j\). On \(E_{ij}\cap\Omega_{\mathrm c}\), both \(N_i(t)\) and \(N_j(t)\) diverge. Let
\(
% \label{eq:tc_pair_eta_def}
\gamma_{ij}
:=
{(p_i-p_j)}/{3}
>0
\). 
By \eqref{eq:tc_score_convergence}, there exists a finite random time \(T_{ij}\) such that, for all \(t\ge T_{ij}\),
\[
\widehat p_i(N_i(t))
>
p_i-\gamma_{ij},
\quad
% \label{eq:tc_pair_high_score}
\widehat p_j(N_j(t))
<
p_j+\gamma_{ij} .
% \label{eq:tc_pair_low_score}
\]
Since
\(
% \label{eq:tc_pair_strict_separation}
p_i-\gamma_{ij}
>
p_j+\gamma_{ij}
\), 
we obtain
\[
% \label{eq:tc_pair_score_dominance}
\widehat p_i(N_i(t))
>
\widehat p_j(N_j(t)),
\qquad
\forall t\ge T_{ij}.
\]
Thus arm \(j\) cannot be selected after \(T_{ij}\), because arm \(i\) has a strictly larger score. This contradicts \(E_j\). Therefore,
\[
% \label{eq:tc_pair_unequal_impossible}
\mathbb P(E_{ij})=0
\qquad
\text{whenever }p_i\neq p_j .
\]

\textit{Step 3. Pairwise impossibility when \(p_i=p_j\).}

The intuition is simple. Suppose the policy keeps switching among arms with the
same mean \(p\). At each switch, the newly selected arm must have a score at
least as large as the currently active arm. But while an arm is not selected, its
score does not change. Hence the score level at which switches occur cannot
increase over time; the sequence of handover score levels is nonincreasing.

At the same time, each tied arm is sampled from a Bernoulli stream with mean
\(p\), so its score will eventually fluctuate below \(p\). If switching continued
forever, these downward fluctuations would eventually push the handover score
level below \(p\). Once this happens, the handover level cannot later rise back
above \(p\), because it is nonincreasing. But if all tied arms were sampled
infinitely often, their scores would converge back to \(p\), so the handover
score level would also have to approach \(p\). This contradiction rules out
perpetual switching among equal-mean arms.

Consider any pair \(i<j\) with \(p_i=p_j=:p\). We show that \(E_{ij}\cap\Omega_{\mathrm c}\) is impossible.

Assume, toward a contradiction, that
\(
% \label{eq:tc_pair_equal_entangle_assumption}
\omega\in E_{ij}\cap\Omega_{\mathrm c}
\). 
Let
\(
% \label{eq:tc_J_of_path_def}
\mathcal K_{\infty}
:=
\{k:N_k(t,\omega)\to\infty\}
\)
be the set of arms sampled infinitely often on this path. Since \(i,j\in\mathcal K_{\infty}\), we have
\(
% \label{eq:tc_J_size_at_least_two}
|\mathcal K_{\infty}|\ge2
\). 
By Step 2, no two arms in \(\mathcal K_{\infty}\) can have different means. Hence
\begin{equation}
\label{eq:tc_all_J_same_mean}
p_k=p,
\qquad
\forall k\in\mathcal K_{\infty} .
\end{equation}

After the last pull of every arm outside \(\mathcal K_{\infty}\), the selected
arm always belongs to \(\mathcal K_{\infty}\). Since
\(|\mathcal K_{\infty}|\ge2\), the selected arm must change infinitely many
times among arms in \(\mathcal K_{\infty}\). Let
\(
h_1<h_2<h_3<\cdots
\) 
be the successive decision times at which a new arm in
\(\mathcal K_{\infty}\) is selected after a different arm in
\(\mathcal K_{\infty}\) was selected previously. Let
\(
B_m:=A_{h_m}
\)
be the newly selected arm at the \(m\)-th handover.

The relevant score at a handover is the score immediately before the new arm is
pulled, because this is the score used by the greedy rule to select the arm.
Define
\[
V_m
:=
\widehat p_{B_m}\!\left(N_{B_m}(h_m-1)\right).
\]
Then, just before the decision at time \(h_m\), arm \(B_m\) is a maximizer among
the arms in \(\mathcal K_{\infty}\), and hence
\begin{equation}
\label{eq:tc_record_score_maximal}
V_m
=
\max_{k\in\mathcal K_{\infty}}
\widehat p_k\!\left(N_k(h_m-1)\right).
\end{equation}

During the interval from the decision at \(h_m\) up to just before the decision at \(h_{m+1}\), every arm other than \(B_m\) is passive. In particular, the next handover arm \(B_{m+1}\) is passive throughout this interval. Therefore, \(
N_{B_{m+1}}(h_{m+1}-1)
=
N_{B_{m+1}}(h_m-1)
\), 
and so
\[\begin{aligned}
V_{m+1}
&=
\widehat p_{B_{m+1}}\!\left(N_{B_{m+1}}(h_{m+1}-1)\right)
\\
&=
\widehat p_{B_{m+1}}\!\left(N_{B_{m+1}}(h_m-1)\right)
\\
&\le
\max_{k\in\mathcal K_{\infty}}
\widehat p_k\!\left(N_k(h_m-1)\right)
\\
&=
V_m .
\end{aligned}\]
Thus
\begin{equation}
\label{eq:tc_R_nonincreasing}
V_1\ge V_2\ge V_3\ge\cdots .
\end{equation}

For every \(k\in\mathcal K_{\infty}\), the law of the iterated logarithm gives infinitely many negative fluctuations:
\begin{equation}
\label{eq:tc_lil_negative}
S_k(n)-pn
<
-\frac12
\sqrt{2p(1-p)n\log\log n}
\qquad
\text{for infinitely many }n .
\end{equation}
Using
\[\begin{aligned}
\widehat p_k(n)-p
&=
\frac{S_k(n)-pn+\alpha-p\beta}{n+\beta},
% \label{eq:tc_equal_mean_score_centering}
\end{aligned}\]
and the fact that \(\alpha-p\beta\) is fixed, \eqref{eq:tc_lil_negative} implies
\begin{equation}
\label{eq:tc_each_arm_below_p_inf_often}
\widehat p_k(n)<p
\qquad
\text{for infinitely many }n .
\end{equation}

We claim that \(V_m<p\) for some finite \(m\). Suppose otherwise:
\begin{equation}
\label{eq:tc_R_never_below_p_assumption}
V_m\ge p,
\qquad
\forall m .
\end{equation}
For each arm \(k\in\mathcal K_{\infty}\), define its first below-\(p\) pull count by
\(
% \label{eq:tc_first_below_p_pull_count}
\nu_k
:=
\inf\left\{
n\ge1:
\widehat p_k(n)<p
\right\}
\). 
By \eqref{eq:tc_each_arm_below_p_inf_often},
\[
% \label{eq:tc_first_below_p_finite}
\nu_k<\infty,
\qquad
k\in\mathcal K_{\infty} .
\]

Consider the first time at which arm \(k\) reaches pull count \(\nu_k\). Immediately after that pull,
\(
% \label{eq:tc_arm_k_frozen_below_p}
\widehat p_k(\nu_k)<p
\). 
If at that time every other arm in \(\mathcal K_{\infty}\) also has score strictly below \(p\), then the maximal score among arms in \(\mathcal K_{\infty}\) is already below \(p\), and therefore \(V_m<p\) at the next handover, contradicting \eqref{eq:tc_R_never_below_p_assumption}. Otherwise, at least one arm in \(\mathcal K_{\infty}\) has score at least \(p\). In that case arm \(k\), whose score is now strictly below \(p\), cannot be selected again as long as \eqref{eq:tc_R_never_below_p_assumption} holds.

Thus, under \eqref{eq:tc_R_never_below_p_assumption}, once arm \(k\) reaches pull count \(\nu_k\), it can never be selected again. Indeed, after that pull its score is strictly below \(p\). If it were selected again at a later handover, then all competing arms in \(\mathcal K_{\infty}\) would have scores no larger than this strictly below-\(p\) frozen score, and hence the maximal handover score would be strictly below \(p\), contradicting \eqref{eq:tc_R_never_below_p_assumption}. Therefore arm \(k\) would be selected only finitely many times. Since \(k\in\mathcal K_{\infty}\) is arbitrary, this already contradicts the definition of \(\mathcal K_{\infty}\), under which every arm in \(\mathcal K_{\infty}\) is sampled infinitely often. Hence \eqref{eq:tc_R_never_below_p_assumption} is impossible.
Therefore, there exists \(m_0<\infty\) such that
\(
% \label{eq:tc_R_below_p_once}
V_{m_0}<p
\). 

Let
\(
% \label{eq:tc_epsilon_record_def}
\varepsilon
:=
p-V_{m_0}
>0
\). 
By \eqref{eq:tc_R_nonincreasing},
\begin{equation}
\label{eq:tc_R_below_p_forever}
V_m\le p-\varepsilon,
\qquad
\forall m\ge m_0 .
\end{equation}

On the other hand, since every arm in \(\mathcal K_{\infty}\) is sampled infinitely often and \(\mathcal K_{\infty}\) is finite,
\[
\min_{k\in\mathcal K_{\infty}}N_k(h_m-1)\to\infty .
\]
By \eqref{eq:tc_score_convergence} and \eqref{eq:tc_all_J_same_mean},
\[
\max_{k\in\mathcal K_{\infty}}
\left|
\widehat p_k\!\left(N_k(h_m-1)\right)-p
\right|
\to0 .
\]
Using \eqref{eq:tc_record_score_maximal}, we get
\(
% \label{eq:tc_R_converges_to_p}
V_m\to p
\). 
This contradicts \eqref{eq:tc_R_below_p_forever}. Therefore, 
\(
% \label{eq:tc_pair_equal_impossible}
\mathbb P(E_{ij})=0
\text{ whenever }p_i=p_j
\).

\textit{Step 4. Finite union over pairs and construction of the absorbing arm.}

Combining the unequal-mean and equal-mean cases, we have
\[
% \label{eq:tc_all_pairs_zero}
\mathbb P(E_{ij})=0,
\qquad
1\le i<j\le K .
\]
By \eqref{eq:tc_multi_entangle_union},
\begin{equation}
\mathbb P(E_{\ge2})
=
\mathbb P\left(
\bigcup_{1\le i<j\le K}E_{ij}
\right)
\le
\sum_{1\le i<j\le K}
\mathbb P(E_{ij})
=
0 .
\label{eq:tc_union_bound_zero}
\end{equation}
On the other hand, at least one arm must be sampled infinitely often. Indeed,
\(
% \label{eq:tc_total_pull_identity}
\sum_{i=1}^{K}N_i(t)=t,\, t\ge1
\), 
and therefore it is impossible that all \(N_i(t)\) remain bounded. Hence
\begin{equation}
\label{eq:tc_at_least_one_infinite}
\mathbb P\left(
\bigcup_{i=1}^{K}E_i
\right)
=1.
\end{equation}
Combining \eqref{eq:tc_union_bound_zero} and \eqref{eq:tc_at_least_one_infinite}, we obtain
\begin{equation}
\label{eq:tc_exactly_one_infinite}
\mathbb P\left(
\sum_{i=1}^{K}\mathbf 1_{E_i}=1
\right)
=1.
\end{equation}

On the probability-one event in \eqref{eq:tc_exactly_one_infinite}, let \(I_{\mathrm{abs}}\) be the unique arm satisfying \(E_{I_{\mathrm{abs}}}\). Every arm \(j\neq I_{\mathrm{abs}}\) has a finite last pull time. Define
\(
% \label{eq:tc_absorption_time_def}
T_0
:=
1+
\max_{j\neq I_{\mathrm{abs}}}
\sup\{t\ge1:A_t=j\}
\), 
with the convention that the supremum over an empty set is \(0\). Then \(T_0<\infty\) almost surely and
\(
% \label{eq:tc_absorption_after_T0}
A_t=I_{\mathrm{abs}},\, \forall t\ge T_0
\). 
This proves \eqref{eq:tc_absorption_statement}.

Finally, by~\eqref{eq:abs_events_def},
\(
Q_i:=\{I_{\mathrm{abs}}=i\},\,
i=1,\ldots,K
\). 
Since \(I_{\mathrm{abs}}\) is almost surely well-defined and unique, the events \(Q_1,\ldots,Q_K\) are disjoint up to null sets and exhaust the sample space up to a null set. Therefore
\[
% \label{eq:tc_partition_conclusion}
\sum_{i=1}^{K}\mathbb P(Q_i)=1.
\]
This proves \eqref{eq:tc_absorbing_partition}.

\end{proof}

\subsection{Regret Decomposition and Core Quantities}
\label{app:regret_decomposition_core_quantities}

We next translate the absorbing-state partition into a finite-horizon regret
decomposition. The purpose of this subsection is not yet to evaluate the absorbing
probabilities, but to identify the finite-horizon quantities that must be controlled in
the rest of the appendix.

For each arm \(i\), let
\[
% \label{eq:rd_time_dependent_count}
N_i(T)
:=
\sum_{t=1}^{T}\mathbf 1\{A_t=i\}
\]
be the number of pulls of arm \(i\) up to horizon \(T\). We also define the terminal
pull count
\[
% \label{eq:rd_terminal_count}
N_i
:=
\lim_{T\to\infty}N_i(T)
\in
\mathbb N\cup\{\infty\}.
\]
The distinction is important: \(N_i(T)\) is horizon-dependent, while \(N_i\) is the
total number of pulls over the infinite trajectory. Recall that \(Q_i\) denotes the
event that the process eventually absorbs into arm \(i\). On \(Q_i\), arm \(i\) is
pulled forever, while every arm \(j\neq i\) is eventually abandoned. Hence,
\[
N_j<\infty
\qquad
\text{on }Q_i,\quad j\neq i.
\]
For finite-horizon regret, the relevant abandoned-arm count is the capped terminal
count. For every \(T<\infty\),
\begin{equation}
\label{eq:rd_count_capped_domination}
N_i(T)
\le
\min\{T,N_i\},
\qquad
i=1,\ldots,K .
\end{equation}

Since arms are indexed so that \(p_1\ge p_2\ge\cdots\ge p_K\), the finite-horizon
regret is
\begin{equation}
R(T)
=
\mathbb E\left[
\sum_{t=1}^{T}(p_1-p_{A_t})
\right]
=
\sum_{i=1}^{K}\delta_i\,\mathbb E[N_i(T)]
=
\sum_{\substack{i=2\\p_i<p_1}}^{K}\delta_i\,\mathbb E[N_i(T)].
\label{eq:rd_regret_def}
\end{equation}

\begin{proposition}[Absorption-based regret skeleton]
\label{prop:rd_regret_skeleton}
For every finite horizon \(T\),
\begin{equation}
R(T)
\ge
T\sum_{\substack{i=2\\p_i<p_1}}^{K}\delta_i\,\mathbb P(Q_i)
-
\sum_{i=1}^{K}\mathbb P(Q_i)
\sum_{\substack{j\neq i\\p_j>p_i}}
(p_j-p_i)\,
\mathbb E[\min\{T,N_j\}\mid Q_i],
\label{eq:rd_regret_lower_skeleton}
\end{equation}
and
\begin{equation}
R(T)
\le
T\sum_{\substack{i=2\\p_i<p_1}}^{K}\delta_i\,\mathbb P(Q_i)
+
\sum_{i=1}^{K}\mathbb P(Q_i)
\sum_{\substack{j\neq i\\p_j<p_i}}
(p_i-p_j)\,
\mathbb E[\min\{T,N_j\}\mid Q_i].
\label{eq:rd_regret_upper_skeleton}
\end{equation}
Terms conditioned on events of probability zero are interpreted as zero
contributions.
\end{proposition}

\begin{proof}\leavevmode
Using the absorbing partition from
Lemma~\ref{lem:topological_collapse_karm},
\begin{equation}
\mathbb E[N_j(T)]
=
\sum_{i=1}^{K}
\mathbb P(Q_i)\,
\mathbb E[N_j(T)\mid Q_i].
\label{eq:rd_count_partition}
\end{equation}
Substituting \eqref{eq:rd_count_partition} into
\eqref{eq:rd_regret_def}, and recalling that \(\delta_i=0\) when \(p_i=p_1\), gives
\[
R(T)
=
\sum_{i=1}^{K}
\mathbb P(Q_i)
\sum_{j=1}^{K}
\delta_j\,
\mathbb E[N_j(T)\mid Q_i].
\]

Fix \(i\in[K]\). Since exactly one arm is pulled in every period,
\[
\sum_{j=1}^{K}N_j(T)=T,
\]
and hence
\[
\mathbb E[N_i(T)\mid Q_i]
=
T-
\sum_{j\neq i}
\mathbb E[N_j(T)\mid Q_i].
\]
It follows that
\begin{equation}
\sum_{j=1}^{K}
\delta_j\,
\mathbb E[N_j(T)\mid Q_i]
=
T\delta_i
+
\sum_{j\neq i}
(\delta_j-\delta_i)\,
\mathbb E[N_j(T)\mid Q_i].
\label{eq:rd_centered_branch_decomposition}
\end{equation}
Because
\[
\delta_j-\delta_i
=
(p_1-p_j)-(p_1-p_i)
=
p_i-p_j,
\]
substituting this identity into
\eqref{eq:rd_centered_branch_decomposition}, and then summing over the
absorbing partition, yields
\begin{equation}
R(T)
=
T\sum_{\substack{i=2\\p_i<p_1}}^{K}\delta_i\,\mathbb P(Q_i)
+
\sum_{i=1}^{K}\mathbb P(Q_i)
\sum_{j\neq i}
(p_i-p_j)\,
\mathbb E[N_j(T)\mid Q_i].
\label{eq:rd_regret_centered_identity}
\end{equation}

For each \(i\in[K]\), the centered correction admits the sign decomposition
\begin{equation}
\sum_{j\neq i}
(p_i-p_j)\,
\mathbb E[N_j(T)\mid Q_i]
=
\sum_{\substack{j\neq i\\p_j<p_i}}
(p_i-p_j)\,
\mathbb E[N_j(T)\mid Q_i]
-
\sum_{\substack{j\neq i\\p_j>p_i}}
(p_j-p_i)\,
\mathbb E[N_j(T)\mid Q_i].
\label{eq:rd_centered_sign_split}
\end{equation}

For the lower bound, discard the nonnegative first sum on the right-hand side
of \eqref{eq:rd_centered_sign_split}. Moreover, for \(j\neq i\), the
capped-count domination \eqref{eq:rd_count_capped_domination} gives
\[
N_j(T)\le\min\{T,N_j\}
\qquad\text{on }Q_i.
\]
Therefore,
\[
\sum_{j\neq i}
(p_i-p_j)\,
\mathbb E[N_j(T)\mid Q_i]
\ge
-
\sum_{\substack{j\neq i\\p_j>p_i}}
(p_j-p_i)\,
\mathbb E[\min\{T,N_j\}\mid Q_i].
\]
Substituting this inequality into
\eqref{eq:rd_regret_centered_identity} proves
\eqref{eq:rd_regret_lower_skeleton}.

For the upper bound, discard instead the nonpositive second sum on the
right-hand side of \eqref{eq:rd_centered_sign_split}. Using again
\eqref{eq:rd_count_capped_domination},
\[
\sum_{j\neq i}
(p_i-p_j)\,
\mathbb E[N_j(T)\mid Q_i]
\le
\sum_{\substack{j\neq i\\p_j<p_i}}
(p_i-p_j)\,
\mathbb E[\min\{T,N_j\}\mid Q_i].
\]
Substituting this inequality into
\eqref{eq:rd_regret_centered_identity} proves
\eqref{eq:rd_regret_upper_skeleton}.

\end{proof}

The decomposition above shows that the leading finite-horizon regret is governed by the
absorbing probabilities
\[
% \label{eq:rd_core_absorbing_probabilities}
\mathbb P(Q_i),
\qquad
i=2,\ldots,K,
\qquad
p_i<p_1,
\]
while the remaining terms are capped transient abandonment costs of the form
\[
% \label{eq:rd_core_transient_costs}
\mathbb E[\min\{T,N_j\}\mid Q_i],
\qquad
j\neq i.
\]
The subsequent sections therefore focus on the two core tasks:
\[\begin{aligned}
&\text{characterizing the absorbing probabilities } \mathbb P(Q_i),
\qquad i=2,\ldots,K,
\qquad
p_i<p_1,
\\
% \label{eq:rd_core_task_probabilities}
&\text{and controlling the capped abandonment costs }
\mathbb E[\min\{T,N_j\}\mid Q_i],
\qquad j\neq i.
% \label{eq:rd_core_task_costs}
\end{aligned}\]

\section{Score-Minimum Reduction and Boundary-Crossing Reformulation}
\label{app:score_minimum_reduction}

This section converts the absorbing events into static comparisons among arm-wise score minima. From this section through Appendix~\ref{app:cumulative_regret_envelope}, we focus on the regularized case \(\alpha>0\). The pure-greedy case \((\alpha,\beta)=(0,0)\) is treated separately in Appendix~\ref{app:pure_greedy_regime}. This restriction avoids a boundary phenomenon that occurs under pure greedy. When \(\alpha=\beta=0\), the score of an arm is the empirical mean \(S_i(n)/n\), which can equal zero whenever the arm has accumulated no successes. For example, if all arms receive a zero reward on their first pull, several arms start exactly at score zero, and further zero rewards may keep an arm at the same boundary for multiple pulls. Thus the relevant event is no longer a simple first arrival at a positive boundary; it may involve repeated contacts with the zero boundary. By contrast, when \(\alpha>0\), every regularized score \((S_i(n)+\alpha)/(n+\beta)\) is strictly positive. The score can become small, but it does not hit zero, so the one-arm score-minimum event can be represented cleanly as a first-passage boundary-crossing event. Therefore, we only treat the regularized regime here.

For each arm \(i\), we first define its potential score minimum \(M_i\) along an independent infinite reward stream, and we introduce the centered score-minimum variable \(L_i:=p_i-M_i\). The absorption result from Appendix~\ref{app:absorption} allows us to compare eventual winners through these potential score minima: arm \(i\) can absorb only if its potential score minimum is no smaller than the competing arms' potential score minima, up to strict or weak tie conventions. This yields an event sandwich for \(Q_i\) in terms of inequalities involving \(M_i\), equivalently \(L_i\). We then show that, for a fixed centered level \(\ell_i\in(0,p_i)\), the one-arm event \(\{L_i>\ell_i\}\) is exactly a boundary-crossing event for a negative-drift Bernoulli random walk with boundary height \(b_i(\ell_i)=\Delta_i+\beta\ell_i\). This reformulation is the bridge from the absorbing-event skeleton to the Lundberg-root analysis developed in the following sections.

\subsection{Event Sandwich via Potential Score Minima}
\label{app:event_sandwich_score_minima}

We now define the static variables that determine the absorbing events. For each arm \(i\), let
\(
% \label{eq:sm_potential_stream_def}
X_{i,1},X_{i,2},\ldots
\)
be an independent infinite reward stream with
\(
% \label{eq:sm_stream_law}
X_{i,n}\sim\mathrm{Bernoulli}(p_i),\, n\ge1
\). 
Define the potential cumulative successes by
\[
% \label{eq:sm_potential_successes}
S_i(n)
:=
\sum_{r=1}^{n}X_{i,r},
\qquad
n\ge1,
\]
and recall that the potential regularized score after \(n\) pulls is
\(
% \label{eq:sm_potential_score}
\widehat p_i(n)
=
{(S_i(n)+\alpha)}/{(n+\beta)}
\). 
The potential score minimum of arm \(i\) is
\(
% \label{eq:sm_score_minimum_def}
M_i
:=
\inf_{n\ge1}\widehat p_i(n)
\). 
Although we write the definition in infimum form, this infimum is an attained minimum. Indeed, \(\widehat p_i(n)\to p_i\) almost surely, while the centered Bernoulli partial sums have arbitrarily large negative fluctuations along a subsequence; hence \(\widehat p_i(n)<p_i\) for some finite \(n\) almost surely. After that strict descent below \(p_i\), convergence back to \(p_i\) implies that only finitely many later indices can improve the minimum. Therefore there exists an almost surely finite random index \(n_i^\star\) such that
\(
% \label{eq:sm_score_minimum_attained}
M_i
=
\widehat p_i(n_i^\star)
\). 

We then define the centered score-minimum variable
\[
% \label{eq:sm_centered_score_minimum_def}
L_i
:=
p_i-M_i.
\]

Because \(\alpha>0\), the potential regularized score is strictly positive:
\(
% \label{eq:sm_score_positive}
\widehat p_i(n)>0,\,
n\ge1
\). 
Moreover, the preceding strict descent below \(p_i\) gives \(M_i<p_i\) almost surely.
Consequently,
\(
% \label{eq:sm_centered_variable_range}
0<L_i<p_i
\text{ almost surely on } Q_i
\). 

For each arm \(i\), define the strict and weak score-minimum winner events
\[
\mathcal W_i^{>}
:=
\left\{
M_i>\max_{j\neq i}M_j
\right\},
\qquad
% \label{eq:sm_strict_winner_event}
\mathcal W_i^{\ge}
:=
\left\{
M_i\ge\max_{j\neq i}M_j
\right\}.
% \label{eq:sm_weak_winner_event}
\]

\begin{lemma}[Absorbing-event sandwich by potential score minima]
\label{lem:sm_absorbing_event_sandwich}
For each arm \(i\),
\begin{equation}
\label{eq:sm_winner_sandwich_M}
\mathcal W_i^{>}
\subseteq
Q_i
\subseteq
\mathcal W_i^{\ge}.
\end{equation}
Equivalently, in centered score-minimum form,
\begin{equation}
\label{eq:sm_winner_sandwich_L}
\left\{
L_i-L_j<p_i-p_j,\ \forall j\neq i
\right\}
\subseteq
Q_i
\subseteq
\left\{
L_i-L_j\le p_i-p_j,\ \forall j\neq i
\right\}.
\end{equation}
Consequently, with \(Q_i^c=\bigcup_{j\neq i}Q_j\),
\begin{equation}
\label{eq:sm_complement_sandwich_M}
\left\{
\exists j\neq i:\ M_j>M_i
\right\}
\subseteq
Q_i^c
\subseteq
\left\{
\exists j\neq i:\ M_j\ge M_i
\right\}.
\end{equation}
\end{lemma}

\begin{proof}\leavevmode
We first prove the weak necessary inclusion. Suppose \(Q_i\) occurs. Then arm
\(i\) is pulled infinitely often. Let \(n_i^\star<\infty\) be an index at which
the potential score minimum of arm \(i\) is attained:
\[
% \label{eq:sm_minimum_attaining_index_in_proof}
\widehat p_i(n_i^\star)=M_i .
\]
Since arm \(i\) is pulled infinitely often on \(Q_i\), the actual trajectory
eventually reaches this pull count. Let \(t_i^\star\) be the decision epoch
immediately after the \(n_i^\star\)-th pull of arm \(i\). At that epoch, the
current score of arm \(i\) is exactly \(M_i\).

We claim that no competing arm can have potential score minimum strictly above
\(M_i\). Suppose, toward a contradiction, that for some \(j\neq i\),
\(
% \label{eq:sm_competitor_minimum_above_Mi_contra}
M_j>M_i
\). 
At the epoch \(t_i^\star\), the current score of arm \(j\) is one of its potential scores, and therefore is at least its potential minimum \(M_j\). Hence
\(
% \label{eq:sm_competitor_current_score_above_Mi}
\widehat p_j(N_j(t_i^\star))
\ge
M_j
>
M_i
=
\widehat p_i(N_i(t_i^\star))
\). 
Moreover, after \(t_i^\star\), as long as arm \(i\) is not selected, its score remains frozen at \(M_i\). Arm \(j\)'s score, whether frozen or updated by future pulls, can never fall below \(M_j\), and hence can never fall below \(M_i\). Thus there is always at least one arm with score strictly larger than arm \(i\)'s frozen score \(M_i\). Consequently, arm \(i\) can never again be a maximizer of the greedy score after \(t_i^\star\).

This contradicts \(Q_i\), because on \(Q_i\) arm \(i\) is selected forever after
some finite time. Therefore,
\(
% \label{eq:sm_Qi_implies_all_competitor_minima_below}
M_j\le M_i,\,
j\neq i
\). 
Equivalently,
\begin{equation}
\label{eq:sm_Qi_implies_weak_winner}
Q_i\subseteq\mathcal W_i^{\ge}.
\end{equation}

We now prove the strict sufficient inclusion. Suppose
\begin{equation}
\label{eq:sm_strict_minimum_advantage}
M_i>\max_{j\neq i}M_j .
\end{equation}
If \(Q_i\) did not occur, then by the absorbing partition from
Lemma~\ref{lem:topological_collapse_karm}, there would exist some \(r\neq i\)
such that \(Q_r\) occurs. Applying the weak necessary inclusion already proved,
but now to the absorbing arm \(r\), gives
\(
% \label{eq:sm_Qr_implies_Mr_ge_Mi}
M_i\le M_r
\). 
This contradicts \eqref{eq:sm_strict_minimum_advantage}. Hence
\begin{equation}
\label{eq:sm_strict_winner_implies_Qi}
\mathcal W_i^{>}
\subseteq
Q_i .
\end{equation}
Combining \eqref{eq:sm_Qi_implies_weak_winner} and
\eqref{eq:sm_strict_winner_implies_Qi} proves
\eqref{eq:sm_winner_sandwich_M}.

It remains to rewrite the comparison in terms of \(L_i\). Since
\(
% \label{eq:sm_M_L_relation}
M_i=p_i-L_i
\), 
we have
\begin{equation}
M_i>M_j
\Longleftrightarrow
p_i-L_i>p_j-L_j
\Longleftrightarrow
L_i-L_j<p_i-p_j,
\label{eq:sm_strict_M_to_L}
\end{equation}
and similarly
\begin{equation}
M_i\ge M_j
\Longleftrightarrow
p_i-L_i\ge p_j-L_j
\Longleftrightarrow
L_i-L_j\le p_i-p_j.
\label{eq:sm_weak_M_to_L}
\end{equation}
Substituting \eqref{eq:sm_strict_M_to_L} and \eqref{eq:sm_weak_M_to_L} into
\eqref{eq:sm_winner_sandwich_M} proves \eqref{eq:sm_winner_sandwich_L}.

Finally, taking complements in \eqref{eq:sm_winner_sandwich_M} gives
\begin{equation}
\label{eq:sm_complement_from_sandwich}
(\mathcal W_i^{\ge})^c
\subseteq
Q_i^c
\subseteq
(\mathcal W_i^{>})^c.
\end{equation}
The two complement events are
\begin{align}
(\mathcal W_i^{\ge})^c
&=
\left\{
\exists j\neq i:\ M_j>M_i
\right\},
\label{eq:sm_complement_weak_event}\\
(\mathcal W_i^{>})^c
&=
\left\{
\exists j\neq i:\ M_j\ge M_i
\right\}.
\label{eq:sm_complement_strict_event}
\end{align}
Substituting \eqref{eq:sm_complement_weak_event} and
\eqref{eq:sm_complement_strict_event} into
\eqref{eq:sm_complement_from_sandwich} proves
\eqref{eq:sm_complement_sandwich_M}.

\end{proof}

\subsection{One-Arm Centered-Minimum Crossing and CGF Setup}
\label{app:one_arm_drawdown_crossing}

The event sandwich in Appendix~\ref{app:event_sandwich_score_minima} reduces the complement \(Q_i^c\) to comparisons among centered score-minimum variables. In particular,
\begin{equation}
\label{eq:bc_Qic_sandwich_recall}
\left\{
\exists j\neq i:\ L_i-L_j>p_i-p_j
\right\}
\subseteq
Q_i^c
\subseteq
\left\{
\exists j\neq i:\ L_i-L_j\ge p_i-p_j
\right\}.
\end{equation}
Thus, after conditioning on the competing variables \(L_j\), the relevant events reduce to one-arm inequalities of the form
\[
L_i
>
p_i-p_j+L_j,
\qquad
% \label{eq:bc_random_strict_level}
L_i
\ge
p_i-p_j+L_j.
% \label{eq:bc_random_weak_level}
\]
The purpose of this subsection is to rewrite these one-arm inequalities as boundary-crossing events. The strict and weak versions correspond to the same random walk and the same boundary, with only a strict--weak crossing convention. The difference will later be absorbed into a bounded discrete-overshoot correction.

Fix an arm \(i\in\{1,\ldots,K\}\) and a deterministic centered level
\(
% \label{eq:bc_level_domain}
\ell_i\in(0,p_i)
\). 
The nontrivial crossing regime is exactly this interior interval. If \(\ell_i\ge p_i\), then \(\{L_i>\ell_i\}\) and \(\{L_i\ge\ell_i\}\) are empty because \(L_i<p_i\). If \(\ell_i\leq0\), then the score threshold \(p_i-\ell_i\) is above the limiting score \(p_i\), and crossing is eventually automatic. Hence the Lundberg-root analysis is needed only for \(\ell_i\in(0,p_i)\).

Recall
\[
% \label{eq:bc_Li_def}
L_i
:=
p_i-\inf_{n\ge1}\widehat p_i(n).
\]
Although the score minimum is written as an infimum, it is attained in every crossing event considered here. Equivalently,
\begin{align}
\left\{
\inf_{n\ge1}\widehat p_i(n)\le p_i-\ell_i
\right\}
&=
\left\{
\exists n\ge1:
\widehat p_i(n)\le p_i-\ell_i
\right\}.
\label{eq:bc_infimum_attained_below_level}
\end{align}

Define
\(
% \label{eq:bc_qi_def}
q_i:=1-p_i
\), 
and recall
\(
% \label{eq:bc_Delta_i_def}
\Delta_i:=\alpha-p_i\beta
\). 
For the level \(\ell_i\), define the boundary height
\[
b_i(\ell_i)
:=
\alpha-(p_i-\ell_i)\beta
=
\Delta_i+\beta\ell_i.
% \label{eq:bc_boundary_height_def}
\]
Under \(\alpha>0\), \(\beta\ge0\), and \(\Delta_i\ge0\),
\(
% \label{eq:bc_boundary_positive}
b_i(\ell_i)>0
\). 

For the potential reward stream of arm \(i\), define the level-dependent increment
\[
% \label{eq:bc_increment_def}
Z_{i,r}(\ell_i)
:=
(p_i-\ell_i)-X_{i,r},
\qquad
r\ge1,
\]
and the associated random walk
\[
U_{i,n}(\ell_i)
:=
\sum_{r=1}^{n}Z_{i,r}(\ell_i)
=
(p_i-\ell_i)n-S_i(n),
\qquad
n\ge1,
% \label{eq:bc_random_walk_def}
\]
with
\(
% \label{eq:bc_random_walk_initial}
U_{i,0}(\ell_i):=0.
\)
The increment \(Z_{i,r}(\ell_i)\) takes the two values
\[
% \label{eq:bc_increment_support}
Z_{i,r}(\ell_i)
=
\begin{cases}
p_i-\ell_i, & X_{i,r}=0,\\
-(q_i+\ell_i), & X_{i,r}=1.
\end{cases}
\]
Its mean is
\(
\mathbb E[Z_{i,r}(\ell_i)]
=
(p_i-\ell_i)-p_i
=
-\ell_i<0.
% \label{eq:bc_increment_mean}
\)
Thus the crossing of the positive boundary \(b_i(\ell_i)\) is an upward large-deviation event for a negative-drift Bernoulli random walk.

\begin{lemma}[One-arm centered-minimum crossing]
\label{lem:bc_one_arm_crossing}
For every \(i\) and every \(\ell_i\in(0,p_i)\),
\begin{equation}
\label{eq:bc_strict_crossing_equivalence}
\{L_i>\ell_i\}
=
\left\{
\sup_{n\ge1}U_{i,n}(\ell_i)>b_i(\ell_i)
\right\}.
\end{equation}
Moreover,
\begin{equation}
\label{eq:bc_weak_crossing_equivalence}
\{L_i\ge\ell_i\}
=
\left\{
\sup_{n\ge1}U_{i,n}(\ell_i)\ge b_i(\ell_i)
\right\}.
\end{equation}
\end{lemma}

\begin{proof}\leavevmode
For the strict event,
\begin{align}
\{L_i>\ell_i\}
&=
\left\{
p_i-\inf_{n\ge1}\widehat p_i(n)>\ell_i
\right\}
\nonumber\\
&=
\left\{
\inf_{n\ge1}\widehat p_i(n)<p_i-\ell_i
\right\}
\nonumber\\
&=
\left\{
\exists n\ge1:
\widehat p_i(n)<p_i-\ell_i
\right\}
\nonumber\\
&=
\left\{
\exists n\ge1:
\frac{S_i(n)+\alpha}{n+\beta}<p_i-\ell_i
\right\}.
\label{eq:bc_strict_event_score_form}
\end{align}
For every \(n\ge1\),
\begin{align}
\frac{S_i(n)+\alpha}{n+\beta}<p_i-\ell_i
&\Longleftrightarrow
S_i(n)+\alpha<(p_i-\ell_i)(n+\beta)
\nonumber\\
&\Longleftrightarrow
(p_i-\ell_i)n-S_i(n)>\alpha-(p_i-\ell_i)\beta
\nonumber\\
&\Longleftrightarrow
U_{i,n}(\ell_i)>b_i(\ell_i).
\label{eq:bc_strict_algebra}
\end{align}
Combining \eqref{eq:bc_strict_event_score_form} and \eqref{eq:bc_strict_algebra} gives \eqref{eq:bc_strict_crossing_equivalence}.

For the weak event, the attainment property \eqref{eq:bc_infimum_attained_below_level} gives
\begin{align}
\{L_i\ge\ell_i\}
&=
\left\{
p_i-\inf_{n\ge1}\widehat p_i(n)\ge\ell_i
\right\}
\nonumber\\
&=
\left\{
\inf_{n\ge1}\widehat p_i(n)\le p_i-\ell_i
\right\}
\nonumber\\
&=
\left\{
\exists n\ge1:
\widehat p_i(n)\le p_i-\ell_i
\right\}
\nonumber\\
&=
\left\{
\exists n\ge1:
\frac{S_i(n)+\alpha}{n+\beta}\le p_i-\ell_i
\right\}.
\label{eq:bc_weak_event_score_form}
\end{align}
For every \(n\ge1\),
\begin{align}
\frac{S_i(n)+\alpha}{n+\beta}\le p_i-\ell_i
&\Longleftrightarrow
S_i(n)+\alpha\le(p_i-\ell_i)(n+\beta)
\nonumber\\
&\Longleftrightarrow
(p_i-\ell_i)n-S_i(n)\ge\alpha-(p_i-\ell_i)\beta
\nonumber\\
&\Longleftrightarrow
U_{i,n}(\ell_i)\ge b_i(\ell_i).
\label{eq:bc_weak_algebra}
\end{align}
Combining \eqref{eq:bc_weak_event_score_form} and \eqref{eq:bc_weak_algebra} proves \eqref{eq:bc_weak_crossing_equivalence}.

\end{proof}

Define the strict and weak first-passage times
\[
\tau_i^{>}(\ell_i)
:=
\inf\left\{
n\ge1:
U_{i,n}(\ell_i)>b_i(\ell_i)
\right\},
\qquad
% \label{eq:bc_strict_tau_def}
\tau_i^{\ge}(\ell_i)
:=
\inf\left\{
n\ge1:
U_{i,n}(\ell_i)\ge b_i(\ell_i)
\right\}.
% \label{eq:bc_weak_tau_def}
\]
Lemma~\ref{lem:bc_one_arm_crossing} gives
\[
\{L_i>\ell_i\}
=
\{\tau_i^{>}(\ell_i)<\infty\},
\qquad
% \label{eq:bc_strict_tau_event}
\{L_i\ge\ell_i\}
=
\{\tau_i^{\ge}(\ell_i)<\infty\}.
% \label{eq:bc_weak_tau_event}
\]

The two first-passage conventions differ only at the boundary. For either convention, write \(\tau_i(\ell_i)\) for the corresponding first-passage time:
\(
% \label{eq:bc_generic_tau_convention}
\tau_i(\ell_i)
\in
\left\{
\tau_i^{>}(\ell_i),
\tau_i^{\ge}(\ell_i)
\right\}
\). 
On the crossing event \(\{\tau_i(\ell_i)<\infty\}\), define the overshoot
\begin{equation}
\label{eq:bc_overshoot_def}
r_i(\ell_i)
:=
U_{i,\tau_i(\ell_i)}(\ell_i)-b_i(\ell_i).
\end{equation}
Because an upward crossing can occur only through a single positive increment of size \(p_i-\ell_i\), the overshoot satisfies
\(
% \label{eq:bc_overshoot_range}
0
\le
r_i(\ell_i)
\le
p_i-\ell_i
\). 
Thus the terminal state at first passage can be written uniformly as
\begin{equation}
\label{eq:bc_terminal_state_unified}
U_{i,\tau_i(\ell_i)}(\ell_i)
=
b_i(\ell_i)+r_i(\ell_i),
\qquad
0\le r_i(\ell_i)\le p_i-\ell_i .
\end{equation}
More precisely, for strict crossing,
\[
% \label{eq:bc_strict_terminal_state}
\tau_i(\ell_i)=\tau_i^{>}(\ell_i)
\quad\Longrightarrow\quad
0<r_i(\ell_i)\le p_i-\ell_i ,
\]
whereas for weak crossing,
\[
% \label{eq:bc_weak_terminal_state}
\tau_i(\ell_i)=\tau_i^{\ge}(\ell_i)
\quad\Longrightarrow\quad
0\le r_i(\ell_i)<p_i-\ell_i .
\]
The common terminal representation \eqref{eq:bc_terminal_state_unified} is sufficient for the Lundberg estimates below. This bounded overshoot is the only correction separating the strict and weak sides of the event sandwich in \eqref{eq:bc_Qic_sandwich_recall}.

We now introduce the cumulant generating function (CGF) of the one-step increment. For \(\theta\in\mathbb R\), define
\[
\Lambda_i(\theta,\ell_i)
:=
\log
\mathbb E\left[
\exp\left\{
\theta Z_{i,1}(\ell_i)
\right\}
\right]
=
\log\left[
q_i \exp\{\theta(p_i-\ell_i)\}
+
p_i \exp\{-\theta(q_i+\ell_i)\}
\right].
% \label{eq:bc_cgf_def}
\]
At the origin,
\(
% \label{eq:bc_cgf_origin}
\Lambda_i(0,\ell_i)=0
\). 
Its first derivative at the origin is the negative drift:
\[
\partial_\theta\Lambda_i(0,\ell_i)
=
\mathbb E[Z_{i,1}(\ell_i)]
=
-\ell_i<0.
% \label{eq:bc_cgf_origin_derivative}
\]
The positive Lundberg root, when it exists, is denoted by
\begin{equation}
\label{eq:bc_lundberg_root_def}
\theta_i(\ell_i)>0,
\qquad
\Lambda_i(\theta_i(\ell_i),\ell_i)=0.
\end{equation}
The corresponding tilted drift is
\[
% \label{eq:bc_tilted_drift_def}
v_i(\ell_i)
:=
\partial_\theta\Lambda_i(\theta_i(\ell_i),\ell_i).
\]
The next section proves the existence and uniqueness of \(\theta_i(\ell_i)\), the strict positivity of \(v_i(\ell_i)\), and the endpoint behavior of both quantities.

\section{Analytic Properties of the Lundberg Root and Tilted Drift}
\label{app:cgf_drawdown_geometry}

We now characterize the one-arm crossing probability from
Appendix~\ref{app:one_arm_drawdown_crossing}. For any arm \(i\) and any level \(\ell_i\in(0,p_i)\), recall that
\[
\{L_i>\ell_i\}
=
\left\{
\sup_{n\ge1}U_{i,n}(\ell_i)>b_i(\ell_i)
\right\},
\qquad
b_i(\ell_i)=\Delta_i+\beta\ell_i .
\]
The probability of this event is governed by the exponential cost of pushing the
negative-drift walk \(U_{i,n}(\ell_i)\) above the boundary
\(b_i(\ell_i)\). We write this cost in root-based form as
\[
\mathbb P(L_i>\ell_i)
\approx
\exp\{-\theta_i(\ell_i)b_i(\ell_i)\}.
\]
The term \(\theta_i(\ell_i)\) is the Lundberg exponent. It is defined as the
positive solution of
\[
\Lambda_i(\theta,\ell_i)=0,
\qquad
\Lambda_i(\theta,\ell_i)
=
\log\left[
q_i \exp\{\theta(p_i-\ell_i)\}
+
p_i \exp\{-\theta(q_i+\ell_i)\}
\right].
\]
This equation is the martingale condition. Indeed, when
\(\theta=\theta_i(\ell_i)\), the process
\(\exp\{\theta U_{i,n}(\ell_i)\}\) has mean one at every step. Stopping this
martingale at the first crossing time gives the root-based crossing identity up
to the bounded overshoot at the boundary. Thus the exponent
\(\theta_i(\ell_i)b_i(\ell_i)\) is the primitive quantity behind the
one-arm drawdown tail.

The section establishes the properties of this exponent needed in the
asymptotic analysis. First, for each \(\ell_i\in(0,p_i)\), the equation
\(\Lambda_i(\theta,\ell_i)=0\) has a unique positive solution. The associated
tilted drift
\(
v_i(\ell_i)
=
\partial_\theta\Lambda_i(\theta_i(\ell_i),\ell_i)
\)
is strictly positive, which will later bound the time needed for the tilted walk
to reach the boundary. Second, as \(\ell_i\downarrow0\),
\[
\theta_i(\ell_i)
=
\frac{2}{p_i(1-p_i)}\ell_i+O(\ell_i^2),
\qquad
v_i(\ell_i)=\ell_i+O(\ell_i^2).
\]
Combining this expansion with \(b_i(\ell_i)=\Delta_i+\beta\ell_i\) gives
\[
\theta_i(\ell_i)b_i(\ell_i)
=
\frac{2\Delta_i}{p_i(1-p_i)}\ell_i
+
O(\Delta_i\ell_i^2+\beta\ell_i^2),
\]
which is the source of the local exponential envelope
\[
\mathbb P(L_i>\ell_i)
\approx
\exp\left\{
-\frac{2\Delta_i}{p_i(1-p_i)}\ell_i
\right\}.
\]
Finally, as \(\ell_i\uparrow p_i\), the positive jump size
\(p_i-\ell_i\) vanishes, the root diverges at scale
\((p_i-\ell_i)^{-1}\), and \(v_i(\ell_i)\asymp p_i-\ell_i\). These endpoint
estimates provide the root and drift controls used in the crossing envelopes,
the local Stieltjes replacement, and the later abandonment-cost bounds.

\subsection{Existence and Uniqueness of the Lundberg Root}
\label{app:root_existence_uniqueness}

We first prove that the positive Lundberg root introduced in \eqref{eq:bc_lundberg_root_def} is well-defined for every arm \(i\) and every level \(\ell_i\in(0,p_i)\). Recall that
\[
% \label{eq:root_cgf_recall}
\Lambda_i(\theta,\ell_i)
=
\log\left[
q_i \exp\{\theta(p_i-\ell_i)\}
+
p_i \exp\{-\theta(q_i+\ell_i)\}
\right].
\]

\begin{lemma}[Existence and uniqueness of the Lundberg root]
\label{lem:root_existence_uniqueness}
For every arm \(i\) and every \(\ell_i\in(0,p_i)\), the equation
\(\Lambda_i(\theta,\ell_i)=0\) has a unique positive solution
\(\theta_i(\ell_i)>0\).
\end{lemma}

\begin{proof}\leavevmode
Consider any \(i\) and any \(\ell_i\in(0,p_i)\). By definition,
\(
% \label{eq:root_cgf_at_zero}
\Lambda_i(0,\ell_i)=0
\). 
Moreover,
\[
\partial_\theta\Lambda_i(0,\ell_i)
=
q_i(p_i-\ell_i)-p_i(q_i+\ell_i)
=
-\ell_i
<0 .
% \label{eq:root_origin_negative_slope}
\]
Hence, by continuity, there exists \(\varepsilon>0\) such that
\(
% \label{eq:root_negative_near_origin}
\Lambda_i(\theta,\ell_i)<0,\,
0<\theta<\varepsilon
\). 

On the other hand, since \(p_i \in [\epsilon_{\mathrm{p}}, 1-\epsilon_{\mathrm{p}}]\), \(\epsilon_{\mathrm{p}}>0
\), and \(p_i-\ell_i>0\), the positive-jump term dominates as
\(\theta\to\infty\). More precisely,
\[\begin{aligned}
\Lambda_i(\theta,\ell_i)
&=
\log\left[
q_i \exp\{\theta(p_i-\ell_i)\}
+
p_i \exp\{-\theta(q_i+\ell_i)\}
\right]
\nonumber\\
&=
\theta(p_i-\ell_i)
+
\log\left[
q_i+p_i\exp\{-\theta\}
\right],
% \label{eq:root_cgf_infty_expansion}
\end{aligned}\]
where we used \(
% \label{eq:root_jump_sum_identity}
(p_i-\ell_i)+(q_i+\ell_i)=1
\). 
Therefore,
\(
% \label{eq:root_cgf_diverges}
\Lambda_i(\theta,\ell_i)\to\infty
\text{ as }\theta\to\infty
\). 
Since \(\theta\mapsto\Lambda_i(\theta,\ell_i)\) is continuous, the intermediate
value theorem implies that there exists at least one positive solution of
\(\Lambda_i(\theta,\ell_i)=0\).

It remains to prove uniqueness. Direct differentiation gives
\begin{align}
\partial_{\theta}^2 \Lambda_i(\theta,\ell_i)
&=
\frac{
q_i p_i
\exp\{\theta(p_i-\ell_i)\}
\exp\{-\theta(q_i+\ell_i)\}
\left[(p_i-\ell_i)+(q_i+\ell_i)\right]^2
}{
\left[
q_i \exp\{\theta(p_i-\ell_i)\}
+
p_i \exp\{-\theta(q_i+\ell_i)\}
\right]^2
}
\nonumber\\
&=
\frac{
q_i p_i \exp\{\theta(p_i-\ell_i)\}\exp\{-\theta(q_i+\ell_i)\}
}{
\left[
q_i \exp\{\theta(p_i-\ell_i)\}
+
p_i \exp\{-\theta(q_i+\ell_i)\}
\right]^2
}
>0 .
\label{eq:root_cgf_strict_convexity}
\end{align}
Thus \(\theta\mapsto\Lambda_i(\theta,\ell_i)\) is strictly convex.

Suppose, toward a contradiction, that there are two distinct positive roots
\(0<\theta_a<\theta_b\). Since
\(
% \label{eq:root_zero_and_positive_root}
\Lambda_i(0,\ell_i)=0,\,
\Lambda_i(\theta_b,\ell_i)=0
\), 
strict convexity implies that
\(
% \label{eq:root_strict_convex_negative_interval}
\Lambda_i(\theta,\ell_i)<0,\,
0<\theta<\theta_b
\). 
In particular,
\(
% \label{eq:root_contradiction_at_theta_a}
\Lambda_i(\theta_a,\ell_i)<0
\), 
which contradicts the assumption that \(\theta_a\) is a root. Therefore the positive root is unique. 
\end{proof}

Recall the tilted drift associated with the positive Lundberg root:
\(
% \label{eq:tilted_drift_recall}
v_i(\ell_i)
:=
\partial_\theta\Lambda_i(\theta_i(\ell_i),\ell_i)
\). 
The preceding convexity argument also implies that this drift is strictly
positive.

\begin{corollary}[Strict positivity of the tilted drift]
\label{cor:tilted_drift_positivity}
For every arm \(i\) and every \(\ell_i\in(0,p_i)\),
\[
% \label{eq:tilted_drift_positive}
v_i(\ell_i)>0 .
\]
\end{corollary}

\begin{proof}\leavevmode
By Lemma~\ref{lem:root_existence_uniqueness}, \(\theta_i(\ell_i)\) is the
unique positive root of \(\Lambda_i(\theta,\ell_i)=0\). Moreover,
\(
% \label{eq:tilted_drift_origin_slope}
\partial_\theta\Lambda_i(0,\ell_i)=-\ell_i<0
\), 
and \(\theta\mapsto\Lambda_i(\theta,\ell_i)\) is strictly convex by
\eqref{eq:root_cgf_strict_convexity}. Since
\begin{equation}
\label{eq:tilted_drift_two_zero_values}
\Lambda_i(0,\ell_i)
=
\Lambda_i(\theta_i(\ell_i),\ell_i)
=
0,
\end{equation}
the derivative at the positive root must be strictly positive. Indeed, if
\(
% \label{eq:tilted_drift_nonpositive_assumption}
\partial_\theta\Lambda_i(\theta_i(\ell_i),\ell_i)\le0
\), 
then strict convexity would imply
\(
% \label{eq:tilted_drift_derivative_negative_interval}
\partial_\theta\Lambda_i(\theta,\ell_i)<0,\,
0<\theta<\theta_i(\ell_i)
\), 
and consequently
\[
\Lambda_i(\theta_i(\ell_i),\ell_i)-\Lambda_i(0,\ell_i)
=
\int_0^{\theta_i(\ell_i)}
\partial_\theta\Lambda_i(\theta,\ell_i)\,d\theta
<0 ,
% \label{eq:tilted_drift_integral_contradiction}
\]
which contradicts \eqref{eq:tilted_drift_two_zero_values}. Therefore, 
\(
% \label{eq:tilted_drift_positive_conclusion}
v_i(\ell_i)
=
\partial_\theta\Lambda_i(\theta_i(\ell_i),\ell_i)
>0
\). 
\end{proof}

\subsection{Lower-Endpoint Behavior as \texorpdfstring{$\ell_i\downarrow0$}{elli down to 0}}
\label{app:lower_endpoint_theta_v}

We next analyze the behavior of \(\theta_i(\ell_i)\) and \(v_i(\ell_i)\) as the drawdown level approaches the lower endpoint. The key point is that the Lundberg equation can be rewritten as an inverse relation between \(\ell_i\) and \(\theta_i(\ell_i)\).

Separate the \(\ell_i\)-dependent linear term in the CGF:
\[
% \label{eq:lower_cgf_decomp}
\Lambda_i(\theta,\ell_i)
=
-\theta\ell_i+\Lambda_{i,0}(\theta),
\]
where
\begin{equation}
\label{eq:lower_Lambda0_def}
\Lambda_{i,0}(\theta)
:=
\theta p_i+\log\left(q_i+p_i \exp\{-\theta\}\right).
\end{equation}
Thus the root equation
\(
% \label{eq:lower_root_equation}
\Lambda_i(\theta_i(\ell_i),\ell_i)=0
\)
is equivalent to
\begin{equation}
\label{eq:lower_g_inverse_relation}
\ell_i
=
\frac{\Lambda_{i,0}(\theta_i(\ell_i))}
{\theta_i(\ell_i)}
=
g_i(\theta_i(\ell_i)),
\end{equation}
where, for \(\theta>0\),
\[
% \label{eq:lower_g_def}
g_i(\theta)
:=
\frac{\Lambda_{i,0}(\theta)}{\theta}.
\]

\begin{lemma}[Lower-endpoint equivalence between level and root]
\label{lem:lower_endpoint_root_equivalence}
The map \(g_i\) is continuous and strictly increasing from \((0,\infty)\) onto \((0,p_i)\). 
Consequently, \(\ell_i\mapsto\theta_i(\ell_i)\) is continuous and strictly increasing on \((0,p_i)\), and
\begin{equation}
\label{eq:lower_endpoint_equivalence}
\ell_i\downarrow0
\qquad
\Longleftrightarrow
\qquad
\theta_i(\ell_i)\downarrow0 .
\end{equation}
\end{lemma}

\begin{proof}\leavevmode
We verify the endpoint limits and monotonicity of \(g_i\). First,
\[
% \label{eq:lower_Lambda0_origin_values}
\Lambda_{i,0}(0)=0,
\qquad
\Lambda_{i,0}'(0)=0 .
\]
Hence, by L'Hopital's rule,
\begin{equation}
\label{eq:lower_g_limit_zero}
\lim_{\theta\downarrow0}g_i(\theta)
=
\lim_{\theta\downarrow0}\Lambda_{i,0}'(\theta)
=
0 .
\end{equation}

Next define
\(
% \label{eq:lower_h_def}
h_i(\theta)
:=
\theta\Lambda_{i,0}'(\theta)-\Lambda_{i,0}(\theta)
\). 
Then
\(
% \label{eq:lower_h_derivative}
h_i'(\theta)
=
\theta\Lambda_{i,0}''(\theta)
\). 
A direct calculation gives
\[
% \label{eq:lower_Lambda0_second_derivative}
\Lambda_{i,0}''(\theta)
=
\frac{p_iq_i \exp\{-\theta\}}
{\left(q_i+p_i \exp\{-\theta\}\right)^2}
>0,
\qquad
\theta>0 .
\]
Therefore,
\(
% \label{eq:lower_h_positive}
h_i'(\theta)>0,\,
\theta>0
\). 
Since \(h_i(0)=0\), it follows that
\(
% \label{eq:lower_h_strict_positive}
h_i(\theta)>0,\,
\theta>0
\). 
Consequently,
\begin{equation}
g_i'(\theta)
=
\frac{\theta\Lambda_{i,0}'(\theta)-\Lambda_{i,0}(\theta)}
{\theta^2}
=
\frac{h_i(\theta)}{\theta^2}
>0,
\qquad
\theta>0 .
\label{eq:lower_g_prime_positive}
\end{equation}
Thus \(g_i\) is strictly increasing.

Finally, from \eqref{eq:lower_Lambda0_def},
\(
g_i(\theta)
=
p_i+
\frac{1}{\theta}
\log\left(q_i+p_i \exp\{-\theta\}\right)
% \label{eq:lower_g_infty_form}
\). 
Since \(q_i\in [\epsilon_{\mathrm{p}}, 1-\epsilon_{\mathrm{p}}]\), \(\epsilon_{\mathrm{p}}>0\), 
\begin{equation}
\label{eq:lower_g_limit_infty}
\lim_{\theta\to\infty}g_i(\theta)
=
p_i .
\end{equation}
Combining \eqref{eq:lower_g_limit_zero}, \eqref{eq:lower_g_prime_positive}, and \eqref{eq:lower_g_limit_infty}, \(g_i\) is a continuous strictly increasing bijection from \((0,\infty)\) onto \((0,p_i)\). By \eqref{eq:lower_g_inverse_relation},
\(
% \label{eq:lower_theta_inverse}
\theta_i(\ell_i)=g_i^{-1}(\ell_i)
\). 
This proves the continuity and monotonicity of \(\theta_i(\ell_i)\), and the endpoint equivalence \eqref{eq:lower_endpoint_equivalence}. 
\end{proof}

We now use this inverse representation to compute the local expansion of the root and the tilted drift. Let
\[
% \label{eq:lower_sigma_def}
\sigma_i^2:=p_iq_i .
\]

\begin{lemma}[Local expansion of the root and tilted drift]
\label{lem:lower_endpoint_theta_v_expansion}
For every arm \(i=1, \ldots, K\), as \(\ell_i\downarrow0\),
\begin{equation}
\label{eq:lower_theta_expansion}
\theta_i(\ell_i)
=
\frac{2}{\sigma_i^2}\ell_i
+
\frac{4(1-2p_i)}{3\sigma_i^4}\ell_i^2
+
O(\ell_i^3),
\end{equation}
and
\begin{equation}
\label{eq:lower_v_expansion}
v_i(\ell_i)
=
\ell_i
-
\frac{2(1-2p_i)}{3\sigma_i^2}\ell_i^2
+
O(\ell_i^3).
\end{equation}
In particular,
\begin{equation}
\label{eq:lower_theta_v_first_order}
\theta_i(\ell_i)
=
\frac{2}{\sigma_i^2}\ell_i
+
O(\ell_i^2),
\qquad
v_i(\ell_i)
=
\ell_i
+
O(\ell_i^2).
\end{equation}
\end{lemma}

\begin{proof}\leavevmode
For \(i\in\{1,\ldots,K\}\), by the standing interior
condition,
\begin{equation}
\label{eq:lower_variance_bounds}
\epsilon_{\mathrm p}(1-\epsilon_{\mathrm p})
\le
\sigma_i^2
=
p_i(1-p_i)
\le
\frac14.
\end{equation}
Thus \(\sigma_i^2\) is bounded away from both zero and infinity, and therefore
\[
\ell_i\downarrow0
\qquad
\Longleftrightarrow
\qquad
\frac{\ell_i}{\sigma_i^2}\downarrow0.
\]

The function \(\Lambda_{i,0}\) is real analytic in a neighborhood of \(\theta=0\). Its normalized derivatives at the origin satisfy
\[
\frac{\Lambda_{i,0}(0)}{\sigma_i^2}
=
\frac{\Lambda_{i,0}'(0)}{\sigma_i^2}
=
0,
\qquad
\frac{\Lambda_{i,0}''(0)}{\sigma_i^2}
=
1,
\qquad
\frac{\Lambda_{i,0}'''(0)}{\sigma_i^2}
=
-(1-2p_i),
\qquad
\frac{\Lambda_{i,0}^{(4)}(0)}{\sigma_i^2}
=
1-6\sigma_i^2.
\]
Because \(\Lambda_{i,0}^{(4)}(\theta)/\sigma_i^2\) is jointly continuous in \((\theta,p_i)\), the standing interior condition~\eqref{eq:app_ordered_means} ensures that it is uniformly bounded over all admissible \(p_i\) and all \(\theta\) in a sufficiently small neighborhood of zero. Hence
\begin{equation}
\label{eq:lower_Lambda0_expansion}
\frac{\Lambda_{i,0}(\theta)}{\sigma_i^2}
=
\frac12\theta^2
-
\frac{1-2p_i}{6}\theta^3
+
O(\theta^4),
\end{equation}
uniformly over all arms and admissible arm means. Recalling that \(g_i(\theta)=\Lambda_{i,0}(\theta)/\theta\), we obtain
\begin{equation}
\label{eq:lower_g_expansion}
\frac{g_i(\theta)}{\sigma_i^2}
=
\frac12\theta
-
\frac{1-2p_i}{6}\theta^2
+
O(\theta^3).
\end{equation}

Moreover,
\[
\left.
\frac{d}{d\theta}
\frac{g_i(\theta)}{\sigma_i^2}
\right|_{\theta=0}
=
\frac12
>0.
\]
Therefore, the analytic inverse function theorem applies at the origin. By
the inverse relation \eqref{eq:lower_g_inverse_relation},
\begin{equation}
\label{eq:lower_normalized_inverse_relation}
\frac{\ell_i}{\sigma_i^2}
=
\frac{
g_i\bigl(\theta_i(\ell_i)\bigr)
}{
\sigma_i^2
},
\end{equation}
and Lemma~\ref{lem:lower_endpoint_root_equivalence} ensures that the resulting
local inverse coincides with
\(\ell_i\mapsto\theta_i(\ell_i)\) for all sufficiently small
\(\ell_i>0\).

Write the local inverse expansion directly in powers of
\(\ell_i/\sigma_i^2\):
\begin{equation}
\label{eq:lower_theta_ansatz}
\theta_i(\ell_i)
=
d_{i,1}\frac{\ell_i}{\sigma_i^2}
+
d_{i,2}
\left(
\frac{\ell_i}{\sigma_i^2}
\right)^2
+
O\left(
\left(
\frac{\ell_i}{\sigma_i^2}
\right)^3
\right).
\end{equation}
Substituting \eqref{eq:lower_theta_ansatz} into
\eqref{eq:lower_normalized_inverse_relation} and using
\eqref{eq:lower_g_expansion} gives
\[\begin{aligned}
\frac{\ell_i}{\sigma_i^2}
=
\frac12
\left[
d_{i,1}\frac{\ell_i}{\sigma_i^2}
+
d_{i,2}
\left(
\frac{\ell_i}{\sigma_i^2}
\right)^2
\right]
-
\frac{1-2p_i}{6}
d_{i,1}^2
\left(
\frac{\ell_i}{\sigma_i^2}
\right)^2
+
O\left(
\left(
\frac{\ell_i}{\sigma_i^2}
\right)^3
\right).
\end{aligned}\]
Matching the first-order coefficient gives
\[
d_{i,1}=2.
\]
Matching the second-order coefficient gives
\[
\frac12d_{i,2}
-
\frac{1-2p_i}{6}d_{i,1}^2
=
0,
\]
and hence
\[
d_{i,2}
=
\frac{1-2p_i}{3}d_{i,1}^2
=
\frac{4(1-2p_i)}{3}.
\]
Consequently,
\[
\theta_i(\ell_i)
=
2\frac{\ell_i}{\sigma_i^2}
+
\frac{4(1-2p_i)}{3}
\left(
\frac{\ell_i}{\sigma_i^2}
\right)^2
+
O\left(
\left(
\frac{\ell_i}{\sigma_i^2}
\right)^3
\right).
\]
Using \eqref{eq:lower_variance_bounds}, this becomes
\[
\theta_i(\ell_i)
=
\frac{2}{\sigma_i^2}\ell_i
+
\frac{4(1-2p_i)}{3\sigma_i^4}\ell_i^2
+
O(\ell_i^3),
\]
which proves \eqref{eq:lower_theta_expansion}.
The derivative of \(g_i(\theta)/\sigma_i^2\) at the origin equals \(1/2\) for every arm, and the higher-order coefficients remain uniformly bounded under the standing interior condition~\eqref{eq:app_ordered_means}. Hence the inverse expansion and its remainder are uniform over all admissible arm means.

It remains to expand the tilted drift. By definition,
\begin{equation}
\label{eq:lower_v_identity}
v_i(\ell_i)
=
\partial_\theta\Lambda_i(\theta_i(\ell_i),\ell_i)
=
-\ell_i
+
\Lambda_{i,0}'(\theta_i(\ell_i)).
\end{equation}
Differentiating \eqref{eq:lower_Lambda0_expansion} gives
\[
% \label{eq:lower_Lambda0_prime_expansion}
\frac{\Lambda_{i,0}'(\theta)}{\sigma_i^2}
=
\theta
-
\frac{1-2p_i}{2}\theta^2
+
O(\theta^3).
\]
Dividing \eqref{eq:lower_v_identity} by \(\sigma_i^2\), we obtain
\[
% \label{eq:lower_v_normalized_identity}
\frac{v_i(\ell_i)}{\sigma_i^2}
=
-\frac{\ell_i}{\sigma_i^2}
+
\theta_i(\ell_i)
-
\frac{1-2p_i}{2}\theta_i^2(\ell_i)
+
O\bigl(\theta_i^3(\ell_i)\bigr).
\]
Using the expansion already obtained for \(\theta_i(\ell_i)\), we have
\[\begin{aligned}
\frac{v_i(\ell_i)}{\sigma_i^2}
&=
-\frac{\ell_i}{\sigma_i^2}
+
2\frac{\ell_i}{\sigma_i^2}
+
\frac{4(1-2p_i)}{3}
\left(
\frac{\ell_i}{\sigma_i^2}
\right)^2
-
\frac{1-2p_i}{2}
\left[
4
\left(
\frac{\ell_i}{\sigma_i^2}
\right)^2
\right]
+
O\left(
\left(
\frac{\ell_i}{\sigma_i^2}
\right)^3
\right)
\nonumber\\
&=
\frac{\ell_i}{\sigma_i^2}
-
\frac{2(1-2p_i)}{3}
\left(
\frac{\ell_i}{\sigma_i^2}
\right)^2
+
O\left(
\left(
\frac{\ell_i}{\sigma_i^2}
\right)^3
\right).
% \label{eq:lower_v_normalized_expansion}
\end{aligned}\]
Multiplying by \(\sigma_i^2\) and using
\eqref{eq:lower_variance_bounds} yields
\[
v_i(\ell_i)
=
\ell_i
-
\frac{2(1-2p_i)}{3\sigma_i^2}\ell_i^2
+
O(\ell_i^3),
\]
which proves \eqref{eq:lower_v_expansion}. The first-order statements in
\eqref{eq:lower_theta_v_first_order} follow immediately.

\end{proof}

\subsection{Upper-Endpoint Behavior as \texorpdfstring{$\ell_i\uparrow p_i$}{elli up to pi}}
\label{app:upper_endpoint_theta_v}

We now record the upper-endpoint behavior of the Lundberg root and the tilted drift. By Lemma~\ref{lem:lower_endpoint_root_equivalence}, the inverse relation
\(\ell_i=g_i(\theta_i(\ell_i))\) implies
\begin{equation}
\label{eq:upper_endpoint_equivalence_recall}
\ell_i\uparrow p_i
\qquad
\Longleftrightarrow
\qquad
\theta_i(\ell_i)\to\infty .
\end{equation}
Thus it remains to identify the precise divergence scale of \(\theta_i(\ell_i)\) and the corresponding scale of \(v_i(\ell_i)\).

\begin{lemma}[Upper-endpoint expansion of the Lundberg root]
\label{lem:upper_endpoint_theta}
As \(\ell_i\uparrow p_i\),
\begin{equation}
\label{eq:upper_theta_gap_product}
\theta_i(\ell_i)(p_i-\ell_i)
=
\log\frac{1}{q_i}
+
o(1).
\end{equation}
Equivalently,
\begin{equation}
\label{eq:upper_theta_expansion}
\theta_i(\ell_i)
=
\frac{\log(1/q_i)}{p_i-\ell_i}
+
o\left(
\frac{1}{p_i-\ell_i}
\right).
\end{equation}
In particular,
\begin{equation}
\label{eq:upper_theta_order}
\theta_i(\ell_i)
\asymp
\frac{1}{p_i-\ell_i}
\qquad
\text{as }\ell_i\uparrow p_i .
\end{equation}
\end{lemma}

\begin{proof}\leavevmode
By \eqref{eq:upper_endpoint_equivalence_recall}, \(\theta_i(\ell_i)\to\infty\) as \(\ell_i\uparrow p_i\). The Lundberg equation gives
\begin{equation}
\label{eq:upper_lundberg_balance}
q_i \exp\{\theta_i(\ell_i)(p_i-\ell_i)\}
+
p_i \exp\{-\theta_i(\ell_i)(q_i+\ell_i)\}
=
1 .
\end{equation}
Since \(q_i+\ell_i\to1\) and \(\theta_i(\ell_i)\to\infty\), we have
\[
% \label{eq:upper_negative_term_vanishes}
p_i \exp\{-\theta_i(\ell_i)(q_i+\ell_i)\}
\to0 .
\]
Therefore \eqref{eq:upper_lundberg_balance} implies
\[
% \label{eq:upper_positive_term_limit}
q_i \exp\{\theta_i(\ell_i)(p_i-\ell_i)\}
\to1 .
\]
Taking logarithms yields
\[
% \label{eq:upper_theta_gap_limit}
\theta_i(\ell_i)(p_i-\ell_i)
\to
\log\frac{1}{q_i}.
\]
This proves \eqref{eq:upper_theta_gap_product}. The expansion \eqref{eq:upper_theta_expansion} and the order estimate \eqref{eq:upper_theta_order} follow immediately.

\end{proof}

\begin{lemma}[Upper-endpoint expansion of the tilted drift]
\label{lem:upper_endpoint_v}
As \(\ell_i\uparrow p_i\),
\begin{equation}
\label{eq:upper_v_expansion}
v_i(\ell_i)
=
p_i-\ell_i
+
o(p_i-\ell_i).
\end{equation}
Equivalently,
\begin{equation}
\label{eq:upper_v_order}
v_i(\ell_i)
\sim
p_i-\ell_i .
\end{equation}
\end{lemma}

\begin{proof}\leavevmode
By definition,
\(
% \label{eq:upper_v_def_recall}
v_i(\ell_i)
=
\partial_\theta\Lambda_i(\theta_i(\ell_i),\ell_i)
\). 
Using the explicit form of \(\Lambda_i\), and using the Lundberg equation to set the denominator equal to one, we obtain
\begin{align}
v_i(\ell_i)
&=
q_i(p_i-\ell_i)\exp\{\theta_i(\ell_i)(p_i-\ell_i)\}
-
p_i(q_i+\ell_i)\exp\{-\theta_i(\ell_i)(q_i+\ell_i)\} .
\label{eq:upper_v_direct_formula}
\end{align}
The Lundberg equation also gives
\begin{equation}
\label{eq:upper_lundberg_rearranged}
q_i \exp\{\theta_i(\ell_i)(p_i-\ell_i)\}
=
1
-
p_i \exp\{-\theta_i(\ell_i)(q_i+\ell_i)\} .
\end{equation}
Substituting \eqref{eq:upper_lundberg_rearranged} into \eqref{eq:upper_v_direct_formula}, we get
\begin{align}
v_i(\ell_i)
&=
(p_i-\ell_i)
\left[
1
-
p_i \exp\{-\theta_i(\ell_i)(q_i+\ell_i)\}
\right]
-
p_i(q_i+\ell_i)\exp\{-\theta_i(\ell_i)(q_i+\ell_i)\}
\nonumber\\
&=
(p_i-\ell_i)
-
p_i
\left[
(p_i-\ell_i)+(q_i+\ell_i)
\right]
\exp\{-\theta_i(\ell_i)(q_i+\ell_i)\}
\nonumber\\
&=
(p_i-\ell_i)
-
p_i \exp\{-\theta_i(\ell_i)(q_i+\ell_i)\} .
\label{eq:upper_v_exact_identity}
\end{align}
By Lemma~\ref{lem:upper_endpoint_theta},
\[
% \label{eq:upper_theta_large_scale}
\theta_i(\ell_i)
=
\frac{\log(1/q_i)}{p_i-\ell_i}
+
o\left(
\frac{1}{p_i-\ell_i}
\right).
\]
Since \(\log(1/q_i)>0\) and $q_i + \ell_i \to 1$, for all \(\ell_i\) sufficiently close to \(p_i\),
\[
% \label{eq:upper_negative_exponent_lower_bound}
\theta_i(\ell_i)(q_i+\ell_i)
\ge
\frac{1}{2}
\frac{\log(1/q_i)}{p_i-\ell_i}.
\]
Therefore,
\[
% \label{eq:upper_exponential_ratio_bound}
0
\le
\frac{
\exp\{-\theta_i(\ell_i)(q_i+\ell_i)\}
}{
p_i-\ell_i
}
\le
\frac{
\exp\left\{
-\dfrac{\log(1/q_i)}{2(p_i-\ell_i)}
\right\}
}{
p_i-\ell_i
}.
\]
The right-hand side converges to zero because the exponential decay in \((p_i-\ell_i)^{-1}\) dominates the logarithmic divergence of \((p_i-\ell_i)^{-1}\). Equivalently,
\[
\log
\left[
\frac{
\exp\left\{
-\dfrac{\log(1/q_i)}{2(p_i-\ell_i)}
\right\}
}{
p_i-\ell_i
}
\right]
=
-\frac{\log(1/q_i)}{2(p_i-\ell_i)}
-\log(p_i-\ell_i)
\longrightarrow
-\infty .
% \label{eq:upper_exponential_ratio_log}
\]
Hence
\begin{equation}
\label{eq:upper_exponential_negligible}
\exp\{-\theta_i(\ell_i)(q_i+\ell_i)\}
=
o(p_i-\ell_i).
\end{equation}
Substituting \eqref{eq:upper_exponential_negligible} into \eqref{eq:upper_v_exact_identity} gives
\[
% \label{eq:upper_v_final}
v_i(\ell_i)
=
p_i-\ell_i
+
o(p_i-\ell_i).
\]
This proves \eqref{eq:upper_v_expansion} and \eqref{eq:upper_v_order}.

\end{proof}

Combining Lemmas~\ref{lem:lower_endpoint_theta_v_expansion} and \ref{lem:upper_endpoint_v}, together with the strict positivity and continuity of \(v_i\) on compact subintervals of \((0,p_i)\), yields the global endpoint order
\begin{equation}
\label{eq:global_tilted_drift_order}
v_i(\ell_i)
\asymp
\min\{\ell_i,p_i-\ell_i\},
\qquad
\ell_i\in(0,p_i).
\end{equation}

\section{Two-Sided Lundberg Bounds with Discrete Overshoot}
\label{app:lundberg_overshoot_bounds}

This section turns the root-based crossing exponent from
Appendix~\ref{app:cgf_drawdown_geometry} into the tail approximation used in the
suboptimal-convergence calculation. For each arm \(i\), the object is the
drawdown tail
\[
\overline F_{i,\circ}(\ell_i)
=
\mathbb P(L_i\circ \ell_i),
\qquad
\circ\in\{>,\ge\}.
\]
Appendix~\ref{app:cgf_drawdown_geometry} identifies the exponent
\(\theta_i(\ell_i)\). We now use the martingale associated with this exponent to
control the crossing probability itself.

At the first crossing time, the walk has reached the boundary
\(b_i(\ell_i)=\Delta_i+\beta\ell_i\) plus an overshoot
\(r_i(\ell_i)\). The overshoot is bounded because a crossing can occur only
through a single positive jump of size \(p_i-\ell_i\). Optional stopping
therefore gives the exact representation
\[
\overline F_{i,\circ}(\ell_i)
=
\frac{
\exp\{-\theta_i(\ell_i)b_i(\ell_i)\}
}{
\mathbb E\!\left[
\exp\{\theta_i(\ell_i)r_i(\ell_i)\}
\,\middle|\,
\tau_i^\circ(\ell_i)<\infty
\right]
}.
\]
The denominator is bounded between one and
\(\exp\{\theta_i(\ell_i)(p_i-\ell_i)\}\), which yields a two-sided
nonasymptotic envelope for both strict and weak drawdown tails.

We then specialize this envelope to the lower-endpoint regime that is relevant
under large regularization. Using
\[
\theta_i(\ell_i)
=
\frac{2}{p_i(1-p_i)}\ell_i+O(\ell_i^2),
\qquad
b_i(\ell_i)=\Delta_i+\beta\ell_i,
\]
we obtain, uniformly for \(0<\ell_i\le\eta_\Delta\),
\[
\overline F_{i,\circ}(\ell_i)
=
\exp\{-\lambda_i\ell_i\}
\left[1+o(1)\right],
\qquad
\lambda_i=\frac{2\Delta_i}{p_i(1-p_i)}.
\]
Outside this local window, the same root-based envelope gives a
super-polynomial tail bound. Thus the exact law of \(L_i\) can be localized near
zero and replaced, inside suitable Stieltjes integrals, by the exponential
reference measure with density \(\lambda_i \exp\{-\lambda_i x\}\). This replacement
is the input used in the subsequent product-tail and suboptimal-convergence
calculations.

\subsection{Root-Based One-Arm Drawdown Envelope}
\label{app:root_based_one_arm_drawdown_envelope}

For \(\circ\in\{>,\ge\}\), define
\[
% \label{eq:rb_tau_convention}
\tau_i^{\circ}(\ell_i)
:=
\begin{cases}
\tau_i^{>}(\ell_i), & \circ=>,\\
\tau_i^{\ge}(\ell_i), & \circ=\ge .
\end{cases}
\]
The corresponding drawdown tails are
\[
% \label{eq:rb_tail_def}
\overline F_{i,>}(\ell_i)
:=
\mathbb P(L_i>\ell_i),
\qquad
\overline F_{i,\ge}(\ell_i)
:=
\mathbb P(L_i\ge\ell_i).
\]
By Lemma~\ref{lem:bc_one_arm_crossing},
\begin{equation}
\label{eq:rb_tail_tau_equivalence}
\overline F_{i,\circ}(\ell_i)
=
\mathbb P(\tau_i^{\circ}(\ell_i)<\infty),
\qquad
\circ\in\{>,\ge\}.
\end{equation}

For the selected convention \(\circ\in\{>,\ge\}\), we use the terminal residual \(r_i(\ell_i)\) already defined in \eqref{eq:bc_overshoot_def}, with \(\tau_i(\ell_i)=\tau_i^{\circ}(\ell_i)\). Thus, on \(\{\tau_i^{\circ}(\ell_i)<\infty\}\), the unified terminal representation \eqref{eq:bc_terminal_state_unified} gives
\begin{equation}
\label{eq:rb_terminal_state_recall}
U_{i,\tau_i^{\circ}(\ell_i)}(\ell_i)
=
b_i(\ell_i)+r_i(\ell_i),
\qquad
0\le r_i(\ell_i)\le p_i-\ell_i .
\end{equation}
For strict crossing, \(0<r_i(\ell_i)\le p_i-\ell_i\), whereas for weak crossing, \(0\le r_i(\ell_i)<p_i-\ell_i\). The coarser bound in \eqref{eq:rb_terminal_state_recall} is sufficient for the Lundberg bounds below.

\begin{lemma}[Root-based first-passage identity]
\label{lem:root_based_first_passage_identity}
For every arm \(i\), every \(\ell_i\in(0,p_i)\), and every
\(\circ\in\{>,\ge\}\),
\begin{equation}
\label{eq:rb_first_passage_identity}
\overline F_{i,\circ}(\ell_i)
=
\frac{
\exp\{-\theta_i(\ell_i)b_i(\ell_i)\}
}{
\mathbb E\left[
\exp\{\theta_i(\ell_i)r_i(\ell_i)\}
\,\middle|\,
\tau_i^{\circ}(\ell_i)<\infty
\right]
}.
\end{equation}
Equivalently,
\begin{equation}
\label{eq:rb_exact_expectation_identity}
1
=
\mathbb E\left[
\exp\left\{
\theta_i(\ell_i)
\bigl(b_i(\ell_i)+r_i(\ell_i)\bigr)
\right\}
\mathbf 1\{\tau_i^{\circ}(\ell_i)<\infty\}
\right].
\end{equation}
\end{lemma}

\begin{proof}\leavevmode
For every arm \(i\), every \(\ell_i\in(0,p_i)\), and \(\circ\in\{>,\ge\}\), to simplify notations inside the proof, write
\(
% \label{eq:rb_short_notation}
\theta_i:=\theta_i(\ell_i),\,
b_i:=b_i(\ell_i),\,
\tau_i:=\tau_i^{\circ}(\ell_i)
\).
Let
\[
% \label{eq:rb_filtration_def}
\mathcal F_{i,n}
:=
\sigma(X_{i,1},\ldots,X_{i,n}),
\qquad
n\ge0 .
\]
Define
\[
% \label{eq:rb_martingale_def}
\mathcal M_{i,n}
:=
\exp\left\{
\theta_i U_{i,n}(\ell_i)
-
n\Lambda_i(\theta_i,\ell_i)
\right\},
\qquad
\mathcal M_{i,0}=1 .
\]
Since \(\theta_i\) is the Lundberg root,
\(
% \label{eq:rb_root_zero}
\Lambda_i(\theta_i,\ell_i)=0
\). 
Therefore, 
\(
% \label{eq:rb_martingale_root_form}
\mathcal M_{i,n}
=
\exp\left\{
\theta_i U_{i,n}(\ell_i)
\right\}
\). 
Because
\(
% \label{eq:rb_walk_increment}
U_{i,n}(\ell_i)
=
U_{i,n-1}(\ell_i)+Z_{i,n}(\ell_i)
\), 
and \(Z_{i,n}(\ell_i)\) is independent of \(\mathcal F_{i,n-1}\), we have
\[\begin{aligned}
\mathbb E[\mathcal M_{i,n}\mid\mathcal F_{i,n-1}]
&=
\exp\left\{
\theta_i U_{i,n-1}(\ell_i)
-
(n-1)\Lambda_i(\theta_i,\ell_i)
\right\}
\times
\mathbb E\left[
\exp\left\{
\theta_i Z_{i,n}(\ell_i)
-
\Lambda_i(\theta_i,\ell_i)
\right\}
\right]
\nonumber\\
&=
\mathcal M_{i,n-1}.
% \label{eq:rb_martingale_property}
\end{aligned}\]
Thus \(\{\mathcal M_{i,n}\}_{n\ge0}\) is a positive mean-one martingale.

For a deterministic horizon \(N\), the stopped time \(\tau_i\wedge N\) is bounded. Optional stopping gives
\(
% \label{eq:rb_optional_stopping}
1
=
\mathbb E[\mathcal M_{i,\tau_i\wedge N}]
\). 
Splitting according to whether crossing has occurred by time \(N\),
\begin{equation}
1
=
\mathbb E\left[
\mathcal M_{i,\tau_i}\mathbf 1\{\tau_i\le N\}
\right]
+
\mathbb E\left[
\mathcal M_{i,N}\mathbf 1\{\tau_i>N\}
\right].
\label{eq:rb_stopped_split}
\end{equation}

Letting \(N\to\infty\), the first term converges by monotone convergence:
\begin{equation}
\label{eq:rb_first_term_limit}
\mathbb E\left[
\mathcal M_{i,\tau_i}\mathbf 1\{\tau_i\le N\}
\right]
\to
\mathbb E\left[
\mathcal M_{i,\tau_i}\mathbf 1\{\tau_i<\infty\}
\right].
\end{equation}
For the second term, on \(\{\tau_i>N\}\), the boundary has not yet been crossed, and hence
\(
% \label{eq:rb_no_crossing_bound}
U_{i,N}(\ell_i)\le b_i
\). 
Therefore, 
\[
% \label{eq:rb_martingale_domination}
0
\le
\mathcal M_{i,N}\mathbf 1\{\tau_i>N\}
\le
\exp\{\theta_i b_i\}.
\]
If \(\tau_i<\infty\), then \(\mathbf 1\{\tau_i>N\}=0\) for all sufficiently large \(N\). If \(\tau_i=\infty\), then by the strong law of large numbers and the negative drift of \(U_{i,n}(\ell_i)\),
\(
% \label{eq:rb_negative_drift_slln}
{U_{i,N}(\ell_i)}/{N}
\to
-\ell_i
\) a.s. 
Thus \(U_{i,N}(\ell_i)\to-\infty\), and consequently
\(
% \label{eq:rb_second_term_pointwise_zero}
\mathcal M_{i,N}\mathbf 1\{\tau_i>N\}
\to0
\) a.s. 
Dominated convergence gives
\begin{equation}
\label{eq:rb_second_term_limit_zero}
\mathbb E\left[
\mathcal M_{i,N}\mathbf 1\{\tau_i>N\}
\right]
\to0 .
\end{equation}
Combining \eqref{eq:rb_stopped_split}, \eqref{eq:rb_first_term_limit}, and \eqref{eq:rb_second_term_limit_zero}, we obtain
\begin{equation}
\label{eq:rb_limit_identity}
1
=
\mathbb E\left[
\mathcal M_{i,\tau_i}\mathbf 1\{\tau_i<\infty\}
\right].
\end{equation}

On \(\{\tau_i<\infty\}\), the terminal representation \eqref{eq:rb_terminal_state_recall} gives 
\(
% \label{eq:rb_martingale_at_crossing}
\mathcal M_{i,\tau_i}
=
\exp\left\{
\theta_i
\bigl(b_i+r_i(\ell_i)\bigr)
\right\}
\). 
Substituting this into \eqref{eq:rb_limit_identity} gives
\[
% \label{eq:rb_expectation_identity_expanded}
1
=
\mathbb E\left[
\exp\left\{
\theta_i
\bigl(b_i+r_i(\ell_i)\bigr)
\right\}
\mathbf 1\{\tau_i<\infty\}
\right],
\]
which is \eqref{eq:rb_exact_expectation_identity}. Factoring out \(\exp\{\theta_i b_i\}\), we get
\[
\begin{aligned}
1
&=
\exp\{\theta_i b_i\}
\mathbb E\left[
\exp\{\theta_i r_i(\ell_i)\}
\mathbf 1\{\tau_i<\infty\}
\right]
\nonumber\\
&=
\exp\{\theta_i b_i\}
\mathbb P(\tau_i<\infty)
\mathbb E\left[
\exp\{\theta_i r_i(\ell_i)\}
\,\middle|\,
\tau_i<\infty
\right].
% \label{eq:rb_factorized_identity}
\end{aligned}
\]
Using \eqref{eq:rb_tail_tau_equivalence} and rearranging proves \eqref{eq:rb_first_passage_identity}.

\end{proof}

\begin{corollary}[Root-based drawdown tail envelope]
\label{cor:root_based_drawdown_tail_envelope}
For every arm \(i\), every \(\ell_i\in(0,p_i)\), and every \(\circ\in\{>,\ge\}\),
\begin{equation}
\label{eq:rb_two_sided_tail_bound}
\exp\left\{
-\theta_i(\ell_i)\bigl[b_i(\ell_i)+p_i-\ell_i\bigr]
\right\}
\le
\overline F_{i,\circ}(\ell_i)
\le
\exp\left\{
-\theta_i(\ell_i)b_i(\ell_i)
\right\}.
\end{equation}
In particular, the coarser but sometimes more convenient bound
\begin{equation}
\label{eq:rb_two_sided_tail_bound_coarse}
\exp\left\{
-\theta_i(\ell_i)\bigl[b_i(\ell_i)+p_i\bigr]
\right\}
\le
\overline F_{i,\circ}(\ell_i)
\le
\exp\left\{
-\theta_i(\ell_i)b_i(\ell_i)
\right\}
\end{equation}
also holds.
\end{corollary}

\begin{proof}\leavevmode
By \eqref{eq:rb_terminal_state_recall} and $\theta_i(\ell_i) > 0$,
\[
% \label{eq:rb_exp_residual_bounds}
1
\le
\exp\{\theta_i(\ell_i)r_i(\ell_i)\}
\le
\exp\{\theta_i(\ell_i)(p_i-\ell_i)\}
\]
on \(\{\tau_i^{\circ}(\ell_i)<\infty\}\). Taking conditional expectations gives
\begin{equation}
\label{eq:rb_denominator_bounds}
1
\le
\mathbb E\left[
\exp\{\theta_i(\ell_i)r_i(\ell_i)\}
\,\middle|\,
\tau_i^{\circ}(\ell_i)<\infty
\right]
\le
\exp\{\theta_i(\ell_i)(p_i-\ell_i)\}.
\end{equation}
Substituting \eqref{eq:rb_denominator_bounds} into \eqref{eq:rb_first_passage_identity} proves \eqref{eq:rb_two_sided_tail_bound}. Since \(p_i-\ell_i\le p_i\), \eqref{eq:rb_two_sided_tail_bound_coarse} follows immediately.

\end{proof}

\subsection{Asymptotic Closed-Form Envelope Near the Lower Endpoint}
\label{app:asymptotic_closed_form_envelope}
We now specialize the root-based envelope to the asymptotic regime used
throughout the subsequent closed-form analysis.

\begin{assumption}[Asymptotic regime]
\label{ass:app_asymptotic_regime}
As \(\Delta_1\to\infty\), the regularization scale satisfies
\begin{equation}
\label{eq:app_asymptotic_regularization_scale}
\Delta_i
=
\Theta(\Delta_1),
\qquad
\beta
=
O(\Delta_1),
\qquad
i=1,\ldots,K.
\end{equation}
Moreover, for a common \(\kappa\in(1/2,1)\), every nonzero pairwise gap
satisfies
\begin{equation}
\label{eq:app_asymptotic_gap_scale}
p_i-p_j
=
\Theta(\Delta_1^{-\kappa}),
\qquad
1\le i<j\le K
\text{ such that }p_i>p_j.
\end{equation}
\end{assumption}

All asymptotic statements below are taken with respect to
\(\Delta_1\to\infty\), while \(K\) and the finite-horizon parameter \(T\)
remain fixed. Condition
\eqref{eq:app_asymptotic_regularization_scale} places the regularization
margins of all arms on the same asymptotic scale. Since
\(
\alpha
=
\Delta_1+p_1\beta
\), 
it also follows that
\(
\alpha
=
\Theta(\Delta_1)
\). 
Furthermore, the standing interior condition~\eqref{eq:app_ordered_means} implies
\[
\lambda_i
:=
\frac{2\Delta_i}{\sigma_i^2}
=
\frac{2\Delta_i}{p_i(1-p_i)}
=
\Theta(\Delta_1),
\qquad
i=1,\ldots,K.
\]

For the subsequent lower-endpoint analysis, define
\begin{equation}
\label{eq:app_asymptotic_eta_def}
\eta_\Delta
:=
\Delta_1^{-\kappa}.
\end{equation}
Because \(\kappa\in(1/2,1)\),
\[
\eta_\Delta\to0,
\qquad
\Delta_1\eta_\Delta\to\infty,
\qquad
\Delta_1\eta_\Delta^2\to0.
\]
Accordingly, every nonzero pairwise gap satisfies
\[
p_i-p_j
=
\Theta(\eta_\Delta),
\qquad
\Delta_1(p_i-p_j)\to\infty,
\qquad
\Delta_1(p_i-p_j)^2\to0.
\]

Since \(p_i\ge\epsilon_{\mathrm p}\) for every arm and
\(\eta_\Delta\to0\), we have \(\eta_\Delta<p_i\) for all
\(i=1,\ldots,K\) and all sufficiently large \(\Delta_1\). Define the
lower-endpoint bulk window and its complementary tail window by
\[
\mathcal B_i
:=
(0,\eta_\Delta],
\qquad
\mathcal T_i
:=
(\eta_\Delta,p_i),
\qquad
i=1,\ldots,K.
\]

\begin{lemma}[Closed-form lower-endpoint envelope]
\label{lem:asymptotic_closed_form_tail}
Under Assumption~\ref{ass:app_asymptotic_regime}, for every arm \(i=1,\ldots,K\) and every \(\circ\in\{>,\ge\}\), the lower-endpoint tail satisfies
\begin{equation}
\label{eq:acf_local_closed_form_tail_quantified}
\overline F_{i,\circ}(\ell_i)
=
\exp\{-\lambda_i\ell_i\}
\left[
1+O(\Delta_1^{1-2\kappa})
\right],
\qquad
\ell_i\in\mathcal B_i,
\end{equation}
uniformly over \(\ell_i\in\mathcal B_i\). Moreover, there exists a constant \(c_i>0\), depending only on \(p_i\), such that for all sufficiently large \(\Delta_1\),
\begin{equation}
\label{eq:acf_tail_suppression_quantified}
\sup_{\ell_i\in\mathcal T_i}
\overline F_{i,\circ}(\ell_i)
\le
\exp\{-c_i\Delta_1\eta_\Delta\}
=
O(\Delta_1^{-\infty}).
\end{equation}
\end{lemma}

\begin{proof}\leavevmode
We first prove the quantified bulk approximation. By Lemma~\ref{lem:lower_endpoint_theta_v_expansion},
\begin{equation}
\label{eq:acf_theta_local_quantified}
\theta_i(\ell_i)
=
\frac{2}{\sigma_i^2}\ell_i
+
O(\ell_i^2)
\qquad
\text{as }\ell_i\downarrow0 .
\end{equation}
Since \(\eta_\Delta=\Delta_1^{-\kappa}\to0\), this expansion holds uniformly for \(\ell_i\in\mathcal B_i=(0,\eta_\Delta]\). Using
\(
% \label{eq:acf_boundary_recall_quantified}
b_i(\ell_i)=\Delta_i+\beta\ell_i
\), 
we obtain, uniformly on \(\mathcal B_i\),
\begin{align}
\theta_i(\ell_i)b_i(\ell_i)
&=
\left[
\frac{2}{\sigma_i^2}\ell_i
+
O(\ell_i^2)
\right]
\left[
\Delta_i+\beta\ell_i
\right]
\nonumber\\
&=
\frac{2\Delta_i}{\sigma_i^2}\ell_i
+
O(\beta\ell_i^2)
+
O(\Delta_i\ell_i^2)
+
O(\beta\ell_i^3)
\nonumber\\
&=
\lambda_i\ell_i
+
O(\Delta_1\ell_i^2),
\label{eq:acf_theta_b_bulk_quantified}
\end{align}
where we used \(\lambda_i=2\Delta_i/\sigma_i^2\), \(\Delta_i=\Theta(\Delta_1)\), \(\beta=O(\Delta_1)\), and \(\ell_i\le\eta_\Delta\to0\).

By Lemma~\ref{lem:root_based_first_passage_identity},
\[
\log \overline F_{i,\circ}(\ell_i)
=
-\theta_i(\ell_i)b_i(\ell_i)
-
\log
\mathbb E\left[
\exp\{\theta_i(\ell_i)r_i(\ell_i)\}
\,\middle|\,
\tau_i^{\circ}(\ell_i)<\infty
\right].
% \label{eq:acf_exact_log_tail_quantified}
\]
The residual bound in \eqref{eq:rb_terminal_state_recall} gives \(0\le r_i(\ell_i) \le p_i-\ell_i \) on \(\{\tau_i^{\circ}(\ell_i)<\infty\}\). Therefore, since \(\theta_i(\ell_i)>0\),
\begin{align}
0
&\le
\log
\mathbb E\left[
\exp\{\theta_i(\ell_i)r_i(\ell_i)\}
\,\middle|\,
\tau_i^{\circ}(\ell_i)<\infty
\right]
\nonumber\\
&\le
\theta_i(\ell_i)(p_i-\ell_i)
\nonumber\\
&=
O(\ell_i),
\label{eq:acf_residual_log_bound_quantified}
\end{align}
uniformly on \(\mathcal B_i\), where the last step uses \eqref{eq:acf_theta_local_quantified}. Combining \eqref{eq:acf_theta_b_bulk_quantified} and \eqref{eq:acf_residual_log_bound_quantified}, we obtain
\begin{equation}
\label{eq:acf_log_tail_quantified}
\log \overline F_{i,\circ}(\ell_i)
=
-\lambda_i\ell_i
+
O(\Delta_1\ell_i^2+\ell_i),
\qquad
\ell_i\in\mathcal B_i .
\end{equation}
Consequently, for some constant \(C_i<\infty\) independent of \(\Delta_1\),
\[
\sup_{\ell_i\in\mathcal B_i}
\left|
\log \overline F_{i,\circ}(\ell_i)+\lambda_i\ell_i
\right|
\le
C_i(\Delta_1\eta_\Delta^2+\eta_\Delta)
=
C_i\left(\Delta_1^{1-2\kappa}+\Delta_1^{-\kappa}\right)
=
O(\Delta_1^{1-2\kappa}),
% \label{eq:acf_uniform_log_error_bound_quantified}
\]
where the last equality follows from \(\kappa<1\). Since \(\kappa>1/2\),
\(
% \label{eq:acf_uniform_log_error_zero_quantified}
\Delta_1^{1-2\kappa}\to0
\). 
Exponentiating \eqref{eq:acf_log_tail_quantified} and using \(\exp\{O(x)\}=1+O(x)\) as \(x\to0\), uniformly in \(\ell_i\in\mathcal B_i\), yields
\[
% \label{eq:acf_local_closed_form_tail_quantified_proof}
\overline F_{i,\circ}(\ell_i)
=
\exp\{-\lambda_i\ell_i\}
\left[
1+O(\Delta_1^{1-2\kappa})
\right],
\qquad
\ell_i\in\mathcal B_i .
\]
This proves the local approximation \eqref{eq:acf_local_closed_form_tail_quantified}.

It remains to prove the tail suppression on \(\mathcal T_i=(\eta_\Delta,p_i)\). By Corollary~\ref{cor:root_based_drawdown_tail_envelope},
\begin{equation}
\label{eq:acf_tail_root_upper_quantified}
\overline F_{i,\circ}(\ell_i)
\le
\exp\{-\theta_i(\ell_i)b_i(\ell_i)\},
\qquad
\ell_i\in(0,p_i).
\end{equation}
By Lemma~\ref{lem:lower_endpoint_root_equivalence}, \(\ell_i\mapsto\theta_i(\ell_i)\) is increasing. Hence, for every \(\ell_i\in\mathcal T_i\),
\begin{equation}
\label{eq:acf_tail_theta_lower_quantified}
\theta_i(\ell_i)
\ge
\theta_i(\eta_\Delta).
\end{equation}
By Lemma~\ref{lem:lower_endpoint_theta_v_expansion},
\[
% \label{eq:acf_theta_eta_expansion_quantified}
\theta_i(\eta_\Delta)
=
\frac{2}{\sigma_i^2}\eta_\Delta
+
O(\eta_\Delta^2).
\]
Since \(\eta_\Delta\to0\), after decreasing the constant if necessary, there exists \(c_i>0\), depending only on \(p_i\), such that
\begin{equation}
\label{eq:acf_theta_eta_lower_quantified}
\theta_i(\eta_\Delta)
\ge
c_i\eta_\Delta
\end{equation}
for all sufficiently large \(\Delta_1\). Moreover,
\begin{equation}
\label{eq:acf_boundary_tail_lower_quantified}
b_i(\ell_i)
=
\Delta_i+\beta\ell_i
\ge
\Delta_i
\ge
\Delta_1 ,
\qquad
\ell_i\in\mathcal T_i,
\end{equation}
where the last inequality follows from \(\Delta_i=\alpha-p_i\beta\ge\alpha-p_1\beta=\Delta_1\) under the ordering \(p_i\le p_1\) and \(\beta\ge0\). Combining \eqref{eq:acf_tail_theta_lower_quantified}, \eqref{eq:acf_theta_eta_lower_quantified}, and \eqref{eq:acf_boundary_tail_lower_quantified}, we get
\[
% \label{eq:acf_tail_exponent_lower_quantified}
\theta_i(\ell_i)b_i(\ell_i)
\ge
c_i\Delta_1\eta_\Delta,
\qquad
\ell_i\in\mathcal T_i .
\]
Substituting this lower bound into \eqref{eq:acf_tail_root_upper_quantified} yields
\[
% \label{eq:acf_tail_suppression_intermediate_quantified}
\sup_{\ell_i\in\mathcal T_i}
\overline F_{i,\circ}(\ell_i)
\le
\exp\{-c_i\Delta_1\eta_\Delta\}.
\]
Since \(\eta_\Delta=\Delta_1^{-\kappa}\), we have
\(
% \label{eq:acf_tail_suppression_exponent_quantified}
\Delta_1\eta_\Delta
=
\Delta_1^{1-\kappa}
\). 
Because \(\kappa<1\), \(\Delta_1^{1-\kappa}\to\infty\). Moreover, for every \(M>0\),
\[
\Delta_1^M
\exp\{-c_i\Delta_1\eta_\Delta\}
=
\exp\left\{
M\log\Delta_1
-
c_i\Delta_1^{1-\kappa}
\right\}
\to0 .
% \label{eq:acf_tail_superpoly_check_quantified}
\]
Thus
\(
% \label{eq:acf_tail_suppression_superpoly_quantified}
\exp\{-c_i\Delta_1\eta_\Delta\}
=
O(\Delta_1^{-\infty})
\), 
and therefore
\[
% \label{eq:acf_tail_suppression_quantified_proof}
\sup_{\ell_i\in\mathcal T_i}
\overline F_{i,\circ}(\ell_i)
\le
\exp\{-c_i\Delta_1\eta_\Delta\}
=
O(\Delta_1^{-\infty}).
\]
This proves \eqref{eq:acf_tail_suppression_quantified}.

\end{proof}

The lemma shows that the lower-endpoint bulk provides the effective closed-form survival envelope
\[
% \label{eq:acf_effective_survival_envelope}
\overline F_{i,\circ}(\ell_i)
=
\exp\{-\lambda_i\ell_i\}
\left[
1+O(\Delta_1^{1-2\kappa})
\right],
\qquad
\ell_i\in(0,\eta_\Delta],
\]
whereas the complement \((\eta_\Delta,p_i)\) is exponentially suppressed by
\eqref{eq:acf_tail_suppression_quantified}. This is a survival-envelope statement, not an
absolute-continuity statement for the exact law of \(L_i\). The exact distribution \(F_i\)
need not admit an ordinary density. We therefore use the exponential form only through the
following local Stieltjes replacement. In the statement below, Stieltjes integrals written as
\(\int_0^{\eta_\Delta}\) are over the interval \([0,\eta_\Delta]\), and \(\int_0^{p_i}\) is over
\([0,p_i)\), which is harmless because \(L_i\in(0,p_i)\) almost surely.

\begin{corollary}[Local Stieltjes replacement through a reference integrand]
\label{cor:acf_local_stieltjes_replacement}
Let \(F_i\) denote the distribution function of \(L_i\). Let
\(G_\Delta^{\mathrm{ex}}:[0,p_i)\to[0,\infty)\) be a measurable exact integrand.
Suppose that there exists a nonnegative reference integrand
\(G_\Delta^{\mathrm{ref}}:[0,\eta_\Delta]\to[0,\infty)\), absolutely continuous on
\([0,\eta_\Delta]\), such that
\begin{equation}
\label{eq:acf_reference_integral_def}
J_\Delta
:=
\int_0^{\eta_\Delta}
G_\Delta^{\mathrm{ref}}(x)\lambda_i \exp\{-\lambda_i x\}\,dx
>0 .
\end{equation}
Assume the local exact-to-reference approximation
\begin{equation}
\label{eq:acf_exact_reference_local_error_condition}
\int_0^{\eta_\Delta}
\left|
G_\Delta^{\mathrm{ex}}(x)
-
G_\Delta^{\mathrm{ref}}(x)
\right|
\,dF_i(x)
=
o(J_\Delta),
\end{equation}
the reference weighted-variation condition
\begin{equation}
\label{eq:acf_reference_variation_condition}
G_\Delta^{\mathrm{ref}}(\eta_\Delta)\exp\{-\lambda_i\eta_\Delta\}
+
\int_0^{\eta_\Delta}
\left|
\frac{d}{dx}G_\Delta^{\mathrm{ref}}(x)
\right|
\exp\{-\lambda_i x\}\,dx
\le
C_G J_\Delta ,
\end{equation}
where \(C_G<\infty\) is independent of \(\Delta_1\), and the exact tail-localization
condition
\begin{equation}
\label{eq:acf_exact_tail_condition}
\mathbb E\left[
G_\Delta^{\mathrm{ex}}(L_i)\mathbf 1\{L_i>\eta_\Delta\}
\right]
=
o(J_\Delta).
\end{equation}
Then
\begin{equation}
\label{eq:acf_reference_stieltjes_replacement}
\int_0^{p_i}
G_\Delta^{\mathrm{ex}}(x)\,dF_i(x)
=
\int_0^{\eta_\Delta}
G_\Delta^{\mathrm{ref}}(x)\lambda_i \exp\{-\lambda_i x\}\,dx
\left[
1+o(1)
\right].
\end{equation}
\end{corollary}

This corollary should be read as a replacement principle for integrals, not as an
absolute-continuity statement for \(F_i\). The exact integrand
\(G_\Delta^{\mathrm{ex}}\) may contain exact survival factors and need not be
smooth. Smoothness is required only for the reference integrand
\(G_\Delta^{\mathrm{ref}}\). In later applications,
\eqref{eq:acf_exact_reference_local_error_condition} is usually verified by a
uniform local approximation
\[
G_\Delta^{\mathrm{ex}}(x)
=
G_\Delta^{\mathrm{ref}}(x)(1+o(1)),
\qquad
0\le x\le\eta_\Delta,
\]
together with the estimate
\[
\int_0^{\eta_\Delta}
G_\Delta^{\mathrm{ref}}(x)\,dF_i(x)
=
J_\Delta(1+o(1)),
\]
which is proved inside the argument below.

\begin{proof}\leavevmode
We write
\[
% \label{eq:acf_distribution_and_strict_survival_recalled_for_reference}
F_i(x):=\mathbb P(L_i\le x),
\qquad
\overline F_{i,>}(x):=\mathbb P(L_i>x),
\qquad
0\le x<p_i .
\]
Since \(0<L_i<p_i\) almost surely, \(F_i(0)=0\), and
\[
% \label{eq:acf_F_strict_survival_relation_reference}
F_i(x)
=
1-\overline F_{i,>}(x),
\qquad
0\le x<p_i .
\]
The lower-endpoint tail envelope gives, uniformly for \(0<x\le\eta_\Delta\),
\[
% \label{eq:acf_survival_error_reference_version}
\overline F_{i,>}(x)
=
\exp\{-\lambda_i x\}
\left[
1+\varepsilon_\Delta(x)
\right],
\qquad
\sup_{0<x\le\eta_\Delta}
|\varepsilon_\Delta(x)|
=
O(\Delta_1^{1-2\kappa})
=
o(1).
\]
Consequently, after increasing the constant if necessary,
\begin{equation}
\label{eq:acf_survival_difference_reference_version}
\left|
\overline F_{i,>}(x)
-
\exp\{-\lambda_i x\}
\right|
\le
a_\Delta \exp\{-\lambda_i x\},
\qquad
0<x\le\eta_\Delta,
\end{equation}
where
\(
% \label{eq:acf_adelta_def}
a_\Delta:=C\Delta_1^{1-2\kappa}=o(1)
\). 

We first reduce the local exact integral to the reference integral. Define
\[
I_\Delta^{\mathrm{ex,loc}}
:=
\int_0^{\eta_\Delta}
G_\Delta^{\mathrm{ex}}(x)\,dF_i(x),
% \label{eq:acf_local_exact_integral_def}
\qquad
I_\Delta^{\mathrm{ref,loc}}
:=
\int_0^{\eta_\Delta}
G_\Delta^{\mathrm{ref}}(x)\,dF_i(x).
% \label{eq:acf_local_reference_integral_def}
\]
By the local exact-to-reference approximation
\eqref{eq:acf_exact_reference_local_error_condition},
\begin{equation}
\left|
I_\Delta^{\mathrm{ex,loc}}
-
I_\Delta^{\mathrm{ref,loc}}
\right|
\le
\int_0^{\eta_\Delta}
\left|
G_\Delta^{\mathrm{ex}}(x)
-
G_\Delta^{\mathrm{ref}}(x)
\right|
\,dF_i(x)
=
o(J_\Delta).
\label{eq:acf_exact_to_ref_local_integral}
\end{equation}
Thus it remains to replace the Stieltjes measure \(dF_i\) by the local exponential
reference measure only for the reference integrand \(G_\Delta^{\mathrm{ref}}\).

We now evaluate \(I_\Delta^{\mathrm{ref,loc}}\). The almost-sure support condition \(L_i>0\) gives \(F_i(0)=0\). Moreover, \(G_\Delta^{\mathrm{ref}}\) is absolutely continuous on \([0,\eta_\Delta]\), so Stieltjes integration by parts gives
\[
% \label{eq:acf_ref_stieltjes_by_parts_first}
\int_0^{\eta_\Delta}
G_\Delta^{\mathrm{ref}}(x)\,dF_i(x)
=
G_\Delta^{\mathrm{ref}}(\eta_\Delta)F_i(\eta_\Delta)
-
\int_0^{\eta_\Delta}
F_i(x)
\frac{d}{dx}G_\Delta^{\mathrm{ref}}(x)\,dx .
\]
Using
\(
% \label{eq:acf_F_strict_survival_relation_reference}
F_i(x)
=
1-\overline F_{i,>}(x)
\), 
we obtain
\begin{align}
I_\Delta^{\mathrm{ref,loc}}
&=
G_\Delta^{\mathrm{ref}}(\eta_\Delta)
\left[
1-\overline F_{i,>}(\eta_\Delta)
\right]
-
\int_0^{\eta_\Delta}
\left[
1-\overline F_{i,>}(x)
\right]
\frac{d}{dx}G_\Delta^{\mathrm{ref}}(x)\,dx
\nonumber\\
&=
G_\Delta^{\mathrm{ref}}(\eta_\Delta)
-
G_\Delta^{\mathrm{ref}}(\eta_\Delta)\overline F_{i,>}(\eta_\Delta)
-
\int_0^{\eta_\Delta}
\frac{d}{dx}G_\Delta^{\mathrm{ref}}(x)\,dx
+
\int_0^{\eta_\Delta}
\frac{d}{dx}G_\Delta^{\mathrm{ref}}(x)\,
\overline F_{i,>}(x)\,dx
\nonumber\\
&=
G_\Delta^{\mathrm{ref}}(0)
-
G_\Delta^{\mathrm{ref}}(\eta_\Delta)\overline F_{i,>}(\eta_\Delta)
+
\int_0^{\eta_\Delta}
\frac{d}{dx}G_\Delta^{\mathrm{ref}}(x)\,
\overline F_{i,>}(x)\,dx .
\label{eq:acf_ref_stieltjes_by_parts_exact}
\end{align}
On the other hand, the exponential reference integral \(J_\Delta\) can be integrated
by parts in the ordinary sense:
\begin{align}
J_\Delta
&=
\int_0^{\eta_\Delta}
G_\Delta^{\mathrm{ref}}(x)\lambda_i \exp\{-\lambda_i x\}\,dx
\nonumber\\
&=
G_\Delta^{\mathrm{ref}}(0)
-
G_\Delta^{\mathrm{ref}}(\eta_\Delta)\exp\{-\lambda_i\eta_\Delta\}
+
\int_0^{\eta_\Delta}
\frac{d}{dx}G_\Delta^{\mathrm{ref}}(x)
\exp\{-\lambda_i x\}\,dx .
\label{eq:acf_ref_exponential_by_parts}
\end{align}
Subtracting \eqref{eq:acf_ref_exponential_by_parts} from
\eqref{eq:acf_ref_stieltjes_by_parts_exact} yields the exact difference identity
\[
I_\Delta^{\mathrm{ref,loc}}-J_\Delta
=
-
G_\Delta^{\mathrm{ref}}(\eta_\Delta)
\left[
\overline F_{i,>}(\eta_\Delta)-\exp\{-\lambda_i\eta_\Delta\}
\right]
+
\int_0^{\eta_\Delta}
\frac{d}{dx}G_\Delta^{\mathrm{ref}}(x)
\left[
\overline F_{i,>}(x)-\exp\{-\lambda_i x\}
\right]\,dx .
% \label{eq:acf_ref_difference_identity}
\]
Taking absolute values and using
\eqref{eq:acf_survival_difference_reference_version}, we get
\[
\left|
I_\Delta^{\mathrm{ref,loc}}-J_\Delta
\right|
\le
a_\Delta
\left[
G_\Delta^{\mathrm{ref}}(\eta_\Delta)\exp\{-\lambda_i\eta_\Delta\}
+
\int_0^{\eta_\Delta}
\left|
\frac{d}{dx}G_\Delta^{\mathrm{ref}}(x)
\right|
\exp\{-\lambda_i x\}\,dx
\right].
% \label{eq:acf_ref_local_error_bound}
\]
By the reference weighted-variation condition
\eqref{eq:acf_reference_variation_condition},
\[
% \label{eq:acf_ref_local_relative_error}
\left|
I_\Delta^{\mathrm{ref,loc}}-J_\Delta
\right|
\le
a_\Delta C_G J_\Delta
=
o(J_\Delta).
\]
Therefore,
\begin{equation}
\label{eq:acf_ref_local_replacement}
I_\Delta^{\mathrm{ref,loc}}
=
J_\Delta(1+o(1)).
\end{equation}
Combining \eqref{eq:acf_exact_to_ref_local_integral} and
\eqref{eq:acf_ref_local_replacement}, we obtain
\begin{equation}
\label{eq:acf_exact_local_replacement}
I_\Delta^{\mathrm{ex,loc}}
=
J_\Delta(1+o(1)).
\end{equation}

It remains to add the complement \((\eta_\Delta,p_i)\). Since \(G_\Delta^{\mathrm{ex}}\) is
nonnegative,
\begin{align}
\int_0^{p_i}
G_\Delta^{\mathrm{ex}}(x)\,dF_i(x)
&=
\int_0^{\eta_\Delta}
G_\Delta^{\mathrm{ex}}(x)\,dF_i(x)
+
\mathbb E\left[
G_\Delta^{\mathrm{ex}}(L_i)\mathbf 1\{L_i>\eta_\Delta\}
\right]
\nonumber\\
&=
I_\Delta^{\mathrm{ex,loc}}
+
\mathbb E\left[
G_\Delta^{\mathrm{ex}}(L_i)\mathbf 1\{L_i>\eta_\Delta\}
\right].
\label{eq:acf_exact_full_split}
\end{align}
By the exact tail-localization condition
\eqref{eq:acf_exact_tail_condition},
\begin{equation}
\label{eq:acf_exact_tail_negligible}
\mathbb E\left[
G_\Delta^{\mathrm{ex}}(L_i)\mathbf 1\{L_i>\eta_\Delta\}
\right]
=
o(J_\Delta).
\end{equation}
Substituting \eqref{eq:acf_exact_local_replacement} and
\eqref{eq:acf_exact_tail_negligible} into \eqref{eq:acf_exact_full_split} gives
\[
% \label{eq:acf_exact_full_replacement}
\int_0^{p_i}
G_\Delta^{\mathrm{ex}}(x)\,dF_i(x)
=
J_\Delta(1+o(1)).
\]
Finally, recalling the definition of \(J_\Delta\) in
\eqref{eq:acf_reference_integral_def}, we obtain
\eqref{eq:acf_reference_stieltjes_replacement}.

\end{proof}

The shorthand
\[
% \label{eq:acf_local_density_envelope}
f_i^{\mathrm{loc}}(\ell)
:=
\lambda_i \exp\{-\lambda_i\ell\}
\]
will be used only for the density of the local exponential reference measure. Thus, an integral written informally against \(f_i^{\mathrm{loc}}\) should always be read as an application of Corollary~\ref{cor:acf_local_stieltjes_replacement}, not as a claim that the exact law \(F_i\) is absolutely continuous.

\begin{corollary}[Product-tail Stieltjes replacement]
\label{cor:acf_product_tail_stieltjes_replacement}
For every arm \(i\) and every \(\circ\in\{>,\ge\}\), suppose that the positive shifted levels appearing on the local window are all \(O(\eta_\Delta)\), in the sense that
\begin{equation}
\label{eq:acf_product_local_shift_condition_direct}
\sup_{\substack{0\le x\le\eta_\Delta,\ r\neq i\\ x+p_r-p_i>0}}
\left(x+p_r-p_i\right)
=
O(\eta_\Delta).
\end{equation}
Then
\begin{align}
\int_0^{p_i}
\prod_{r\neq i}
\overline F_{r,\circ}(x+p_r-p_i)
\,dF_i(x)
=
\int_0^{\eta_\Delta}
\exp\left\{
-\sum_{r\neq i}
\lambda_r(x+p_r-p_i)_+
\right\}
\lambda_i \exp\{-\lambda_i x\}\,dx
\left[
1+o(1)
\right],
\label{eq:acf_product_tail_replacement_direct}
\end{align}
where shifted levels \(x+p_r-p_i\le0\) contribute survival factor one.
\end{corollary}

\begin{proof}\leavevmode
We apply Corollary~\ref{cor:acf_local_stieltjes_replacement} with
\begin{equation}
\label{eq:acf_product_Gex_identification}
G_\Delta^{\mathrm{ex}}(x)
=
\prod_{r\neq i}
\overline F_{r,\circ}(x+p_r-p_i),
\end{equation}
\[
% \label{eq:acf_product_Gref_identification}
G_\Delta^{\mathrm{ref}}(x)
=
\exp\left\{
-\sum_{r\neq i}
\lambda_r(x+p_r-p_i)_+
\right\},
\]
and
\begin{equation}
\label{eq:acf_product_J_identification}
J_\Delta
=
\int_0^{\eta_\Delta}
\exp\left\{
-\sum_{r\neq i}
\lambda_r(x+p_r-p_i)_+
\right\}
\lambda_i \exp\{-\lambda_i x\}\,dx .
\end{equation}
We now verify the three requirements in
Corollary~\ref{cor:acf_local_stieltjes_replacement}: the local
exact-to-reference approximation, the weighted-variation condition for
\(G_\Delta^{\mathrm{ref}}\), and the exact tail-localization condition.

First consider the local window \(0\le x\le\eta_\Delta\). For any \(r\neq i\), if
\(x+p_r-p_i\le0\), then the corresponding exact survival factor is one by the
nonpositive-level convention, and the corresponding exponential reference factor is
also one:
\[
% \label{eq:acf_product_nonpositive_factor_equal_direct}
\overline F_{r,\circ}(x+p_r-p_i)
=
1
=
\exp\left\{
-\lambda_r(x+p_r-p_i)_+
\right\}.
\]
If \(x+p_r-p_i>0\), then by
\eqref{eq:acf_product_local_shift_condition_direct} the shifted level lies in the
lower-endpoint region, uniformly over \(0\le x\le\eta_\Delta\) and \(r\neq i\). Hence the
lower-endpoint tail envelope gives
\[
% \label{eq:acf_product_positive_factor_local_direct}
\overline F_{r,\circ}(x+p_r-p_i)
=
\exp\left\{
-\lambda_r(x+p_r-p_i)
\right\}
\left[
1+O(\Delta_1^{1-2\kappa})
\right],
\]
uniformly over all such \(x\) and \(r\). Since \(K\) is fixed and
\(\kappa>1/2\), multiplying the finitely many factors gives
\begin{align}
\prod_{r\neq i}
\overline F_{r,\circ}(x+p_r-p_i)
&=
\exp\left\{
-\sum_{r\neq i}
\lambda_r(x+p_r-p_i)_+
\right\}
\left[
1+O(\Delta_1^{1-2\kappa})
\right]
\nonumber\\
&=
G_\Delta^{\mathrm{ref}}(x)
\left[
1+o(1)
\right],
\label{eq:acf_product_uniform_local_direct}
\end{align}
uniformly for \(0\le x\le\eta_\Delta\).

We next check the weighted-variation condition for \(G_\Delta^{\mathrm{ref}}\).
The map \(x\mapsto (x+p_r-p_i)_+\) is absolutely continuous. Therefore
\(G_\Delta^{\mathrm{ref}}\) is absolutely continuous on \([0,\eta_\Delta]\), and for almost
every \(x\in[0,\eta_\Delta]\),
\[\begin{aligned}
\left|
\frac{d}{dx}G_\Delta^{\mathrm{ref}}(x)
\right|
&=
\left|
\frac{d}{dx}
\exp\left\{
-\sum_{r\neq i}
\lambda_r(x+p_r-p_i)_+
\right\}
\right|
\nonumber\\
&\le
\left(
\sum_{r\neq i}\lambda_r
\right)
\exp\left\{
-\sum_{r\neq i}
\lambda_r(x+p_r-p_i)_+
\right\}
\nonumber\\
&=
\left(
\sum_{r\neq i}\lambda_r
\right)
G_\Delta^{\mathrm{ref}}(x).
% \label{eq:acf_product_derivative_bound_direct}
\end{aligned}\]
Multiplying by \(\exp\{-\lambda_i x\}\) and integrating gives
\begin{equation}
\int_0^{\eta_\Delta}
\left|
\frac{d}{dx}G_\Delta^{\mathrm{ref}}(x)
\right|
\exp\{-\lambda_i x\}\,dx
\le
\left(
\sum_{r\neq i}\lambda_r
\right)
\int_0^{\eta_\Delta}
G_\Delta^{\mathrm{ref}}(x)\exp\{-\lambda_i x\}\,dx
=
\frac{\sum_{r\neq i}\lambda_r}{\lambda_i}
J_\Delta .
\label{eq:acf_product_variation_bound_direct}
\end{equation}
Since \(K\) is fixed and \(\lambda_r=\Theta(\Delta_1)\) for every \(r\), the ratio
\(
({\sum_{r\neq i}\lambda_r})/{\lambda_i}
\)
is uniformly bounded.

The endpoint term in the weighted-variation condition is controlled by the same
reference integral. Since \(\lambda_i\eta_\Delta\to\infty\), for all sufficiently large
\(\Delta_1\), \(\lambda_i^{-1}\le\eta_\Delta\). The function
\(
x\mapsto
G_\Delta^{\mathrm{ref}}(x)\exp\{-\lambda_i x\}
\) 
is nonincreasing on \([0,\eta_\Delta]\). Hence
\begin{equation}
J_\Delta
=
\int_0^{\eta_\Delta}
G_\Delta^{\mathrm{ref}}(x)\lambda_i \exp\{-\lambda_i x\}\,dx
\ge
\lambda_i
\int_{\eta_\Delta-\lambda_i^{-1}}^{\eta_\Delta}
G_\Delta^{\mathrm{ref}}(x)\exp\{-\lambda_i x\}\,dx
\ge
G_\Delta^{\mathrm{ref}}(\eta_\Delta)\exp\{-\lambda_i\eta_\Delta\}.
\label{eq:acf_product_endpoint_bound_direct}
\end{equation}
Combining \eqref{eq:acf_product_variation_bound_direct} and
\eqref{eq:acf_product_endpoint_bound_direct}, we obtain
\begin{equation}
\label{eq:acf_product_weighted_variation_condition_direct}
G_\Delta^{\mathrm{ref}}(\eta_\Delta)\exp\{-\lambda_i\eta_\Delta\}
+
\int_0^{\eta_\Delta}
\left|
\frac{d}{dx}G_\Delta^{\mathrm{ref}}(x)
\right|
\exp\{-\lambda_i x\}\,dx
\le
C J_\Delta .
\end{equation}
Thus the weighted-variation condition of
Corollary~\ref{cor:acf_local_stieltjes_replacement} is satisfied.

We now verify the local exact-to-reference condition. By
\eqref{eq:acf_product_uniform_local_direct},
\begin{equation}
\int_0^{\eta_\Delta}
\left|
G_\Delta^{\mathrm{ex}}(x)
-
G_\Delta^{\mathrm{ref}}(x)
\right|
\,dF_i(x)
\le
o(1)
\int_0^{\eta_\Delta}
G_\Delta^{\mathrm{ref}}(x)\,dF_i(x).
\label{eq:acf_product_L1_pre_direct}
\end{equation}
Applying the local Stieltjes replacement estimate
\eqref{eq:acf_ref_local_replacement} from the proof of
Corollary~\ref{cor:acf_local_stieltjes_replacement} to the present reference
integrand \(G_\Delta^{\mathrm{ref}}\), with \(J_\Delta\) defined by
\eqref{eq:acf_product_J_identification}, gives
\begin{equation}
\label{eq:acf_product_reference_local_under_F_direct}
\int_0^{\eta_\Delta}
G_\Delta^{\mathrm{ref}}(x)\,dF_i(x)
=
J_\Delta
\left[
1+o(1)
\right].
\end{equation}
Substituting \eqref{eq:acf_product_reference_local_under_F_direct} into
\eqref{eq:acf_product_L1_pre_direct}, we get
\[
% \label{eq:acf_product_L1_condition_direct}
\int_0^{\eta_\Delta}
\left|
G_\Delta^{\mathrm{ex}}(x)
-
G_\Delta^{\mathrm{ref}}(x)
\right|
\,dF_i(x)
=
o(J_\Delta).
\]
This is exactly the local exact-to-reference approximation required in
Corollary~\ref{cor:acf_local_stieltjes_replacement}.

It remains to prove the exact tail-localization condition. Since each one-arm survival
factor is at most one,
\[
\int_{(\eta_\Delta,p_i)}
\prod_{r\neq i}
\overline F_{r,\circ}(x+p_r-p_i)
\,dF_i(x)
\le
\int_{(\eta_\Delta,p_i)}
\prod_{\substack{r\neq i\\p_r>p_i}}
\overline F_{r,\circ}(p_r-p_i)
\,dF_i(x).
% \label{eq:acf_product_tail_monotone_start}
\]
Indeed, if \(p_r>p_i\), then \(x+p_r-p_i\ge p_r-p_i\) for \(x>\eta_\Delta\), and the
survival function is nonincreasing; if \(p_r\le p_i\), the corresponding factor is
bounded by one. By \eqref{eq:acf_product_local_shift_condition_direct}, every positive
gap \(p_r-p_i\) with \(p_r>p_i\) is \(O(\eta_\Delta)\). Hence the lower-endpoint tail envelope
gives
\[
% \label{eq:acf_product_better_factor_tail_bound}
\overline F_{r,\circ}(p_r-p_i)
\le
C
\exp\left\{
-\lambda_r(p_r-p_i)
\right\},
\qquad
p_r>p_i .
\]
Moreover, the one-arm tail suppression for \(L_i\) gives
\[
% \label{eq:acf_product_Li_tail_eta_bound}
\mathbb P(L_i>\eta_\Delta)
\le
\exp\{-c\Delta_1\eta_\Delta\}
=
O(\Delta_1^{-\infty}).
\]
Therefore
\begin{equation}
\int_{(\eta_\Delta,p_i)}
\prod_{r\neq i}
\overline F_{r,\circ}(x+p_r-p_i)
\,dF_i(x)
\le
C
\exp\left\{
-\sum_{\substack{r\neq i\\p_r>p_i}}
\lambda_r(p_r-p_i)
\right\}
\exp\{-c\Delta_1\eta_\Delta\}.
\label{eq:acf_product_tail_upper_final}
\end{equation}

We now lower-bound \(J_\Delta\) by the same better-arm exponential factor. Fix a
small constant \(a>0\). Since \(\eta_\Delta\gg\Delta_1^{-1}\), we have
\(a/\Delta_1\le\eta_\Delta\) for all sufficiently large \(\Delta_1\). For
\(0\le x\le a/\Delta_1\),
\[
\sum_{r\neq i}\lambda_r(x+p_r-p_i)_+
\le
\sum_{\substack{r\neq i\\p_r>p_i}}
\lambda_r(p_r-p_i)
+
x\sum_{r\neq i}\lambda_r
\le
\sum_{\substack{r\neq i\\p_r>p_i}}
\lambda_r(p_r-p_i)
+
C'a,
% \label{eq:acf_product_reference_exponent_small_window}
\]
because \(K\) is fixed and \(\lambda_r=\Theta(\Delta_1)\). Also
\(\lambda_i=\Theta(\Delta_1)\). Hence
\begin{align}
J_\Delta
&=
\int_0^{\eta_\Delta}
\exp\left\{
-\sum_{r\neq i}
\lambda_r(x+p_r-p_i)_+
\right\}
\lambda_i \exp\{-\lambda_i x\}\,dx
\nonumber\\
&\ge
\int_0^{a/\Delta_1}
\exp\left\{
-\sum_{r\neq i}
\lambda_r(x+p_r-p_i)_+
\right\}
\lambda_i \exp\{-\lambda_i x\}\,dx
\nonumber\\
&\ge
c'
\exp\left\{
-\sum_{\substack{r\neq i\\p_r>p_i}}
\lambda_r(p_r-p_i)
\right\}.
\label{eq:acf_product_reference_integral_lower_final}
\end{align}
Dividing \eqref{eq:acf_product_tail_upper_final} by
\eqref{eq:acf_product_reference_integral_lower_final}, we obtain
\[\begin{aligned}
\frac{
\displaystyle
\int_{(\eta_\Delta,p_i)}
\prod_{r\neq i}
\overline F_{r,\circ}(x+p_r-p_i)
\,dF_i(x)
}{
J_\Delta
}
\le
C\exp\{-c\Delta_1\eta_\Delta\}/c'
=
o(1).
% \label{eq:acf_product_tail_relative_final}
\end{aligned}\]
This proves the exact tail-localization condition required in
Corollary~\ref{cor:acf_local_stieltjes_replacement}.

All assumptions of Corollary~\ref{cor:acf_local_stieltjes_replacement} have now been
verified for the choices
\eqref{eq:acf_product_Gex_identification}--\eqref{eq:acf_product_J_identification}.
Therefore,
\[\begin{aligned}
\int_0^{p_i}
\prod_{r\neq i}
\overline F_{r,\circ}(x+p_r-p_i)
\,dF_i(x)
=
J_\Delta
\left[
1+o(1)
\right].
% \label{eq:acf_product_final_by_reference_corollary}
\end{aligned}\]
Substituting the expression for \(J_\Delta\) from
\eqref{eq:acf_product_J_identification} gives
\eqref{eq:acf_product_tail_replacement_direct}.

\end{proof}

\section{Asymptotic Suboptimal-Convergence Formula}
\label{app:asymptotic_absorbing_probability_formula}

This section derives an asymptotic closed-form expression for the probability that a suboptimal arm becomes the eventual winner. The proof starts from the score-minimum sandwich, which reduces suboptimal convergence to comparisons among independent potential score minima. This representation separates the constraints imposed by arms with larger means from those imposed by arms with smaller means, and leads to a Stieltjes integral formula for the winner event in terms of the centered score-minimum variables.

The main asymptotic step is to localize this Stieltjes integral near the lower endpoint of the candidate winning arm. In the large-regularization regime, the relevant centered score minimum concentrates on this local scale under the winner comparison, while the complementary region is negligible relative to the local contribution. On the local scale, the one-arm Lundberg envelopes reduce uniformly to elementary exponential forms. The resulting reference integral is then evaluated by splitting the positive-part terms according to the ordered gaps between arm means, yielding the desired asymptotic suboptimal-convergence formula with relative \(1+o(1)\) error.

\subsection{Root-Based Score-Minimum Representation}
\label{app:root_based_drawdown_representation}

We first express the absorbing probability through the arm-wise centered score-minimum variables. Recall that
\(
% \label{eq:mfc_M_L_relation_recall}
M_k=p_k-L_k,\,
k=1,\ldots,K
\). 
The outer Stieltjes measure below is the distribution \(F_i\) of the centered score-minimum variable \(L_i\) of the candidate absorbing arm \(i\). On the candidate branch \(Q_i\), the relevant centered score minimum satisfies
\(
% \label{eq:mfc_Li_support}
L_i\in(0,p_i) \text{ a.s. on }Q_i
\). 
Thus the outer Stieltjes integrals are taken over \([0,p_i)\).

\begin{proposition}[Root-based score-minimum envelope for \(Q_i\)]
\label{prop:root_based_drawdown_representation}
For every arm \(i\), adopt the following conventions. Whenever \(\ell_i+p_j-p_i\le0\), set \(\theta_j(\ell_i+p_j-p_i)=0\). Whenever \(0<\ell_i+p_j-p_i<p_j\), \(\theta_j(\ell_i+p_j-p_i)\) denotes the unique positive solution of
\(
% \label{eq:mfc_shifted_root_equation}
q_j \exp\{\theta(p_i-\ell_i)\}
+
p_j \exp\{-\theta(q_i+\ell_i)\}
=
1
\). 
Then
\begin{align}
&\int_{[0,p_i)}
\exp\left\{
-
\bigl[\Delta_i+\beta\ell_i+p_i-\ell_i\bigr]
\sum_{j\neq i}
\theta_j(\ell_i+p_j-p_i)
\right\}
\,dF_i(\ell_i)
\nonumber \\
\le &
\mathbb P(Q_i)
\label{eq:mfc_root_based_Qi_envelope}
\\
\le &
\int_{[0,p_i)}
\exp\left\{
-
\bigl[\Delta_i+\beta\ell_i\bigr]
\sum_{j\neq i}
\theta_j(\ell_i+p_j-p_i)
\right\}
\,dF_i(\ell_i). \nonumber
\end{align}
\end{proposition}

\begin{proof}\leavevmode
Recall the strict and weak score-minimum winner events
\[
\mathcal W_i^{>}
=
\left\{
M_i>M_j,\ \forall j\neq i
\right\},
\qquad
% \label{eq:mfc_W_strict_def}
\mathcal W_i^{\ge}
=
\left\{
M_i\ge M_j,\ \forall j\neq i
\right\}.
% \label{eq:mfc_W_weak_def}
\]
By Lemma~\ref{lem:sm_absorbing_event_sandwich},
\begin{equation}
\label{eq:mfc_absorbing_sandwich_recall}
\mathcal W_i^{>}
\subseteq
Q_i
\subseteq
\mathcal W_i^{\ge}.
\end{equation}

Using \(M_k=p_k-L_k\), for every \(j\neq i\),
\[
M_i>M_j
\Longleftrightarrow
p_i-L_i>p_j-L_j
\Longleftrightarrow
L_j>L_i+p_j-p_i,
% \label{eq:mfc_strict_comparison_L}
\]
and similarly
\[\begin{aligned}
M_i\ge M_j
&\Longleftrightarrow
L_j\ge L_i+p_j-p_i .
% \label{eq:mfc_weak_comparison_L}
\end{aligned}\]
Therefore,
\[
\mathcal W_i^{>}
=
\left\{
L_j>L_i+p_j-p_i,\ \forall j\neq i
\right\},
\qquad
% \label{eq:mfc_strict_W_L_form}
\mathcal W_i^{\ge}
=
\left\{
L_j\ge L_i+p_j-p_i,\ \forall j\neq i
\right\}.
% \label{eq:mfc_weak_W_L_form}
\]

Recall the strict and weak one-arm tails
\[
% \label{eq:mfc_tail_def_in_proof}
\overline F_{j,>}(x):=\mathbb P(L_j>x),
\qquad
\overline F_{j,\ge}(x):=\mathbb P(L_j\ge x).
\]

When the shifted level \(x=\ell_i+p_j-p_i\) is nonpositive, the comparison \(L_j>x\) or \(L_j\ge x\) is automatic for the score-minimum event under consideration. Thus both strict and weak tail factors are taken to be one:
\[
% \label{eq:mfc_nonpositive_tail_factor_one}
\overline F_{j,>}(x)
=
\overline F_{j,\ge}(x)
=
1,
\qquad
x\le0 .
\]
Equivalently, a nonpositive shifted level contributes no exponential penalty, and we encode this by the convention \(\theta_j(x)=0\) for \(x\le0\).

Because the potential reward streams are independent across arms, the centered score-minimum variables \(L_1,\ldots,L_K\) are independent. Conditioning on \(L_i=\ell_i\), we obtain
\begin{align}
\mathbb P(\mathcal W_i^{>})
=
\int_{[0,p_i)}
\mathbb P\left(
L_j>\ell_i+p_j-p_i,\ \forall j\neq i
\right)
\,dF_i(\ell_i)
=
\int_{[0,p_i)}
\prod_{j\neq i}
\overline F_{j,>}\!\left(\ell_i+p_j-p_i\right)
\,dF_i(\ell_i).
\label{eq:mfc_strict_winner_stieltjes}
\end{align}
Likewise,
\begin{align}
\mathbb P(\mathcal W_i^{\ge})
=
\int_{[0,p_i)}
\mathbb P\left(
L_j\ge\ell_i+p_j-p_i,\ \forall j\neq i
\right)
\,dF_i(\ell_i)
=
\int_{[0,p_i)}
\prod_{j\neq i}
\overline F_{j,\ge}\!\left(\ell_i+p_j-p_i\right)
\,dF_i(\ell_i).
\label{eq:mfc_weak_winner_stieltjes}
\end{align}

We now apply the one-arm root-based envelopes. Since \(\ell_i\in[0,p_i)\), every shifted level satisfies
\(
% \label{eq:mfc_shifted_level_upper_range}
\ell_i+p_j-p_i<p_j
\). 
Thus the upper endpoint never occurs. When \(\ell_i+p_j-p_i\le0\), the tail factor is one and the convention \(\theta_j(\ell_i+p_j-p_i)=0\) makes the corresponding exponential factor equal to one.

Now consider the nontrivial case
\(
% \label{eq:mfc_positive_shifted_level_case}
0<\ell_i+p_j-p_i<p_j
\). 
The Lundberg root of arm \(j\) at level \(\ell_i+p_j-p_i\) is defined by
\[
% \label{eq:mfc_shifted_root_from_Lambda}
\Lambda_j\!\left(
\theta_j(\ell_i+p_j-p_i),
\ell_i+p_j-p_i
\right)
=
0 .
\]
Using the explicit CGF, this equation becomes
\[\begin{aligned}
q_j
\exp\left\{
\theta_j(\ell_i+p_j-p_i)
\left[p_j-(\ell_i+p_j-p_i)\right]
\right\}
+
p_j
\exp\left\{
-\theta_j(\ell_i+p_j-p_i)
\left[q_j+\ell_i+p_j-p_i\right]
\right\}
=
1 .
% \label{eq:mfc_shifted_root_expanded}
\end{aligned}\]
Since
\(
% \label{eq:mfc_shifted_jump_simplification}
p_j-(\ell_i+p_j-p_i)=p_i-\ell_i,\,
q_j+\ell_i+p_j-p_i=q_i+\ell_i
\), 
this reduces to
\[
% \label{eq:mfc_shifted_root_equation_in_proof}
q_j \exp\{\theta_j(\ell_i+p_j-p_i)(p_i-\ell_i)\}
+
p_j \exp\{-\theta_j(\ell_i+p_j-p_i)(q_i+\ell_i)\}
=
1 .
\]
Thus \(\theta_j(\ell_i+p_j-p_i)\) is precisely the unique positive solution specified in Proposition~\ref{prop:root_based_drawdown_representation}.

For positive shifted levels, Corollary~\ref{cor:root_based_drawdown_tail_envelope} gives
\begin{align}
\overline F_{j,>}(\ell_i+p_j-p_i)
&\ge
\exp\left\{
-\theta_j(\ell_i+p_j-p_i)
\left[
b_j(\ell_i+p_j-p_i)
+
p_j-(\ell_i+p_j-p_i)
\right]
\right\},
\label{eq:mfc_strict_tail_lower_root}\\
\overline F_{j,\ge}(\ell_i+p_j-p_i)
&\le
\exp\left\{
-\theta_j(\ell_i+p_j-p_i)
b_j(\ell_i+p_j-p_i)
\right\}.
\nonumber
% \label{eq:mfc_weak_tail_upper_root}
\end{align}
The shifted affine boundary simplifies exactly:
\[\begin{aligned}
b_j(\ell_i+p_j-p_i)
&=
\Delta_j+\beta(\ell_i+p_j-p_i)
\nonumber\\
&=
\alpha-p_j\beta+\beta\ell_i+\beta p_j-\beta p_i
\nonumber\\
&=
\alpha-p_i\beta+\beta\ell_i
\nonumber\\
&=
\Delta_i+\beta\ell_i.
% \label{eq:mfc_boundary_shift_identity}
\end{aligned}\]
Moreover,
\begin{equation}
\label{eq:mfc_positive_jump_shift_identity}
p_j-(\ell_i+p_j-p_i)
=
p_i-\ell_i .
\end{equation}
Combining \eqref{eq:mfc_strict_tail_lower_root}--\eqref{eq:mfc_positive_jump_shift_identity}, and using the zero-root convention for nonpositive shifted levels, yields
\begin{align}
\prod_{j\neq i}
\overline F_{j,>}(\ell_i+p_j-p_i)
&\ge
\exp\left\{
-
\bigl[\Delta_i+\beta\ell_i+p_i-\ell_i\bigr]
\sum_{j\neq i}
\theta_j(\ell_i+p_j-p_i)
\right\},
\label{eq:mfc_product_strict_lower}\\
\prod_{j\neq i}
\overline F_{j,\ge}(\ell_i+p_j-p_i)
&\le
\exp\left\{
-
\bigl[\Delta_i+\beta\ell_i\bigr]
\sum_{j\neq i}
\theta_j(\ell_i+p_j-p_i)
\right\}.
\label{eq:mfc_product_weak_upper}
\end{align}
Substituting \eqref{eq:mfc_product_strict_lower} into \eqref{eq:mfc_strict_winner_stieltjes}, substituting \eqref{eq:mfc_product_weak_upper} into \eqref{eq:mfc_weak_winner_stieltjes}, and using \eqref{eq:mfc_absorbing_sandwich_recall}, proves \eqref{eq:mfc_root_based_Qi_envelope}.

\end{proof}

\subsection{Asymptotic Closed-Form Evaluation}
\label{app:asymptotic_closed_form}

We now evaluate the score-minimum Stieltjes representation in Proposition~\ref{prop:root_based_drawdown_representation} under Assumption~\ref{ass:app_asymptotic_regime}. Consider a candidate absorbing arm \(i\). On the local window \(0\le\ell_i\le\eta_\Delta\), the assumed pairwise-gap scale ensures that every positive shifted level \(\ell_i+p_r-p_i\), \(r\neq i\), remains within the lower-endpoint region. Therefore, the local root expansion in Lemma~\ref{lem:lower_endpoint_theta_v_expansion} and the tail envelope in Lemma~\ref{lem:asymptotic_closed_form_tail} apply to all shifted levels appearing in the local Stieltjes representation.

\begin{proposition}[Macroscopic suboptimal-convergence formula]
\label{prop:mfc_asymptotic_closed_form}
Under Assumption~\ref{ass:app_asymptotic_regime}, for every arm \(i\ge2\) with \(p_i<p_1\),
\begin{align}
\mathbb P(Q_i)
&=
\sum_{m=i}^{K}
\frac{1}{m}
\exp\left\{
-\sum_{h=1}^{m}\lambda_h\delta_{hm}
\right\}
\left[
1-
\exp\left\{
-\left(
\sum_{h=1}^{m}\lambda_h
\right)\delta_{m,m+1}
\right\}
\right]
(1+o(1)),
\label{eq:mfc_closed_form_simplified}
\end{align}
where empty sums are interpreted as zero, \(\delta_{hh}:=0\), and
\(\delta_{K,K+1}:=\infty\). For \(m=K\), the second exponential in the bracket
is interpreted as zero.
\end{proposition}

\begin{proof}\leavevmode
By the score-minimum sandwich and the Stieltjes representation in Proposition~\ref{prop:root_based_drawdown_representation}, for \(\circ\in\{>,\ge\}\),
\[
% \label{eq:mfc_Wcirc_stieltjes}
\mathbb P(\mathcal W_i^{\circ})
=
\int_{[0,p_i)}
\prod_{j\neq i}
\overline F_{j,\circ}(\ell_i+p_j-p_i)
\,dF_i(\ell_i),
\]
and
\begin{equation}
\label{eq:mfc_Wcirc_sandwich}
\mathbb P(\mathcal W_i^{>})
\le
\mathbb P(Q_i)
\le
\mathbb P(\mathcal W_i^{\ge}).
\end{equation}
We first evaluate \(\mathbb P(\mathcal W_i^{\circ})\). For \(0\le\ell_i\le\eta_\Delta\), the pairwise-gap scale in \eqref{eq:app_asymptotic_gap_scale} ensures that every positive shifted level \(\ell_i+p_j-p_i\) is \(O(\eta_\Delta)\). Hence Corollary~\ref{cor:acf_product_tail_stieltjes_replacement} gives
\begin{align}
\mathbb P(\mathcal W_i^{\circ})
&=
\int_0^{\eta_\Delta}
\exp\left\{
-\sum_{j\neq i}
\lambda_j(\ell_i+p_j-p_i)_+
\right\}
\lambda_i \exp\{-\lambda_i\ell_i\}\,d\ell_i
(1+o(1)).
\label{eq:mfc_product_tail_replacement_applied}
\end{align}
Using the ordering \(p_1\ge\cdots\ge p_K\), the shifted positive parts are
\begin{equation}
\label{eq:mfc_shifted_positive_part_split}
(\ell_i+p_j-p_i)_+
=
\begin{cases}
\ell_i+\delta_{ji}, & j<i,\\
(\ell_i-\delta_{ij})_+, & j>i .
\end{cases}
\end{equation}
Substituting \eqref{eq:mfc_shifted_positive_part_split} into \eqref{eq:mfc_product_tail_replacement_applied}, we obtain
\[\begin{aligned}
\mathbb P(\mathcal W_i^{\circ})
&=
\lambda_i
\exp\left\{
-\sum_{j<i}\lambda_j\delta_{ji}
\right\}
\int_0^{\eta_\Delta}
\exp\left\{
-
\left(
\lambda_i+\sum_{j<i}\lambda_j
\right)\ell_i
-
\sum_{j>i}\lambda_j(\ell_i-\delta_{ij})_+
\right\}
\,d\ell_i
(1+o(1)).
% \label{eq:mfc_bulk_integral_eta_precise}
\end{aligned}\]

We next extend the upper integration limit from \(\eta_\Delta\) to infinity. By the standing interior condition~\eqref{eq:app_ordered_means} and the regularization scale \eqref{eq:app_asymptotic_regularization_scale}, uniformly over \(k\in[K]\), \(\lambda_k=2\Delta_k/\sigma_k^2=\Theta(\Delta_1)\). Hence \(\lambda_i+\sum_{j<i}\lambda_j=\Theta(\Delta_1)\). Since \((\ell_i-\delta_{ij})_+\ge0\),
\[\begin{aligned}
\int_{\eta_\Delta}^\infty
\exp\left\{
-
\left(
\lambda_i+\sum_{j<i}\lambda_j
\right)\ell_i
-
\sum_{j>i}\lambda_j(\ell_i-\delta_{ij})_+
\right\}
\,d\ell_i&\le
\int_{\eta_\Delta}^\infty
\exp\left\{
-
\left(
\lambda_i+\sum_{j<i}\lambda_j
\right)\ell_i
\right\}
\,d\ell_i
\nonumber\\
&=
\frac{
\exp\left\{
-
\left(
\lambda_i+\sum_{j<i}\lambda_j
\right)\eta_\Delta
\right\}
}{
\lambda_i+\sum_{j<i}\lambda_j
}
\nonumber\\
&\le
C\Delta_1^{-1}\exp\{-\Delta_1\eta_\Delta/C\}.
% \label{eq:mfc_extend_to_infty_tail_upper}
\end{aligned}\]
Conversely, fix a small constant \(a>0\). Since \(\eta_\Delta=\Delta_1^{-\kappa}\) by \eqref{eq:app_asymptotic_eta_def} and \(\kappa<1\), we have \(a/\Delta_1\le\eta_\Delta\) for all sufficiently large \(\Delta_1\). Hence
\[\begin{aligned}
\int_0^{\eta_\Delta}
\exp\left\{
-
\left(
\lambda_i+\sum_{j<i}\lambda_j
\right)\ell_i
-
\sum_{j>i}\lambda_j(\ell_i-\delta_{ij})_+
\right\}
\,d\ell_i
&\ge
\int_0^{a/\Delta_1}
\exp\left\{
-
\left(
\lambda_i+\sum_{j<i}\lambda_j+\sum_{j>i}\lambda_j
\right)\ell_i
\right\}
\,d\ell_i
\nonumber\\
&\ge
c\Delta_1^{-1}.
% \label{eq:mfc_extend_to_infty_bulk_lower}
\end{aligned}\]
Therefore, because \(\Delta_1\eta_\Delta\to\infty\),
\[\begin{aligned}
&\int_{\eta_\Delta}^\infty
\exp\left\{
-
\left(
\lambda_i+\sum_{j<i}\lambda_j
\right)\ell_i
-
\sum_{j>i}\lambda_j(\ell_i-\delta_{ij})_+
\right\}
\,d\ell_i
\\=&
o(1)
\int_0^{\eta_\Delta}
\exp\left\{
-
\left(
\lambda_i+\sum_{j<i}\lambda_j
\right)\ell_i
-
\sum_{j>i}\lambda_j(\ell_i-\delta_{ij})_+
\right\}
\,d\ell_i .
% \label{eq:mfc_extend_to_infty_relative}
\end{aligned}\]
Consequently,
\begin{align}
\mathbb P(\mathcal W_i^{\circ})
&=
\lambda_i
\exp\left\{
-\sum_{j<i}\lambda_j\delta_{ji}
\right\}
\int_0^\infty
\exp\left\{
-
\left(
\lambda_i+\sum_{j<i}\lambda_j
\right)\ell_i
-
\sum_{j>i}\lambda_j(\ell_i-\delta_{ij})_+
\right\}
\,d\ell_i
(1+o(1)).
\label{eq:mfc_integral_extended_final}
\end{align}

It remains to evaluate the elementary integral. Because the arms are ordered by nonincreasing means, for the arm \(i\) under consideration, the gaps \(\delta_{ik}=p_i-p_k\), \(k>i\), are nondecreasing in \(k\). With the endpoint convention \(\delta_{ii}=0\) and \(\delta_{i,K+1}=\infty\), we have
\[
% \label{eq:mfc_gap_order_for_piecewise}
0=\delta_{ii}
\le
\delta_{i,i+1}
\le
\cdots
\le
\delta_{iK}
\le
\delta_{i,K+1}
=
\infty .
\]
Split \([0,\infty)\) according to these breakpoints:
\[
% \label{eq:mfc_piecewise_intervals}
[0,\infty)
=
\bigcup_{m=i}^{K}
[\delta_{im},\delta_{i,m+1}).
\]
If consecutive arms have equal means, then the corresponding breakpoints coincide and the resulting interval is empty; such intervals contribute zero to the piecewise integral.
For \(m\in\{i,\ldots,K\}\) and \(\ell_i\in[\delta_{im},\delta_{i,m+1})\),
\[
% \label{eq:mfc_positive_part_piecewise}
\sum_{j>i}\lambda_j(\ell_i-\delta_{ij})_+
=
\sum_{k=i+1}^{m}\lambda_k(\ell_i-\delta_{ik}),
\]
and therefore
\[
-
\left(
\lambda_i+\sum_{j<i}\lambda_j
\right)\ell_i
-
\sum_{j>i}\lambda_j(\ell_i-\delta_{ij})_+
=
-
\left(
\sum_{k=1}^{m}\lambda_k
\right)\ell_i
+
\sum_{k=i+1}^{m}\lambda_k\delta_{ik}.
% \label{eq:mfc_exponent_piecewise_simplified}
\]
Thus
\begin{align}
&\int_0^\infty
\exp\left\{
-
\left(
\lambda_i+\sum_{j<i}\lambda_j
\right)\ell_i
-
\sum_{j>i}\lambda_j(\ell_i-\delta_{ij})_+
\right\}
\,d\ell_i
\nonumber\\
&=
\sum_{m=i}^{K}
\exp\left\{
\sum_{k=i+1}^{m}\lambda_k\delta_{ik}
\right\}
\int_{\delta_{im}}^{\delta_{i,m+1}}
\exp\left\{
-
\left(
\sum_{k=1}^{m}\lambda_k
\right)\ell_i
\right\}
\,d\ell_i
\nonumber\\
&=
\sum_{m=i}^{K}
\frac{
\exp\left\{
\sum_{k=i+1}^{m}\lambda_k\delta_{ik}
\right\}
}{
\sum_{k=1}^{m}\lambda_k
}
\left[
\exp\left\{
-
\left(
\sum_{k=1}^{m}\lambda_k
\right)
\delta_{im}
\right\}
-
\exp\left\{
-
\left(
\sum_{k=1}^{m}\lambda_k
\right)
\delta_{i,m+1}
\right\}
\right].
\label{eq:mfc_piecewise_integral_evaluation}
\end{align}
For \(m=K\), the second exponential in the bracket is interpreted as zero because \(\delta_{i,K+1}=\infty\).

Substituting \eqref{eq:mfc_piecewise_integral_evaluation} into \eqref{eq:mfc_integral_extended_final}, we get, for both \(\circ=>\) and \(\circ=\ge\),
\begin{align}
\mathbb P(\mathcal W_i^{\circ})
&=
\lambda_i
\exp\left\{
-\sum_{j<i}\lambda_j\delta_{ji}
\right\}
\sum_{m=i}^{K}
\frac{
\exp\left\{
\sum_{k=i+1}^{m}\lambda_k\delta_{ik}
\right\}
}{
\sum_{k=1}^{m}\lambda_k
}
\nonumber\\
&\quad\times
\left[
\exp\left\{
-
\left(
\sum_{k=1}^{m}\lambda_k
\right)
\delta_{im}
\right\}
-
\exp\left\{
-
\left(
\sum_{k=1}^{m}\lambda_k
\right)
\delta_{i,m+1}
\right\}
\right]
(1+o(1)).
\label{eq:mfc_Wcirc_piecewise_final}
\end{align}

We now simplify the common leading expression in
\eqref{eq:mfc_Wcirc_piecewise_final}. 

For \(m\in\{i,\ldots,K\}\), consider first the first exponential term inside
the bracket of \eqref{eq:mfc_Wcirc_piecewise_final}. Its total exponent,
including the prefactor outside the bracket, is
\[\begin{aligned}
&-\sum_{j<i}\lambda_j\delta_{ji}
+
\sum_{k=i+1}^{m}\lambda_k\delta_{ik}
-
\left(
\sum_{h=1}^{m}\lambda_h
\right)\delta_{im}.
\end{aligned}\]
Equivalently, this is the negative of
\[\begin{aligned}
A_{i,m}
&:=
\sum_{j<i}\lambda_j\delta_{ji}
-
\sum_{k=i+1}^{m}\lambda_k\delta_{ik}
+
\left(
\sum_{h=1}^{m}\lambda_h
\right)\delta_{im}.
\end{aligned}\]
We simplify \(A_{i,m}\). For every \(j<i\),
\[
\delta_{ji}+\delta_{im}
=
(p_j-p_i)+(p_i-p_m)
=
p_j-p_m
=
\delta_{jm}.
\]
For every \(k=i+1,\ldots,m\),
\[
\delta_{im}-\delta_{ik}
=
(p_i-p_m)-(p_i-p_k)
=
p_k-p_m
=
\delta_{km}.
\]
Therefore,
\[\begin{aligned}
A_{i,m}
&=
\sum_{j<i}\lambda_j(\delta_{ji}+\delta_{im})
+
\lambda_i\delta_{im}
+
\sum_{k=i+1}^{m}\lambda_k(\delta_{im}-\delta_{ik})
\\
&=
\sum_{j<i}\lambda_j\delta_{jm}
+
\lambda_i\delta_{im}
+
\sum_{k=i+1}^{m}\lambda_k\delta_{km}
\\
&=
\sum_{h=1}^{m}\lambda_h\delta_{hm},
\end{aligned}\]
Hence the first exponential term in the bracket
contributes
\(
\exp\left\{
-\sum_{h=1}^{m}\lambda_h\delta_{hm}
\right\}
\).

Now consider the second exponential term in the bracket of
\eqref{eq:mfc_Wcirc_piecewise_final}. For \(m<K\), its total exponent,
including the same prefactor outside the bracket, is the negative of
\[\begin{aligned}
B_{i,m}
&:=
\sum_{j<i}\lambda_j\delta_{ji}
-
\sum_{k=i+1}^{m}\lambda_k\delta_{ik}
+
\left(
\sum_{h=1}^{m}\lambda_h
\right)\delta_{i,m+1}.
\end{aligned}\]
The same algebra gives, for \(j<i\), 
\(
\delta_{ji}+\delta_{i,m+1}
=
\delta_{j,m+1}
\), 
and, for \(k=i+1,\ldots,m\), 
\(
\delta_{i,m+1}-\delta_{ik}
=
\delta_{k,m+1}
\). 
Thus
\[\begin{aligned}
B_{i,m}
&=
\sum_{j<i}\lambda_j\delta_{j,m+1}
+
\lambda_i\delta_{i,m+1}
+
\sum_{k=i+1}^{m}\lambda_k\delta_{k,m+1}
\\
&=
\sum_{h=1}^{m}\lambda_h\delta_{h,m+1}.
\end{aligned}\]
Since \(h\le m\), we have
\(
\delta_{h,m+1}
=
\delta_{hm}+\delta_{m,m+1}
\). 
Therefore,
\[\begin{aligned}
B_{i,m}
&=
\sum_{h=1}^{m}\lambda_h\delta_{hm}
+
\left(
\sum_{h=1}^{m}\lambda_h
\right)\delta_{m,m+1}.
\end{aligned}\]
For \(m=K\), the second exponential term is interpreted as zero, which is
equivalent to the convention \(\delta_{K,K+1}=\infty\).

Substituting these exponent identities into
\eqref{eq:mfc_Wcirc_piecewise_final}, we obtain the rate-exact compact form
\begin{align}
\mathbb P(\mathcal W_i^{\circ})
=
\sum_{m=i}^{K}
\frac{\lambda_i}{\sum_{h=1}^{m}\lambda_h}
\exp\left\{
-\sum_{h=1}^{m}\lambda_h\delta_{hm}
\right\}\times
\left[
1-
\exp\left\{
-\left(
\sum_{h=1}^{m}\lambda_h
\right)\delta_{m,m+1}
\right\}
\right]
(1+o(1)).
\label{eq:mfc_Wcirc_compact_rate_exact}
\end{align}

We next simplify the prefactor \(\lambda_i/\sum_{h=1}^{m}\lambda_h\). By \eqref{eq:app_asymptotic_gap_scale}, uniformly over the finitely many indices \(h\le m\),
\[
|p_h-p_i|
=
O(\eta_\Delta)
=
o(1),
\]
where the conclusion is immediate when \(p_h=p_i\).
Moreover,
\[
\begin{aligned}
\sigma_h^2-\sigma_i^2
&=
p_h(1-p_h)-p_i(1-p_i)
\\
&=
(p_h-p_i)\bigl[1-(p_h+p_i)\bigr].
\end{aligned}
\]
By the standing interior condition~\eqref{eq:app_ordered_means},
\[
\left|\sigma_h^2-\sigma_i^2\right|
\le
|p_h-p_i|
=
O(\eta_\Delta),
\]
while
\(
\sigma_i^2
\ge
\epsilon_{\mathrm p}(1-\epsilon_{\mathrm p})
>
0
\). 
Consequently,
\[
\frac{\sigma_h^2}{\sigma_i^2}
=
1+O(\eta_\Delta)
=
1+o(1),
\]
uniformly over \(h\le m\).

Similarly,
\[
\Delta_h-\Delta_i
=
-(p_h-p_i)\beta.
\]
By \eqref{eq:app_asymptotic_regularization_scale},
\[
\Delta_i=\Theta(\Delta_1), 
\qquad
\beta=O(\Delta_1).
\]
Therefore,
\[
\left|
\frac{\Delta_h}{\Delta_i}-1
\right|
=
\frac{|p_h-p_i|\beta}{\Delta_i}
=
O(\eta_\Delta)
=
o(1),
\]
uniformly over \(h\le m\). Combining the preceding estimates gives
\[
\frac{\lambda_h}{\lambda_i}
=
\frac{\Delta_h}{\Delta_i}
\frac{\sigma_i^2}{\sigma_h^2}
=
1+O(\eta_\Delta)
=
1+o(1).
\]
Hence,
\[
\lambda_h
=
\lambda_i(1+o(1)),
\qquad
h=1,\ldots,m.
\]
Because \(K\) is fixed,
\[
\sum_{h=1}^{m}\lambda_h
=
m\lambda_i(1+o(1)).
\]
Consequently,
\begin{equation}
\label{eq:mfc_lambda_ratio_one_over_m}
\frac{\lambda_i}{\sum_{h=1}^{m}\lambda_h}
=
\frac{1}{m}(1+o(1)).
\end{equation}
Substituting \eqref{eq:mfc_lambda_ratio_one_over_m} into
\eqref{eq:mfc_Wcirc_compact_rate_exact} yields
\begin{align}
\mathbb P(\mathcal W_i^{\circ})
=
\sum_{m=i}^{K}
\frac{1}{m}
\exp\left\{
-\sum_{h=1}^{m}\lambda_h\delta_{hm}
\right\}
\times
\left[
1-
\exp\left\{
-\left(
\sum_{h=1}^{m}\lambda_h
\right)\delta_{m,m+1}
\right\}
\right]
(1+o(1)).
\label{eq:mfc_Wcirc_simplified_final}
\end{align}

The simplified expression in \eqref{eq:mfc_Wcirc_simplified_final} holds for both
\(\circ=>\) and \(\circ=\ge\). This is because
Corollary~\ref{cor:acf_product_tail_stieltjes_replacement} applies to the strict
and weak Stieltjes integrals in the same way: both integrals are localized on
the same lower-endpoint window, both exact product tails are replaced by the
same exponential reference product, and the contribution outside the localized
window is negligible. Consequently,
\[
% \label{eq:mfc_strict_weak_winner_equivalence}
\mathbb P(\mathcal W_i^{\ge})
=
\mathbb P(\mathcal W_i^{>})(1+o(1)),
\]
and hence
\begin{equation}
\label{eq:mfc_strict_weak_gap_negligible}
\mathbb P(\mathcal W_i^{\ge})
-
\mathbb P(\mathcal W_i^{>})
=
o\!\left(\mathbb P(\mathcal W_i^{>})\right).
\end{equation}
Combining \eqref{eq:mfc_Wcirc_simplified_final} and \eqref{eq:mfc_strict_weak_gap_negligible} with the sandwich
\eqref{eq:mfc_Wcirc_sandwich} proves \eqref{eq:mfc_closed_form_simplified}.

\end{proof}

\section{Conditional Abandonment Costs Given the Absorbing Arm}
\label{app:conditional_abandonment_costs}

This section estimates the finite-horizon abandonment costs incurred before a nonabsorbing arm is discarded. Consider any two distinct arms \(i\neq j\). On the event \(Q_i\), arm \(i\) is the unique absorbing arm, while arm \(j\) is eventually abandoned. Conditioning on the realized centered score minimum \(L_i=\ell_i\) of the absorbing arm fixes the score floor \(p_i-\ell_i\). The evolution of arm \(j\), relative to this floor, can then be compared with a one-arm boundary-crossing problem. This comparison gives an upper representation for the capped count \(\min\{T,N_j\}\), but only after using a strict crossing convention; equality with the absorbing floor may still allow tie-breaking, and is therefore kept as a finite-horizon correction.

The first step is to derive a pairwise winner-conditioned Stieltjes upper bound. The strict score-minimum factors define the denominator, while the numerator uses a strict-crossing time only for the abandoned arm \(j\) and keeps weak winner factors for all remaining competitors. The only equality correction that must be retained is the pairwise boundary event \(M_j=M_i\). This keeps the conditioning on \(Q_i\) explicit and avoids treating the adaptive count \(N_j\) as an exact one-arm hitting time.

The second step is to bound the strict capped first-passage envelope pointwise. When \(p_i>p_j\), the comparison drift is \(\delta_{ij}-\ell_i\), so the critical interface \(\ell_i=\delta_{ij}\) separates a positive-drift region from a negative-drift region. The positive-drift side is controlled by a direct Wald bound, the negative-drift side by a tilted Lundberg--Wald bound, and the interface itself only by the finite-horizon cap \(T\). When \(p_j>p_i\), the comparison walk has negative drift for every \(\ell_i>0\), so the tilted Lundberg--Wald bound applies throughout. When \(p_i=p_j\), no nonzero drift scale is available, and the finite-horizon cap gives the only general bound.

The final step is asymptotic. Under the macroscopic lower-endpoint regime, the centered score minimum \(L_i\) is localized through winner-weighted Stieltjes ratios rather than through a direct conditional law of \(L_i\) given \(Q_i\). This localization turns the strict root-based Stieltjes upper envelope into explicit closed-form conditional-abandonment bounds. The equality correction needed for the abandoned arm \(j\) is kept in pairwise form. More precisely, only the boundary event \(M_j=M_i\) can invalidate the strict-crossing upper bound for \(N_j\). This pairwise boundary event is shown to be negligible on the suboptimal-convergence scale of arm \(j\), which is essential when the absorbing arm is optimal. The optimal absorbing branches are then refined separately, because they generate the only transient correction that remains visible in the final regret upper envelope.

\subsection{Winner-Conditioned Capped Boundary-Crossing Upper Representation}
\label{app:cac_representation}

For any two distinct arms \(i\neq j\), on \(Q_i\), arm \(i\) is eventually pulled forever and arm \(j\) is eventually abandoned. For finite-horizon regret, we only need the capped count \(\min\{T,N_j\}\). The goal of this subsection is to dominate this capped count by a one-arm comparison time and then integrate this domination over the centered score minimum \(L_i\) of the absorbing arm.

Recall that
\(
% \label{eq:cac_winner_floor}
M_i=p_i-L_i
\). 
Condition on \(L_i=\ell_i\). The corresponding score floor is \(p_i-\ell_i\). Define the comparison walk of arm \(j\) against this fixed floor by
\begin{equation}
\label{eq:cac_comparison_walk_def}
U_{j|i,n}(\ell_i)
:=
(p_i-\ell_i)n-S_j(n),
\qquad
n\ge1,
\end{equation}
with boundary height
\begin{equation}
\label{eq:cac_comparison_boundary}
b_i(\ell_i)
:=
\alpha-(p_i-\ell_i)\beta
=
\Delta_i+\beta\ell_i .
\end{equation}

We use two comparison times. The weak first-passage time is
\[
% \label{eq:cac_weak_tau_def}
\tau_{j|i}^{\ge}(\ell_i)
:=
\inf\left\{
n\ge1:
U_{j|i,n}(\ell_i)\ge b_i(\ell_i)
\right\},
\]
which is the first potential pull count \(n\) for which the regularized score of arm \(j\) is no larger than \(p_i-\ell_i\). The strict first-passage time is
\begin{equation}
\label{eq:cac_strict_tau_def}
\tau_{j|i}^{>}(\ell_i)
:=
\inf\left\{
n\ge1:
U_{j|i,n}(\ell_i)> b_i(\ell_i)
\right\},
\end{equation}
which is the first potential pull count \(n\) for which the regularized score of arm \(j\) is strictly below \(p_i-\ell_i\). The strict time is the one that can be used to upper-bound the actual abandoned count. Equality with the floor may still allow tie-breaking to select arm \(j\), so the weak time alone is not a pathwise upper bound on \(N_j\).

The drift of the comparison walk is
\[\begin{aligned}
d_{j|i}(\ell_i)
:=
\mathbb E\!\left[
U_{j|i,n}(\ell_i)-U_{j|i,n-1}(\ell_i)
\right]
=
p_i-p_j-\ell_i .
% \label{eq:cac_effective_drift}
\end{aligned}\]
Thus, if \(p_i>p_j\), then
\(
% \label{eq:cac_lower_mean_drift}
d_{j|i}(\ell_i)
=
\delta_{ij}-\ell_i
\), 
so the drift changes sign at \(\ell_i=\delta_{ij}\). If \(p_j>p_i\), then
\(
% \label{eq:cac_higher_mean_drift}
d_{j|i}(\ell_i)
=
-(\delta_{ji}+\ell_i)<0
\), 
so the comparison is always in the negative-drift regime.

We now derive the pathwise capped envelope. For the ordered pair
\((i,j)\), define the pairwise equality part of the score-minimum comparison by
\[
% \label{eq:cac_pairwise_tie_part_def}
\mathcal E_{ij}^{=}
:=
\mathcal W_i^{\ge}\cap\{M_j=M_i\}.
\]
This is the only boundary event relevant for the abandoned count of arm \(j\).
Indeed, equality between \(M_i\) and some other competitor \(M_r\), \(r\neq j\),
does not prevent the strict-crossing argument from upper-bounding \(N_j\).

Since \(Q_i\subseteq\mathcal W_i^{\ge}\), on \(Q_i\) either \(M_j<M_i\) or
\(M_j=M_i\). On the event \(M_j<M_i\), the strict comparison with arm \(j\) is
available. In particular, if the actual trajectory ever reaches the
\(\tau_{j|i}^{>}(L_i)\)-th potential pull of arm \(j\), then arm \(j\)'s score
is strictly below \(M_i\), while arm \(i\)'s current score is always at least
its potential minimum \(M_i\). Hence arm \(j\) can never again be a maximizer
after that strict crossing. On the pairwise equality event \(M_j=M_i\), we use
only the finite-horizon cap. Therefore,
\[\begin{aligned}
\min\{T,N_j\}\mathbf 1_{Q_i}
&\le
\min\{T,\tau_{j|i}^{>}(L_i)\}
\mathbf 1_{\mathcal W_i^{\ge}\cap\{M_j<M_i\}}
+
T\mathbf 1_{\mathcal E_{ij}^{=}} .
% \label{eq:cac_pathwise_pairwise_capped_upper}
\end{aligned}\]
The factor \(T\) is kept explicitly as the finite-horizon cap. Throughout the asymptotic estimates below, the limit is taken only in the regularization scale \(\Delta_1\to\infty\). The horizon \(T\) is treated as an external finite-horizon parameter and is not absorbed into any \(o(\cdot)\) term.

Conditioning on \(L_i=\ell_i\) and using independence of the potential reward
streams, the numerator satisfies
\begin{align}
\mathbb E\!\left[
\min\{T,N_j\}\mathbf 1_{Q_i}
\right]
\le
\int_{[0,p_i)}
\mathbb E\!\left[
\min\{T,\tau_{j|i}^{>}(\ell_i)\}
\mathbf 1\{\tau_{j|i}^{>}(\ell_i)<\infty\}
\right]
\prod_{\substack{r=1\\r\neq i,j}}^{K}
\overline F_{r,\ge}(\ell_i+p_r-p_i)
\,dF_i(\ell_i)
+
T\mathbb P(\mathcal E_{ij}^{=}).
\label{eq:cac_pairwise_capped_upper_numerator_stieltjes}
\end{align}
Here the strict event for arm \(j\) is contained in the one-arm factor
\(\mathbf 1\{\tau_{j|i}^{>}(\ell_i)<\infty\}\), while the remaining competitors
enter only through weak winner factors.

For the denominator, the strict side of the score-minimum sandwich gives
\begin{equation}
\label{eq:cac_denominator_strict_lower}
\mathbb P(Q_i)
\ge
\mathbb P(\mathcal W_i^>)
=
\int_{[0,p_i)}
\prod_{r\neq i}
\overline F_{r,>}(\ell_i+p_r-p_i)
\,dF_i(\ell_i).
\end{equation}
Combining \eqref{eq:cac_pairwise_capped_upper_numerator_stieltjes} and
\eqref{eq:cac_denominator_strict_lower}, we obtain the pairwise
winner-conditioned capped upper representation
\begin{align}
\mathbb E[\min\{T,N_j\}\mid Q_i]
&\le
\frac{
\displaystyle
\int_{[0,p_i)}
\mathbb E\!\left[
\min\{T,\tau_{j|i}^{>}(\ell_i)\}
\mathbf 1\{\tau_{j|i}^{>}(\ell_i)<\infty\}
\right]
\prod_{\substack{r=1\\r\neq i,j}}^{K}
\overline F_{r,\ge}(\ell_i+p_r-p_i)
\,dF_i(\ell_i)
+
T\mathbb P(\mathcal E_{ij}^{=})
}{
\displaystyle
\mathbb P(\mathcal W_i^>)
}.
\label{eq:cac_capped_winner_conditioned_upper_ratio}
\end{align}
This representation is an upper envelope for the capped abandonment cost.

\subsection{Pointwise Capped First-Passage Envelopes}
\label{app:cac_exact_root_based}

We now bound the one-arm capped factor appearing in \eqref{eq:cac_capped_winner_conditioned_upper_ratio}. Consider any \(i\neq j\) and \(\ell_i\in(0,p_i)\). Throughout this subsection, \(\tau_{j|i}^{>}(\ell_i)\), \(U_{j|i,n}(\ell_i)\), and \(b_i(\ell_i)\) are defined by \eqref{eq:cac_comparison_walk_def}, \eqref{eq:cac_comparison_boundary}, and \eqref{eq:cac_strict_tau_def}.

On the strict crossing event, the terminal state has the bounded-overshoot form
\begin{equation}
\label{eq:cac_terminal_state}
U_{j|i,\tau_{j|i}^{>}(\ell_i)}(\ell_i)
=
b_i(\ell_i)+r_{j|i}^{>}(\ell_i),
\qquad
0< r_{j|i}^{>}(\ell_i)\le p_i-\ell_i .
\end{equation}
The coarse bound \(r_{j|i}^{>}(\ell_i)\le p_i-\ell_i\) is the only overshoot information needed below.

When the comparison drift is negative, the relevant shifted level for arm \(j\) is
\(
% \label{eq:cac_shifted_level_y_def}
y
:=
\ell_i+p_j-p_i
\). 
The negative-drift regime is exactly \(0<y<p_j\). In this case, let \(\theta_j(y)>0\) be the Lundberg root of arm \(j\) at level \(y\), i.e.,
\begin{equation}
\label{eq:cac_shifted_lundberg_root_explicit}
q_j \exp\{\theta_j(y)(p_j-y)\}
+
p_j \exp\{-\theta_j(y)(q_j+y)\}
=
1 .
\end{equation}
The associated tilted drift is
\[\begin{aligned}
v_j(y)
&:=
\partial_\theta\Lambda_j(\theta_j(y),y)
\nonumber\\
&=
q_j(p_j-y)\exp\{\theta_j(y)(p_j-y)\}
-
p_j(q_j+y)\exp\{-\theta_j(y)(q_j+y)\}
\nonumber\\
&=
p_j-y
-
p_j \exp\{-\theta_j(y)(q_j+y)\}.
% \label{eq:cac_shifted_tilted_drift_def}
\end{aligned}\]
The last equality follows from the Lundberg balance \eqref{eq:cac_shifted_lundberg_root_explicit}. In particular, \(v_j(y)>0\) by Corollary~\ref{cor:tilted_drift_positivity}.

Since
\[
% \label{eq:cac_shifted_boundary_identity}
b_j(y)
=
\Delta_j+\beta y
=
\Delta_i+\beta\ell_i
=
b_i(\ell_i),
\]
the comparison time \(\tau_{j|i}^{>}(\ell_i)\) is the strict first-passage time of the one-arm crossing problem for arm \(j\) at shifted level \(y\).

\begin{lemma}[Pointwise capped first-passage envelope]
\label{lem:cac_pointwise_capped_first_passage}
For every \(i\neq j\), every \(\ell_i\in(0,p_i)\), and every finite horizon \(T\),
\begin{equation}
\label{eq:cac_pointwise_capped_factor_bound}
\mathbb E\!\left[
\min\{T,\tau_{j|i}^{>}(\ell_i)\}
\mathbf 1\{\tau_{j|i}^{>}(\ell_i)<\infty\}
\right]
\le
\mathcal H_{j|i,T}(\ell_i),
\end{equation}
where \(\mathcal H_{j|i,T}\) is defined as follows. If \(p_i>p_j\), then
\[
% \label{eq:cac_H_lower_mean_case_capped}
\mathcal H_{j|i,T}(\ell_i)
:=
\begin{cases}
\displaystyle
\min\left\{
T,\,
\dfrac{
b_i(\ell_i)+p_i-\ell_i
}{
\delta_{ij}-\ell_i
}
\right\},
&
0<\ell_i<\delta_{ij},
\\[3ex]
T,
&
\ell_i=\delta_{ij},
\\[2ex]
\displaystyle
\min\left\{
T,\,
\exp\{-\theta_j(\ell_i-\delta_{ij})b_i(\ell_i)\}
\dfrac{
b_i(\ell_i)+p_i-\ell_i
}{
v_j(\ell_i-\delta_{ij})
}
\right\},
&
\delta_{ij}<\ell_i<p_i .
\end{cases}
\]
If \(p_j>p_i\), then
\[
% \label{eq:cac_H_higher_mean_case_capped}
\mathcal H_{j|i,T}(\ell_i)
:=
\min\left\{
T,\,
\exp\{-\theta_j(\ell_i+\delta_{ji})b_i(\ell_i)\}
\dfrac{
b_i(\ell_i)+p_i-\ell_i
}{
v_j(\ell_i+\delta_{ji})
}
\right\},
\qquad
0<\ell_i<p_i .
\]
If \(p_i=p_j\), then
\[
% \label{eq:cac_H_equal_mean_case_capped}
\mathcal H_{j|i,T}(\ell_i)
:=
T,
\qquad
0<\ell_i<p_i .
\]
\end{lemma}

\begin{proof}\leavevmode
We consider the three cases separately.

First suppose \(p_i>p_j\) and \(0<\ell_i<\delta_{ij}\). Then the comparison drift is positive:
\[
% \label{eq:cac_positive_drift_mu}
d_{j|i}(\ell_i)=\delta_{ij}-\ell_i>0 .
\]
By the strong law, \(\tau_{j|i}^{>}(\ell_i)<\infty\) almost surely. Since the increments are bounded and have strictly positive mean \(\delta_{ij}-\ell_i\), this first-passage time has finite expectation. Wald's identity, together with the terminal bound \eqref{eq:cac_terminal_state}, therefore gives
\[
(\delta_{ij}-\ell_i)\,
\mathbb E[\tau_{j|i}^{>}(\ell_i)]
=
\mathbb E\!\left[
U_{j|i,\tau_{j|i}^{>}(\ell_i)}(\ell_i)
\right]
\le
b_i(\ell_i)+p_i-\ell_i .
% \label{eq:cac_positive_drift_wald_bound}
\]
Therefore
\[
% \label{eq:cac_positive_drift_capped_bound}
\mathbb E\!\left[
\min\{T,\tau_{j|i}^{>}(\ell_i)\}
\right]
\le
\min\left\{
T,\,
\frac{b_i(\ell_i)+p_i-\ell_i}{\delta_{ij}-\ell_i}
\right\}.
\]

At the interface \(p_i>p_j\) and \(\ell_i=\delta_{ij}\), the comparison drift is zero. We do not use a zero-drift hitting-time expectation. The finite-horizon cap gives directly
\[
% \label{eq:cac_critical_capped_bound}
\mathbb E\!\left[
\min\{T,\tau_{j|i}^{>}(\delta_{ij})\}
\mathbf 1\{\tau_{j|i}^{>}(\delta_{ij})<\infty\}
\right]
\le
T .
\]

It remains to treat the negative-drift case \(0<y=\ell_i+p_j-p_i<p_j\). This includes \(\ell_i>\delta_{ij}\) when \(p_i>p_j\), all \(\ell_i\in(0,p_i)\) when \(p_j>p_i\), and all \(\ell_i>0\) when \(p_i=p_j\). The increment distribution of \(U_{j|i,n}(\ell_i)\) is the one-arm crossing distribution for arm \(j\) at level \(y\). 
Under the tilted measure generated by the martingale \(\exp\{\theta_j(y)U_{j|i,n}(\ell_i)\}\), the comparison walk has positive drift \(v_j(y)\). Since the tilted increments remain bounded, the first-passage time has finite expectation under the tilted measure. Wald's identity and \eqref{eq:cac_terminal_state} therefore give
\[
% \label{eq:cac_tilted_mean_time_bound}
\mathbb E^{\theta}\!\left[
\tau_{j|i}^{>}(\ell_i)
\right]
\le
\frac{b_i(\ell_i)+p_i-\ell_i}{v_j(y)} .
\]
Changing measure back at the crossing time gives
\[\begin{aligned}
\mathbb E\!\left[
\tau_{j|i}^{>}(\ell_i)
\mathbf 1\{\tau_{j|i}^{>}(\ell_i)<\infty\}
\right]&=
\mathbb E^{\theta}\!\left[
\tau_{j|i}^{>}(\ell_i)
\exp\left\{
-\theta_j(y)
U_{j|i,\tau_{j|i}^{>}(\ell_i)}(\ell_i)
\right\}
\right]
\nonumber\\
&\le
\exp\{-\theta_j(y)b_i(\ell_i)\}
\mathbb E^{\theta}\!\left[
\tau_{j|i}^{>}(\ell_i)
\right]
\nonumber\\
&\le
\exp\{-\theta_j(y)b_i(\ell_i)\}
\frac{b_i(\ell_i)+p_i-\ell_i}{v_j(y)} .
% \label{eq:cac_tilted_wald_weighted_bound}
\end{aligned}\]
Since the capped expectation is also at most \(T\), we obtain
\[\mathbb E\!\left[
\min\{T,\tau_{j|i}^{>}(\ell_i)\}
\mathbf 1\{\tau_{j|i}^{>}(\ell_i)<\infty\}
\right]
\le
\min\left\{
T,\,
\exp\{-\theta_j(y)b_i(\ell_i)\}
\frac{b_i(\ell_i)+p_i-\ell_i}{v_j(y)}
\right\}.
% \label{eq:cac_tilted_wald_capped_bound}
\]
For \(p_i=p_j\), the simpler bound \(\mathcal H_{j|i,T}(\ell_i)=T\) is weaker but sufficient and no shifted-level notation is needed. In the two cases where the negative-drift tilted estimate is used, the shifted level is \(y=\ell_i-\delta_{ij}\) when \(p_i>p_j\) and \(\ell_i>\delta_{ij}\), and \(y=\ell_i+\delta_{ji}\) when \(p_j>p_i\). Combining the positive-drift, interface, and negative-drift estimates proves \eqref{eq:cac_pointwise_capped_factor_bound}.   

\end{proof}

\begin{proposition}[Root-based integrated capped upper envelope]
\label{prop:cac_root_based_integrated_upper}
For every \(i\neq j\) and every finite horizon \(T\),
\begin{align}
\mathbb E[\min\{T,N_j\}\mid Q_i]
&\le
\frac{
\displaystyle
\int_{[0,p_i)}
\mathcal H_{j|i,T}(\ell_i)
\prod_{\substack{r=1\\r\neq i,j}}^{K}
\overline F_{r,\ge}(\ell_i+p_r-p_i)
\,dF_i(\ell_i)
+
T\mathbb P(\mathcal E_{ij}^{=})
}{
\displaystyle
\mathbb P(\mathcal W_i^>)
}.
\label{eq:cac_exact_integrated_upper}
\end{align}
Moreover, the pairwise equality probability admits the Stieltjes representation
\begin{align}
\mathbb P(\mathcal E_{ij}^{=})
=
\int_{[0,p_i)}
\left[
\overline F_{j,\ge}(\ell_i+p_j-p_i)
-
\overline F_{j,>}(\ell_i+p_j-p_i)
\right]
\times
\prod_{\substack{r=1\\r\neq i,j}}^{K}
\overline F_{r,\ge}(\ell_i+p_r-p_i)
\,dF_i(\ell_i).
\label{eq:cac_pairwise_tie_probability_stieltjes}
\end{align}
Shifted levels below zero contribute equal strict and weak tail factors, and
therefore make zero contribution to the difference in
\eqref{eq:cac_pairwise_tie_probability_stieltjes}.
\end{proposition}

\begin{proof}\leavevmode
By \eqref{eq:cac_capped_winner_conditioned_upper_ratio}, it is enough to bound
the one-arm capped factor
\[
% \label{eq:cac_time_factor_recall}
\mathbb E\!\left[
\min\{T,\tau_{j|i}^{>}(\ell_i)\}
\mathbf 1\{\tau_{j|i}^{>}(\ell_i)<\infty\}
\right].
\]
Lemma~\ref{lem:cac_pointwise_capped_first_passage} gives the pointwise bound
by \(\mathcal H_{j|i,T}(\ell_i)\). Substituting this pointwise bound into
\eqref{eq:cac_capped_winner_conditioned_upper_ratio} proves
\eqref{eq:cac_exact_integrated_upper}.

It remains to derive \eqref{eq:cac_pairwise_tie_probability_stieltjes}. By
definition,
\(
% \label{eq:cac_pairwise_tie_event_recall}
\mathcal E_{ij}^{=}
=
\mathcal W_i^{\ge}\cap\{M_j=M_i\}
\). 
Conditioning on \(L_i=\ell_i\), the event \(M_j=M_i\) is equivalent to
\(
% \label{eq:cac_pairwise_tie_L_form}
L_j=\ell_i+p_j-p_i
\). 
The remaining competitors must satisfy the weak winner inequalities
\[
% \label{eq:cac_pairwise_remaining_weak_constraints}
L_r\ge \ell_i+p_r-p_i,
\qquad r\neq i,j.
\]
By independence of the potential reward streams, we obtain
\[\begin{aligned}
\mathbb P(\mathcal E_{ij}^{=})
&=
\int_{[0,p_i)}
\mathbb P\left(L_j=\ell_i+p_j-p_i\right)
\prod_{\substack{r=1\\r\neq i,j}}^{K}
\overline F_{r,\ge}(\ell_i+p_r-p_i)
\,dF_i(\ell_i)
\nonumber\\
&=
\int_{[0,p_i)}
\left[
\overline F_{j,\ge}(\ell_i+p_j-p_i)
-
\overline F_{j,>}(\ell_i+p_j-p_i)
\right]
\times
\prod_{\substack{r=1\\r\neq i,j}}^{K}
\overline F_{r,\ge}(\ell_i+p_r-p_i)
\,dF_i(\ell_i).
% \label{eq:cac_pairwise_tie_probability_stieltjes_proof}
\end{aligned}\]
This proves \eqref{eq:cac_pairwise_tie_probability_stieltjes}.

\end{proof}

\subsection{Winner-Conditioned Endpoint Localization}
\label{app:cac_winner_conditioned_localization}

We next record the endpoint localization of the absorbing arm under winner conditioning. This result controls moments and tail probabilities of \(L_i\) under the winner-weighted Stieltjes ratios appearing in the capped abandonment envelope. The resulting pairwise abandonment estimates are derived in the next subsection.

The following arguments apply under Assumption~\ref{ass:app_asymptotic_regime} to every absorbing arm \(i\). Since \(\kappa<1\), we have
\(
\Delta_1^{-1}=o(\eta_\Delta)
\). 
Thus every nonzero pairwise gap is asymptotically larger than the \(O(\Delta_1^{-1})\) endpoint window on which \(L_i\) is localized. In particular, on \(0\le\ell\le a_0/\Delta_1\), a lower-mean competitor \(r\) with \(p_i>p_r\) satisfies
\(
\ell+p_r-p_i<0
\)
for all sufficiently large \(\Delta_1\), so its tail factor remains equal to one. At the same time, every positive shifted level arising on this window remains within the lower-endpoint scale \(\eta_\Delta\), and hence the local root and tail expansions apply. Equal-mean competitors are permitted and are treated separately below.

For \(\circ\in\{>,\ge\}\), define the winner weight
\[
% \label{eq:cac_winner_weight_def}
\Phi_i^{\circ}(\ell)
:=
\prod_{r\neq i}
\overline F_{r,\circ}(\ell+p_r-p_i),
\]
where shifted levels below zero contribute tail factor one. 

Let
\[
% \label{eq:cac_strict_winner_denominator_def}
\mathcal D_i^{>}
:=
\mathbb P(\mathcal W_i^>)
=
\int_{[0,p_i)}
\Phi_i^{>}(\ell)\,dF_i(\ell).
\]

\begin{lemma}[Winner-weighted endpoint localization]
\label{lem:cac_winner_conditioned_endpoint_localization}
Under Assumption~\ref{ass:app_asymptotic_regime}, for every arm \(i\), there exist constants \(C,c>0\), independent of \(\Delta_1\), such that\begin{equation}
\label{eq:cac_winner_weighted_tail}
\frac{
\displaystyle
\int_x^{p_i}
\Phi_i^{\ge}(\ell)\,dF_i(\ell)
}{
\mathcal D_i^{>}
}
\le
C \exp\{-c\Delta_1x\},
\qquad
0\le x\le\eta_\Delta .
\end{equation}
Consequently, for every fixed integer \(m\ge1\) and every \(a\in(0,\eta_\Delta]\),
\begin{equation}
\label{eq:cac_winner_weighted_truncated_moment}
\frac{
\displaystyle
\int_0^a
\ell^m \Phi_i^{\ge}(\ell)\,dF_i(\ell)
}{
\mathcal D_i^{>}
}
=
O(\Delta_1^{-m}).
\end{equation}
The same two estimates hold with \(\Phi_i^{>}\) in place of \(\Phi_i^{\ge}\) in the
numerator. In particular,
\begin{equation}
\label{eq:cac_winner_conditioned_tail}
\mathbb P(L_i>x\mid Q_i)
\le
C \exp\{-c\Delta_1x\},
\qquad
0\le x\le\eta_\Delta ,
\end{equation}
and
\begin{equation}
\label{eq:cac_winner_conditioned_truncated_moment}
\mathbb E\!\left[
L_i^m\mathbf 1\{L_i\le a\}
\,\middle|\,
Q_i
\right]
=
O(\Delta_1^{-m}).
\end{equation}
\end{lemma}

\begin{proof}\leavevmode
Recall that, for \(\circ\in\{>,\ge\}\),
\[
\Phi_i^{\circ}(\ell)
=
\prod_{r\neq i}
\overline F_{r,\circ}(\ell+p_r-p_i),
\]
where shifted levels below zero contribute tail factor one. Also recall
\[
\mathcal D_i^{>}
=
\mathbb P(\mathcal W_i^>)
=
\int_{[0,p_i)}
\Phi_i^{>}(\ell)\,dF_i(\ell).
\]

\paragraph{Numerator upper bound.}
Since every shifted level \(\ell+p_r-p_i\) is nondecreasing in \(\ell\), each
one-arm tail factor is nonincreasing in \(\ell\). Hence
\[
% \label{eq:cac_H_monotone}
\Phi_i^{\ge}(\ell)
\le
\Phi_i^{\ge}(0),
\qquad
\ell\ge0 .
\]
We compare the weak and strict weights at zero. If \(p_r-p_i\le0\), both tail
factors are one. If \(p_r-p_i>0\), Corollary~\ref{cor:root_based_drawdown_tail_envelope}
gives
\[\begin{aligned}
\frac{
\overline F_{r,\ge}(p_r-p_i)
}{
\overline F_{r,>}(p_r-p_i)
}
&\le
\exp\{\theta_r(p_r-p_i)(p_r-(p_r-p_i))\}
\nonumber\\
&=
\exp\{\theta_r(p_r-p_i)p_i\}.
% \label{eq:cac_strict_weak_single_factor_ratio_pre}
\end{aligned}\]
The Lundberg equation at level \(p_r-p_i\) implies
\(
% \label{eq:cac_strict_weak_root_bound}
q_r \exp\{\theta_r(p_r-p_i)p_i\}
\le 1
\), 
and therefore
\[
% \label{eq:cac_strict_weak_single_factor_ratio}
\frac{
\overline F_{r,\ge}(p_r-p_i)
}{
\overline F_{r,>}(p_r-p_i)
}
\le
\frac{1}{q_r}.
\]
Since \(K\) is fixed, multiplying the factorwise bounds yields
\(
% \label{eq:cac_strict_weak_weight_comparison}
\Phi_i^{\ge}(0)
\le
C \Phi_i^{>}(0)
\). 
Using the one-arm lower-endpoint tail envelope for \(L_i\), uniformly for
\(0\le x\le\eta_\Delta\),
\(
% \label{eq:cac_unconditional_Li_tail_local}
\mathbb P(L_i>x)
\le
C \exp\{-c\Delta_1x\}
\). 
Therefore,
\begin{align}
\int_x^{p_i}\Phi_i^{\ge}(\ell)\,dF_i(\ell)
\le
\Phi_i^{\ge}(0)\mathbb P(L_i>x)
\le
C \Phi_i^{>}(0)\exp\{-c\Delta_1x\}.
\label{eq:cac_numerator_tail_bound}
\end{align}

\paragraph{Denominator lower bound.}
We now prove that the denominator is not much smaller than the zero-level
winner weight. More precisely, we show that there exists \(c>0\), independent
of \(\Delta_1\), such that
\begin{equation}
\label{eq:cac_H_lower_small_window}
\Phi_i^{>}(\ell)
\ge
c\Phi_i^{>}(0),
\qquad
0\le \ell\le \frac{a_0}{\Delta_1},
\end{equation}
where \(a_0>0\) is a fixed constant to be chosen later. Since \(\Phi_i^{>}\) is a
product over competitors, it is enough to show that each factor
\(
\overline F_{r,>}(\ell+p_r-p_i)
\)
is bounded below by a positive constant times its value at \(\ell=0\),
uniformly over \(0\le\ell\le a_0/\Delta_1\).

First consider a lower-mean competitor \(r\) with \(p_i>p_r\). The shifted
level is
\(
% \label{eq:cac_lower_competitor_shift}
\ell+p_r-p_i
=
\ell-\delta_{ir}
\). 
The purpose here is to show that this factor stays equal to one throughout the
small window. By \eqref{eq:app_asymptotic_gap_scale}, \(\delta_{ir}=\Theta(\eta_\Delta)\). Since \(\eta_\Delta=\Delta_1^{-\kappa}\) by \eqref{eq:app_asymptotic_eta_def} and \(\kappa<1\), \(\Delta_1\delta_{ir}=\Theta(\Delta_1^{1-\kappa})\to\infty\). Hence, for every fixed \(a_0\),
\(
% \label{eq:cac_lower_competitor_window_below_gap}
{a_0}/{\Delta_1}
<
\delta_{ir}
\) 
for all sufficiently large \(\Delta_1\). Therefore,
\[
% \label{eq:cac_lower_competitor_factor_one}
\ell+p_r-p_i\le0,
\qquad
0\le\ell\le\frac{a_0}{\Delta_1}.
\]
By the nonpositive-level convention, the corresponding strict tail factor is
one throughout the window, and it is also one at \(\ell=0\). Thus lower-mean
competitors cause no loss in \eqref{eq:cac_H_lower_small_window}.

Next consider an equal-mean competitor \(r\) with \(p_r=p_i\). The shifted level
is \(\ell\), and the goal is to show that this small positive shift only
changes the strict tail by a constant factor. Since
\(\overline F_{r,>}(0)=1\), it is enough to lower-bound
\(\overline F_{r,>}(\ell)\). By the lower side of
Corollary~\ref{cor:root_based_drawdown_tail_envelope},
\begin{equation}
\label{eq:cac_equal_factor_root_lower}
\overline F_{r,>}(\ell)
\ge
\exp\left\{
-\theta_r(\ell)\bigl[b_r(\ell)+p_r-\ell\bigr]
\right\}.
\end{equation}
On \(0\le\ell\le a_0/\Delta_1\), the lower-endpoint expansion gives
\(
% \label{eq:cac_equal_theta_local_bound}
\theta_r(\ell)\le C\ell
\). 
Moreover, using \(\Delta_r=\Theta(\Delta_1)\), \(\beta=O(\Delta_1)\), and
\(\ell\le a_0/\Delta_1\),
\begin{align}
b_r(\ell)+p_r-\ell
=
\Delta_r+\beta\ell+p_r-\ell
\le
C\Delta_1 .
\label{eq:cac_equal_boundary_bound}
\end{align}
Combining \eqref{eq:cac_equal_factor_root_lower}--\eqref{eq:cac_equal_boundary_bound},
we obtain a constant
\(
c_{\mathrm{eq}}:=\exp\{-C_{\mathrm{eq}}a_0\}>0
\)
such that
\[
% \label{eq:cac_equal_mean_factor_lower}
\overline F_{r,>}(\ell)
\ge
c_{\mathrm{eq}},
\qquad
0\le\ell\le\frac{a_0}{\Delta_1}.
\]
Thus equal-mean competitors also lose only a constant factor.

Finally consider a higher-mean competitor \(r\) with \(p_r>p_i\). The shifted
level moves from \(\delta_{ri}=p_r-p_i\) to \(\delta_{ri}+\ell\). Our goal is
to prove that this \(O(\Delta_1^{-1})\) shift changes the strict tail only by a
constant factor:
\begin{equation}
\label{eq:cac_higher_factor_goal}
\overline F_{r,>}(\delta_{ri}+\ell)
\ge
c_{\mathrm{high},r}\,
\overline F_{r,>}(\delta_{ri}),
\qquad
0\le\ell\le\frac{a_0}{\Delta_1}.
\end{equation}
By \eqref{eq:app_asymptotic_gap_scale}, \(\delta_{ri}=\Theta(\eta_\Delta)\). Moreover, since \(\kappa<1\), \({a_0}/{\Delta_1}=o(\eta_\Delta)\). Therefore,
\[
% \label{eq:cac_higher_shifted_level_local}
0<\delta_{ri}\le \delta_{ri}+\ell
\le
C_{\mathrm{loc}}\eta_\Delta,
\qquad
0\le\ell\le\frac{a_0}{\Delta_1},
\]
for all sufficiently large \(\Delta_1\). Thus the whole shifted-level interval
lies in the lower-endpoint region. In particular,
Lemma~\ref{lem:lower_endpoint_theta_v_expansion} gives, uniformly over this
interval,
\begin{equation}
\label{eq:cac_higher_theta_v_local}
\theta_r(y)\le C_\theta y,
\qquad
v_r(y)\ge c_v y,
\qquad
\frac{\theta_r(y)}{v_r(y)}\le C_{\theta/v},
\qquad
0<y\le C_{\mathrm{loc}}\eta_\Delta.
\end{equation}
Also, since \(p_r-(\delta_{ri}+\ell)=p_i-\ell\), we have
\[
% \label{eq:cac_higher_away_from_upper}
p_r-(\delta_{ri}+\ell)
=
p_i-\ell
\ge
\frac{p_i}{2},
\qquad
0\le\ell\le\frac{a_0}{\Delta_1},
\]
for all sufficiently large \(\Delta_1\). Hence the upper endpoint is not
involved.

By the two-sided root envelope,
\begin{align}
\overline F_{r,>}(\delta_{ri}+\ell)
&\ge
\exp\left\{
-\theta_r(\delta_{ri}+\ell)
\bigl[
b_r(\delta_{ri}+\ell)+p_i-\ell
\bigr]
\right\},
\label{eq:cac_higher_numerator_lower}\\
\overline F_{r,>}(\delta_{ri})
&\le
\exp\left\{
-\theta_r(\delta_{ri})b_r(\delta_{ri})
\right\}.
\label{eq:cac_higher_denominator_upper}
\end{align}
Dividing \eqref{eq:cac_higher_numerator_lower} by
\eqref{eq:cac_higher_denominator_upper} gives
\begin{align}
\frac{\overline F_{r,>}(\delta_{ri}+\ell)}
{\overline F_{r,>}(\delta_{ri})}
&\ge
\exp\left\{
-\left[
\theta_r(\delta_{ri}+\ell)b_r(\delta_{ri}+\ell)
-
\theta_r(\delta_{ri})b_r(\delta_{ri})
\right]
\right\}
\nonumber\\
&\quad\times
\exp\left\{
-\theta_r(\delta_{ri}+\ell)(p_i-\ell)
\right\}.
\label{eq:cac_higher_factor_ratio}
\end{align}
We bound the two exponential factors separately.

For the first factor, differentiate the product \(\theta_r(y)b_r(y)\).
Differentiating the root equation \(\Lambda_r(\theta_r(y),y)=0\) gives
\[
% \label{eq:cac_theta_prime_identity}
\theta_r'(y)
=
\frac{\theta_r(y)}{v_r(y)},
\]
because \(\partial_y\Lambda_r(\theta,y)=-\theta\) and
\(\partial_\theta\Lambda_r(\theta_r(y),y)=v_r(y)\). Therefore
\begin{align}
\frac{d}{dy}\{\theta_r(y)b_r(y)\}
=
\theta_r'(y)b_r(y)+\theta_r(y)b_r'(y)
=
\theta_r(y)
\left[
\frac{b_r(y)}{v_r(y)}
+
\beta
\right].
\label{eq:cac_A_derivative_identity}
\end{align}
On \(0<y\le C_{\mathrm{loc}}\eta_\Delta\), \eqref{eq:cac_higher_theta_v_local}
gives
\(
\theta_r(y)/v_r(y)\le C_{\theta/v}
\) and 
\(
\theta_r(y)\le C_\theta\eta_\Delta
\). 
Moreover,
\[
% \label{eq:cac_higher_boundary_scale}
b_r(y)=\Delta_r+\beta y=O(\Delta_1),
\qquad
\beta=O(\Delta_1).
\]
Substituting these bounds into \eqref{eq:cac_A_derivative_identity}, we obtain
a constant \(C_{\mathrm{der}}<\infty\) such that
\[
% \label{eq:cac_A_derivative_bound}
\sup_{\delta_{ri}\le y\le \delta_{ri}+a_0/\Delta_1}
\left|
\frac{d}{dy}\{\theta_r(y)b_r(y)\}
\right|
\le
C_{\mathrm{der}}\Delta_1 .
\]
Therefore,
\[\begin{aligned}
\theta_r(\delta_{ri}+\ell)b_r(\delta_{ri}+\ell)
-
\theta_r(\delta_{ri})b_r(\delta_{ri})
&=
\int_{\delta_{ri}}^{\delta_{ri}+\ell}
\frac{d}{dy}\{\theta_r(y)b_r(y)\}\,dy
\nonumber\\
&\le
C_{\mathrm{der}}\Delta_1\ell
\le
C_{\mathrm{der}}a_0 .
% \label{eq:cac_A_increment_bound}
\end{aligned}\]
Thus the first exponential factor in \eqref{eq:cac_higher_factor_ratio} is
bounded below by
\begin{equation}
\label{eq:cac_first_exponential_lower}
\exp\{-C_{\mathrm{der}}a_0\}.
\end{equation}

For the second factor, the Lundberg equation implies
\[
% \label{eq:cac_positive_jump_exponent_bound}
q_r \exp\{\theta_r(y)(p_r-y)\}
\le
1,
\qquad
0<y<p_r.
\]
With \(y=\delta_{ri}+\ell\), this becomes
\begin{equation}
\label{eq:cac_second_exponential_lower}
\exp\{-\theta_r(\delta_{ri}+\ell)(p_i-\ell)\}
\ge
q_r .
\end{equation}
Combining \eqref{eq:cac_first_exponential_lower} and
\eqref{eq:cac_second_exponential_lower}, we get
\[
% \label{eq:cac_higher_factor_ratio_lower_final}
\frac{\overline F_{r,>}(\delta_{ri}+\ell)}
{\overline F_{r,>}(\delta_{ri})}
\ge
q_r \exp\{-C_{\mathrm{der}}a_0\},
\qquad
0\le\ell\le\frac{a_0}{\Delta_1}.
\]
Thus \eqref{eq:cac_higher_factor_goal} holds with
\[
c_{\mathrm{high},r}:=q_r \exp\{-C_{\mathrm{der}}a_0\}.
\]

Combining the lower-mean, equal-mean, and higher-mean competitor cases, define
\[
c_\Phi
:=
\left(
\min\left\{
1,\,
c_{\mathrm{eq}},\,
\min_{\substack{r\neq i\\p_r>p_i}}c_{\mathrm{high},r}
\right\}
\right)^{K-1},
\]
with the convention that the minimum over an empty set is \(1\). Here
\(c_{\mathrm{eq}}\) is the constant factor obtained in the equal-mean case.
Then every competitor factor in \(\Phi_i^{>}(\ell)\) is bounded below by its
zero-level value times the corresponding constant factor. Hence
\[
% \label{eq:cac_H_lower_small_window_conclusion}
\Phi_i^{>}(\ell)
\ge
c_\Phi \Phi_i^{>}(0),
\qquad
0\le\ell\le\frac{a_0}{\Delta_1},
\]
which proves \eqref{eq:cac_H_lower_small_window}.

It remains to convert this local weight bound into a denominator bound. Using
\eqref{eq:cac_H_lower_small_window},
\[\begin{aligned}
\mathcal D_i^{>}
=
\int_{[0,p_i)}
\Phi_i^{>}(\ell)\,dF_i(\ell)
\ge
\int_0^{a_0/\Delta_1}
\Phi_i^{>}(\ell)\,dF_i(\ell)
\ge
c_\Phi\Phi_i^{>}(0)
\mathbb P\!\left(L_i\le\frac{a_0}{\Delta_1}\right).
% \label{eq:cac_denominator_lower_pre}
\end{aligned}\]
Since \(a_0/\Delta_1\le\eta_\Delta\) for all sufficiently large \(\Delta_1\), the
local one-arm tail envelope gives
\[
% \label{eq:cac_small_window_tail}
\mathbb P\!\left(L_i>\frac{a_0}{\Delta_1}\right)
=
\exp\left\{
-\lambda_i\frac{a_0}{\Delta_1}
\right\}
\left[
1+O(\Delta_1^{1-2\kappa})
\right].
\]
Because \(\lambda_i=\Theta(\Delta_1)\), choosing \(a_0\) sufficiently large
yields
\[
% \label{eq:cac_small_window_mass_lower}
\mathbb P\!\left(L_i\le\frac{a_0}{\Delta_1}\right)
\ge
\frac12
\]
for all large \(\Delta_1\). Therefore
\begin{equation}
\label{eq:cac_denominator_lower}
\mathcal D_i^{>}
\ge
c_{\mathcal D}\Phi_i^{>}(0),
\qquad
c_{\mathcal D}:=\frac{c_\Phi}{2}.
\end{equation}

\paragraph{Conclusion.}
Combining \eqref{eq:cac_numerator_tail_bound} and
\eqref{eq:cac_denominator_lower} gives \eqref{eq:cac_winner_weighted_tail}. The
same tail bound with \(\Phi_i^{>}\) in the numerator follows because
\(\Phi_i^{>}(\ell)\le \Phi_i^{\ge}(\ell)\).

For the moment bound, first note that for \(0\le \ell\le a\),
\(
% \label{eq:cac_layer_cake_pointwise_identity}
\ell^m
=
m\int_0^a x^{m-1}\mathbf 1\{x\le \ell\}\,dx
\). 
Therefore, by Tonelli's theorem,
\[\begin{aligned}
\frac{
\displaystyle
\int_0^a
\ell^m \Phi_i^{\ge}(\ell)\,dF_i(\ell)
}{
\mathcal D_i^{>}
}
&=
\frac{
\displaystyle
\int_0^a
\left[
m\int_0^a x^{m-1}\mathbf 1\{x\le \ell\}\,dx
\right]
\Phi_i^{\ge}(\ell)\,dF_i(\ell)
}{
\mathcal D_i^{>}
}
\nonumber\\
&=
m\int_0^a
x^{m-1}
\frac{
\displaystyle
\int_x^a
\Phi_i^{\ge}(\ell)\,dF_i(\ell)
}{
\mathcal D_i^{>}
}
\,dx
\nonumber\\
&\le
m\int_0^a
x^{m-1}
\frac{
\displaystyle
\int_x^{p_i}
\Phi_i^{\ge}(\ell)\,dF_i(\ell)
}{
\mathcal D_i^{>}
}
\,dx .
% \label{eq:cac_layer_cake_weighted_moment}
\end{aligned}\]
Using \eqref{eq:cac_winner_weighted_tail},
\[\begin{aligned}
\frac{
\displaystyle
\int_0^a
\ell^m \Phi_i^{\ge}(\ell)\,dF_i(\ell)
}{
\mathcal D_i^{>}
}
&\le
Cm\int_0^a x^{m-1}\exp\{-c\Delta_1x\}\,dx
\nonumber\\
&\le
Cm\int_0^\infty x^{m-1}\exp\{-c\Delta_1x\}\,dx
=
O(\Delta_1^{-m}).
% \label{eq:cac_winner_weighted_moment_final}
\end{aligned}\]
This proves \eqref{eq:cac_winner_weighted_truncated_moment}. The same moment
bound with \(\Phi_i^{>}\) in the numerator follows again from
\(\Phi_i^{>}(\ell)\le \Phi_i^{\ge}(\ell)\).

Finally, by the score-minimum sandwich,
\[\begin{aligned}
\mathbb P(L_i>x\mid Q_i)
&\le
\frac{
\displaystyle
\int_x^{p_i}
\Phi_i^{\ge}(\ell)\,dF_i(\ell)
}{
\mathcal D_i^{>}
},
% \label{eq:cac_cond_tail_ratio_start}
\end{aligned}\]
which gives \eqref{eq:cac_winner_conditioned_tail}. The conditional moment
bound \eqref{eq:cac_winner_conditioned_truncated_moment} follows from
\eqref{eq:cac_winner_conditioned_tail} by the same layer-cake argument.

\end{proof}

The same estimates hold with \(\Phi_i^{>}\) in place of \(\Phi_i^{\ge}\) in the numerator, because \(\Phi_i^{>}(\ell)\le \Phi_i^{\ge}(\ell)\). Thus, in later applications, either strict or weak winner-weighted numerator integrals over lower-endpoint windows have moments of order \(O(\Delta_1^{-m})\) after normalization by \(\mathcal D_i^{>}=\mathbb P(\mathcal W_i^>)\).

\subsection{Asymptotic Conditional-Abandonment Envelope}
\label{app:cac_asymptotic_closed_form}

We now evaluate the root-based upper envelope in Proposition~\ref{prop:cac_root_based_integrated_upper} under Assumption~\ref{ass:app_asymptotic_regime} for every pair \(i\neq j\). The pointwise bounds from Lemma~\ref{lem:cac_pointwise_capped_first_passage} are integrated against the winner weights, while Lemma~\ref{lem:cac_winner_conditioned_endpoint_localization} controls the resulting endpoint moments that arise from this integration.

Throughout this subsection, \(T\) denotes the finite-horizon cap in \(\min\{T,N_j\}\) and is kept as an external parameter. All \(O(\cdot)\) and \(o(\cdot)\) terms refer to the limit \(\Delta_1\to\infty\), with every displayed factor of \(T\) left explicit.

By \eqref{eq:app_asymptotic_gap_scale}, every nonzero pairwise gap is of order \(\eta_\Delta\). Hence all positive shifted levels entering the local Stieltjes replacements remain within the lower-endpoint region, while \(\Delta_1\eta_\Delta\to\infty\) separates every nonzero comparison gap from the \(O(\Delta_1^{-1})\) endpoint window of the absorbing arm. When \(p_i=p_j\), no nonzero drift scale is available, and the finite-horizon cap is used directly.

\begin{proposition}[Asymptotic capped conditional-abandonment upper envelope]
\label{prop:cac_asymptotic_conditional_abandonment_upper}
Under Assumption~\ref{ass:app_asymptotic_regime}, for every pair \(i\neq j\) and every finite-horizon parameter \(T\), if \(p_i=p_j\), then
\begin{equation}
\label{eq:cac_equal_mean_trivial_cap}
\mathbb E[N_j(T)\mid Q_i]
\le
\mathbb E[\min\{T,N_j\}\mid Q_i]
\le
T .
\end{equation}
If \(p_i\neq p_j\), then 
\begin{equation}
\label{eq:cac_piecewise_upper_main}
\mathbb E[N_j(T)\mid Q_i]
\le
\mathbb E[\min\{T,N_j\}\mid Q_i]
\le
\begin{cases}
\displaystyle
\frac{\Delta_i}{\delta_{ij}}
\left(1+O(\Delta_1^{\kappa-1})\right)
+
T\,o\!\left(\frac{\mathbb P(Q_j)}{\mathbb P(Q_i)}\right),
&
p_i>p_j,
\\[2.8ex]
\displaystyle
\frac{\Delta_i}{\delta_{ji}}
\left(1+o(1)\right)
+
T\,o(1),
&
p_j>p_i .
\end{cases}
\end{equation}
Terms conditioned on probability-zero events are interpreted as zero
contributions.
\end{proposition}

\begin{proof}\leavevmode
The first inequality follows from \(N_j(T)\le \min\{T,N_j\}\). If \(p_i=p_j\),
then the finite-horizon cap gives directly
\(
% \label{eq:cac_equal_finite_horizon_bound}
\mathbb E[\min\{T,N_j\}\mid Q_i]\le T
\), 
which proves \eqref{eq:cac_equal_mean_trivial_cap}. Hence, for the rest of the
proof, assume \(p_i\neq p_j\).

By \eqref{eq:app_asymptotic_gap_scale} and \eqref{eq:app_asymptotic_eta_def}, every nonzero comparison gap is of order \(\eta_\Delta\), and hence is separated from the \(O(\Delta_1^{-1})\) endpoint window because \(\kappa<1\). All positive shifted levels appearing on the lower-endpoint windows are also \(O(\eta_\Delta)\). Therefore, Lemma~\ref{lem:cac_winner_conditioned_endpoint_localization} applies to the winner-weighted Stieltjes ratios below.

We prove the noncritical bounds in \eqref{eq:cac_piecewise_upper_main} by
evaluating the root-based integrated envelope in
Proposition~\ref{prop:cac_root_based_integrated_upper}.
For a Borel set \(B\subseteq[0,p_i)\), write
\[
% \label{eq:cac_asymptotic_integral_piece_def}
\mathcal I_{j|i,T}(B)
:=
\int_B
\mathcal H_{j|i,T}(\ell)
\prod_{\substack{r=1\\r\neq i,j}}^{K}
\overline F_{r,\ge}(\ell+p_r-p_i)
\,dF_i(\ell),
\]
and
\[
% \label{eq:cac_asymptotic_denominator_def}
\mathcal D_i^{>}
:=
\mathbb P(\mathcal W_i^>)
=
\int_{[0,p_i)}
\prod_{r\neq i}
\overline F_{r,>}(\ell+p_r-p_i)
\,dF_i(\ell).
\]
The strict and weak winner Stieltjes replacements for arm \(i\) have the same
leading order. Since
\(
% \label{eq:cac_winner_sandwich_recall}
\mathcal W_i^>
\subseteq
Q_i
\subseteq
\mathcal W_i^\ge
\), 
we have
\(
% \label{eq:cac_Qi_Di_equivalent}
\mathbb P(Q_i)
=
\mathcal D_i^{>}(1+o(1))
\). 
Thus ratios with denominator \(\mathcal D_i^{>}\) and ratios with denominator
\(\mathbb P(Q_i)\) are interchangeable at the precision used below.

We repeatedly use the weighted endpoint moment bound from
Lemma~\ref{lem:cac_winner_conditioned_endpoint_localization}: for every
\(a\le\eta_\Delta\) and every fixed integer \(m\ge0\),
\begin{equation}
\label{eq:cac_winner_weighted_moment_ratio}
\frac{
\displaystyle
\int_0^a
\ell^m
\prod_{r\neq i}
\overline F_{r,\ge}(\ell+p_r-p_i)
\,dF_i(\ell)
}{
\mathcal D_i^{>}
}
=
O(\Delta_1^{-m}),
\end{equation}
where \(m=0\) means an \(O(1)\) bound.

\paragraph{Positive-drift branch: \(p_i>p_j\).}
Recall \(\delta_{ij}:=p_i-p_j>0\). By \eqref{eq:app_asymptotic_gap_scale} and \eqref{eq:app_asymptotic_eta_def},
\[
% \label{eq:cac_positive_gap_scale}
\delta_{ij}
=
\Theta(\eta_\Delta),
\qquad
\Delta_1\delta_{ij}\to\infty,
\qquad
\Delta_1\delta_{ij}^2\to0 .
\]

Choose
\[
% \label{eq:cac_positive_singular_window_choice}
\omega_\Delta:=\Delta_1^{-1-\xi},
\qquad
0<\xi<\kappa .
\]
Then
\[
% \label{eq:cac_positive_singular_window_scale}
\omega_\Delta=o(\delta_{ij}),
\qquad
\omega_\Delta=o(\eta_\Delta),
\qquad
\Delta_1\omega_\Delta=o(1).
\]

Fix a small constant
\(
% \label{eq:cac_positive_epsilon_choice}
0<\varepsilon<\frac12\min\{p_i,p_j\}
\). 
For all sufficiently large \(\Delta_1\), we split interval
\([0,p_i)\) into following pieces:
\begin{itemize}
\item Bulk region \(0\le \ell\le\delta_{ij}/2\);
\item Left noncritical region
\(\delta_{ij}/2<\ell<\delta_{ij}-\omega_\Delta\);
\item Singular layer
\(|\ell-\delta_{ij}|\le\omega_\Delta\);
\item Right noncritical region
\(\delta_{ij}+\omega_\Delta<\ell\le p_i-\varepsilon\);
\item Upper-endpoint region \(p_i-\varepsilon<\ell<p_i\).
\end{itemize}
We evaluate the bulk explicitly, show that the two noncritical regions and the
upper-endpoint region are negligible, and then control the singular layer by a
finite-horizon cap.

\emph{Bulk region \(0\le\ell\le\delta_{ij}/2\).}
On this interval,
\[
% \label{eq:cac_positive_bulk_j_factor_absent}
\ell+p_j-p_i
=
\ell-\delta_{ij}
\le
-\frac{\delta_{ij}}{2}<0,
\]
so the \(j\)-factor in the winner weight is equal to one. By the
positive-drift part of Lemma~\ref{lem:cac_pointwise_capped_first_passage},
\[
% \label{eq:cac_positive_bulk_wald_integrated}
\mathcal H_{j|i,T}(\ell)
\le
\frac{\Delta_i+\beta\ell+p_i}{\delta_{ij}-\ell},
\qquad
0\le\ell\le\frac{\delta_{ij}}{2}.
\]
For this display, set
\[
% \label{eq:cac_fij_def}
f_{ij}(\ell)
:=
\frac{\Delta_i+\beta\ell+p_i}{\delta_{ij}-\ell}.
\]
Then
\[
f_{ij}(0)
=
\frac{\Delta_i+p_i}{\delta_{ij}},
\qquad
% \label{eq:cac_fij_zero}
f_{ij}'(0)
=
\frac{\beta}{\delta_{ij}}
+
\frac{\Delta_i+p_i}{\delta_{ij}^2}
=
O\!\left(\frac{\Delta_1}{\delta_{ij}^2}\right),
\qquad
% \label{eq:cac_fij_prime_zero}
\sup_{0\le \ell\le\delta_{ij}/2}|f_{ij}''(\ell)|
=
O\!\left(\frac{\Delta_1}{\delta_{ij}^3}\right).
% \label{eq:cac_fij_second_bound}
\]
Using Taylor's formula and \eqref{eq:cac_winner_weighted_moment_ratio},
\begin{align}
\frac{\mathcal I_{j|i,T}([0,\delta_{ij}/2])}{\mathcal D_i^{>}}
&\le
\frac{\Delta_i}{\delta_{ij}}
+
O(\delta_{ij}^{-1})
+
O(\delta_{ij}^{-2})
+
O\!\left(\frac{1}{\Delta_1\delta_{ij}^3}\right)
\nonumber\\
&=
\frac{\Delta_i}{\delta_{ij}}
\left(1+O(\Delta_1^{\kappa-1})\right).
\label{eq:cac_positive_bulk_final}
\end{align}

\emph{Left noncritical region
\(\delta_{ij}/2<\ell<\delta_{ij}-\omega_\Delta\).}
On this region the comparison drift is still positive, and
\(
% \label{eq:cac_left_noncritical_gap_lower}
\delta_{ij}-\ell
\ge
\omega_\Delta
\). 
Moreover, the shifted level of arm \(j\) is nonpositive:
\(
% \label{eq:cac_left_noncritical_j_factor_one}
\ell+p_j-p_i
=
\ell-\delta_{ij}
<
0
\). 
Thus the \(j\)-factor in the winner weight is equal to one. The positive-drift part of Lemma~\ref{lem:cac_pointwise_capped_first_passage} gives
\[\begin{aligned}
\mathcal H_{j|i,T}(\ell)
\le
\frac{\Delta_i+\beta\ell+p_i-\ell}{\delta_{ij}-\ell}
\le
\frac{C\Delta_1}{\omega_\Delta},
\qquad
\frac{\delta_{ij}}{2}
<
\ell
<
\delta_{ij}-\omega_\Delta .
% \label{eq:cac_left_noncritical_pointwise_bound}
\end{aligned}\]
Therefore,
\begin{align}
\frac{
\mathcal I_{j|i,T}
\bigl((\delta_{ij}/2,\delta_{ij}-\omega_\Delta)\bigr)
}{
\mathcal D_i^{>}
}&\le
\frac{C\Delta_1}{\omega_\Delta}
\frac{
\displaystyle
\int_{\delta_{ij}/2}^{\delta_{ij}-\omega_\Delta}
\prod_{\substack{r=1\\r\neq i,j}}^{K}
\overline F_{r,\ge}(\ell+p_r-p_i)
\,dF_i(\ell)
}{
\mathcal D_i^{>}
}
\nonumber\\
&=
\frac{C\Delta_1}{\omega_\Delta}
\frac{
\displaystyle
\int_{\delta_{ij}/2}^{\delta_{ij}-\omega_\Delta}
\prod_{r\neq i}
\overline F_{r,\ge}(\ell+p_r-p_i)
\,dF_i(\ell)
}{
\mathcal D_i^{>}
}
\nonumber\\
&\le
\frac{C\Delta_1}{\omega_\Delta}
\frac{
\displaystyle
\int_{\delta_{ij}/2}^{p_i}
\prod_{r\neq i}
\overline F_{r,\ge}(\ell+p_r-p_i)
\,dF_i(\ell)
}{
\mathcal D_i^{>}
}
\nonumber\\
&\le
\frac{C\Delta_1}{\omega_\Delta}
\exp\{-c\Delta_1\delta_{ij}\}
=
O(\Delta_1^{-\infty}).
\label{eq:cac_left_noncritical_contribution}
\end{align}
The last inequality uses Lemma~\ref{lem:cac_winner_conditioned_endpoint_localization}
with \(x=\delta_{ij}/2\).

\emph{Right noncritical region
\(\delta_{ij}+\omega_\Delta<\ell\le p_i-\varepsilon\).}
Here the shifted level of arm \(j\) is positive:
\[
% \label{eq:cac_right_shifted_level_def}
y
:=
\ell+p_j-p_i
=
\ell-\delta_{ij}
\ge
\omega_\Delta .
\]
The restriction \(\ell\le p_i-\varepsilon\) also gives
\(
% \label{eq:cac_right_away_from_upper_endpoint}
p_j-y
=
p_i-\ell
\ge
\varepsilon
\). 
Therefore, by the global tilted-drift order
\eqref{eq:global_tilted_drift_order},
\begin{equation}
\label{eq:cac_right_tilted_drift_lower}
v_j(y)
\ge
c\min\{y,p_j-y\}
\ge
c\omega_\Delta .
\end{equation}
The negative-drift part of Lemma~\ref{lem:cac_pointwise_capped_first_passage}
gives
\[
% \label{eq:cac_right_tilted_pointwise}
\mathcal H_{j|i,T}(\ell)
\le
\exp\{-\theta_j(y)b_i(\ell)\}
\frac{\Delta_i+\beta\ell+p_i-\ell}{v_j(y)} .
\]
Using \eqref{eq:cac_right_tilted_drift_lower} and
\(\Delta_i+\beta\ell+p_i-\ell=O(\Delta_1)\), we obtain
\begin{equation}
\label{eq:cac_right_tilted_prefactor_bound}
\mathcal H_{j|i,T}(\ell)
\le
\frac{C\Delta_1}{\omega_\Delta}
\exp\{-\theta_j(y)b_i(\ell)\}.
\end{equation}

We next compare the tilted exponential with the strict \(j\)-tail in the
winner denominator. Since
\[
% \label{eq:cac_right_boundary_identity}
b_j(y)
=
\Delta_j+\beta y
=
\Delta_i+\beta\ell
=
b_i(\ell),
\]
the lower side of the root envelope gives
\[
% \label{eq:cac_right_j_tail_lower}
\overline F_{j,>}(y)
\ge
\exp\left\{
-\theta_j(y)\bigl[b_i(\ell)+p_j-y\bigr]
\right\}.
\]
Hence
\begin{equation}
\frac{
\exp\{-\theta_j(y)b_i(\ell)\}
}{
\overline F_{j,>}(y)
}
\le
\exp\{\theta_j(y)(p_j-y)\}
\le
\frac{1}{q_j},
\label{eq:cac_right_tail_cancellation}
\end{equation}
where the last inequality follows from the Lundberg equation
\(q_j \exp\{\theta_j(y)(p_j-y)\}\le1\). Combining
\eqref{eq:cac_right_tilted_prefactor_bound} and
\eqref{eq:cac_right_tail_cancellation}, we have the effective pointwise bound
\[
% \label{eq:cac_right_effective_pointwise_bound}
\mathcal H_{j|i,T}(\ell)
\le
\frac{C\Delta_1}{\omega_\Delta}
\overline F_{j,>}(\ell-\delta_{ij}),
\qquad
\delta_{ij}+\omega_\Delta<\ell\le p_i-\varepsilon .
\]
Therefore,
\begin{align}
\frac{
\mathcal I_{j|i,T}
\bigl((\delta_{ij}+\omega_\Delta,p_i-\varepsilon]\bigr)
}{
\mathcal D_i^{>}
}&\le
\frac{C\Delta_1}{\omega_\Delta}
\frac{
\displaystyle
\int_{\delta_{ij}+\omega_\Delta}^{p_i-\varepsilon}
\overline F_{j,>}(\ell-\delta_{ij})
\prod_{\substack{r=1\\r\neq i,j}}^{K}
\overline F_{r,\ge}(\ell+p_r-p_i)
\,dF_i(\ell)
}{
\mathcal D_i^{>}
}
\nonumber\\
&\le
\frac{C\Delta_1}{\omega_\Delta}
\frac{
\displaystyle
\int_{\delta_{ij}+\omega_\Delta}^{p_i}
\prod_{r\neq i}
\overline F_{r,\ge}(\ell+p_r-p_i)
\,dF_i(\ell)
}{
\mathcal D_i^{>}
}
\nonumber\\
&\le
\frac{C\Delta_1}{\omega_\Delta}
\frac{
\displaystyle
\int_{\delta_{ij}/2}^{p_i}
\prod_{r\neq i}
\overline F_{r,\ge}(\ell+p_r-p_i)
\,dF_i(\ell)
}{
\mathcal D_i^{>}
}
\nonumber\\
&\le
\frac{C\Delta_1}{\omega_\Delta}
\exp\{-c\Delta_1\delta_{ij}\}
=
O(\Delta_1^{-\infty}).
\label{eq:cac_right_noncritical_contribution}
\end{align}
Combining \eqref{eq:cac_left_noncritical_contribution} and
\eqref{eq:cac_right_noncritical_contribution}, we obtain
\begin{equation}
\label{eq:cac_positive_noncritical_tail}
\frac{
\mathcal I_{j|i,T}
\bigl(
\{\ell>\delta_{ij}/2,\ |\ell-\delta_{ij}|>\omega_\Delta,\ \ell\le p_i-\varepsilon\}
\bigr)
}{
\mathcal D_i^{>}
}
=
O(\Delta_1^{-\infty}).
\end{equation}

\emph{Upper-endpoint region \(p_i-\varepsilon<\ell<p_i\).}
Set
\(
% \label{eq:cac_upper_endpoint_s_def}
s:=p_i-\ell
\). 
Then \(s\in(0,\varepsilon)\), and the shifted level of arm \(j\) is
\begin{equation}
\label{eq:cac_upper_endpoint_shifted_level}
\ell+p_j-p_i
=
p_j-s .
\end{equation}
Thus this is an upper-endpoint crossing for arm \(j\). By
\eqref{eq:cac_upper_endpoint_shifted_level} and the upper-endpoint tilted-drift
order \eqref{eq:upper_v_order}, after possibly reducing \(\varepsilon>0\), there
exists \(c>0\) such that
\[
% \label{eq:cac_upper_endpoint_v_lower}
v_j(p_j-s)\ge cs,
\qquad
0<s\le\varepsilon .
\]
The negative-drift part of Lemma~\ref{lem:cac_pointwise_capped_first_passage}
therefore gives
\begin{align}
\mathcal H_{j|i,T}(\ell)
&\le
\exp\{-\theta_j(p_j-s)b_i(\ell)\}
\frac{b_i(\ell)+s}{v_j(p_j-s)}
\nonumber\\
&\le
\frac{C\Delta_1}{s}
\exp\{-\theta_j(p_j-s)b_i(\ell)\}.
\label{eq:cac_upper_endpoint_pointwise_pre}
\end{align}
Here we used \(b_i(\ell)+s=O(\Delta_1)\).

We now compare the tilted exponential with the strict \(j\)-tail. Since
\[
% \label{eq:cac_upper_endpoint_boundary_identity}
b_j(p_j-s)
=
\Delta_j+\beta(p_j-s)
=
\Delta_i+\beta(p_i-s)
=
b_i(\ell),
\]
the lower side of Corollary~\ref{cor:root_based_drawdown_tail_envelope}
implies
\[
% \label{eq:cac_upper_endpoint_j_tail_lower}
\overline F_{j,>}(p_j-s)
\ge
\exp\left\{
-\theta_j(p_j-s)\bigl[b_i(\ell)+s\bigr]
\right\}.
\]
Hence
\begin{equation}
\frac{
\exp\{-\theta_j(p_j-s)b_i(\ell)\}
}{
\overline F_{j,>}(p_j-s)
}
\le
\exp\{\theta_j(p_j-s)s\}
\le
\frac{1}{q_j},
\label{eq:cac_upper_endpoint_tail_cancellation}
\end{equation}
where the last inequality follows from the Lundberg equation
\(q_j \exp\{\theta_j(p_j-s)s\}\le1\). Combining
\eqref{eq:cac_upper_endpoint_pointwise_pre} and
\eqref{eq:cac_upper_endpoint_tail_cancellation}, we obtain the effective
pointwise bound
\[
% \label{eq:cac_upper_endpoint_effective_pointwise}
\mathcal H_{j|i,T}(\ell)
\le
\frac{C\Delta_1}{p_i-\ell}
\overline F_{j,>}(\ell+p_j-p_i),
\qquad
p_i-\varepsilon<\ell<p_i .
\]
Therefore,
\[\begin{aligned}
\frac{
\mathcal I_{j|i,T}(\{p_i-\varepsilon<\ell<p_i\})
}{
\mathcal D_i^{>}
}&\le
C\Delta_1
\frac{
\displaystyle
\int_{p_i-\varepsilon}^{p_i}
\frac{1}{p_i-\ell}
\overline F_{j,>}(\ell+p_j-p_i)
\prod_{\substack{r=1\\r\neq i,j}}^{K}
\overline F_{r,\ge}(\ell+p_r-p_i)
\,dF_i(\ell)
}{
\mathcal D_i^{>}
}
\nonumber\\
&\le
C\Delta_1
\frac{
\displaystyle
\int_{p_i-\varepsilon}^{p_i}
\frac{1}{p_i-\ell}
\prod_{r\neq i}
\overline F_{r,\ge}(\ell+p_r-p_i)
\,dF_i(\ell)
}{
\mathcal D_i^{>}
}.
% \label{eq:cac_upper_endpoint_integral_start}
\end{aligned}\]

We next record the winner-weighted upper-endpoint tail bound. For
\(0<s\le\varepsilon\),
\[\frac{
\displaystyle
\int_{p_i-s}^{p_i}
\prod_{r\neq i}
\overline F_{r,\ge}(\ell+p_r-p_i)
\,dF_i(\ell)
}{
\mathcal D_i^{>}
}
\le
\frac{\mathbb P(L_i>p_i-s)}{\mathcal D_i^{>}}.
% \label{eq:cac_upper_endpoint_weighted_tail_start}
\]
By the upper-endpoint root expansion for arm \(i\),
\[
% \label{eq:cac_upper_endpoint_Li_tail}
\mathbb P(L_i>p_i-s)
\le
\exp\left\{
-\theta_i(p_i-s)b_i(p_i-s)
\right\}
\le
\exp\left\{
-\frac{c\Delta_1}{s}
\right\}.
\]
On the other hand, the denominator lower bound in
Lemma~\ref{lem:cac_winner_conditioned_endpoint_localization}, together with the intermediate-gap scale of the positive shifted levels in \(\Phi_i^>(0)\),
gives
\(
% \label{eq:cac_upper_endpoint_denominator_lower}
\mathcal D_i^{>}
\ge
\exp\{-C\Delta_1\eta_\Delta\}
\). 
Since \(\eta_\Delta\to0\), after decreasing \(c>0\) if necessary,
\[
% \label{eq:cac_upper_endpoint_weighted_tail}
\frac{
\displaystyle
\int_{p_i-s}^{p_i}
\prod_{r\neq i}
\overline F_{r,\ge}(\ell+p_r-p_i)
\,dF_i(\ell)
}{
\mathcal D_i^{>}
}
\le
C\exp\left\{
-\frac{c\Delta_1}{s}
\right\},
\qquad
0<s\le\varepsilon .
\]

It remains to integrate the algebraic prefactor \((p_i-\ell)^{-1}\). Let
\(s_n:=2^{-n}\varepsilon\), \(n=0,1,2,\ldots\), and split the upper-endpoint
region into dyadic shells
\(
A_n:=\{p_i-s_n<\ell\le p_i-s_{n+1}\}
\). 
On \(A_n\), we have \((p_i-\ell)^{-1}\le s_{n+1}^{-1}\). Moreover, by the upper-endpoint tail bound for \(L_i\), the denominator lower bound \eqref{eq:cac_denominator_lower}, and the fact that \(\Delta_1\eta_\Delta=o(\Delta_1/s_n)\) uniformly for \(s_n\le\varepsilon\), there exists \(c>0\) such that
\begin{equation}
\label{eq:cac_endpoint_weighted_mass_ratio}
\frac{
\displaystyle
\int_{p_i-s_n}^{p_i}
\prod_{r\neq i}
\overline F_{r,\ge}(\ell+p_r-p_i)
\,dF_i(\ell)
}{
\mathcal D_i^{>}
}
\le
\exp\left\{
-\frac{c\Delta_1}{s_n}
\right\},
\qquad
n=0,1,2,\ldots .
\end{equation}
Therefore,
\begin{align}
\frac{
\mathcal I_{j|i,T}(\{p_i-\varepsilon<\ell<p_i\})
}{
\mathcal D_i^{>}
}&\le
C\Delta_1
\sum_{n=0}^{\infty}
\frac{1}{s_{n+1}}
\frac{
\displaystyle
\int_{p_i-s_n}^{p_i}
\prod_{r\neq i}
\overline F_{r,\ge}(\ell+p_r-p_i)
\,dF_i(\ell)
}{
\mathcal D_i^{>}
}
\nonumber\\
&\le
C\Delta_1
\sum_{n=0}^{\infty}
\frac{1}{s_{n+1}}
\exp\left\{
-\frac{c\Delta_1}{s_n}
\right\}
\nonumber\\
&\le
C\Delta_1
\sum_{n=0}^{\infty}
\frac{2^{n+1}}{\varepsilon}
\exp\left\{
-\frac{c2^n\Delta_1}{\varepsilon}
\right\}
\nonumber\\
&=
O(\Delta_1^{-\infty}).
\label{eq:cac_endpoint_contribution}
\end{align}

\emph{Singular layer \(|\ell-\delta_{ij}|\le\omega_\Delta\).}
This layer contains the only region, up to the noncritical and upper-endpoint
parts already shown negligible, where the strict comparison with arm \(j\) may
fail to give a useful pathwise upper bound. We therefore control it by a
finite-horizon cap. The size of this layer is compared with the strict-winner
scale of arm \(j\) through lower-endpoint crossing envelopes.

Define the singular-layer event
\[
% \label{eq:cac_singular_layer_event_def}
\mathcal A_{ij,\Delta}
:=
Q_i\cap
\left\{
|L_i-\delta_{ij}|\le\omega_\Delta
\right\}.
\]
Since \(Q_i\subseteq\mathcal W_i^{\ge}\), conditioning on \(L_i=\ell\) gives the
exact weak-layer upper bound
\begin{align}
\mathbb P(\mathcal A_{ij,\Delta})
&\le
\int_{[\delta_{ij}-\omega_\Delta,\;\delta_{ij}+\omega_\Delta]}
\prod_{h\neq i}
\overline F_{h,\ge}(\ell+p_h-p_i)
\,dF_i(\ell).
\label{eq:cac_singular_exact_weak_layer_start}
\end{align}

We now bound the factors in the integrand. On the layer
\(|\ell-\delta_{ij}|\le\omega_\Delta\), the \(h=j\) factor is bounded by one.
For \(h>j\), the corresponding factor is also bounded by one; under strict
ordering the shifted level is in fact nonpositive for all sufficiently large
\(\Delta_1\). For \(h<j\), \(h\neq i\), we have
\(
% \label{eq:cac_singular_shift_identity}
\ell+p_h-p_i
=
\delta_{hj}+(\ell-\delta_{ij})
\). 
If \(\delta_{hj}>0\), then this shifted level is
\(\delta_{hj}+O(\omega_\Delta)\). Since
\(\delta_{hj}=O(\eta_\Delta)\), \(\omega_\Delta=o(\eta_\Delta)\), and
\(\Delta_1\omega_\Delta=o(1)\), the lower-endpoint tail envelope gives,
uniformly on the singular layer,
\[
% \label{eq:cac_singular_h_factor_bound_positive_gap}
\overline F_{h,\ge}(\ell+p_h-p_i)
\le
\exp\left\{
-\lambda_h\delta_{hj}
+o(1)
\right\}.
\]
If \(\delta_{hj}=0\), the same upper bound is trivial because the right-hand
side is \(\exp\{o(1)\}\) and the survival factor is at most one. Therefore,
uniformly over the singular layer,
\[
% \label{eq:cac_singular_product_factor_bound}
\prod_{\substack{h<j\\h\neq i}}
\overline F_{h,\ge}(\ell+p_h-p_i)
\le
\exp\left\{
-\sum_{\substack{h<j\\h\neq i}}
\lambda_h\delta_{hj}
+o(1)
\right\}.
\]
Combining this with \eqref{eq:cac_singular_exact_weak_layer_start}, we obtain
\begin{align}
\mathbb P(\mathcal A_{ij,\Delta})
&\le
\exp\left\{
-\sum_{\substack{h<j\\h\neq i}}
\lambda_h\delta_{hj}
+o(1)
\right\}
\left[
F_i(\delta_{ij}+\omega_\Delta)
-
F_i\bigl((\delta_{ij}-\omega_\Delta)^-\bigr)
\right]
\nonumber\\
&=
\exp\left\{
-\sum_{\substack{h<j\\h\neq i}}
\lambda_h\delta_{hj}
+o(1)
\right\}
\left[
\overline F_{i,\ge}(\delta_{ij}-\omega_\Delta)
-
\overline F_{i,>}(\delta_{ij}+\omega_\Delta)
\right].
\label{eq:cac_singular_layer_reduced_to_Li_mass}
\end{align}

It remains to bound the \(L_i\)-mass of the short window. Because \(\delta_{ij}=\Theta(\eta_\Delta)\) and \(\omega_\Delta=o(\eta_\Delta)\), both endpoints \(\delta_{ij}\pm\omega_\Delta\) lie on the lower-endpoint scale for all sufficiently large \(\Delta_1\). Hence, for both \(\circ\in\{>,\ge\}\), the local lower-endpoint approximation gives
\[
\overline F_{i,\circ}(x)
=
\exp\{-\lambda_i x\}
\left[
1+O(\Delta_1^{1-2\kappa})
\right],
\qquad
x\in
\left\{
\delta_{ij}-\omega_\Delta,\,
\delta_{ij}+\omega_\Delta
\right\}.
\]
Therefore,
\begin{align}
F_i(\delta_{ij}+\omega_\Delta)
-
F_i\bigl((\delta_{ij}-\omega_\Delta)^-\bigr)
&=
\overline F_{i,\ge}(\delta_{ij}-\omega_\Delta)
-
\overline F_{i,>}(\delta_{ij}+\omega_\Delta)
\nonumber\\
&\le
C
\int_{\delta_{ij}-\omega_\Delta}^{\delta_{ij}+\omega_\Delta}
\lambda_i\exp\{-\lambda_i x\}\,dx
+
C\Delta_1^{1-2\kappa}
\exp\{-\lambda_i(\delta_{ij}-\omega_\Delta)\}
\nonumber\\
&\le
C
\left[
\lambda_i\omega_\Delta
+
\Delta_1^{1-2\kappa}
\right]
\exp\{-\lambda_i(\delta_{ij}-\omega_\Delta)\}
\nonumber\\
&\le
C
\left[
\lambda_i\omega_\Delta
+
\Delta_1^{1-2\kappa}
\right]
\exp\{-\lambda_i\delta_{ij}+o(1)\}.
\label{eq:cac_singular_Li_window_mass_bound}
\end{align}
In the last step we used \(\lambda_i\omega_\Delta=o(1)\). The integral term is the mass of the exponential reference measure over the window, while the \(\Delta_1^{1-2\kappa}\) term is the uniform lower-endpoint approximation error. Both terms are \(o(1)\), because \(\lambda_i=\Theta(\Delta_1)\), \(\omega_\Delta=\Delta_1^{-1-\xi}\) with \(\xi>0\), and \(\kappa>1/2\).

Substituting \eqref{eq:cac_singular_Li_window_mass_bound} into
\eqref{eq:cac_singular_layer_reduced_to_Li_mass} gives
\begin{align}
\mathbb P(\mathcal A_{ij,\Delta})
&\le
C
\left[
\lambda_i\omega_\Delta
+
\Delta_1^{1-2\kappa}
\right]
\exp\left\{
-\lambda_i\delta_{ij}
-
\sum_{\substack{h<j\\h\neq i}}
\lambda_h\delta_{hj}
+o(1)
\right\}
\nonumber\\
&=
C
\left[
\lambda_i\omega_\Delta
+
\Delta_1^{1-2\kappa}
\right]
\exp\left\{
-\sum_{h<j}\lambda_h\delta_{hj}
+o(1)
\right\}
\nonumber\\
&=
o\left(
\exp\left\{
-\sum_{h<j}\lambda_h\delta_{hj}
+o(1)
\right\}
\right).
\label{eq:cac_singular_layer_upper_envelope_bound}
\end{align}

We next compare this upper bound with a restricted strict-winner lower bound
for arm \(j\). Fix a small constant \(a>0\). Since
\(a/\Delta_1\le\eta_\Delta\) for all sufficiently large \(\Delta_1\), the contribution
to \(\mathbb P(\mathcal W_j^{>})\) from \(0\le x\le a/\Delta_1\) gives
\[\begin{aligned}
\mathbb P(Q_j)
&\ge
\mathbb P(\mathcal W_j^{>})
\nonumber\\
&\ge
c
\int_0^{a/\Delta_1}
\exp\left\{
-\sum_{h<j}\lambda_h(x+\delta_{hj})
\right\}
\lambda_j \exp\{-\lambda_j x\}\,dx
\left[1+o(1)\right].
% \label{eq:cac_singular_Qj_lower_restricted}
\end{aligned}\]
Here the fixed constant \(c>0\) accounts for competitor factors not displayed
in the integral; on the \(O(\Delta_1^{-1})\) window these factors are either
automatic or bounded below by a positive constant. Evaluating the integral gives
\[\begin{aligned}
&\int_0^{a/\Delta_1}
\exp\left\{
-\sum_{h<j}\lambda_h(x+\delta_{hj})
\right\}
\lambda_j \exp\{-\lambda_j x\}\,dx
\nonumber\\
&\qquad
=
\exp\left\{
-\sum_{h<j}\lambda_h\delta_{hj}
\right\}
\frac{\lambda_j}{\lambda_j+\sum_{h<j}\lambda_h}
\left[
1-
\exp\left\{
-\left(\lambda_j+\sum_{h<j}\lambda_h\right)
\frac{a}{\Delta_1}
\right\}
\right].
\end{aligned}\]
Since \(K\) is fixed and \(\lambda_h=\Theta(\Delta_1)\) for every \(h\), the
last two factors are bounded below by a positive constant. Hence
\begin{align}
\mathbb P(Q_j)
&\ge
c
\exp\left\{
-\sum_{h<j}\lambda_h\delta_{hj}
\right\}.
\label{eq:cac_singular_Qj_lower_scale}
\end{align}
Combining \eqref{eq:cac_singular_layer_upper_envelope_bound} and
\eqref{eq:cac_singular_Qj_lower_scale}, we obtain
\(
% \label{eq:cac_singular_joint_o_Qj}
\mathbb P(\mathcal A_{ij,\Delta})
=
o(\mathbb P(Q_j))
\). 

Finally, since
\(
% \label{eq:cac_singular_Di_Qi_equivalence}
\mathcal D_i^{>}:=\mathbb P(\mathcal W_i^{>}),\,
\mathbb P(Q_i)=\mathcal D_i^{>}(1+o(1))
\), 
the finite-horizon contribution of the singular layer is
\begin{equation}
T
\frac{
\mathbb P\left(Q_i,\ |L_i-\delta_{ij}|\le\omega_\Delta\right)
}{
\mathbb P(Q_i)
}
=
T
\frac{
\mathbb P(\mathcal A_{ij,\Delta})
}{
\mathbb P(Q_i)
}
=
T\,o\!\left(
\frac{\mathbb P(Q_j)}{\mathbb P(Q_i)}
\right).
\label{eq:cac_singular_layer_contribution}
\end{equation}

\emph{Pairwise equality correction.}
It remains to control the pairwise equality correction
\(\mathcal E_{ij}^{=}\). We prove that, in the positive-drift branch
\(p_i>p_j\),
\begin{equation}
\label{eq:cac_pairwise_tie_positive_target}
\mathbb P(\mathcal E_{ij}^{=})
=
o(\mathbb P(Q_j)).
\end{equation}
This estimate is stronger than the global strict--weak winner comparison for arm \(i\), and it is the estimate needed on every branch \(Q_i\) satisfying \(p_i=p_1\).

By \eqref{eq:cac_pairwise_tie_probability_stieltjes}, the pairwise equality
probability is
\begin{equation}
\mathbb P(\mathcal E_{ij}^{=})
=
\int_{\delta_{ij}}^{p_i}
\left[
\overline F_{j,\ge}(\ell-\delta_{ij})
-
\overline F_{j,>}(\ell-\delta_{ij})
\right]
\times
\prod_{\substack{h=1\\h\neq i,j}}^{K}
\overline F_{h,\ge}(\ell+p_h-p_i)
\,dF_i(\ell).
\label{eq:cac_pairwise_tie_positive_stieltjes}
\end{equation}
The lower limit \(\delta_{ij}\) is harmless because, for
\(\ell<\delta_{ij}\), the shifted level of arm \(j\) is negative and the strict and weak tail factors are both equal to one.

Set
\(
% \label{eq:cac_pairwise_tie_y_def}
y:=\ell-\delta_{ij}
\), 
then
\[
% \label{eq:cac_pairwise_tie_shift_rewrite}
\ell+p_h-p_i
=
y+p_h-p_j,
\qquad h\neq i,j.
\]
On the local region \(0\le y\le\eta_\Delta\), all positive shifted levels in
\eqref{eq:cac_pairwise_tie_positive_stieltjes} are \(O(\eta_\Delta)\). For \(\circ\in\{>,\ge\}\), the lower-endpoint tail envelope gives, uniformly
for \(0\le y\le\eta_\Delta\),
\[
\overline F_{j,\circ}(y)
=
\exp\{-\lambda_j y\}
\left[
1+\varepsilon_{j,\circ}(y)
\right],
\qquad
\sup_{0\le y\le\eta_\Delta}
|\varepsilon_{j,\circ}(y)|
\le
C_0\Delta_1^{1-2\kappa}.
\]
Since the weak tail contains the strict tail,
\(\overline F_{j,\ge}(y)\ge \overline F_{j,>}(y)\). Therefore
\begin{align}
0
&\le
\overline F_{j,\ge}(y)-\overline F_{j,>}(y)
\nonumber\\
&=
\exp\{-\lambda_j y\}
\left[
\varepsilon_{j,\ge}(y)-\varepsilon_{j,>}(y)
\right]
\nonumber\\
&\le
2C_0\Delta_1^{1-2\kappa}\exp\{-\lambda_j y\}.
\label{eq:cac_pairwise_tail_difference_bound_pre}
\end{align}
Thus, after increasing the constant if necessary,
\begin{equation}
\label{eq:cac_pairwise_tail_difference_bound}
0
\le
\overline F_{j,\ge}(y)-\overline F_{j,>}(y)
\le
a_\Delta \exp\{-\lambda_j y\},
\qquad
a_\Delta:=C\Delta_1^{1-2\kappa}=o(1),
\end{equation}
uniformly for \(0\le y\le\eta_\Delta\).

Using \eqref{eq:cac_pairwise_tail_difference_bound}, the local Stieltjes
replacement for \(L_i\), and the same lower-endpoint replacements for the
remaining weak factors, we obtain
\[
\mathbb P(\mathcal E_{ij}^{=})
\le
a_\Delta
\int_0^{\eta_\Delta}
\exp\left\{
-\lambda_j y
-
\sum_{\substack{h=1\\h\neq i,j}}^{K}
\lambda_h(y+p_h-p_j)_+
\right\}
\lambda_i \exp\{-\lambda_i(\delta_{ij}+y)\}
\,dy
\,[1+o(1)]
+
o(\mathcal J_{ij,\Delta}),
% \label{eq:cac_pairwise_tie_local_bound}
\]
where
\begin{align}
\mathcal J_{ij,\Delta}
&:=
\int_0^{\eta_\Delta}
\exp\left\{
-\lambda_j y
-
\sum_{\substack{h=1\\h\neq i,j}}^{K}
\lambda_h(y+p_h-p_j)_+
\right\}
\lambda_i \exp\{-\lambda_i(\delta_{ij}+y)\}
\,dy .
\label{eq:cac_pairwise_reference_integral}
\end{align}
The contribution of \(y>\eta_\Delta\), together with the upper-endpoint part, is
\(o(\mathcal J_{ij,\Delta})\) by the same tail-localization argument used in
Corollary~\ref{cor:acf_product_tail_stieltjes_replacement}. Therefore
\begin{equation}
\label{eq:cac_pairwise_tie_o_reference}
\mathbb P(\mathcal E_{ij}^{=})
=
o(\mathcal J_{ij,\Delta}).
\end{equation}

We now compare \(\mathcal J_{ij,\Delta}\) with the suboptimal-convergence scale of
arm \(j\). The lower-endpoint reference integral for the strict winner event
of arm \(j\) is
\begin{align}
\mathcal J_{j,\Delta}
&:=
\int_0^{\eta_\Delta}
\exp\left\{
-\lambda_i(y+\delta_{ij})
-
\sum_{\substack{h=1\\h\neq i,j}}^{K}
\lambda_h(y+p_h-p_j)_+
\right\}
\lambda_j \exp\{-\lambda_j y\}
\,dy .
\label{eq:cac_pairwise_Qj_reference_integral}
\end{align}
Comparing \eqref{eq:cac_pairwise_reference_integral} and
\eqref{eq:cac_pairwise_Qj_reference_integral}, we have
\(
% \label{eq:cac_pairwise_reference_ratio}
\mathcal J_{ij,\Delta}
=
{\lambda_i \mathcal J_{j,\Delta}}/{\lambda_j} 
\). 
Since \(\lambda_i=\Theta(\Delta_1)\) and \(\lambda_j=\Theta(\Delta_1)\),
\(
% \label{eq:cac_pairwise_lambda_ratio_bounded}
{\lambda_i}/{\lambda_j}
=
\Theta(1)
\). 
Moreover, by the local Stieltjes replacement and the strict/weak winner
sandwich for arm \(j\),
\(
% \label{eq:cac_pairwise_Qj_reference_equivalence}
\mathbb P(Q_j)
=
\mathcal J_{j,\Delta}[1+o(1)]
\). 
Hence
\begin{equation}
\label{eq:cac_pairwise_reference_O_Qj}
\mathcal J_{ij,\Delta}
=
O(\mathbb P(Q_j)).
\end{equation}
Combining \eqref{eq:cac_pairwise_tie_o_reference} and
\eqref{eq:cac_pairwise_reference_O_Qj} proves
\eqref{eq:cac_pairwise_tie_positive_target}.

Consequently, using \(\mathcal D_i^{>}=\mathbb P(\mathcal W_i^{>})\) and
\(\mathbb P(Q_i)=\mathcal D_i^{>}(1+o(1))\),
\begin{align}
T
\frac{
\mathbb P(\mathcal E_{ij}^{=})
}{
\mathcal D_i^{>}
}
&=
T\,o\!\left(
\frac{\mathbb P(Q_j)}{\mathbb P(Q_i)}
\right).
\label{eq:cac_pairwise_tie_positive_contribution}
\end{align}

Combining the bulk estimate \eqref{eq:cac_positive_bulk_final}, the
noncritical estimate \eqref{eq:cac_positive_noncritical_tail}, the
upper-endpoint estimate \eqref{eq:cac_endpoint_contribution}, the singular-layer
contribution \eqref{eq:cac_singular_layer_contribution}, and the pairwise
equality correction \eqref{eq:cac_pairwise_tie_positive_contribution}, we obtain
\[
% \label{eq:cac_positive_branch_final}
\mathbb E[\min\{T,N_j\}\mid Q_i]
\le
\frac{\Delta_i}{\delta_{ij}}
\left(1+O(\Delta_1^{\kappa-1})\right)
+
T\,o\!\left(
\frac{\mathbb P(Q_j)}{\mathbb P(Q_i)}
\right).
\]
This proves the branch \(p_i>p_j\).

\paragraph{Large-deviation branch: \(p_j>p_i\).}
Recall 
\(
% \label{eq:cac_large_deviation_gap_def}
\delta_{ji}:=p_j-p_i>0
\). 
The shifted level in Lemma~\ref{lem:cac_pointwise_capped_first_passage} is
\(
% \label{eq:cac_large_deviation_shifted_level_def}
y=\ell+\delta_{ji}
\). 
By \eqref{eq:app_asymptotic_gap_scale} and \eqref{eq:app_asymptotic_eta_def},
\[
% \label{eq:cac_large_deviation_gap_scale}
\delta_{ji}
=
\Theta(\eta_\Delta),
\qquad
\Delta_1\delta_{ji}\to\infty,
\qquad
\Delta_1\delta_{ji}^2\to0 .
\]
There is no interface singularity in this branch, because \(y=\ell+\delta_{ji}>0\)
for every \(\ell\ge0\). We first evaluate the lower-endpoint contribution
\(0\le\ell\le\eta_\Delta\), and then show that the remaining contribution
\(\ell>\eta_\Delta\) is negligible.

On \(0\le\ell\le\eta_\Delta\), we have
\[
% \label{eq:cac_large_deviation_y_local}
\delta_{ji}
\le
y
=
\ell+\delta_{ji}
\le
\delta_{ji}+\eta_\Delta
=
O(\eta_\Delta).
\]
Thus every positive shifted level used in the local part of this branch is
inside the same \(O(\eta_\Delta)\) lower-endpoint window. In particular, the
expansions of \(\theta_j(y)\), \(v_j(y)\), and the one-arm tail envelopes are
used only uniformly over \(0<y\le C\eta_\Delta\).

Therefore, the lower-endpoint expansions apply uniformly. In particular,
\begin{equation}
\label{eq:cac_large_deviation_v_lower}
v_j(y)
\ge
y(1-O(\eta_\Delta))
=
(\ell+\delta_{ji})(1-O(\eta_\Delta)).
\end{equation}
Moreover,
\begin{equation}
\label{eq:cac_large_deviation_boundary_identity}
b_j(y)
=
\Delta_j+\beta(\ell+\delta_{ji})
=
\Delta_i+\beta\ell
=
b_i(\ell).
\end{equation}
The negative-drift part of Lemma~\ref{lem:cac_pointwise_capped_first_passage}
therefore gives
\begin{equation}
\label{eq:cac_large_deviation_pointwise_start}
\mathcal H_{j|i,T}(\ell)
\le
\exp\{-\theta_j(y)b_i(\ell)\}
\frac{b_i(\ell)+p_i-\ell}{v_j(y)} .
\end{equation}

We next compare the tilted exponential in
\eqref{eq:cac_large_deviation_pointwise_start} with the strict \(j\)-tail
that is absent from \(\mathcal I_{j|i,T}\). By the lower side of
Corollary~\ref{cor:root_based_drawdown_tail_envelope} and
\eqref{eq:cac_large_deviation_boundary_identity},
\[
% \label{eq:cac_large_deviation_j_tail_lower}
\overline F_{j,>}(y)
\ge
\exp\left\{
-\theta_j(y)\bigl[b_i(\ell)+p_j-y\bigr]
\right\}.
\]
Since \(p_j-y=p_i-\ell\), we have
\begin{align}
\frac{
\exp\{-\theta_j(y)b_i(\ell)\}
}{
\overline F_{j,>}(y)
}
&\le
\exp\{\theta_j(y)(p_i-\ell)\}
\nonumber\\
&=
\exp\{\theta_j(y)(p_j-y)\}
\nonumber\\
&\le
\frac{1}{q_j},
\label{eq:cac_large_deviation_tail_ratio_cancellation}
\end{align}
where the last inequality follows from the Lundberg equation
\(q_j \exp\{\theta_j(y)(p_j-y)\}\le1\).

Combining \eqref{eq:cac_large_deviation_pointwise_start},
\eqref{eq:cac_large_deviation_v_lower}, and
\eqref{eq:cac_large_deviation_tail_ratio_cancellation}, and using
\(y=\ell+\delta_{ji}\), we obtain, uniformly for
\(0\le\ell\le\eta_\Delta\),
\[
% \label{eq:cac_large_deviation_effective_pointwise}
\mathcal H_{j|i,T}(\ell)
\le
\frac{b_i(\ell)+p_i-\ell}{\ell+\delta_{ji}}
\left(1+O(\eta_\Delta)\right)
\overline F_{j,>}(\ell+\delta_{ji}).
\]
Since \(\ell+\delta_{ji}\ge\delta_{ji}\) and
\(b_i(\ell)+p_i-\ell=\Delta_i+\beta\ell+p_i-\ell\le
\Delta_i+\beta\ell+p_i\), this implies
\begin{equation}
\label{eq:cac_large_deviation_effective_pointwise_simple}
\mathcal H_{j|i,T}(\ell)
\le
\frac{\Delta_i+\beta\ell+p_i}{\delta_{ji}}
\left(1+O(\eta_\Delta)\right)
\overline F_{j,>}(\ell+\delta_{ji}),
\qquad
0\le\ell\le\eta_\Delta .
\end{equation}

Substituting \eqref{eq:cac_large_deviation_effective_pointwise_simple} into
the integral over \([0,\eta_\Delta]\), we get
\[\begin{aligned}
\frac{\mathcal I_{j|i,T}([0,\eta_\Delta])}{\mathcal D_i^{>}}
&\le
\frac{1+O(\eta_\Delta)}{\delta_{ji}}
\frac{
\displaystyle
\int_0^{\eta_\Delta}
(\Delta_i+\beta\ell+p_i)
\overline F_{j,>}(\ell+\delta_{ji})
\prod_{\substack{r=1\\r\neq i,j}}^K
\overline F_{r,\ge}(\ell+p_r-p_i)
\,dF_i(\ell)
}{
\mathcal D_i^{>}
}
\nonumber\\
&\le
\frac{1+O(\eta_\Delta)}{\delta_{ji}}
\frac{
\displaystyle
\int_0^{\eta_\Delta}
(\Delta_i+\beta\ell+p_i)
\prod_{r\neq i}
\overline F_{r,\ge}(\ell+p_r-p_i)
\,dF_i(\ell)
}{
\mathcal D_i^{>}
}.
% \label{eq:cac_large_deviation_local_integral_start}
\end{aligned}\]
The last step uses
\(\overline F_{j,>}(\ell+\delta_{ji})
\le
\overline F_{j,\ge}(\ell+\delta_{ji})\).

We now apply the winner-weighted moment bounds. The leading-order equivalence
between strict and weak winner Stieltjes replacements, together with the tail
bound in Lemma~\ref{lem:cac_winner_conditioned_endpoint_localization}, gives
\[
% \label{eq:cac_large_deviation_zero_moment}
\frac{
\displaystyle
\int_0^{\eta_\Delta}
\prod_{r\neq i}
\overline F_{r,\ge}(\ell+p_r-p_i)
\,dF_i(\ell)
}{
\mathcal D_i^{>}
}
=
1+o(1).
\]
The \(m=1\) case of
Lemma~\ref{lem:cac_winner_conditioned_endpoint_localization} gives
\[
% \label{eq:cac_large_deviation_first_moment}
\frac{
\displaystyle
\int_0^{\eta_\Delta}
\ell
\prod_{r\neq i}
\overline F_{r,\ge}(\ell+p_r-p_i)
\,dF_i(\ell)
}{
\mathcal D_i^{>}
}
=
O(\Delta_1^{-1}).
\]
Therefore, using \(\beta=O(\Delta_1)\) and \(\Delta_i=\Theta(\Delta_1)\),
\begin{align}
\frac{\mathcal I_{j|i,T}([0,\eta_\Delta])}{\mathcal D_i^{>}}
&\le
\frac{1+O(\eta_\Delta)}{\delta_{ji}}
\left[
(\Delta_i+p_i)(1+o(1))
+
\beta O(\Delta_1^{-1})
\right]
\nonumber\\
&=
\frac{\Delta_i}{\delta_{ji}}
\left(1+o(1)\right).
\label{eq:cac_large_deviation_local_final}
\end{align}

It remains to prove that the contribution of \((\eta_\Delta,p_i)\) is negligible.
For \(\ell>\eta_\Delta\), the shifted level satisfies \(y=\ell+\delta_{ji}\ge\eta_\Delta\).
After the same strict-tail cancellation as in
\eqref{eq:cac_large_deviation_tail_ratio_cancellation}, the effective
pointwise envelope is bounded by
\begin{equation}
\label{eq:cac_large_deviation_tail_effective_pointwise}
\mathcal H_{j|i,T}(\ell)
\le
C\Delta_1
\left[
\frac{1}{\eta_\Delta}
+
\frac{1}{p_i-\ell}
\right]
\overline F_{j,>}(\ell+\delta_{ji}),
\qquad
\eta_\Delta<\ell<p_i .
\end{equation}
Indeed, this follows from
\(
v_j(\ell+\delta_{ji})
\asymp
\min\{\ell+\delta_{ji},\,p_i-\ell\}
\), 
together with \(\ell+\delta_{ji}\ge\eta_\Delta\) and \(b_i(\ell)+p_i-\ell=O(\Delta_1)\).

Substituting \eqref{eq:cac_large_deviation_tail_effective_pointwise} gives
\begin{align}
\frac{\mathcal I_{j|i,T}((\eta_\Delta,p_i))}{\mathcal D_i^{>}}
&\le
C\Delta_1
\frac{
\displaystyle
\int_{\eta_\Delta}^{p_i}
\left[
\eta_\Delta^{-1}+(p_i-\ell)^{-1}
\right]
\overline F_{j,>}(\ell+\delta_{ji})
\prod_{\substack{r=1\\r\neq i,j}}^K
\overline F_{r,\ge}(\ell+p_r-p_i)
\,dF_i(\ell)
}{
\mathcal D_i^{>}
}
\nonumber\\
&\le
C\Delta_1
\frac{
\displaystyle
\int_{\eta_\Delta}^{p_i}
\left[
\eta_\Delta^{-1}+(p_i-\ell)^{-1}
\right]
\prod_{r\neq i}
\overline F_{r,\ge}(\ell+p_r-p_i)
\,dF_i(\ell)
}{
\mathcal D_i^{>}
}.
\label{eq:cac_large_deviation_tail_integral_start}
\end{align}
The term with \(\eta_\Delta^{-1}\) is bounded by
\begin{equation}
C\frac{\Delta_1}{\eta_\Delta}
\frac{
\displaystyle
\int_{\eta_\Delta}^{p_i}
\prod_{r\neq i}
\overline F_{r,\ge}(\ell+p_r-p_i)
\,dF_i(\ell)
}{
\mathcal D_i^{>}
}
\le
C\frac{\Delta_1}{\eta_\Delta}
\exp\{-c\Delta_1\eta_\Delta\}
=
O(\Delta_1^{-\infty})
\label{eq:cac_large_deviation_tail_eta_part}
\end{equation}
by Lemma~\ref{lem:cac_winner_conditioned_endpoint_localization} with \(x=\eta_\Delta\).

For the term with \((p_i-\ell)^{-1}\), split
\((\eta_\Delta,p_i)\) into the part away from the upper endpoint and the
upper-endpoint part. On \((\eta_\Delta,p_i-\varepsilon]\), the algebraic factor
is bounded by \(\varepsilon^{-1}\). Hence the weighted tail-localization
estimate gives
\begin{equation}
\label{eq:cac_large_deviation_tail_away_from_endpoint}
C\Delta_1
\frac{
\displaystyle
\int_{\eta_\Delta}^{p_i-\varepsilon}
(p_i-\ell)^{-1}
\prod_{r\neq i}
\overline F_{r,\ge}(\ell+p_r-p_i)
\,dF_i(\ell)
}{
\mathcal D_i^{>}
}
=
O(\Delta_1^{-\infty}).
\end{equation}
It remains to consider \(p_i-\varepsilon<\ell<p_i\). For this part, we use the
same upper-endpoint dyadic estimate as in
\eqref{eq:cac_endpoint_contribution}. Namely, for \(0<s\le\varepsilon\),
\begin{equation}
\label{eq:cac_large_deviation_upper_weighted_tail}
\frac{
\displaystyle
\int_{p_i-s}^{p_i}
\prod_{r\neq i}
\overline F_{r,\ge}(\ell+p_r-p_i)
\,dF_i(\ell)
}{
\mathcal D_i^{>}
}
\le
C\exp\left\{
-\frac{c\Delta_1}{s}
\right\}.
\end{equation}
Let \(s_n:=2^{-n}\varepsilon\), \(n=0,1,2,\ldots\). Since
\((p_i-\ell)^{-1}\le s_{n+1}^{-1}\) on
\(\{p_i-s_n<\ell\le p_i-s_{n+1}\}\), \eqref{eq:cac_large_deviation_upper_weighted_tail}
implies
\begin{align}
C\Delta_1
\frac{
\displaystyle
\int_{p_i-\varepsilon}^{p_i}
(p_i-\ell)^{-1}
\prod_{r\neq i}
\overline F_{r,\ge}(\ell+p_r-p_i)
\,dF_i(\ell)
}{
\mathcal D_i^{>}
}&
\le
C\Delta_1
\sum_{n=0}^{\infty}
\frac{1}{s_{n+1}}
\exp\left\{
-\frac{c\Delta_1}{s_n}
\right\}
\nonumber\\
&
\le
C\Delta_1
\sum_{n=0}^{\infty}
\frac{2^{n+1}}{\varepsilon}
\exp\left\{
-\frac{c2^n\Delta_1}{\varepsilon}
\right\}
=
O(\Delta_1^{-\infty}).
\label{eq:cac_large_deviation_tail_endpoint_part}
\end{align}
Combining \eqref{eq:cac_large_deviation_tail_integral_start},
\eqref{eq:cac_large_deviation_tail_eta_part},
\eqref{eq:cac_large_deviation_tail_away_from_endpoint}, and
\eqref{eq:cac_large_deviation_tail_endpoint_part}, we obtain
\begin{equation}
\label{eq:cac_large_deviation_tail_contribution}
\frac{\mathcal I_{j|i,T}((\eta_\Delta,p_i))}{\mathcal D_i^{>}}
=
O(\Delta_1^{-\infty}).
\end{equation}

Combining \eqref{eq:cac_large_deviation_local_final} and
\eqref{eq:cac_large_deviation_tail_contribution}, and using
\(\eta_\Delta=\Delta_1^{-\kappa}\), yields
\[
% \label{eq:cac_large_deviation_closed_final}
\frac{\mathcal I_{j|i,T}([0,p_i))}{\mathcal D_i^{>}}
\le
\frac{\Delta_i}{\delta_{ji}}
\left(1+o(1)\right).
\]

The strict-crossing bound only misses the pairwise boundary part where
\(M_j=M_i\). This pairwise boundary part is \(\mathcal E_{ij}^{=}\), and it is
contained in the global strict--weak gap
\(\mathcal W_i^{\ge}\setminus\mathcal W_i^{>}\). Therefore, by
\eqref{eq:mfc_strict_weak_gap_negligible}, applied to arm \(i\),
\[
% \label{eq:cac_large_deviation_pairwise_tie_negligible}
\mathbb P(\mathcal E_{ij}^{=})
\le
\mathbb P(\mathcal W_i^{\ge})
-
\mathbb P(\mathcal W_i^{>})
=
o(\mathcal D_i^{>}),
\qquad
\mathcal D_i^{>}:=\mathbb P(\mathcal W_i^{>}).
\]
Together with \(\mathbb P(Q_i)=\mathcal D_i^{>}(1+o(1))\), its finite-horizon contribution is
\[
% \label{eq:cac_large_deviation_boundary_correction}
T
\frac{
\mathbb P(\mathcal E_{ij}^{=})
}{
\mathbb P(Q_i)
}
=
T\,o(1).
\]
Therefore,
\[
% \label{eq:cac_large_deviation_branch_final}
\mathbb E[\min\{T,N_j\}\mid Q_i]
\le
\frac{\Delta_i}{\delta_{ji}}
\left(1+o(1)\right)
+
T\,o(1),
\]
which proves the branch \(p_j>p_i\).

This proves both noncritical branches in \eqref{eq:cac_piecewise_upper_main} and completes the proof.

\end{proof}

\section{Finite-Horizon Cumulative Regret Envelope}
\label{app:cumulative_regret_envelope}

This section converts the absorption-based regret skeleton into explicit finite-horizon regret envelopes. The leading term is the suboptimal-absorption regret: on \(Q_i\), a suboptimal arm \(i\ge2\) is eventually pulled for almost all large times, producing the contribution \(T\delta_i\mathbb P(Q_i)\). The remaining terms are finite transient corrections, coming from pulls spent away from the absorbing arm before the trajectory settles.

We first state a root-based two-sided envelope by combining three ingredients already established above: the absorption-based regret decomposition, the root-based score-minimum envelopes for \(\mathbb P(Q_i)\), and the winner-conditioned strict-crossing upper envelope for capped abandonment counts. The equality part of the score-minimum comparison is kept as a finite-horizon correction. This step is an assembly step; no new boundary-crossing estimate is needed.

We then specialize the envelope to the macroscopic lower-endpoint regime. The absorbing probabilities are replaced by their closed-form approximations \(\widetilde{\mathbb P}(Q_i)\), obtained by rearranging the suboptimal-convergence formula into a form adapted to the regret sum. The transient correction on suboptimal absorbing branches is super-polynomially small, while branches absorbing into an optimal arm contribute the only non-negligible deterministic abandonment correction. The equality corrections are controlled in pairwise form. In the positive-drift case, the boundary event \(M_j=M_i\) is shown to be negligible on the suboptimal-convergence scale of arm \(j\), which allows the corresponding optimal-branch equality corrections to be absorbed into \(T\mathcal P_\Delta o(1)\). This yields a finite-horizon regret sandwich centered at \(T\sum_{\substack{i=2, p_i<p_1}}^{K}\delta_i\widetilde{\mathbb P}(Q_i)\), with an explicit upper transient term.

Throughout this section, \(T\) is an external finite-horizon parameter. The asymptotic notation refers only to the limit \(\Delta_1\to\infty\) under the stated regularization and gap regime. We therefore keep the horizon dependence explicit through factors such as \(T\), \(T\,o(\mathbb P(Q_j))\), and \(T\mathcal P_\Delta o(1)\), and we do not assign a separate asymptotic scale to \(T\).

\subsection{Root-Based Two-Sided Regret Envelope}
\label{app:regret_root_based_envelope}

We first combine the absorption-based regret skeleton with the root-based
score-minimum envelopes and the pairwise winner-conditioned strict-crossing
upper envelope for capped abandonment counts. This subsection is only an
assembly step.

Throughout this subsection, we use the shifted-root convention from
Proposition~\ref{prop:root_based_drawdown_representation}: whenever
\(\ell_i+p_r-p_i\le0\), the corresponding root
\(\theta_r(\ell_i+p_r-p_i)\) is interpreted as \(0\), and the corresponding
tail factor is interpreted as one. For each arm \(i\), define
\[\begin{aligned}
\mathcal P_i^{-}
&:=
\int_{[0,p_i)}
\exp\left\{
-
\bigl[\Delta_i+\beta\ell_i+p_i-\ell_i\bigr]
\sum_{r\neq i}\theta_r(\ell_i+p_r-p_i)
\right\}
\,dF_i(\ell_i),
\\
% \label{eq:regret_root_P_minus_def}
\mathcal P_i^{+}
&:=
\int_{[0,p_i)}
\exp\left\{
-
\bigl[\Delta_i+\beta\ell_i\bigr]
\sum_{r\neq i}\theta_r(\ell_i+p_r-p_i)
\right\}
\,dF_i(\ell_i).
% \label{eq:regret_root_P_plus_def}
\end{aligned}\]
The root-tail bounds and the winner sandwich imply
\begin{equation}
\label{eq:regret_root_P_envelope}
\mathcal P_i^{-}
\le
\mathbb P(\mathcal W_i^{>})
\le
\mathbb P(Q_i)
\le
\mathbb P(\mathcal W_i^{\ge})
\le
\mathcal P_i^{+}.
\end{equation}

We next recall the pairwise equality event from the conditional-abandonment
representation. For \(i\neq k\), let
\[
% \label{eq:regret_pairwise_tie_event_recall}
\mathcal E_{ik}^{=}
:=
\mathcal W_i^{\ge}\cap\{M_k=M_i\}.
\]
This is the only equality event that can invalidate the strict-crossing upper
bound for the abandoned count of arm \(k\) on the absorbing branch \(Q_i\).
The exact Stieltjes representation from
\eqref{eq:cac_pairwise_tie_probability_stieltjes} gives
\[
\mathbb P(\mathcal E_{ik}^{=})
=
\int_{[0,p_i)}
\left[
\overline F_{k,\ge}(\ell_i+p_k-p_i)
-
\overline F_{k,>}(\ell_i+p_k-p_i)
\right]
\times
\prod_{\substack{r=1\\r\neq i,k}}^{K}
\overline F_{r,\ge}(\ell_i+p_r-p_i)
\,dF_i(\ell_i).
% \label{eq:regret_pairwise_tie_stieltjes_recall}
\]
Shifted levels below zero contribute equal strict and weak tail factors and
therefore make zero contribution to the difference.

For \(i\neq k\), define the joint transient upper envelope
\[\begin{aligned}
\mathcal C_{k|i,T}^{+}
&:=
\int_{[0,p_i)}
\mathcal H_{k|i,T}(\ell_i)
\prod_{\substack{r=1\\ r\neq i,k}}^{K}
\overline F_{r,\ge}(\ell_i+p_r-p_i)
\,dF_i(\ell_i)
+
T\mathbb P(\mathcal E_{ik}^{=}).
% \label{eq:regret_root_C_plus_def}
\end{aligned}\]
Here \(\mathcal H_{k|i,T}\) is the strict-crossing capped first-passage
envelope from Lemma~\ref{lem:cac_pointwise_capped_first_passage}. The weak tails
for the remaining competitors match the pairwise winner-conditioned
representation. The final term controls only the boundary event \(M_k=M_i\),
which is the only equality obstruction relevant for the abandoned count of arm \(k\). The factor \(T\) is kept explicitly as the finite-horizon multiplier.

The corrected strict-crossing numerator bound gives
\begin{align}
\mathbb E\!\left[
\min\{T,N_k\}\mathbf 1_{Q_i}
\right]
&\le
\int_{[0,p_i)}
\mathcal H_{k|i,T}(\ell_i)
\prod_{\substack{r=1\\ r\neq i,k}}^{K}
\overline F_{r,\ge}(\ell_i+p_r-p_i)
\,dF_i(\ell_i)
+
T\mathbb P(\mathcal E_{ik}^{=})
\nonumber\\
&=
\mathcal C_{k|i,T}^{+}.
\label{eq:regret_root_joint_transient_envelope}
\end{align}
This is the joint version of the pairwise conditional-abandonment envelope.

\begin{proposition}[Root-based two-sided regret envelope]
\label{prop:regret_root_based_two_sided}
For every finite horizon \(T\),
\begin{equation}
R(T)
\ge
T\sum_{\substack{i=2\\p_i<p_1}}^{K}\delta_i\mathcal P_i^{-}
-
\sum_{i=1}^{K}
\sum_{\substack{j\neq i\\p_j>p_i}}
(p_j-p_i)\,
\mathcal C_{j|i,T}^{+},
\label{eq:regret_root_based_lower}
\end{equation}
and
\begin{equation}
R(T)
\le
T\sum_{\substack{i=2\\p_i<p_1}}^{K}\delta_i\mathcal P_i^{+}
+
\sum_{i=1}^{K}
\sum_{\substack{j\neq i\\p_j<p_i}}
(p_i-p_j)\,
\mathcal C_{j|i,T}^{+}.
\label{eq:regret_root_based_upper}
\end{equation}
\end{proposition}

\begin{proof}\leavevmode
We first consider the lower bound. By
Proposition~\ref{prop:rd_regret_skeleton},
\[
\begin{aligned}
R(T)
&\ge
T\sum_{\substack{i=2\\p_i<p_1}}^{K}\delta_i\mathbb P(Q_i)
-
\sum_{i=1}^{K}
\mathbb P(Q_i)
\sum_{\substack{j\neq i\\p_j>p_i}}
(p_j-p_i)\,
\mathbb E[\min\{T,N_j\}\mid Q_i]
\\
&=
T\sum_{\substack{i=2\\p_i<p_1}}^{K}\delta_i\mathbb P(Q_i)
-
\sum_{i=1}^{K}
\sum_{\substack{j\neq i\\p_j>p_i}}
(p_j-p_i)\,
\mathbb E\!\left[
\min\{T,N_j\}\mathbf 1_{Q_i}
\right].
\end{aligned}
\]
Using the lower probability envelope in
\eqref{eq:regret_root_P_envelope} and the joint transient upper envelope in
\eqref{eq:regret_root_joint_transient_envelope}, we obtain
\[
R(T)
\ge
T\sum_{\substack{i=2\\p_i<p_1}}^{K}\delta_i\mathcal P_i^{-}
-
\sum_{i=1}^{K}
\sum_{\substack{j\neq i\\p_j>p_i}}
(p_j-p_i)\,
\mathcal C_{j|i,T}^{+},
\]
which proves \eqref{eq:regret_root_based_lower}.

For the upper bound, Proposition~\ref{prop:rd_regret_skeleton} gives
\[
\begin{aligned}
R(T)
&\le
T\sum_{\substack{i=2\\p_i<p_1}}^{K}\delta_i\mathbb P(Q_i)
+
\sum_{i=1}^{K}
\mathbb P(Q_i)
\sum_{\substack{j\neq i\\p_j<p_i}}
(p_i-p_j)\,
\mathbb E[\min\{T,N_j\}\mid Q_i]
\\
&=
T\sum_{\substack{i=2\\p_i<p_1}}^{K}\delta_i\mathbb P(Q_i)
+
\sum_{i=1}^{K}
\sum_{\substack{j\neq i\\p_j<p_i}}
(p_i-p_j)\,
\mathbb E\!\left[
\min\{T,N_j\}\mathbf 1_{Q_i}
\right].
\end{aligned}
\]
Using the upper probability envelope in
\eqref{eq:regret_root_P_envelope} and applying
\eqref{eq:regret_root_joint_transient_envelope} with absorbing arm \(i\)
and abandoned arm \(j\), we obtain
\[
R(T)
\le
T\sum_{\substack{i=2\\p_i<p_1}}^{K}\delta_i\mathcal P_i^{+}
+
\sum_{i=1}^{K}
\sum_{\substack{j\neq i\\p_j<p_i}}
(p_i-p_j)\,
\mathcal C_{j|i,T}^{+},
\]
which proves \eqref{eq:regret_root_based_upper}.

\end{proof}

The bounds in Proposition~\ref{prop:regret_root_based_two_sided} have the same
structure as the absorption-based regret skeleton. The terms
\(\mathcal P_i^{-}\) and \(\mathcal P_i^{+}\) control the linear regret from
suboptimal absorbing arms, while the terms \(\mathcal C_{k|i,T}^{+}\) control
the finite-horizon transient corrections. Absorption into an arm tied with an optimal arm has zero linear regret because its regret gap is zero. Likewise, transient pulls between equal-mean arms do not enter either side of the centered regret envelope. The equality correction is included in pairwise form through \(T\mathbb P(\mathcal E_{ik}^{=})\), because only the
boundary event \(M_k=M_i\) can obstruct the strict-crossing upper bound for the
abandoned count of arm \(k\).

\subsection{Asymptotic Closed-Form Regret Envelope}
\label{app:regret_asymptotic_closed_form}

We now specialize the root-based regret envelope to the macroscopic lower-endpoint regime under Assumption~\ref{ass:app_asymptotic_regime}. The common intermediate scale in \eqref{eq:app_asymptotic_gap_scale} ensures that Proposition~\ref{prop:cac_asymptotic_conditional_abandonment_upper} applies to every distinct-mean pair. Equal-mean pairs do not contribute to the centered transient corrections because their coefficient \(p_i-p_j\) is zero.

Recall
\[
% \label{eq:regret_asymptotic_lambda_convention}
\lambda_h:=\frac{2\Delta_h}{p_h(1-p_h)},
\qquad
\delta_{ij}:=p_i-p_j\quad (i<j),
\qquad
\delta_i:=\delta_{1i}.
\]
For every \(i=2,\ldots,K\) with \(p_i<p_1\), define the closed-form suboptimal-absorption approximation
\begin{equation}
\widetilde{\mathbb P}(Q_i)
:=
\sum_{m=i}^{K}\frac{1}{m}
\exp\left\{
-\sum_{h=1}^{m}\lambda_h\delta_{hm}
\right\}
\left[
1
-
\exp\left\{
-\left(
\sum_{h=1}^{m}\lambda_h
\right)\delta_{m,m+1}
\right\}
\right],
\label{eq:regret_asymptotic_PQi_tilde_def}
\end{equation}
where \(\delta_{hh}:=0\), and for \(m=K\), the bracket is interpreted as \(1\). Empty sums are interpreted as zero. Finally, set
\[
\mathcal P_\Delta
:=
\sum_{\substack{i=2\\p_i<p_1}}^{K}
\delta_i\widetilde{\mathbb P}(Q_i),
\qquad
% \label{eq:regret_P_delta_def}
\mathcal C_\Delta
:=
(K-1)\Delta_1.
% \label{eq:regret_C_delta_def}
\]

\begin{proposition}[Asymptotic closed-form regret sandwich]
\label{prop:regret_asymptotic_closed_form}
Under Assumption~\ref{ass:app_asymptotic_regime}, for every finite-horizon parameter \(T\),
\begin{equation}
\label{eq:regret_asymptotic_sandwich}
T\mathcal P_\Delta(1-o(1))
-
O(\Delta_1^{-\infty})
\le
R(T)
\le
T\mathcal P_\Delta(1+o(1))
+
\mathcal C_\Delta(1+o(1)).
\end{equation}
Here the relative \(o(1)\) factors are with respect to
\(\Delta_1\to\infty\). The horizon \(T\) is kept as an external
finite-horizon parameter and is not assigned a separate asymptotic scale.
\end{proposition}

\begin{proof}\leavevmode
Let
\[
\mathcal I_{\mathrm{sub}}
:=
\{i\in\{2,\ldots,K\}:p_i<p_1\}
\]
denote the set of suboptimal arms.

We first identify the closed-form approximation of the absorbing
probabilities. By Proposition~\ref{prop:mfc_asymptotic_closed_form}, for every \(i\in\mathcal I_{\mathrm{sub}}\),
\begin{equation}
\label{eq:regret_PQi_tilde_equivalence}
\mathbb P(Q_i)
=
\widetilde{\mathbb P}(Q_i)(1+o(1)),
\end{equation}
where \(\widetilde{\mathbb P}(Q_i)\) is defined in
\eqref{eq:regret_asymptotic_PQi_tilde_def}. Since \(K\) is fixed, the relative
error in \eqref{eq:regret_PQi_tilde_equivalence} is uniform over
\(i\in\mathcal I_{\mathrm{sub}}\). Therefore,
\begin{equation}
\begin{aligned}
\sum_{\substack{i=2\\p_i<p_1}}^{K}\delta_i\mathbb P(Q_i)
&=
\sum_{i\in\mathcal I_{\mathrm{sub}}}
\delta_i\mathbb P(Q_i)
\\
&=
\sum_{i\in\mathcal I_{\mathrm{sub}}}
\delta_i\widetilde{\mathbb P}(Q_i)(1+o(1))
\\
&=
\mathcal P_\Delta(1+o(1)).
\end{aligned}
\label{eq:regret_sum_PQi_tilde_equivalence}
\end{equation}

We also record a super-polynomial upper bound for every suboptimal absorbing
probability. By the weak side of the score-minimum sandwich,
\(Q_i\subseteq\{M_i\ge M_1\}\) for
\(i\in\mathcal I_{\mathrm{sub}}\). Since \(M_r=p_r-L_r\), this implies
\[
p_i-L_i
\ge
p_1-L_1,
\]
and hence
\[
L_1
\ge
L_i+\delta_i
\ge
\delta_i
\qquad
\text{on }Q_i.
\]
Therefore, by the one-arm lower-endpoint envelope,
\[
\mathbb P(Q_i)
\le
\mathbb P(L_1\ge\delta_i)
\le
\exp\{-c\Delta_1\delta_i\}
=
O(\Delta_1^{-\infty}),
\qquad
i\in\mathcal I_{\mathrm{sub}},
\]
because
\(
\Delta_1\delta_i
=
\Theta(\Delta_1^{1-\kappa})
\to\infty
\). 
Since \(K\) is fixed, summing over the suboptimal arms gives
\begin{equation}
\sum_{i\in\mathcal I_{\mathrm{sub}}}\mathbb P(Q_i)
=
O(\Delta_1^{-\infty}).
\label{eq:regret_suboptimal_absorption_superpoly}
\end{equation}

We next control the centered transient correction in the regret lower bound.
For every pair \(j\neq i\) satisfying \(p_j>p_i\),
Proposition~\ref{prop:cac_asymptotic_conditional_abandonment_upper} gives
\[
\mathbb E[\min\{T,N_j\}\mid Q_i]
\le
\frac{\Delta_i}{p_j-p_i}(1+o(1))
+
T\,o(1).
\]
Multiplying by \(p_j-p_i\), we obtain
\[
(p_j-p_i)
\mathbb E[\min\{T,N_j\}\mid Q_i]
\le
\Delta_i(1+o(1))
+
T(p_j-p_i)o(1).
\]
The condition \(p_j>p_i\) implies \(p_i<p_1\), so only
\(i\in\mathcal I_{\mathrm{sub}}\) contributes. Moreover,
\[
p_j-p_i
\le
p_1-p_i
=
\delta_i.
\]
Consequently, using the uniform regularization scale
\(\Delta_i=\Theta(\Delta_1)\), the finiteness of \(K\), and
\eqref{eq:regret_suboptimal_absorption_superpoly},
\begin{equation}
\begin{aligned}
\sum_{i=1}^{K}\mathbb P(Q_i)
\sum_{\substack{j\neq i\\p_j>p_i}}
(p_j-p_i)\,
\mathbb E[\min\{T,N_j\}\mid Q_i]
&\le
C\Delta_1
\sum_{i\in\mathcal I_{\mathrm{sub}}}\mathbb P(Q_i)
+
T\sum_{i\in\mathcal I_{\mathrm{sub}}}
\delta_i\mathbb P(Q_i)o(1)
\\
&=
O(\Delta_1^{-\infty})
+
T\mathcal P_\Delta o(1),
\end{aligned}
\label{eq:regret_common_lower_transient_total}
\end{equation}
where the last equality follows from
\eqref{eq:regret_sum_PQi_tilde_equivalence}.

We now control the upper transient correction on suboptimal absorbing
branches. For \(i\in\mathcal I_{\mathrm{sub}}\) and \(p_j<p_i\),
Proposition~\ref{prop:cac_asymptotic_conditional_abandonment_upper} gives
\[
\mathbb E[\min\{T,N_j\}\mid Q_i]
\le
\frac{\Delta_i}{p_i-p_j}
\left(1+O(\Delta_1^{\kappa-1})\right)
+
T\,o\!\left(
\frac{\mathbb P(Q_j)}{\mathbb P(Q_i)}
\right).
\]
Multiplying by \((p_i-p_j)\mathbb P(Q_i)\) yields
\[
(p_i-p_j)\mathbb P(Q_i)
\mathbb E[\min\{T,N_j\}\mid Q_i]
\le
\Delta_i\mathbb P(Q_i)
\left(1+O(\Delta_1^{\kappa-1})\right)
+
T(p_i-p_j)\mathbb P(Q_j)o(1).
\]
Since \(p_j<p_i\le p_1\), arm \(j\) is suboptimal and
\[
p_i-p_j
\le
p_1-p_j
=
\delta_j.
\]
It follows from \eqref{eq:regret_suboptimal_absorption_superpoly},
the finiteness of \(K\), and
\eqref{eq:regret_sum_PQi_tilde_equivalence} that
\begin{equation}
\begin{aligned}
\sum_{i\in\mathcal I_{\mathrm{sub}}}\mathbb P(Q_i)
\sum_{\substack{j\neq i\\p_j<p_i}}
(p_i-p_j)\,
\mathbb E[\min\{T,N_j\}\mid Q_i]
&\le
C\Delta_1
\sum_{i\in\mathcal I_{\mathrm{sub}}}\mathbb P(Q_i)
+
T\sum_{j\in\mathcal I_{\mathrm{sub}}}
\delta_j\mathbb P(Q_j)o(1)
\\
&=
O(\Delta_1^{-\infty})
+
T\mathcal P_\Delta o(1).
\end{aligned}
\label{eq:regret_common_upper_transient_total}
\end{equation}

It remains to control the upper transient correction on branches absorbing
into an optimal arm. If \(p_i=p_1\) and
\(j\in\mathcal I_{\mathrm{sub}}\), then
\[
p_i-p_j=\delta_j,
\qquad
\Delta_i=\Delta_1.
\]
Therefore,
Proposition~\ref{prop:cac_asymptotic_conditional_abandonment_upper} gives
\[
\mathbb E[\min\{T,N_j\}\mid Q_i]
\le
\frac{\Delta_1}{\delta_j}
\left(1+O(\Delta_1^{\kappa-1})\right)
+
T\,o\!\left(
\frac{\mathbb P(Q_j)}{\mathbb P(Q_i)}
\right).
\]
Multiplying by \(\delta_j\mathbb P(Q_i)\) gives
\[
\begin{aligned}
\delta_j\mathbb P(Q_i)
\mathbb E[\min\{T,N_j\}\mid Q_i]
\le
\Delta_1\mathbb P(Q_i)
\left(1+O(\Delta_1^{\kappa-1})\right)
+
T\delta_j\mathbb P(Q_j)o(1).
\end{aligned}
\]
Summing over all \(i\) satisfying \(p_i=p_1\) and all
\(j\in\mathcal I_{\mathrm{sub}}\), and using
\[
\sum_{\substack{i=1\\p_i=p_1}}^{K}\mathbb P(Q_i)
\le
1,
\qquad
|\mathcal I_{\mathrm{sub}}|
\le
K-1,
\]
we obtain
\begin{equation}
\begin{aligned}
\sum_{\substack{i=1\\p_i=p_1}}^{K}\mathbb P(Q_i)
\sum_{j\in\mathcal I_{\mathrm{sub}}}
\delta_j\,
\mathbb E[\min\{T,N_j\}\mid Q_i]
&\le
(K-1)\Delta_1
\left(1+O(\Delta_1^{\kappa-1})\right)
+
T\sum_{j\in\mathcal I_{\mathrm{sub}}}
\delta_j\mathbb P(Q_j)o(1)
\\
&=
\mathcal C_\Delta(1+o(1))
+
T\mathcal P_\Delta o(1).
\end{aligned}
\label{eq:regret_asymptotic_typical_branch_first_order}
\end{equation}

We now prove the regret lower bound. By
Proposition~\ref{prop:rd_regret_skeleton},
\[
R(T)
\ge
T\sum_{\substack{i=2\\p_i<p_1}}^{K}\delta_i\mathbb P(Q_i)
-
\sum_{i=1}^{K}\mathbb P(Q_i)
\sum_{\substack{j\neq i\\p_j>p_i}}
(p_j-p_i)\,
\mathbb E[\min\{T,N_j\}\mid Q_i].
\]
Using \eqref{eq:regret_common_lower_transient_total} and then
\eqref{eq:regret_sum_PQi_tilde_equivalence}, we obtain
\begin{equation}
\begin{aligned}
R(T)
&\ge
T\mathcal P_\Delta(1+o(1))
-
T\mathcal P_\Delta o(1)
-
O(\Delta_1^{-\infty})
\\
&=
T\mathcal P_\Delta(1-o(1))
-
O(\Delta_1^{-\infty}).
\end{aligned}
\label{eq:regret_lower_before_probability_insert}
\end{equation}

We next prove the regret upper bound. By
Proposition~\ref{prop:rd_regret_skeleton},
\[
\begin{aligned}
R(T)
&\le
T\sum_{\substack{i=2\\p_i<p_1}}^{K}\delta_i\mathbb P(Q_i)
\\
&\quad+
\sum_{\substack{i=1\\p_i=p_1}}^{K}\mathbb P(Q_i)
\sum_{j\in\mathcal I_{\mathrm{sub}}}
\delta_j\,
\mathbb E[\min\{T,N_j\}\mid Q_i]
\\
&\quad+
\sum_{i\in\mathcal I_{\mathrm{sub}}}\mathbb P(Q_i)
\sum_{\substack{j\neq i\\p_j<p_i}}
(p_i-p_j)\,
\mathbb E[\min\{T,N_j\}\mid Q_i].
\end{aligned}
\]
Using \eqref{eq:regret_sum_PQi_tilde_equivalence},
\eqref{eq:regret_common_upper_transient_total}, and
\eqref{eq:regret_asymptotic_typical_branch_first_order}, we obtain
\begin{align}
R(T)
&\le
T\mathcal P_\Delta(1+o(1))
+
\mathcal C_\Delta(1+o(1))
+
T\mathcal P_\Delta o(1)
+
O(\Delta_1^{-\infty})
\nonumber\\
&=
T\mathcal P_\Delta(1+o(1))
+
\mathcal C_\Delta(1+o(1)),
\label{eq:regret_upper_before_probability_insert}
\end{align}
where the super-polynomial term is absorbed into the final \(o(1)\)
remainder.

Combining \eqref{eq:regret_lower_before_probability_insert} and
\eqref{eq:regret_upper_before_probability_insert} proves
\eqref{eq:regret_asymptotic_sandwich}.

\end{proof}

Thus the finite-horizon regret is centered at the explicit suboptimal-absorption term \(T\mathcal P_\Delta\). The lower centered correction and the upper corrections from suboptimal absorbing branches are \(O(\Delta_1^{-\infty})+T\mathcal P_\Delta o(1)\), while the remaining upper correction arises from branches absorbing into an optimal arm and is bounded by the deterministic transient term \(\mathcal C_\Delta(1+o(1))\). The pairwise equality corrections generated by the strict-crossing representation are controlled pairwise: for each optimal absorbing arm \(i\) and suboptimal arm \(j\), the event \(M_j=M_i\) contributes only \(T\,o(\mathbb P(Q_j))\), and hence its regret contribution is absorbed into \(T\mathcal P_\Delta o(1)\).

\section{Pure Greedy Regime}
\label{app:pure_greedy_regime}

This section treats the pure greedy regime \((\alpha,\beta)=(0,0)\). In this case the score-minimum boundary-crossing formulation used for regularized greedy becomes singular at zero, because an arm with no observed success has empirical score exactly zero. The policy still pulls every arm once at initialization, so the correct starting point is the initial success set \(\mathcal S_0:=\{i:X_i=1\}\).

The pure-greedy trajectory decomposes into three branches. If \(\mathcal S_0=\varnothing\), all empirical scores are zero and the policy enters an absolute-zero renewal phase until the first success occurs. If \(|\mathcal S_0|=1\), the unique initially successful arm immediately dominates all zero-score arms and absorbs. If \(|\mathcal S_0|\ge2\), all initially unsuccessful arms are permanently discarded, and the initially successful arms evolve as a warm-started fixed \((1,1)\)-regularized greedy subproblem on the active set \(\mathcal S_0\). In this reduced subproblem, every active arm starts with score \(1\), corresponding to one prior success and one prior pull.

We first formalize this initial success-set bifurcation. We then evaluate the all-zero and singleton branches explicitly, reduce the multi-success branches to warm-started fixed \((1,1)\)-regularized subproblems, and combine the branchwise quantities into a finite-horizon regret sandwich. The multi-success subproblems are kept in root-based or numerical form rather than replaced by the macroscopic asymptotic envelopes, since their effective margins \(\Delta_i'=1-p_i\) are fixed and do not diverge.

\subsection{Initial Success-Set Bifurcation}
\label{app:pg_initial_success_set}

In the pure greedy regime, the empirical score after the first mandatory pull is exactly the first observed reward. Therefore the trajectory after initialization is determined by the initial success set
\(
% \label{eq:pg_initial_success_set_def}
\mathcal S_0:=\{i:X_i=1\}
\), 
where
\(
% \label{eq:pg_initial_outcome_def}
X_i\sim\mathrm{Bernoulli}(p_i),\,
i=1,\ldots,K
\), 
are independent first-pull rewards. For \(A\subseteq\{1,\ldots,K\}\), define
\(
% \label{eq:pg_initial_success_event_def}
E_A:=\{\mathcal S_0=A\}
\). 
Then
\[
% \label{eq:pg_initial_success_event_probability}
\mathbb P(E_A)
=
\prod_{i\in A}p_i
\prod_{j\notin A}(1-p_j).
\]

\begin{lemma}[Initial success-set bifurcation]
\label{lem:pg_initial_success_set_bifurcation}
Under the pure greedy policy, after the first mandatory pull of each arm, exactly one of the following cases occurs.

\begin{enumerate}
\item If \(A=\varnothing\), then all empirical scores are zero. The policy repeatedly breaks ties among all \(K\) arms until the first success occurs.

\item If \(A=\{i\}\), then arm \(i\) absorbs immediately:
\begin{equation}
\label{eq:pg_singleton_absorption}
E_{\{i\}}\subseteq Q_i,
\qquad
N_j=1\quad\text{for all }j\neq i .
\end{equation}

\item If \(|A|\ge2\), then every arm outside \(A\) is abandoned after its initial pull:
\begin{equation}
\label{eq:pg_multisuccess_outside_abandoned}
N_j=1,
\qquad
j\notin A .
\end{equation}
Conditional on \(E_A\), the future competition among the arms in \(A\) is exactly the warm-started \((1,1)\)-regularized greedy process on \(A\). More precisely, each active arm starts the reduced process with score \(1\). After \(n\ge0\) additional pulls of arm \(i\in A\), its score is 
\[  
% \label{eq:pg_substate_prior_score}
\widehat p_i^A(n) = \frac{1+S_i^A(n)}{1+n} = \frac{S_i^A(n)+\alpha'}{n+\beta'} \quad \text{with } \alpha'=1,\ \beta'=1 . 
\] 
Thus the reduced margins are \[  
% \label{eq:pg_substate_prior}
\Delta_i' = \alpha'-p_i\beta' = 1-p_i, \qquad i\in A . 
\]
\end{enumerate}
\end{lemma}

\begin{proof}\leavevmode
After initialization, the empirical score of arm \(i\) is \(X_i\). If \(A=\varnothing\), all scores are zero, so the greedy rule has no strict maximizer and continues to tie-break among all arms until a selected arm succeeds.

If \(A=\{i\}\), then arm \(i\) has score \(1\), while every arm \(j\neq i\) has score \(0\). As long as arm \(i\) is pulled, its empirical score remains strictly positive because its cumulative number of successes is at least one. Hence every zero-score arm is permanently dominated, and arm \(i\) is selected forever. This proves \eqref{eq:pg_singleton_absorption}.

If \(|A|\ge2\), then each arm in \(A\) has one initial success and therefore has empirical score \(1\), while each arm outside \(A\) has score \(0\). Thus no arm outside \(A\) can ever be selected again, proving \eqref{eq:pg_multisuccess_outside_abandoned}. For \(i\in A\), let \(n\ge0\) be the number of additional pulls after initialization and let \(S_i^A(n)\) be the number of additional successes. The empirical score is 
\[ 
% \label{eq:pg_substate_score_identity}
\frac{1+S_i^A(n)}{1+n} = \frac{S_i^A(n)+\alpha'}{n+\beta'} \quad \text{with }(\alpha',\beta')=(1,1). 
\] 
Therefore the remaining dynamics on \(A\) coincide with a warm-started \((1,1)\)-regularized greedy process, and the corresponding margins are \(\Delta_i'=1-p_i\). Finally, this warm start is compatible with the score-minimum reduction used below. At the beginning of the reduced process every active arm has score \(1\). Since \(p_i\in(0,1)\), an active arm that remains at score \(1\) remains a maximizer whenever it is present; under uniform tie-breaking among maximizers, each active arm is pulled at least once after the branch starts, almost surely. Moreover, the initial reduced-time score \(1\) cannot be the score minimum, because the relevant limiting mean is \(p_i<1\). Hence including the initial warm-start state or starting the score-minimum analysis from the first additional pull gives the same potential score minimum almost surely. 
\end{proof}

This lemma separates the pure-greedy trajectory into an all-zero renewal branch, deterministic singleton branches, and fixed \((1,1)\)-regularized multi-success subproblems.

\subsection{Absolute-Zero Renewal and Singleton Absorption}
\label{app:pg_zero_renewal_singleton_absorption}

We now evaluate the two branches that can be resolved explicitly: the all-zero branch \(E_{\varnothing}\) and the singleton branches \(E_{\{i\}}\).

\paragraph{Absolute-zero renewal.}

On \(E_{\varnothing}\), all arms have empirical score zero after initialization. The policy therefore breaks ties uniformly among all \(K\) arms until the first success occurs. Let
\begin{equation}
\label{eq:pg_sum_p_def}
P_{\Sigma}:=\sum_{r=1}^{K}p_r .
\end{equation}
In each zero-renewal round, arm \(i\) is selected with probability \(1/K\) and succeeds with probability \(p_i\). Hence the probability that a given round ends with a success of arm \(i\) is \(p_i/K\), while the probability of any success is \(P_{\Sigma}/K\). Therefore,
\[
% \label{eq:pg_zero_branch_abs_prob}
\mathbb P(Q_i\mid E_{\varnothing})
=
\frac{p_i}{P_{\Sigma}},
\qquad
i=1,\ldots,K .
\]

Let \(G_0\) be the number of failed zero-renewal rounds before the first success. Since each round succeeds with probability \(P_{\Sigma}/K\),
\[
% \label{eq:pg_zero_branch_failure_count}
\mathbb E[G_0\mid E_{\varnothing}]
=
\frac{K-P_{\Sigma}}{P_{\Sigma}}.
\]
Conditional on a failed renewal round, the selected arm is not uniformly
distributed. Before conditioning on failure, the tie-breaking rule selects each
arm with probability \(1/K\). However, if arm \(j\) is selected, the round fails
with probability \(1-p_j\). Hence
\[
\mathbb P(\text{arm }j\text{ selected and failure})
=
\frac{1-p_j}{K},
\qquad
\mathbb P(\text{failure})
=
\frac{1}{K}\sum_{h=1}^K(1-p_h)
=
\frac{K-P_{\Sigma}}{K}.
\]
Therefore,
\[
% \label{eq:pg_zero_branch_failure_label}
\mathbb P(\text{arm }j\text{ selected}\mid \text{failure})
=
\frac{1-p_j}{K-P_{\Sigma}}.
\]
Thus failed renewal rounds are biased toward arms with larger failure
probability \(1-p_j\).
The failure labels before the terminating success are independent of the label of the terminating success. Hence, conditional on \(E_{\varnothing}\) and \(Q_i\), the expected number of failed renewal pulls assigned to arm \(j\) is
\[
% \label{eq:pg_zero_branch_extra_failures}
\frac{K-P_{\Sigma}}{P_{\Sigma}}
\cdot
\frac{1-p_j}{K-P_{\Sigma}}
=
\frac{1-p_j}{P_{\Sigma}}.
\]
Therefore, if \(j\neq i\),
\[
% \label{eq:pg_zero_branch_loser_count}
\mathbb E[N_j\mid E_{\varnothing},Q_i]
=
1+\frac{1-p_j}{P_{\Sigma}}.
\]
Equivalently, the all-zero branch contributes
\begin{align}
\mathbb P(E_{\varnothing}\cap Q_i)
&=
\left[
\prod_{r=1}^{K}(1-p_r)
\right]
\frac{p_i}{P_{\Sigma}},
\label{eq:pg_zero_branch_joint_prob}
\\
\mathbb E[N_j\mathbf 1_{E_{\varnothing}\cap Q_i}]
&=
\left[
\prod_{r=1}^{K}(1-p_r)
\right]
\frac{p_i}{P_{\Sigma}}
\left(
1+\frac{1-p_j}{P_{\Sigma}}
\right),
\qquad
j\neq i .
\label{eq:pg_zero_branch_joint_count}
\end{align}

\paragraph{Singleton absorption.}

Now consider the singleton branch \(E_{\{i\}}\). In this branch, arm \(i\) has empirical score \(1\) after initialization, while every other arm has empirical score \(0\). By Lemma~\ref{lem:pg_initial_success_set_bifurcation}, arm \(i\) absorbs immediately and every arm \(j\neq i\) is pulled exactly once. Hence
\begin{align}
\mathbb P(E_{\{i\}}\cap Q_i)
&=
p_i
\prod_{r\neq i}(1-p_r),
\label{eq:pg_singleton_branch_joint_prob}
\\
\mathbb P(E_{\{i\}}\cap Q_j)
&=
0,
\qquad
j\neq i,
\nonumber
% \label{eq:pg_singleton_branch_wrong_winner_prob}
\\
\mathbb E[N_j\mathbf 1_{E_{\{i\}}\cap Q_i}]
&=
p_i
\prod_{r\neq i}(1-p_r),
\qquad
j\neq i .
\label{eq:pg_singleton_branch_joint_count}
\end{align}
The remaining branches are those with at least two initial successes. On those branches, the pure-greedy trajectory reduces to a warm-started fixed \((1,1)\)-regularized subproblem on the initially successful arms.

\subsection{Multi-Success Branches as Fixed \texorpdfstring{\((1,1)\)}{(1,1)} Subproblems}
\label{app:pg_multisuccess_regularized_reduction}

We now consider the branches \(E_A\) with \(|A|\ge2\). On such a branch, every arm outside \(A\) is abandoned after its initial pull, while the arms in \(A\) evolve as a warm-started fixed \((1,1)\)-regularized greedy process. In this reduced process, every active arm starts with score \(1\), corresponding to one prior success and one prior pull, and subsequent pulls are counted on the additional-pull clock. This subsection records the branchwise absorbing probabilities and capped transient counts generated by this reduced process.

For each \(A\subseteq\{1,\ldots,K\}\) with \(|A|\ge2\), let \(\mathbb P_A^{(1,1)}\) denote the law of this warm-started reduced \((1,1)\)-regularized process on the active set \(A\). Under \(\mathbb P_A^{(1,1)}\), the score of arm \(h\in A\) after \(n\ge0\) additional pulls is 
\[ 
% \label{eq:pg_warm_started_score_def}
\widehat p_h^A(n) = \frac{1+S_h^A(n)}{1+n}. 
\] 
Let \(\widetilde Q_i^A\) be the event that this reduced process absorbs into arm \(i\in A\), and let \(\widetilde N_j^A\) be the number of additional pulls of arm \(j\in A\) after the initial pure-greedy pull. Define
\[\begin{aligned}
\psi_i^A
&:=
\mathbb P_A^{(1,1)}(\widetilde Q_i^A),
\\
% \label{eq:pg_multisuccess_q_def}
\mu_{j|i,T}^A
&:=
\mathbb E_A^{(1,1)}
\!\left[
\min\{T,\widetilde N_j^A\}
\,\middle|\,
\widetilde Q_i^A
\right],
\qquad
i,j\in A,\ j\neq i .
% \label{eq:pg_multisuccess_mu_capped_def}
\end{aligned}\]
These are fixed-parameter quantities. Since the reduced margins are
\(
% \label{eq:pg_substate_margin_def}
\Delta_h'
=
1-p_h,\,
h\in A
\), 
they do not diverge with \(\Delta_1\). Therefore the macroscopic closed-form
envelopes are not applied to the multi-success branches.

We next define root-based envelopes for \(\psi_i^A\). Let
\(F_h^A\), \(\overline F_{h,>}^{A}\), \(\overline F_{h,\ge}^{A}\), and
\(\theta_h^A\) denote the centered score-minimum distribution, strict and weak
score-minimum tails, and Lundberg root of arm \(h\) in the reduced
\((1,1)\)-regularized process. Shifted levels below zero contribute tail
factor one and root value zero.

In the warm-started reduced process, the potential score minimum can be written
as
\[
% \label{eq:pg_warm_started_minimum_def}
M_h^A
:=
\inf_{n\ge0}
\frac{1+S_h^A(n)}{1+n}.
\]
The initial value at \(n=0\) is \(1\). Since \(p_h<1\), this initial value does
not affect the attained score minimum almost surely. Thus the one-arm
boundary-crossing representation may equivalently be applied on the
additional-pull clock after the first subsequent pull. For a centered level
\(x\in(0,p_h)\), the reduced boundary is
\(
% \label{eq:pg_warm_started_boundary_def}
b_h^A(x)
=
1-p_h+x
\), 
which is the \((1,1)\)-regularized boundary with reduced margin
\(\Delta_h'=1-p_h\).

For \(i\in A\), define
\[\begin{aligned}
\underline{\psi}_i^A
&:=
\int_{[0,p_i)}
\exp\left\{
-
\sum_{\substack{h\in A\\h\neq i}}
\theta_h^A(x+p_h-p_i)
\right\}
\,dF_i^A(x),
\\
% \label{eq:pg_substate_q_lower_for_regret}
\overline{\psi}_i^A
&:=
\int_{[0,p_i)}
\exp\left\{
-
(1-p_i+x)
\sum_{\substack{h\in A\\h\neq i}}
\theta_h^A(x+p_h-p_i)
\right\}
\,dF_i^A(x).
% \label{eq:pg_substate_q_upper_for_regret}
\end{aligned}\]
These are direct applications of the root-based score-minimum envelope to the
reduced process. Indeed, in the reduced \((1,1)\)-regularized process, the
boundary for the candidate absorbing arm \(i\) at centered level \(x\) is
\(
% \label{eq:pg_substate_candidate_boundary_identity}
b_i^A(x)
=
1-p_i+x
\). 
For a competitor \(h\in A\), the shifted level is \(x+p_h-p_i\), and the
shifted boundary satisfies
\[\begin{aligned}
b_h^A(x+p_h-p_i)
&=
1-p_h+x+p_h-p_i
\nonumber\\
&=
1-p_i+x.
% \label{eq:pg_substate_shifted_boundary_identity}
\end{aligned}\]
The corresponding positive jump size is
\(
% \label{eq:pg_substate_positive_jump_identity}
p_h-(x+p_h-p_i)
=
p_i-x
\). 
Hence the upper root envelope uses the exponent factor
\(b_h^A(x+p_h-p_i)=1-p_i+x\), while the lower root envelope uses
\[
% \label{eq:pg_substate_lower_factor_simplification}
b_h^A(x+p_h-p_i)+p_h-(x+p_h-p_i)
=
1.
\]
Together with the strict/weak winner sandwich in the reduced process, this
gives
\begin{equation}
\label{eq:pg_substate_q_bounds_for_regret}
\underline{\psi}_i^A
\le
\psi_i^A
\le
\overline{\psi}_i^A .
\end{equation}

We also need a fixed-parameter capped transient envelope. Let
\(\mathcal H_{j|i,T}^A(x)\) be the strict-crossing capped first-passage
envelope from Lemma~\ref{lem:cac_pointwise_capped_first_passage}, applied to
the reduced active set \(A\) with \((\alpha',\beta')=(1,1)\). In particular,
the critical interface is bounded only by the finite-horizon cap \(T\).

Let \(M_h^A\) and
\(
% \label{eq:pg_substate_centered_minimum_def}
L_h^A:=p_h-M_h^A,\, h\in A
\), 
denote the potential score minimum and the centered score-minimum variable of
arm \(h\) in the reduced \((1,1)\)-regularized process on \(A\). Define the
strict and weak score-minimum winner events in this reduced process by
\[
\widetilde{\mathcal W}_{i}^{A,>}
:=
\left\{
M_i^A>M_h^A,\ \forall h\in A,\ h\neq i
\right\},
\qquad
% \label{eq:pg_substate_strict_winner_event_def}
\widetilde{\mathcal W}_{i}^{A,\ge}
:=
\left\{
M_i^A\ge M_h^A,\ \forall h\in A,\ h\neq i
\right\}.
% \label{eq:pg_substate_weak_winner_event_def}
\]
The reduced absorbing event satisfies the same score-minimum sandwich as in the
regularized process: 
\(
% \label{eq:pg_substate_winner_sandwich}
\widetilde{\mathcal W}_{i}^{A,>}
\subseteq
\widetilde Q_i^A
\subseteq
\widetilde{\mathcal W}_{i}^{A,\ge}
\). 

For \(i,j\in A\), \(j\neq i\), define the pairwise equality event in the
reduced process by
\[
% \label{eq:pg_substate_pairwise_tie_event_def}
\mathcal E_{ij}^{A,=}
:=
\widetilde{\mathcal W}_{i}^{A,\ge}
\cap
\{M_j^A=M_i^A\}.
\]
This is the only equality event that can invalidate the strict-crossing upper
bound for the abandoned count of arm \(j\) on the reduced absorbing branch
\(\widetilde Q_i^A\). Equality between \(M_i^A\) and another active competitor
\(M_h^A\), \(h\neq j\), does not affect the strict comparison with arm \(j\).

Equivalently, after conditioning on \(L_i^A=x\), the equality
\(M_j^A=M_i^A\) is the event
\(
% \label{eq:pg_substate_pairwise_tie_L_form}
L_j^A=x+p_j-p_i
\). 
The remaining active competitors must satisfy the weak winner inequalities
\[
% \label{eq:pg_substate_pairwise_remaining_weak_constraints}
L_h^A\ge x+p_h-p_i,
\qquad h\in A,\ h\neq i,j .
\]

Thus, by independence of the reduced potential streams, the pairwise equality
probability admits the Stieltjes representation
\[\begin{aligned}
e_{ij}^{A,=}
&:=
\mathbb P_A^{(1,1)}(\mathcal E_{ij}^{A,=})
\nonumber\\
&=
\int_{[0,p_i)}
\left[
\overline F_{j,\ge}^{A}(x+p_j-p_i)
-
\overline F_{j,>}^{A}(x+p_j-p_i)
\right]
\prod_{\substack{h\in A\\h\neq i,j}}
\overline F_{h,\ge}^{A}(x+p_h-p_i)
\,dF_i^A(x).
% \label{eq:pg_substate_pairwise_tie_stieltjes}
\end{aligned}\]
As before, shifted levels below zero contribute equal strict and weak tail
factors and hence make zero contribution to the difference.

Define the reduced pairwise capped transient envelope by
\begin{align}
\overline\mu_{j|i,T}^A
&:=
\frac{
\displaystyle
\int_{[0,p_i)}
\mathcal H_{j|i,T}^A(x)
\prod_{\substack{h\in A\\h\neq i,j}}
\overline F_{h,\ge}^{A}(x+p_h-p_i)
\,dF_i^A(x)
+
T e_{ij}^{A,=}
}{
\displaystyle
\underline{\psi}_i^A
}.
\label{eq:pg_substate_mu_root_upper}
\end{align}
This is the reduced-process pairwise version of Proposition~\ref{prop:cac_root_based_integrated_upper}. The strict-crossing time is used only for the abandoned arm \(j\), while the remaining active competitors enter through weak winner factors. The finite-horizon correction is also pairwise: it controls only the boundary event \(M_j^A=M_i^A\), which is the only equality obstruction relevant for the abandoned count of arm \(j\).

The same pathwise argument as in Proposition~\ref{prop:cac_root_based_integrated_upper} gives
\begin{align}
\psi_i^A \mu_{j|i,T}^A
&=
\mathbb E_A^{(1,1)}
\!\left[
\min\{T,\widetilde N_j^A\}
\mathbf 1_{\widetilde Q_i^A}
\right]
\nonumber\\
&\le
\int_{[0,p_i)}
\mathcal H_{j|i,T}^A(x)
\prod_{\substack{h\in A\\h\neq i,j}}
\overline F_{h,\ge}^{A}(x+p_h-p_i)
\,dF_i^A(x)
+
T e_{ij}^{A,=}.
\label{eq:pg_substate_pairwise_joint_count_upper}
\end{align}
Since \(\psi_i^A\ge \underline{\psi}_i^A\), \eqref{eq:pg_substate_mu_root_upper}
and \eqref{eq:pg_substate_pairwise_joint_count_upper} imply
\(
% \label{eq:pg_substate_mu_upper_bound}
\mu_{j|i,T}^A
\le
\overline\mu_{j|i,T}^A
\). 
The quantities \(\underline{\psi}_i^A\), \(\overline{\psi}_i^A\),
\(e_{ij}^{A,=}\), and \(\overline\mu_{j|i,T}^A\) are fixed-parameter root-based or Stieltjes quantities and can be evaluated numerically by solving the one-arm root equations and approximating the associated Stieltjes integrals.

\begin{lemma}[Multi-success branch composition]
\label{lem:pg_multisuccess_branch_composition}
For every \(A\subseteq\{1,\ldots,K\}\) with \(|A|\ge2\), define
\begin{equation}
\label{eq:pg_multisuccess_branch_probability}
\pi_A
:=
\prod_{h\in A}p_h
\prod_{h\notin A}(1-p_h).
\end{equation}
Then, for \(i\in A\),
\begin{equation}
\label{eq:pg_multisuccess_joint_probability}
\mathbb P(E_A\cap Q_i)
=
\pi_A \psi_i^A,
\end{equation}
while for \(i\notin A\),
\begin{equation}
\label{eq:pg_multisuccess_no_outside_winner}
\mathbb P(E_A\cap Q_i)=0.
\end{equation}
Moreover, for \(i\in A\) and \(j\neq i\),
\begin{equation}
\label{eq:pg_multisuccess_joint_capped_count_upper}
\mathbb E\!\left[
\min\{T,N_j\}\mathbf 1_{E_A\cap Q_i}
\right]
\le
\pi_A \overline{\psi}_i^A
\begin{cases}
1+\overline\mu_{j|i,T}^A, & j\in A,\\
1, & j\notin A .
\end{cases}
\end{equation}
\end{lemma}

\begin{proof}\leavevmode
Conditional on \(E_A\), the future dynamics on \(A\) coincide with the warm-started reduced \((1,1)\)-regularized process defined above, and no arm outside \(A\) can be selected again.
Hence the absorbing arm must belong to \(A\), and
\(\mathbb P(Q_i\mid E_A)=\psi_i^A\) for \(i\in A\). Multiplying by
\(\mathbb P(E_A)=\pi_A\) gives
\eqref{eq:pg_multisuccess_joint_probability} and
\eqref{eq:pg_multisuccess_no_outside_winner}.

If \(j\notin A\), then arm \(j\) is pulled once during initialization and never
again. If \(j\in A\) and \(j\neq i\), then on \(E_A\cap Q_i\), arm \(j\) has
one initial pull plus its additional pulls inside the reduced process.
Therefore
\[
% \label{eq:pg_multisuccess_capped_count_pathwise}
\min\{T,N_j\}
\le
\begin{cases}
1+\min\{T,\widetilde N_j^A\}, & j\in A,\\
1, & j\notin A .
\end{cases}
\]
Using \(\psi_i^A\le\overline{\psi}_i^A\) and
\(\mu_{j|i,T}^A\le\overline\mu_{j|i,T}^A\) gives
\eqref{eq:pg_multisuccess_joint_capped_count_upper}.

\end{proof}

Thus every multi-success branch is reduced to warm-started fixed
\((1,1)\)-regularized quantities on the active set \(A\), weighted by the explicit initial-branch
probability \(\pi_A\).

\subsection{Finite-Horizon Regret Synthesis and Linear-Regret Floor}
\label{app:pg_regret_synthesis}

We now combine the all-zero, singleton, and multi-success branches. Recall \(P_{\Sigma}=\sum_{h=1}^{K}p_h\) from \eqref{eq:pg_sum_p_def}, and keep the convention \(\delta_i=p_1-p_i\). The pure-greedy finite-horizon regret is
\[
% \label{eq:pg_regret_definition}
R_{\mathrm{pg}}(T)
=
\sum_{j=2}^{K}\delta_j\,\mathbb E[N_j(T)] .
\]
The branchwise formulas below use the mandatory-initialization convention of this appendix. If a global horizon is counted from before initialization, the same statements hold for \(T\ge K\), up to the harmless deterministic initialization adjustment.

For each arm \(i\), define the lower and upper pure-greedy absorbing-probability envelopes
\[\begin{aligned}
\underline{\mathbb P}_{\mathrm{pg}}(Q_i)
&:=
\left[
\prod_{h=1}^{K}(1-p_h)
\right]
\frac{p_i}{P_{\Sigma}}
+
p_i\prod_{h\neq i}(1-p_h)
+
\sum_{\substack{A\subseteq\{1,\ldots,K\}\\ |A|\ge2,\ i\in A}}
\pi_A\,\underline{\psi}_i^A,
\\
% \label{eq:pg_absorption_probability_lower}
\overline{\mathbb P}_{\mathrm{pg}}(Q_i)
&:=
\left[
\prod_{h=1}^{K}(1-p_h)
\right]
\frac{p_i}{P_{\Sigma}}
+
p_i\prod_{h\neq i}(1-p_h)
+
\sum_{\substack{A\subseteq\{1,\ldots,K\}\\ |A|\ge2,\ i\in A}}
\pi_A\,\overline{\psi}_i^A,
% \label{eq:pg_absorption_probability_upper}
\end{aligned}\]
where \(\pi_A\) is defined in \eqref{eq:pg_multisuccess_branch_probability}. By \eqref{eq:pg_substate_q_bounds_for_regret}, these quantities satisfy
\begin{equation}
\label{eq:pg_absorption_probability_envelope}
\underline{\mathbb P}_{\mathrm{pg}}(Q_i)
\le
\mathbb P(Q_i)
\le
\overline{\mathbb P}_{\mathrm{pg}}(Q_i).
\end{equation}

For \(i\neq j\), define the pure-greedy joint capped-count upper envelope
\begin{align}
\overline{\mathcal U}_{j|i,T}^{\mathrm{pg}}
:=
\left[
\prod_{h=1}^{K}(1-p_h)
\right]
\frac{p_i}{P_{\Sigma}}
\left(
1+\frac{1-p_j}{P_{\Sigma}}
\right)
+
p_i\prod_{h\neq i}(1-p_h)
+
\sum_{\substack{A\subseteq\{1,\ldots,K\}\\ |A|\ge2,\ i\in A}}
\pi_A\,\overline{\psi}_i^A
\begin{cases}
1+\overline\mu_{j|i,T}^A, & j\in A,\\
1, & j\notin A .
\end{cases}
\label{eq:pg_joint_capped_count_upper_def}
\end{align}
Here \(\overline\mu_{j|i,T}^A\) is the pairwise strict-crossing reduced-process envelope from \eqref{eq:pg_substate_mu_root_upper}, including the reduced pairwise equality correction \(T e_{ij}^{A,=}\). Combining \eqref{eq:pg_zero_branch_joint_count}, \eqref{eq:pg_singleton_branch_joint_count}, and \eqref{eq:pg_multisuccess_joint_capped_count_upper}, we have
\begin{equation}
\label{eq:pg_joint_capped_count_upper}
\mathbb E[\min\{T,N_j\}\mathbf 1_{Q_i}]
\le
\overline{\mathcal U}_{j|i,T}^{\mathrm{pg}},
\qquad
i\neq j .
\end{equation}

\begin{proposition}[Pure-greedy finite-horizon regret sandwich]
\label{prop:pg_regret_synthesis_linear_floor}
For pure greedy \((\alpha,\beta)=(0,0)\), for every finite horizon \(T\),
\begin{equation}
R_{\mathrm{pg}}(T)
\ge
T\sum_{\substack{i=2\\p_i<p_1}}^{K}\delta_i
\underline{\mathbb P}_{\mathrm{pg}}(Q_i)
-
\sum_{i=1}^{K}
\sum_{\substack{j\neq i\\p_j>p_i}}
(p_j-p_i)\,
\overline{\mathcal U}_{j|i,T}^{\mathrm{pg}},
\label{eq:pg_regret_lower_sandwich}
\end{equation}
and
\begin{equation}
R_{\mathrm{pg}}(T)
\le
T\sum_{\substack{i=2\\p_i<p_1}}^{K}\delta_i
\overline{\mathbb P}_{\mathrm{pg}}(Q_i)
+
\sum_{i=1}^{K}
\sum_{\substack{j\neq i\\p_j<p_i}}
(p_i-p_j)\,
\overline{\mathcal U}_{j|i,T}^{\mathrm{pg}}.
\label{eq:pg_regret_upper_sandwich}
\end{equation}
\end{proposition}

\begin{proof}\leavevmode
The initialization events \(E_A\), \(A\subseteq\{1,\ldots,K\}\), form a partition. The all-zero branch contributions are given by \eqref{eq:pg_zero_branch_joint_prob} and \eqref{eq:pg_zero_branch_joint_count}; the singleton contributions are given by \eqref{eq:pg_singleton_branch_joint_prob}--\eqref{eq:pg_singleton_branch_joint_count}; and the multi-success contributions are bounded by Lemma~\ref{lem:pg_multisuccess_branch_composition}. Summing over branches gives \eqref{eq:pg_absorption_probability_envelope} and \eqref{eq:pg_joint_capped_count_upper}.

The regret skeleton from Proposition~\ref{prop:rd_regret_skeleton} applies to pure greedy as well, since absorption holds also in the pure-greedy regime. Its lower side gives
\[\begin{aligned}
R_{\mathrm{pg}}(T)
\ge
T\sum_{\substack{i=2\\p_i<p_1}}^{K}\delta_i\mathbb P(Q_i)
-
\sum_{i=1}^{K}
\sum_{\substack{j\neq i\\p_j>p_i}}
(p_j-p_i)\,
\mathbb E\!\left[
\min\{T,N_j\}\mathbf 1_{Q_i}
\right].
\end{aligned}\]
Using \(\mathbb P(Q_i)\ge\underline{\mathbb P}_{\mathrm{pg}}(Q_i)\) and \eqref{eq:pg_joint_capped_count_upper} proves \eqref{eq:pg_regret_lower_sandwich}.

Similarly, the upper side of Proposition~\ref{prop:rd_regret_skeleton} gives
\[\begin{aligned}
R_{\mathrm{pg}}(T)
\le
T\sum_{\substack{i=2\\p_i<p_1}}^{K}\delta_i\mathbb P(Q_i)
+
\sum_{i=1}^{K}
\sum_{\substack{j\neq i\\p_j<p_i}}
(p_i-p_j)\,
\mathbb E\!\left[
\min\{T,N_j\}\mathbf 1_{Q_i}
\right].
\end{aligned}\]
Using \(\mathbb P(Q_i)\le\overline{\mathbb P}_{\mathrm{pg}}(Q_i)\) and \eqref{eq:pg_joint_capped_count_upper} proves \eqref{eq:pg_regret_upper_sandwich}.

\end{proof}

\begin{remark}[Singleton linear-regret floor]
\label{rem:pg_singleton_linear_regret_floor}
The pure-greedy policy has a positive finite-horizon linear-regret floor
whenever at least one suboptimal arm has \(p_i\in(0,1)\). Indeed, on the
singleton branch \(E_{\{i\}}\) with \(p_i<p_1\), arm \(i\) absorbs immediately
after the initialization phase. Hence, for any finite horizon \(T\ge K\),
\(N_i(T)\ge T-K+1 \) on \(E_{\{i\}}\). 
Since
\(
\mathbb P(E_{\{i\}})
=
p_i\prod_{h\neq i}(1-p_h)
\), 
we obtain the finite-horizon lower bound
\[
% \label{eq:pg_singleton_linear_floor}
R_{\mathrm{pg}}(T)
\ge
(T-K+1)
\sum_{\substack{i=2\\p_i<p_1}}^{K}
\delta_i
p_i
\prod_{h\neq i}(1-p_h).
\]
Thus the regret of pure greedy grows at least linearly in the finite horizon
\(T\), independently of the multi-success subproblem bounds.
\end{remark}

\begin{remark}[Numerical evaluation of fixed \((1,1)\)-subproblem envelopes]
\label{rem:pg_numerical_stieltjes_evaluation}
The quantities \(\underline{\psi}_i^A\), \(\overline{\psi}_i^A\), and \(\overline\mu_{j|i,T}^A\) are fixed-parameter root-based envelopes. They are not replaced by the macroscopic closed-form formula because the reduced margins \(\Delta_h'=1-p_h\) do not diverge. Numerically, one solves the one-arm Lundberg root equations for the reduced \((1,1)\)-process and evaluates the corresponding Stieltjes integrals by discretization or quadrature.
\end{remark}

\section{Regret-Certificate Calibration for Regularized Greedy}
\label{app:calibration_proofs}

This section records the regret certificate used to calibrate the regularized greedy policies. The certificate is obtained from the closed-form upper side of Theorem~\ref{thm:main_sandwich} by dropping the asymptotic remainders and keeping the two leading finite-horizon components. The horizon \(T\) is treated as a design parameter, and the arm means are treated as fixed inputs to the certificate. Thus the calibration rule can be applied to a fixed instance even though the regret envelope itself is justified under the regularization-asymptotic regime.

The upper side of Theorem~\ref{thm:main_sandwich} decomposes regret into a suboptimal-absorption term and a transient-abandonment correction. The suboptimal-absorption term is
\(
T\sum_{i=2,\, p_i<p_1}^{K}
\delta_i\,\widetilde{\mathbb P}(Q_i)
\), 
where \(\widetilde{\mathbb P}(Q_i)\) is the closed-form suboptimal-convergence approximation defined in \eqref{eq:regret_asymptotic_PQi_tilde_def}. The leading transient correction is \((K-1)\Delta_1\). Dropping the asymptotic remainders in the regret upper envelope therefore gives the calibration certificate
\begin{equation}
\label{eq:app_cal_certificate_def}
R_{\mathrm{cert}}(T;\alpha,\beta)
:=
T\sum_{\substack{i=2\\p_i<p_1}}^{K}
\delta_i\,\widetilde{\mathbb P}(Q_i)
+
(K-1)\Delta_1.
\end{equation}
This certificate preserves the main finite-horizon trade-off identified by the theory: larger regularization reduces the suboptimal-absorption probabilities in \(\widetilde{\mathbb P}(Q_i)\), while increasing the transient cost through \(\Delta_1\).

\subsection{Fixed-Tilt Optimization over the Regularization Strength}
\label{app:calibration_fixed_tilt_strength}

We next fix the tilt and optimize the regularization strength. Write
\[
\beta=\zeta\alpha,
\qquad
0\le \zeta<\frac{1}{p_1}.
\]
For this fixed tilt, define
\[
\chi_i(\zeta):=1-p_i\zeta,
\qquad
\omega_i(\zeta):=
\frac{2\chi_i(\zeta)}{p_i(1-p_i)},
\qquad
i=1,\ldots,K.
\]
Then
\[
\Delta_i=\alpha\chi_i(\zeta),
\qquad
\lambda_i=\alpha\omega_i(\zeta).
\]
Because \(\zeta<1/p_1\) and \(p_i\le p_1\), all \(\chi_i(\zeta)\) are positive.

For \(m=2,\ldots,K\), define
\[
\varrho_m(\zeta)
:=
\sum_{h=1}^{m}
\omega_h(\zeta)\delta_{hm}.
\]
Also set \(\varrho_{K+1}(\zeta):=\infty\). Substituting
\(\lambda_h=\alpha\omega_h(\zeta)\) into
\eqref{eq:regret_asymptotic_PQi_tilde_def} gives
\[
\widetilde{\mathbb P}(Q_i;\alpha,\zeta)
=
\sum_{m=i}^{K}
\frac{1}{m}
\left[
\exp\{-\alpha\varrho_m(\zeta)\}
-
\exp\{-\alpha\varrho_{m+1}(\zeta)\}
\right],
\qquad
i=2,\ldots,K
\text{ such that }p_i<p_1.
\]
Indeed,
\[
\sum_{h=1}^{m}\omega_h(\zeta)\delta_{hm}
+
\left(\sum_{h=1}^{m}\omega_h(\zeta)\right)\delta_{m,m+1}
=
\sum_{h=1}^{m}\omega_h(\zeta)\delta_{h,m+1}
=
\sum_{h=1}^{m+1}\omega_h(\zeta)\delta_{h,m+1},
\]
because \(\delta_{m+1,m+1}=0\). This gives the telescoping form above. 

Define
\[
\Xi_m
:=
\sum_{\substack{i=2\\p_i<p_1}}^{m}\delta_i,
\qquad
m=2,\ldots,K,
\qquad
\Xi_1:=0,
\]
and
\[
\gamma_m
:=
\frac{\Xi_m}{m}
-
\frac{\Xi_{m-1}}{m-1},
\qquad
m=2,\ldots,K.
\]
Equivalently,
\[\begin{aligned}
\gamma_m
=
\frac{1}{m}\sum_{\substack{i=2\\p_i<p_1}}^{m}\delta_i
-
\frac{1}{m-1}\sum_{\substack{i=2\\p_i<p_1}}^{m-1}\delta_i
=
\frac{(m-1)\delta_m-\sum_{\substack{i=2,p_i<p_1}}^{m-1}\delta_i}{m(m-1)}.
\end{aligned}\]
Since \(\delta_2\le\cdots\le\delta_m\), if \(\delta_m=0\), then
\(\delta_i=0\) for every \(i\le m\), and hence \(\gamma_m=0\). If
\(\delta_m>0\), then
\[
\sum_{\substack{i=2\\p_i<p_1}}^{m-1}\delta_i
\le
(m-2)\delta_m
<
(m-1)\delta_m,
\]
so \(\gamma_m>0\). Therefore,
\[
\gamma_m\ge0,
\qquad
\gamma_m>0
\quad\Longleftrightarrow\quad
p_m<p_1.
\]

We now sum the fixed-tilt expansion of
\(\widetilde{\mathbb P}(Q_i;\alpha,\zeta)\) against the regret gaps. Using the
convention \(\exp\{-\alpha\varrho_{K+1}(\zeta)\}=0\) and \(\delta_i=0\) whenever \(p_i=p_1\), we obtain
\[\begin{aligned}
\sum_{\substack{i=2\\p_i<p_1}}^{K}
\delta_i
\widetilde{\mathbb P}(Q_i;\alpha,\zeta)
&=
\sum_{\substack{i=2\\p_i<p_1}}^{K}
\delta_i
\sum_{m=i}^{K}
\frac{1}{m}
\left[
\exp\{-\alpha\varrho_m(\zeta)\}
-
\exp\{-\alpha\varrho_{m+1}(\zeta)\}
\right]
\\
&=
\sum_{m=2}^{K}
\frac{1}{m}
\left(
\sum_{\substack{i=2\\p_i<p_1}}^{m}\delta_i
\right)
\left[
\exp\{-\alpha\varrho_m(\zeta)\}
-
\exp\{-\alpha\varrho_{m+1}(\zeta)\}
\right]
\\
&=
\sum_{m=2}^{K}
\frac{\Xi_m}{m}
\exp\{-\alpha\varrho_m(\zeta)\}
-
\sum_{m=2}^{K}
\frac{\Xi_m}{m}
\exp\{-\alpha\varrho_{m+1}(\zeta)\}
\\
&=
\sum_{m=2}^{K}
\frac{\Xi_m}{m}
\exp\{-\alpha\varrho_m(\zeta)\}
-
\sum_{m=3}^{K+1}
\frac{\Xi_{m-1}}{m-1}
\exp\{-\alpha\varrho_m(\zeta)\}
\\
&=
\sum_{m=2}^{K}
\left[
\frac{\Xi_m}{m}
-
\frac{\Xi_{m-1}}{m-1}
\right]
\exp\{-\alpha\varrho_m(\zeta)\}
\\
&=
\sum_{m=2}^{K}
\gamma_m \exp\{-\alpha\varrho_m(\zeta)\}.
\end{aligned}\]
In the fourth equality, the term with \(m=K+1\) vanishes because
\(\varrho_{K+1}(\zeta)=\infty\). Thus the weighted suboptimal-absorption term
collapses into a nonnegative weighted sum of exponentials.

Therefore, using the certificate in \eqref{eq:app_cal_certificate_def}, the
fixed-tilt certificate is
\[
R_{\mathrm{cert}}(T;\alpha,\zeta)
=
T\sum_{m=2}^{K}
\gamma_m \exp\{-\alpha\varrho_m(\zeta)\}
+
(K-1)\alpha\chi_1(\zeta).
\]
For later use, define
\[
T_0(\zeta)
:=
\frac{(K-1)\chi_1(\zeta)}
{\sum_{m=2}^{K}\gamma_m \varrho_m(\zeta)}
\]
and
\[
\varrho_{\min}(\zeta)
:=
\min_{\substack{2\le m\le K\\p_m<p_1}}
\varrho_m(\zeta).
\]

\begin{proposition}[Fixed-tilt calibration and logarithmic certificate growth]
\label{prop:cal_leading_certificate_main}
Suppose that at least one arm is suboptimal, so that \(p_K<p_1\), and fix
\(0\le\zeta<1/p_1\). If \(T\le T_0(\zeta)\), then the minimizer of
\(R_{\mathrm{cert}}(T;\alpha,\zeta)\) over \(\alpha\ge0\) is
\(\alpha^\star(T,\zeta)=0\). If \(T>T_0(\zeta)\), then
\(\alpha^\star(T,\zeta)>0\) is uniquely characterized by
\begin{equation}
\label{eq:cal_alpha_star_foc}
\sum_{m=2}^{K}
\gamma_m \varrho_m(\zeta)\exp\{-\alpha^\star(T,\zeta)\varrho_m(\zeta)\}
=
\frac{(K-1)\chi_1(\zeta)}{T}.
\end{equation}
Moreover, for \(T>T_0(\zeta)\),
\begin{equation}
\label{eq:app_cal_log_growth_bound}
R_{\mathrm{cert}}(T;\alpha^\star(T,\zeta),\zeta)
\le
\frac{(K-1)\chi_1(\zeta)}{\varrho_{\min}(\zeta)}
\left[
1+
\log\left(
\frac{T}{T_0(\zeta)}
\right)
\right].
\end{equation}
Thus, for fixed arm means and fixed feasible tilt, the minimized certificate is
at most logarithmic in \(T\). The unique positive solution in
\eqref{eq:cal_alpha_star_foc} can be computed by bisection in time polynomial
in \(K\) and \(\log(1/\varepsilon_\alpha)\), where
\(\varepsilon_\alpha\) denotes the desired accuracy for
\(\alpha^\star(T,\zeta)\).
\end{proposition}

\begin{proof}\leavevmode
Differentiating the fixed-tilt certificate with respect to \(\alpha\) gives
\[
\partial_\alpha
R_{\mathrm{cert}}(T;\alpha,\zeta)
=
(K-1)\chi_1(\zeta)
-
T\sum_{m=2}^{K}
\gamma_m \varrho_m(\zeta)\exp\{-\alpha\varrho_m(\zeta)\}.
\]
Differentiating once more gives
\[
\partial_\alpha^2
R_{\mathrm{cert}}(T;\alpha,\zeta)
=
T\sum_{m=2}^{K}
\gamma_m \varrho_m^2(\zeta)\exp\{-\alpha\varrho_m(\zeta)\}.
\]
For \(m\ge2\), both \(\gamma_m\) and \(\varrho_m(\zeta)\) are nonnegative.
Moreover, if \(p_m<p_1\), then \(\gamma_m>0\), and
\[
\varrho_m(\zeta)
=
\sum_{h=1}^{m}\omega_h(\zeta)(p_h-p_m)>0,
\]
because the term with \(h=1\) is strictly positive. Since at least one arm is
suboptimal, at least one term in
\[
\sum_{m=2}^{K}
\gamma_m\varrho_m^2(\zeta)
\exp\{-\alpha\varrho_m(\zeta)\}
\]
is strictly positive. Hence
\(
\partial_\alpha^2
R_{\mathrm{cert}}(T;\alpha,\zeta)>0
\). 
Thus \(R_{\mathrm{cert}}(T;\alpha,\zeta)\) is strictly convex in \(\alpha\), and
its derivative is strictly increasing.

At \(\alpha=0\),
\[\begin{aligned}
\partial_\alpha
R_{\mathrm{cert}}(T;0,\zeta)
&=
(K-1)\chi_1(\zeta)
-
T\sum_{m=2}^{K}\gamma_m \varrho_m(\zeta).
\end{aligned}\]
This derivative is nonnegative exactly when \(T\le T_0(\zeta)\). In that case,
strict convexity implies that the constrained minimizer over \(\alpha\ge0\) is
\(\alpha^\star(T,\zeta)=0\).

If \(T>T_0(\zeta)\), the right derivative at zero is negative. Moreover,
\[
\lim_{\alpha\to\infty}
\partial_\alpha
R_{\mathrm{cert}}(T;\alpha,\zeta)
=
(K-1)\chi_1(\zeta)>0.
\]
Since the derivative is continuous and strictly increasing, it has a unique
zero. This zero is the unique minimizer and satisfies
\eqref{eq:cal_alpha_star_foc}. The left-hand side of
\eqref{eq:cal_alpha_star_foc} is continuous and strictly decreasing in
\(\alpha\), so bisection computes the solution. Each evaluation requires a
finite sum over \(m=2,\ldots,K\), so the computation is polynomial in \(K\) and
\(\log(1/\varepsilon_\alpha)\).

It remains to prove the logarithmic growth bound. Let
\[
c_\zeta:=(K-1)\chi_1(\zeta),
\qquad
S_\zeta:=\sum_{m=2}^{K}\gamma_m \varrho_m(\zeta).
\]
Then \(T_0(\zeta)=c_\zeta/S_\zeta\). At the interior minimizer,
\[
T\sum_{m=2}^{K}
\gamma_m \varrho_m(\zeta)\exp\{-\alpha^\star \varrho_m(\zeta)\}
=
c_\zeta.
\]
For every \(m\) such that \(\gamma_m>0\),
\(\varrho_m(\zeta)\ge\varrho_{\min}(\zeta)\), while indices with
\(\gamma_m=0\) make no contribution. We have 
\begin{equation}
\label{eq:suboptimal_absorption_bound}
T\sum_{m=2}^{K}
\gamma_m \exp\{-\alpha^\star \varrho_m(\zeta)\}
\le
\frac{T}{\varrho_{\min}(\zeta)}
\sum_{m=2}^{K}
\gamma_m \varrho_m(\zeta)\exp\{-\alpha^\star \varrho_m(\zeta)\}
=
\frac{c_\zeta}{\varrho_{\min}(\zeta)}.
\end{equation}
Also,
\[
\sum_{m=2}^{K}
\gamma_m \varrho_m(\zeta)\exp\{-\alpha^\star\varrho_m(\zeta)\}
\le
S_\zeta \exp\{-\alpha^\star \varrho_{\min}(\zeta)\}.
\]
Together with the first-order condition, this gives
\[
\frac{c_\zeta}{T}
\le
S_\zeta \exp\{-\alpha^\star \varrho_{\min}(\zeta)\}.
\]
Hence
\begin{equation}
\label{eq:cal_alpha_star_bound}
\alpha^\star
\le
\frac{1}{\varrho_{\min}(\zeta)}
\log\left(
\frac{T S_\zeta}{c_\zeta}
\right)
=
\frac{1}{\varrho_{\min}(\zeta)}
\log\left(
\frac{T}{T_0(\zeta)}
\right).
\end{equation}
Combining \eqref{eq:suboptimal_absorption_bound} and \eqref{eq:cal_alpha_star_bound} with
\[
R_{\mathrm{cert}}(T;\alpha^\star,\zeta)
=
T\sum_{m=2}^{K}\gamma_m \exp\{-\alpha^\star \varrho_m(\zeta)\}
+
c_\zeta\alpha^\star
\]
proves \eqref{eq:app_cal_log_growth_bound}.

\end{proof}

\subsection{Tilt Monotonicity and the Practical Near-Boundary Choice}
\label{app:calibration_tilt_monotonicity}

We now discuss the choice of \(\zeta\). Consider the interior calibration regime in which
\[
\alpha^\star(T,\zeta)>0,
\qquad
\Delta_1^\star(T,\zeta)>0.
\]

The previous subsection optimizes the certificate over \(\alpha\) at a fixed
tilt. For every fixed \(0\le\zeta<1/p_1\),
\[
\Delta_1=\alpha\chi_1(\zeta),
\qquad
\chi_1(\zeta)=1-p_1\zeta>0.
\]
Thus \(\alpha\mapsto\Delta_1\) is one-to-one on \((0,\infty)\). Optimizing over
\(\alpha>0\) is therefore equivalent to optimizing over \(\Delta_1>0\), with
\(
\Delta_1^\star(T,\zeta)
=
\chi_1(\zeta)\alpha^\star(T,\zeta)
\). 
We use this equivalent \(\Delta_1\)-parametrization only to make the dependence
on \(\zeta\) transparent.

For fixed \(\Delta_1>0\) and \(\zeta\), we have
\[
\alpha=\frac{\Delta_1}{\chi_1(\zeta)},
\qquad
\beta=\frac{\zeta\Delta_1}{\chi_1(\zeta)}.
\]
Define the normalized rates
\[
\overline\varrho_m(\zeta)
:=
\frac{\varrho_m(\zeta)}{\chi_1(\zeta)}
=
\sum_{h=1}^{m}
\frac{2}{p_h(1-p_h)}
\frac{1-p_h\zeta}{1-p_1\zeta}
\delta_{hm},
\qquad
m=2,\ldots,K.
\]
Then the same fixed-tilt certificate can be written as
\[
R_{\mathrm{cert}}(T;\Delta_1,\zeta)
=
T\sum_{m=2}^{K}
\gamma_m \exp\{-\Delta_1\overline\varrho_m(\zeta)\}
+
(K-1)\Delta_1.
\]
This is exactly the certificate in the previous subsection after the
one-to-one change of variable \(\Delta_1=\alpha\chi_1(\zeta)\).

Let
\(
\Delta_1^\star(T,\zeta)
\in
\arg\min_{\Delta_1>0}
R_{\mathrm{cert}}(T;\Delta_1,\zeta)
\), 
and define the profiled certificate by
\(
R_{\mathrm{cert}}^\star(T;\zeta)
:=
R_{\mathrm{cert}}(T;\Delta_1^\star(T,\zeta),\zeta)
\). 

\begin{proposition}[Tilt monotonicity of the profiled certificate]
\label{prop:app_cal_zeta_monotonicity}
The profiled certificate
\(R_{\mathrm{cert}}^\star(T;\zeta)\) is nonincreasing in
\(\zeta\in[0,1/p_1)\). If the arm means contain at least three distinct
values, then it is strictly decreasing:
\[
\frac{d}{d\zeta}
R_{\mathrm{cert}}^\star(T;\zeta)
=
\partial_\zeta
R_{\mathrm{cert}}(T;\Delta_1^\star(T,\zeta),\zeta)
<
0.
\]
If the arm means contain only two distinct values, then the profiled
certificate is flat in \(\zeta\) under the
\(\Delta_1\)-parametrization.
\end{proposition}

\begin{proof}\leavevmode
For each arm \(h\),
\[
\frac{d}{d\zeta}
\left[
\frac{1-p_h\zeta}{1-p_1\zeta}
\right]
=
\frac{
-p_h(1-p_1\zeta)
+
p_1(1-p_h\zeta)
}{
(1-p_1\zeta)^2}
=
\frac{p_1-p_h}{(1-p_1\zeta)^2}.
\]
This derivative is nonnegative for every \(h\), and it is strictly positive
exactly when \(p_h<p_1\).

Since
\[
\overline\varrho_m(\zeta)
=
\sum_{h=1}^{m}
\frac{2\delta_{hm}}{p_h(1-p_h)}
\frac{1-p_h\zeta}{1-p_1\zeta},
\]
and all coefficients \(2\delta_{hm}/[p_h(1-p_h)]\) are nonnegative, we have
\[
\overline\varrho_m'(\zeta)\ge0,
\qquad
m=2,\ldots,K.
\]
If the arm means contain at least three distinct values, there exist indices
\(h<m\) such that \(p_1>p_h>p_m\). For this pair,
\[
p_1-p_h>0,
\qquad
\delta_{hm}=p_h-p_m>0,
\]
so \(\overline\varrho_m'(\zeta)>0\). Moreover, \(p_m<p_1\) implies
\(\gamma_m>0\).

For fixed \(\Delta_1>0\), differentiating the \(\Delta_1\)-parametrized
certificate gives
\[
\partial_\zeta
R_{\mathrm{cert}}(T;\Delta_1,\zeta)
=
-
T\Delta_1
\sum_{m=2}^{K}
\gamma_m
\overline\varrho_m'(\zeta)
\exp\{-\Delta_1\overline\varrho_m(\zeta)\}.
\]
Since \(T>0\), \(\Delta_1>0\), and both \(\gamma_m\) and \(\overline\varrho_m'(\zeta)\) are nonnegative, the partial derivative is nonpositive. If the arm means contain at least three distinct values, there exists an index \(m\) for which
\[
\gamma_m\overline\varrho_m'(\zeta)>0,
\]
so the partial derivative is strictly negative.

At the interior optimizer,
\[
\partial_{\Delta_1}
R_{\mathrm{cert}}(T;\Delta_1^\star(T,\zeta),\zeta)=0.
\]
Therefore, by the envelope theorem,
\[\begin{aligned}
\frac{d}{d\zeta}
R_{\mathrm{cert}}^\star(T;\zeta)
&=
\partial_{\Delta_1}
R_{\mathrm{cert}}(T;\Delta_1^\star(T,\zeta),\zeta)
\frac{d\Delta_1^\star(T,\zeta)}{d\zeta}
+
\partial_\zeta
R_{\mathrm{cert}}(T;\Delta_1^\star(T,\zeta),\zeta)
\\
&=
\partial_\zeta
R_{\mathrm{cert}}(T;\Delta_1^\star(T,\zeta),\zeta).
\end{aligned}\]
If the arm means contain only two distinct values, then for every \(m\) and
\(h\le m\), either \(p_h=p_1\), so that \(p_1-p_h=0\), or
\(p_h=p_m\), so that \(\delta_{hm}=p_h-p_m=0\). Therefore, every summand in
\[
\overline\varrho_m'(\zeta)
=
\sum_{h=1}^{m}
\frac{2(p_1-p_h)}{p_h(1-p_h)(1-p_1\zeta)^2}
\delta_{hm}
\]
vanishes, and hence
\[
\overline\varrho_m'(\zeta)=0
\qquad
\text{for every }m.
\]
Thus the certificate is independent of \(\zeta\) under the
\(\Delta_1\)-parametrization, and so is its profiled value.
\end{proof}

Proposition~\ref{prop:app_cal_zeta_monotonicity} shows that the certificate prefers larger tilts, and when the arm means contain at least three distinct values, it strictly prefers tilts closer to the feasibility boundary \(1/p_1\). The
boundary itself is not admissible, and taking the backoff too small is
incompatible with the scaling behind the certificate.

Let
\(
\zeta={1}/{p_1}-\epsilon^\circ
\). 
Then
\[
\chi_1(\zeta)=p_1\epsilon^\circ,
\qquad
\Delta_1=\alpha p_1\epsilon^\circ.
\]
If \(\epsilon^\circ\to0\) while \(\Delta_1\) is treated as the regularization
scale, then
\(
\alpha={\Delta_1}/{(p_1\epsilon^\circ)}
\) 
becomes much larger than \(\Delta_1\), and the same issue appears for
\(\beta=\zeta\alpha\). Moreover, for \(i\ge2\) with \(p_i < p_1\),
\[
\frac{\Delta_i}{\Delta_1}
=
\frac{1-p_i\zeta}{1-p_1\zeta}
=
\frac{1-p_i/p_1+p_i\epsilon^\circ}{p_1\epsilon^\circ},
\]
which diverges as \(\epsilon^\circ\to0\). Thus an asymptotically vanishing
backoff would violate the scaling conditions
\(\alpha=\Theta(\Delta_1)\), \(\beta=O(\Delta_1)\), and
\(\Delta_i=\Theta(\Delta_1)\).

For this reason, the implemented calibration uses a fixed moderate backoff. We
set
\(
\zeta_{\epsilon^\circ}
:=
{1}/{p_1}-\epsilon^\circ
\) with \(\epsilon^\circ\) fixed, for example \(\epsilon^\circ=0.2\). The calibrated
pair is then
\[
\alpha_{\epsilon^\circ}^\star(T)
\in
\arg\min_{\alpha>0}
R_{\mathrm{cert}}(T;\alpha,\zeta_{\epsilon^\circ}),
\qquad
\beta_{\epsilon^\circ}^\star(T)
=
\zeta_{\epsilon^\circ}\alpha_{\epsilon^\circ}^\star(T).
\]
With estimated inputs, the same rule is applied after replacing \(p_1\) and the
ordered arm means by their current design estimates. The role of
\(\epsilon^\circ\) is to keep the tilt close to the certificate-preferred
boundary while preserving a regularization scale compatible with the asymptotic
derivation.

\end{document}